\pdfoutput=1
\documentclass{article}
\usepackage{iclr2027_conference,times}
\iclrfinalcopy
\usepackage{amsmath,amssymb,amsthm,mathtools}
\usepackage{booktabs,longtable,graphicx,tabularx}
\usepackage{tikz}
\usetikzlibrary{arrows.meta}
\usepackage{enumitem}
\usepackage{float}
\usepackage{aliascnt}
\usepackage{hyperref}
\usepackage{url}
\usepackage[capitalise,nameinlink,noabbrev]{cleveref}

\input{glyphtounicode}
\usepackage{microtype}
\usepackage{placeins}
\usepackage{flafter}
\usepackage{needspace}
\allowdisplaybreaks
\makeatletter
\long\def\@makecaption#1#2{%
  \vskip\abovecaptionskip
  \small
  \sbox\@tempboxa{#1: #2}%
  \ifdim \wd\@tempboxa >\hsize
    #1: #2\par
  \else
    \global\@minipagefalse
    \hb@xt@\hsize{\hfil\box\@tempboxa\hfil}%
  \fi
  \vskip\belowcaptionskip}
\def\thm@space@setup{\thm@preskip=3.5pt \thm@postskip=3.5pt}
\makeatother
\newtheorem{theorem}{Theorem}[section]
\newcommand{\sharedtheorem}[2]{%
  \newaliascnt{#1}{theorem}%
  \newtheorem{#1}[#1]{#2}%
  \aliascntresetthe{#1}%
  \crefname{#1}{#2}{#2s}%
  \Crefname{#1}{#2}{#2s}}
\sharedtheorem{lemma}{Lemma}
\sharedtheorem{proposition}{Proposition}
\sharedtheorem{corollary}{Corollary}
\crefname{corollary}{Corollary}{Corollaries}
\Crefname{corollary}{Corollary}{Corollaries}
\theoremstyle{definition}
\sharedtheorem{definition}{Definition}
\sharedtheorem{assumption}{Assumption}
\theoremstyle{remark}
\sharedtheorem{remark}{Remark}

\DeclareMathOperator{\TM}{TM}
\DeclareMathOperator{\Up}{U}
\DeclareMathOperator{\Dn}{L}
\DeclareMathOperator{\OS}{os}
\DeclareMathOperator{\Unif}{Unif}

\newcommand{\cX}{\mathcal X}
\newcommand{\cY}{\mathcal Y}
\newcommand{\cH}{\mathcal H}
\newcommand{\cA}{\mathcal A}
\newcommand{\cC}{\mathcal C}
\newcommand{\R}{\mathbb R}
\newcommand{\E}{\mathbb E}
\newcommand{\Pp}{\mathbb P}
\newcommand{\ind}{\mathbf 1}
\newcommand{\Smax}{S_{\max}}

\newcommand{\calib}{\mathrm{cal}}
\newcommand{\qry}{\mathrm{qry}}
\newcommand{\rlo}{\mathrm{lo}}
\newcommand{\rhi}{\mathrm{hi}}
\newcommand{\rev}[1]{{#1}}
\title{Byzantine-Robust Federated RAG via\\
Aligned Calibration and Fixed-Membership Conformal Prediction}
\author{Prasanjit Dubey$^{1}$, Aritra Guha$^{2}$ \& Xiaoming Huo$^{1}$ \\
$^{1}$H. Milton Stewart School of Industrial and Systems Engineering,\\
Georgia Institute of Technology, Atlanta, GA 30332, U.S.A.\\
$^{2}$AT\&T Chief Data Office}
\hypersetup{
  pdftitle={Byzantine-Robust Federated RAG via Aligned Calibration and Fixed-Membership Conformal Prediction},
  pdfauthor={Prasanjit Dubey, Aritra Guha and Xiaoming Huo}
}

\newcommand{\RealStableMinSizeSaving}{0.057}
\newcommand{\RealStableMaxSizeSaving}{0.077}

\newcommand{\RealOracleMinSize}{3.509}
\newcommand{\RealOracleMaxSize}{3.567}
\newcommand{\RealCommonMinAccuracy}{24.10}
\newcommand{\RealCommonMaxAccuracy}{28.12}
\newcommand{\RealViolationMinCoverage}{68.74}
\newcommand{\RealViolationMaxCoverage}{75.92}
\newcommand{\RealRobSizeLowerCells}{44}

\newcommand{\RealRobCoverageLowerCells}{43}

{}

\newcommand{\SynLowHighSize}{0.965}
\newcommand{\SynLowHighEmptyPercent}{8.92}
\newcommand{\SynHighLowFsSize}{1.836}
\newcommand{\SynHighLowPmergeSize}{1.487}

\newcommand{\SynViolatedCoverage}{77.13}

\newcommand{\SynStableFsReductionPercent}{19.19}

{}

\begin{document}
\maketitle
\fancyhead{}
\renewcommand{\headrulewidth}{0pt}
\begin{abstract}
Language models answer questions more accurately when they can consult relevant
documents, an approach called retrieval-augmented generation (RAG). Many valuable
collections, such as medical records or company files, cannot be pooled in one place
because of privacy rules or ownership. Federated RAG leaves each collection with its
owner: each owner, or \emph{node}, searches its own documents and scores candidate
answers, and a central hub combines the scores. Because the hub cannot inspect the
nodes, some of them, called Byzantine, may be compromised, faulty, or misled by
malicious instructions hidden in documents, and report arbitrary scores. We want the
hub to return a set of candidate answers that contains the correct one with a chosen
probability, the guarantee offered by conformal prediction. Conformal prediction keeps
every answer whose score passes a cutoff, and the hub sets that cutoff in a calibration
step, using questions with known answers. An unknown group of nodes, no larger than a
declared bound, may misreport both in this step and when new questions are answered.
Existing methods assume every node is honest or protect only the calibration step. Our
method rests on a simple observation: the honest nodes are the same in both steps. The
hub therefore has all nodes score the same calibration questions. It keeps a candidate
answer only if some plausible group of honest nodes, using its own scores in both steps,
would keep it. We prove three
guarantees. The resulting sets contain the correct answer with the chosen probability
in finite samples, whatever the Byzantine nodes report. No method using the same
information can return smaller sets without risking the loss of an answer the honest
nodes support. If nodes fail at random, the guarantee weakens only by the probability
that more nodes fail than declared. We tested the method in simulations and on real
question-answering tasks, including medical exams. We also used a panel of language
models as nodes, some of them hijacked by injected instructions. Our answer sets reached
the target probability whenever no more nodes misbehaved than declared. They were also
clearly smaller than those of simpler methods offering the same protection, most of all
when the declared bound was generous. Simply averaging the nodes' scores could miss the
target. In practice, an operator can therefore declare a cautious bound on the number of
bad nodes without paying much for it in the size of the answer sets.
\end{abstract}

\section{Introduction}
\label{sec:introduction}

Many of the document collections that could make language models more reliable cannot be
pooled in one place. Retrieval-augmented generation (RAG) lets a language model answer a
question using documents retrieved for it~\citep{lewis2020retrieval}. Yet valuable
collections, such as the records of different hospitals or the files of different firms,
are often kept apart by privacy rules or ownership. Federated RAG leaves each collection
with its owner~\citep{chakraborty2025fedragmap,he2025pfedrag,dhasade2026ragroute}: each
owner, or \emph{node}, retrieves from its own collection and scores candidate answers,
and a central hub combines the scores. For its output to be trusted, the hub
should return a set of answers that contains the correct one with a chosen probability.
Conformal prediction provides such sets~\citep{vovk2005algorithmic,angelopoulos2023gentle}.
It sets a cutoff from the scores of the correct answers to \emph{calibration questions},
whose answers are known, and keeps every candidate scoring at or below it (a lower score
means a more plausible answer).

The challenge is that the hub cannot inspect the nodes, so some may misreport, whether
compromised, faulty, or misled by instructions planted in retrieved text (\emph{prompt
injection}). Consider ``Which dry material conducts electricity: copper, wood, glass, or
rubber?'' sent to 16 nodes with separate corpus shards. At most two nodes may report
arbitrary scores, both when the cutoff is calibrated (the \emph{calibration phase}) and
when the question is answered (the \emph{query phase}). Such nodes are \emph{Byzantine},
the rest are honest, and the hub does not know which are which. The natural baseline, a
plain average of all nodes' reported scores (their \emph{reports}), is vulnerable. Small
calibration reports lower the cutoff and large query reports push the correct answer out; in our experiments this makes
the set miss the correct answer far more often than the target allows
(\cref{fig:unguarded-coverage}a in \cref{sec:experiments}). Sending scores with few bits
to save communication also perturbs even honest scores by rounding.

\begin{figure}[t]
\centering
\begingroup%
\definecolor{settingBlue}{HTML}{0072B2}%
\definecolor{settingOrange}{HTML}{D55E00}%
\definecolor{settingInk}{HTML}{263238}%
\definecolor{settingGray}{HTML}{69747B}%
\begin{tikzpicture}[
  x=1pt,y=0.85pt,
  font=\fontsize{8}{9.5}\selectfont,
  text=settingInk,
  line width=0.6pt,
  box/.style={draw=settingGray!55,rounded corners=3pt,fill=white},
  flow/.style={-{Stealth[length=3.5pt,width=3pt]},draw=settingGray},
  labeltext/.style={align=center,inner sep=0pt},
  heading/.style={labeltext,font=\fontsize{8}{9.5}\selectfont\bfseries},
  service/.style={draw=settingBlue,fill=settingBlue!7,rounded corners=2pt,
    minimum width=23pt,minimum height=15.5pt,inner sep=0pt},
  corrupt/.style={service,draw=settingOrange,fill=settingOrange!7,densely dashed}
]
\path[use as bounding box] (0,12) rectangle (396,272);

\path[box,fill=settingGray!5] (0,241) rectangle (396,271);
\node[labeltext,font=\fontsize{9}{10.5}\selectfont] at (198,260)
  {Which dry material conducts electricity?};
\node[labeltext] at (198,248) {Copper \quad Wood \quad Glass \quad Rubber};

\node[heading] at (62,225) {(a) Federated scoring};
\node[heading] at (205,225) {(b) Hub: fixed-set rule};
\node[heading] at (337,225) {(c) Answer sets};
\draw[settingGray!25] (134,9) -- (134,216);
\draw[settingGray!25] (275,9) -- (275,216);

\node[labeltext,text width=124pt] at (62,200)
  {Same calibration questions\\and new query at all nodes};
\draw[flow] (62,186) -- (62,177);
\foreach \i in {1,...,16} {
  \pgfmathtruncatemacro{\row}{int((\i-1)/4)}
  \pgfmathtruncatemacro{\col}{mod(\i-1,4)}
  \pgfmathsetmacro{\xx}{20+28*\col}
  \pgfmathsetmacro{\yy}{163-22*\row}
  \ifnum\i<3
    \node[corrupt] at (\xx,\yy) {\i};
    \draw[settingOrange,line width=0.9pt]
      (\xx+6,\yy+4.4) -- (\xx+10,\yy+9.1)
      (\xx+6,\yy+9.1) -- (\xx+10,\yy+4.4);
  \else
    \node[service] at (\xx,\yy) {\i};
  \fi
}
\node[labeltext,text width=124pt] at (62,65.5)
  {At most two Byzantine nodes\\Same identities in both phases\\Unknown to the hub};
\node[labeltext,text width=123pt] at (62,30)
  {Each node retrieves and scores\\using its own documents.};
\draw[flow] (118,130) -- (140,130);
\node[labeltext,fill=white,inner sep=1pt] at (127,141) {scores};

\path[box,draw=settingBlue,fill=settingBlue!7] (146,183) rectangle (265,211);
\node[labeltext] at (205.5,197)
  {Try a possible honest group\\of 14--16 nodes};
\draw[flow] (205.5,182) -- (205.5,167);
\path[box] (151,132) rectangle (260,165);
\node[labeltext] at (205.5,148.5)
  {Calibration: correct answers\\Group means $\to$ cutoff};
\path[box] (151,86) rectangle (260,119);
\node[labeltext] at (205.5,102.5)
  {Query: candidate answer\\Same group's mean};
\draw[flow,draw=settingBlue] (151,190) -- (143,190) -- (143,102.5) -- (150,102.5);
\draw[flow] (260,148.5) -- (268,148.5) -- (268,56) -- (261,56);
\draw[flow] (205.5,85) -- (205.5,74);
\path[box] (151,40) rectangle (260,72);
\node[labeltext] at (205.5,56)
  {Query mean $\le$ cutoff\\$+$ error allowance?};
\node[labeltext,text width=125pt,font=\fontsize{8}{9.5}\selectfont] at (205.5,25)
  {Retain the candidate if\\\textbf{any group} supports it.};

\node[labeltext,text width=112pt] at (337,199)
  {Worked example\\No rounding error};
\path[box,draw=settingBlue,fill=settingBlue!4,line width=0.9pt]
  (284,137) rectangle (390,182);
\node[heading,text=settingBlue] at (337,170) {Fixed-set};
\node[labeltext,font=\fontsize{9}{10.5}\selectfont] at (337,151)
  {$\{\text{copper}\}$};
\path[box,fill=settingGray!4] (284,66) rectangle (390,126);
\node[labeltext] at (337,112) {Joint-threshold\\or deletion};
\node[labeltext,font=\fontsize{9}{10.5}\selectfont] at (337,83)
  {$\{\text{copper},\text{wood}\}$};
\node[labeltext,text width=113pt,font=\fontsize{8}{9.5}\selectfont] at (337,28)
  {One extra option removed\\by fixed membership.};
\end{tikzpicture}%
\endgroup%
{}
\caption{Requiring one possible honest group to support a candidate in both phases (fixed
membership, b) removes an answer that two relaxations keep because they may use
different groups for the cutoff and the query (c; \cref{sec:deletion-method}). The hub does
not know which two nodes (marked $\times$) are Byzantine. (a) All 16 nodes score the same calibration questions and the query. (b) A group supports a candidate if its mean query score is at most the cutoff set by its own calibration scores, plus a rounding allowance; the proposed fixed-set rule keeps exactly the candidates that some group supports. (c) No group supports wood, so fixed-set returns $\{\text{copper}\}$, while joint-threshold and deletion return $\{\text{copper},\text{wood}\}$.}
\label{fig:problem-setting}
\end{figure}
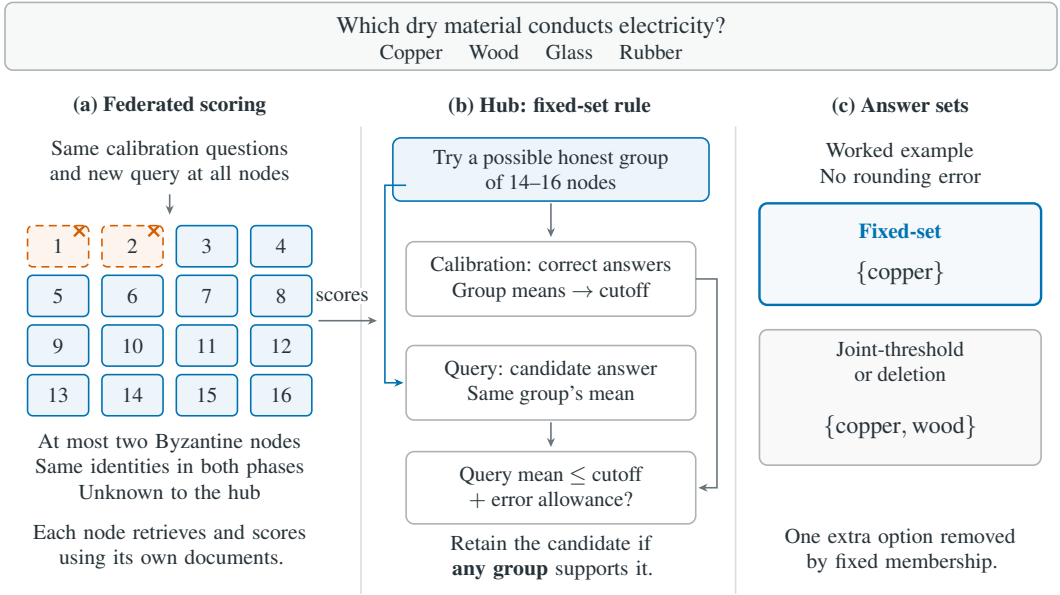

We cast the problem as split conformal prediction with contaminated inputs
(\cref{sec:setup}). If the hub knew which nodes were honest and saw their exact scores, it
could average those scores and apply standard split conformal prediction; we call the
resulting set the \emph{honest-mean oracle}. The hub cannot compute it, but it knows a
declared \emph{budget} $A$: an upper bound, fixed in advance, on the number of Byzantine
nodes. We seek a prediction set
that meets three requirements. (i) It contains the correct answer with probability at least
$1-\alpha$, whatever the Byzantine nodes report. (ii) It contains every answer the
oracle might contain, given what the reports reveal. (iii) It is otherwise as small
as possible. Requirement (ii) is
stronger than (i). Coverage holds only on average over questions, so a set can meet it
while dropping answers the honest nodes support, as does discarding nodes that look
suspicious (\cref{sec:experiments}).

Existing methods guarantee parts of this problem but not the whole. TRAQ and C-RAG
certify answer sets or risk for RAG~\citep{li2024traq,kang2024crag}, and federated
conformal methods assume honest reporting~\citep{lu2023federated,humbert2023oneshot,zhu2024wfcp}.
Rob-FCP and PRISM-FCP defend calibration but not query scores~\citep{kang2024robfcp,lari2026prism}.
FC-RAG calibrates each node on its own questions and assumes the resulting summaries
approximate the all-node mean's cutoff~\citep{dubeyhuo2026bandwidth}. Robust conformal
methods certify sets against bounded perturbations or data
poisoning~\citep{gendler2022adversarial,zargarbashi2024robust,scholten2025reliable}.
None says which answers to keep when one unknown honest set is shared by every
calibration question and the query (\cref{sec:related}). Together, these methods leave three gaps. First, most assume honest reporting. Second, current defenses protect only calibration or bound each score separately, though nodes can misreport in both phases. Third, calibrating each node separately, as FC-RAG does, cannot determine the honest-mean oracle's cutoff (\cref{sec:alignment}).

We propose \emph{fixed-set} inference, which exploits the fact that the same nodes are
honest in the calibration phase and in the query phase. All nodes score the same
calibration questions. With $K$ nodes in total, any set of at least $K-A$ of them could
be the honest group; call it a \emph{possible honest group}. Such a group
\emph{supports} a candidate answer if the
group's average query score for it falls at or below the cutoff computed from that same
group's calibration scores, allowing for known rounding error. Fixed-set keeps exactly
the candidates that some possible honest group supports. It closes each gap in turn: its coverage holds whatever up to $A$ nodes report, one possible honest group must support a candidate in both phases, and shared calibration questions retain how the nodes' scores vary together. Rules that let the group
change between the two phases can keep answers that no single honest group supports,
such as wood in \cref{fig:problem-setting}. Unlike robust conformal methods that bound
each score separately, fixed-set holds one honest group fixed throughout.
\Cref{sec:method} describes the method and two simpler relaxations of it.

The method brings three practical benefits, which are our contributions.
\begin{enumerate}[leftmargin=1.3em,itemsep=0pt,topsep=1pt,parsep=1pt]
 \item \emph{Reliable answers without trusting any single node.} The answer set contains
 the correct answer with the chosen probability whatever the misbehaving nodes report,
 provided no more of them misbehave than declared. If nodes instead fail at random, the
 guarantee drops only by the chance that too many fail, which an operator can compute and
 advertise in advance (\cref{sec:theory}). Plain averaging offers no such protection: in
 our experiments it misses the target both under constructed attacks and when language
 models are hijacked by injected instructions (\cref{sec:experiments}).
 \item \emph{The smallest answer sets that are safe.} Among methods that see only the
 nodes' reports, none can return fewer answers without risking the loss of an answer the
 honest nodes support (\cref{sec:theory}). Every answer removed is one fewer candidate
 for a person or a downstream system to check. \rev{With a stronger scorer the sets are also
 small: with Qwen2.5-14B as every node's scorer, fixed-set returns a single answer for 63\%
 and 83\% of the questions of MMLU-Med and OpenBookQA (\cref{tab:consolidated}).}
 \item \emph{Caution at little cost.} Because the method tracks which groups of nodes
 could be honest, declaring a generous budget for the number of bad nodes adds few answers.
 In simulations where Byzantine nodes send the largest score in both phases, raising the
 budget from two to six adds about one fifth of an answer per question, against about one
 full answer for a simpler robust method. Operators can therefore err on the side of
 safety. On real question-answering tasks, including medical exams, the method also
 returns smaller sets than the simpler one; with language-model judges as nodes, its
 advantage again grows with the budget (\cref{sec:experiments}).
\end{enumerate}
{}
\section{Problem formulation}
\label{sec:setup}

\textbf{Frozen federation.}\quad
Each of $K$ nodes scores a finite list of candidate answers with a frozen pipeline, one
whose retrieval and scoring settings are fixed before calibration, and an unknown subset
of the nodes is Byzantine. Freezing makes a node's \emph{clean} score, the one it
computes before any rounding or tampering, mean the same thing during calibration and
prediction. Let $\cX$ be the set of questions and $\cY$ a finite set of $M$ candidate answers. In a
multiple-choice question $x\in\cX$ with $M$ listed options, $\cY=\{1,\ldots,M\}$ and
answer $y$ is the $y$th listed option. Write $[q]=\{1,\ldots,q\}$. Among $K$ authenticated nodes, a fixed
unknown set $\cA\subseteq[K]$ is Byzantine and $\cH=[K]\setminus\cA$ honest. The actual
number of Byzantine nodes, $a=|\cA|$, obeys the declared integer budget $A$
($0\le a\le A<K$), and $h=|\cH|=K-a$. Each honest node freezes its corpus, retriever,
model, prompt, and normalization before calibration. Their composition is a score
$s_i:\cX\times\cY\to\R$, smaller favoring inclusion; in our testbeds it is one minus the
model's probability for the answer given the node's passages. The components' training
data are independent of the calibration split. The honest-mean score is
\begin{equation}
  s_{\cH}(x,y)=\frac1h\sum_{i\in\cH}s_i(x,y).
  \label{eq:honest-mean}
\end{equation}

\textbf{Directional reconstruction errors.}\quad
Known bounds on transmission error let the hub protect coverage even when honest scores
are rounded to a few bits before they are sent. Each phase has its own transmission
\emph{channel} $r\in\{\calib,\qry\}$ (calibration or query). The hub decodes the bits it
receives into a number, the \emph{decoded report}; for honest node $i$'s clean score $u$
it is $Q_{i,r}(u)$. Here $Q_{i,r}$ is a deterministic quantizer \emph{registered} (fixed and
recorded) before calibration. Nonnegative $\rho_r^-$ and $\rho_r^+$ bound how far the
report falls below or exceeds $u$ (the \emph{reconstruction error}):
\begin{equation}
 -\rho_r^-\le Q_{i,r}(u)-u\le\rho_r^+.
 \label{eq:directional-quantizer}
\end{equation}
Quantizers may differ by node and phase; scores need not be bounded
(\cref{sec:protocol-details,sec:bit-audit}).

\textbf{Arbitrary messages, fixed identities.}\quad
Byzantine reports may depend on anything, labels and honest reports included, and may
switch strategy between phases. We assume honest nodes respond and that the hub,
corpora, and honest score functions are trusted and frozen. A node misled by an injected
prompt counts as Byzantine. Out of scope are poisoning of calibration data or of honest
nodes' corpora or models, changing membership, Sybil attacks (forged identities), and
unmodeled distribution shift (\cref{sec:protocol-details}).
\Cref{cor:random-membership} relaxes the cap $a\le A$.
\rev{\Cref{app:revision-theory} relaxes two of these exclusions. Text planted in a node's corpus
before calibration, without knowledge of the calibration questions, and left unchanged needs
no budget: the node then reports one frozen, if worse, score in both phases
(\cref{cor:phase-blind}). And deletion keeps its coverage when the corrupted nodes change
from question to question (\cref{thm:rotating}).}

\textbf{The honest-mean oracle and the clean rank event.}\quad
The \emph{honest-mean oracle} is the set the hub would obtain if it knew which
nodes were honest and could read their exact scores. Fix $n\ge1$ and let
$Z_j=(X_j,Y_j)$, $j\in[n+1]$, be clean examples exchangeable conditional on the frozen
system and honest set: their joint law is unchanged by reordering, and i.i.d.\ examples
are a special case. The first $n$ calibrate the rule and the last is the future
question. Write $R_j=s_{\cH}(X_j,Y_j)$, fix $\alpha\in(0,1)$, and set
$k=\lceil(n+1)(1-\alpha)\rceil$. For $k\le n$, let $R_{(k)}$ be the $k$th smallest
calibration score, an \emph{order statistic}. The split-conformal honest-mean
oracle~\citep{lei2018distribution} is $\cC^\circ_\alpha(x)=\{y:s_{\cH}(x,y)\le R_{(k)}\}$;
for $k=n+1$, it and all three rules proposed in \cref{sec:method} return $\cY$.

One event drives every guarantee below. Call
\begin{equation}
 \{R_{n+1}\le R_{(k)}\}
 \quad\text{equivalently}\quad
 \{Y_{n+1}\in\cC^\circ_\alpha(X_{n+1})\}
 \label{eq:clean-rank-event}
\end{equation}
the \emph{clean rank event}: the future question's clean honest-mean score falls at or
below the clean calibration cutoff. Exchangeability makes $R_{n+1}$ equally likely to take
each of the $n+1$ ranks (ties only help), so the event has probability at least
$k/(n+1)\ge1-\alpha$ (\cref{lem:rank}).

\textbf{Problem statement.}\quad
Using only the decoded reports, the budget $A$, and the error bounds, the hub must return
a set $\cC_\alpha(x)$ that meets three requirements. (i) \emph{Coverage}:
$\Pp\{Y_{n+1}\in\cC_\alpha(X_{n+1})\}\ge1-\alpha$, whatever the Byzantine nodes send.
(ii) \emph{Clean-score integrity}: the set keeps every $y$ for which the clean rank event
could hold under some possible honest group (any set of at least $K-A$ nodes) and clean
scores consistent with the reports
(made precise in \cref{thm:fixed-set}). (iii) \emph{Minimal size}: the set is as small
as possible given (i) and (ii). Integrity implies coverage, because the oracle's answers
are always kept. \Cref{app:checklist} lists the assumptions of
\cref{thm:coverage,thm:fixed-set} and what they do \emph{not} require.
\Cref{app:notation} tabulates the notation.
{}
\section{Proposed method: aligned fixed-membership inference}
\label{sec:method}

The method exploits \emph{fixed membership}: the honest set $\cH$ is the same in the
calibration and query phases. \Cref{sec:alignment} specifies the data the hub collects,
\cref{sec:fixed-membership-method} the fixed-set rule, and \cref{sec:deletion-method} two
relaxations that let the honest group change. Each rule rests on an \emph{envelope}, a
bound on the honest mean that holds for every possible honest group.

\subsection{Aligned calibration}
\label{sec:alignment}

For every calibration example $j\in[n]$, the hub sends $X_j$ and its candidates, but not
$Y_j$, to all $K$ nodes; at a query $x$ it does the same. This design is
\emph{alignment}. Let $\widetilde s^{\calib}_{ij}(y)$ and $\widetilde s^{\qry}_i(x,y)$ be
node $i$'s decoded calibration and query scores: $Q_{i,r}\circ s_i$ on channel $r$ if
$i\in\cH$, and arbitrary finite values if $i\in\cA$. The hub keeps
\begin{equation}
 V=(v_{ij})\in\R^{K\times n},\quad v_{ij}=\widetilde s^{\calib}_{ij}(Y_j);\qquad
 w(x,y)\in\R^K,\quad w_i(x,y)=\widetilde s^{\qry}_i(x,y).
 \label{eq:aligned-data}
\end{equation}
For $u\in\R^K$ and $H\subseteq[K]$, write $\bar u_H=|H|^{-1}\sum_{i\in H}u_i$.

Joint information is necessary: the law of the nodes' mean score is determined by their joint
law, but not by their per-node marginal laws (\cref{prop:nonidentification};
cf.~\citealp{makarov1982fixedmarginals}). Column $j$ of $V$ holds all nodes' scores on one
example, so $V$ retains the joint law. For
$U\sim\Unif[0,1]$ and $T(u)=1-|2u-1|$, the pairs $(U,U)$ and $(U,T(U))$ have the same
uniform marginals, yet the cumulative distribution functions (CDFs) of their means differ
by up to $1/4$. At a 75\% target,
per-node calibration is exact for the first pair but covers 99.7\% for the second; aligned
calibration covers 75.1\% for both (\cref{tab:synthetic-diagnostics}). Other routes to the joint law, such as sending cross-moments, need a correctly specified model of how the nodes' scores depend on each other; alignment needs no such model (\cref{rem:nonidentification-scope}).

\subsection{The fixed-set rule}
\label{sec:fixed-membership-method}

The possible honest groups are the \emph{feasible} subsets
\begin{equation}
 \mathfrak H_A=\{H\subseteq[K]:|H|\ge K-A\},\qquad
 c_j(H;V)=\frac1{|H|}\sum_{i\in H}v_{ij},
 \label{eq:feasible-honest-sets}
\end{equation}
where $c_j(H;V)$ is the mean calibration report of group $H$ on example $j$. Groups larger than $K-A$ are included because
$a$ may be smaller than $A$. For $k\le n$, let $\OS_k$ return the $k$th smallest of $n$
values. Group $H$'s calibration cutoff, and the largest such cutoff, are
\begin{equation}
 T_k(H;V)=\OS_k\{c_1(H;V),\ldots,c_n(H;V)\},\qquad
 \tau_{A,k}(V)=\max_{H\in\mathfrak H_A}T_k(H;V).
 \label{eq:joint-threshold}
\end{equation}
With clean reports, $T_k(\cH;V)$ is the oracle cutoff $R_{(k)}$. The padding
$g_\rho=\rho_{\calib}^-+\rho_{\qry}^+$ allows for calibration reports rounded down and
query reports rounded up. Group $H$ \emph{supports} $y$ if
$\bar w_H(x,y)\le T_k(H;V)+g_\rho$, and fixed-set keeps exactly the supported candidates:
\begin{align}
 \Delta_{A,k}(V,w)
 &=\min_{H\in\mathfrak H_A}\bigl\{\bar w_H-T_k(H;V)\bigr\},
 \label{eq:fixed-set-margin}\\
 \cC^{\rm fs}_\alpha(x)
 &=\left\{y:\Delta_{A,k}\{V,w(x,y)\}\le g_\rho\right\}
 =\bigl\{y:\exists H\in\mathfrak H_A,\ \bar w_H(x,y)\le T_k(H;V)+g_\rho\bigr\}.
 \label{eq:fixed-set-rule}
\end{align}

\subsection{Two relaxations}
\label{sec:deletion-method}

The two relaxations, \emph{joint-threshold} and \emph{deletion}, let the query use a
different group from calibration, and deletion also lets each calibration example use
its own group:
\begin{align*}
 y\in\cC^{\rm jt}_\alpha(x)
 &\iff\exists H,H'\in\mathfrak H_A:\
 \bar w_{H'}(x,y)\le T_k(H;V)+g_\rho,\\
 y\in\cC^{\rm del}_\alpha(x)
 &\iff\exists H',H_1,\ldots,H_n\in\mathfrak H_A:\
 \bar w_{H'}(x,y)\le\OS_k\bigl\{c_1(H_1;V),\ldots,c_n(H_n;V)\bigr\}+g_\rho.
\end{align*}
Each relaxation allows more choices of groups than the rule before it, so
$\cC^{\rm fs}_\alpha\subseteq\cC^{\rm jt}_\alpha\subseteq\cC^{\rm del}_\alpha$
(\cref{thm:fixed-set}). Both have closed forms. With $v\in\R^K$ sorted as
$v_{(1)}\le\cdots\le v_{(K)}$, the extreme feasible-subset means are the deletion means
\begin{align}
 \Up_A(v)&=\max_{H\in\mathfrak H_A}\bar v_H=\frac1{K-A}\sum_{\ell=A+1}^{K}v_{(\ell)},
 &
 \Dn_A(v)&=\min_{H\in\mathfrak H_A}\bar v_H=\frac1{K-A}\sum_{\ell=1}^{K-A}v_{(\ell)},
 \label{eq:deletion-means}
\end{align}
which drop the $A$ smallest and the $A$ largest reports, respectively. Minimizing over
$H'$ and maximizing over $H$ gives
\begin{equation}
 \cC^{\rm jt}_\alpha(x)
 =\left\{y:\Dn_A\{w(x,y)\}\le\tau_{A,k}(V)+g_\rho\right\},
 \label{eq:joint-threshold-set}
\end{equation}
and, because $\OS_k$ is nondecreasing in each argument, maximizing over each $H_j$
separately gives, with $v_{\cdot j}$ the $j$th column of $V$,
$\widehat R_j=\Up_A(v_{\cdot j})$, $\widehat R_{(k)}$ the $k$th smallest of
$\widehat R_1,\ldots,\widehat R_n$, and $\widehat s^{\qry}_{\rm del}(x,y)=\Dn_A\{w(x,y)\}$,
\begin{equation}
 \widehat q^{\rm del}_\alpha=\widehat R_{(k)}+g_\rho,\qquad
 \cC^{\rm del}_\alpha(x)=
 \{y:\widehat s^{\qry}_{\rm del}(x,y)\le\widehat q^{\rm del}_\alpha\}.
 \label{eq:deletion-set}
\end{equation}
Joint-threshold thus compares every query with one scalar threshold $\tau_{A,k}(V)$
computed once from calibration.

\textbf{Worked example.}\quad
Fixed membership removes wood from the introductory question's answer set (\cref{fig:problem-setting}), although both relaxations
keep it. Take $K=16$, $A=2$, $n=9$, $\alpha=0.1$, $k=9$, and no
rounding. Nodes 1--2 report $4/5$ on every calibration example and $9/10$ for every
option; nodes 3--16 report $0$ on calibration and $0,\ 1/10,\ 1/5,\ 3/10$ for copper,
wood, glass, and rubber. Every column of $V$ is the same, so
$\tau_{A,k}=\widehat R_{(k)}=4/35$ (the two $4/5$ reports averaged with twelve zeros).
For wood, $\Dn_A(w)=1/10\le4/35$, so both relaxations keep it. But every node's wood
report exceeds its own calibration report by $1/10$, so
$\bar w_H-T_k(H;V)=1/10>0$ for every $H$, and fixed-set drops wood. With nodes 3--16
honest, fixed-set and the oracle return $\{\text{copper}\}$ and both relaxations
$\{\text{copper},\text{wood}\}$: the relaxations' largest cutoff and smallest query mean
come from different groups, and no single group supports wood.

\textbf{Which rule to deploy.}\quad
Tighter sets cost more hub computation but no extra communication. Fixed-set enumerates
$|\mathfrak H_A|=\sum_{d\le A}\binom{K}{d}$ subsets: 137 at the default $K=16$, $A=2$;
697 at $A=3$; $43{,}745$ at $K=64$, $A=3$. Use fixed-set where that is affordable, as
here (\cref{sec:experiments}); joint-threshold when a single reusable threshold is
preferable, since it enumerates only during calibration; and deletion, which needs no
enumeration, at larger $A$ (\cref{sec:computation-details,app:exact-grid-decisions}).
\rev{Deletion is also the rule to deploy if the compromised nodes may change from one question
to the next (\cref{thm:rotating}); fixed-set and joint-threshold assume a fixed honest group
(\cref{prop:two-node}).}
{}
\section{Theoretical guarantees: coverage, minimality, and set size}
\label{sec:theory}

Our guarantees answer the three questions a user of the rules would ask. Does the set
still contain the correct answer whatever Byzantine nodes send
(\cref{thm:coverage,cor:random-membership})? Could any safe rule return a smaller set
(\cref{thm:fixed-set})? How many extra answers does robustness cost
(\cref{thm:set-sandwich})? All three concern one property: retaining every candidate the
honest-mean oracle would include. A \emph{transcript} is the full collection of decoded
reports, $V$ and every $w(x,y)$. It is \emph{admissible} if at most $A$ nodes are
Byzantine and every honest report obeys \cref{eq:directional-quantizer}. A statement is
\emph{pathwise} if it holds for every calibration sample, query, and admissible
transcript, before any averaging.

\subsection{Coverage under arbitrary reports}

The deletion means bracket the honest mean with no distributional model for
Byzantine reports: being the extreme feasible-subset means (\cref{sec:deletion-method}),
they satisfy $\Dn_A(v)\le\bar v_{\cH}\le\Up_A(v)$ whenever at most $A$ entries of $v$ come from nodes that are
Byzantine. Padding by $\rho_r^-$ and $\rho_r^+$ keeps the bracket when honest reports
are rounded, and neither bound can be tightened without more information
(\cref{thm:extremal-envelopes}).

\begin{theorem}[Finite-sample validity of the deletion rule]
\label{thm:coverage}
Assume $Z_1,\ldots,Z_n,Z_{n+1}$ are exchangeable conditional on the frozen
honest score functions and fixed honest set.  Under $a\le A<K$, suppose the
directional reconstruction bounds in \cref{eq:directional-quantizer} hold.
Then, uniformly over all finite Byzantine calibration and query reports,
$\Pp\{Y_{n+1}\in\cC^{\rm del}_\alpha(X_{n+1})\}\ge k/(n+1)\ge1-\alpha$.
Validity needs no bounded scores, continuity, density, cross-node independence, or
attack model.
\end{theorem}

Thus arbitrary reporting in both phases cannot push coverage below the chosen level. The
padded envelopes bracket the clean means on both channels, and order statistics preserve
the inequality. Exchangeability is needed only for the clean rank event, as in certified
robust conformal inference (\citealp{gendler2022adversarial,yan2024rscpplus}; \cref{app:proof-coverage}).

When nodes fail at random, coverage degrades gracefully. A \emph{failure law}, a
probability distribution over which nodes are Byzantine, can replace the cap $a\le A$
without changing the rule (the proof also uses \cref{thm:fixed-set}).
\begin{corollary}[Coverage under a random failure law]
\label{cor:random-membership}
Suppose the Byzantine set $\cA$ is drawn once, before calibration, from any probability law under
which the clean examples are exchangeable conditional on each realized value of $\cA$
(for instance, independently of the data). Suppose also that every node's reports obey
\cref{eq:directional-quantizer} whenever that node is honest, and that $A<K$ is fixed in
advance. Let
$\varepsilon_A=\Pp\{|\cA|>A\}$. Then for each of the three proposed rules, uniformly over
finite Byzantine reports, which may depend on $\cA$,
\begin{equation}
 \Pp\{Y_{n+1}\in\cC_\alpha(X_{n+1})\}\ \ge\ \frac{k}{n+1}\,(1-\varepsilon_A)
 \ \ge\ 1-\alpha-\varepsilon_A.
 \label{eq:random-membership-coverage}
\end{equation}
\end{corollary}

The guarantee weakens only when more nodes fail than the budget declares, and by at most
that event's probability. For independent failures, $\varepsilon_A$ is the Poisson--binomial
probability that more than $A$ nodes fail (\cref{app:random-membership}). It is
nondecreasing in every node's failure probability, so upper bounds on those give a valid
coverage \emph{floor}, the right side of \cref{eq:random-membership-coverage}, to
advertise. If the bounds are wrong, the corollary still holds at the true
$\varepsilon_A$. \rev{For independent failures, the budget stays small: 16 nodes failing at
1--10\% each need $A=2$--$5$ for $\varepsilon_A\le 0.01$ (\cref{tab:budget-choice}).}

\subsection{Fixed-set is the smallest rule with clean-score integrity}

Fixed-set improves on joint-threshold only by pairing each group's cutoff with that
same group's query mean: joint-threshold already uses the sharpest scalar bound on the
clean calibration order statistic, $\tau_{A,k}(V)+\rho_{\calib}^-$
(\cref{lem:sharp-joint-envelope}). A \emph{report-only} rule uses only the reports, the declared budget, and the error bounds.

\begin{theorem}[Rule hierarchy and minimality]
\label{thm:fixed-set}
Assume the conditions of \cref{thm:coverage}.  For $k\le n$:
\begin{enumerate}[itemsep=1pt,topsep=2pt]
 \item For every query and admissible transcript,
 \begin{equation}
   \cC^\circ_\alpha(x)
   \subseteq\cC^{\rm fs}_\alpha(x)
   \subseteq\cC^{\rm jt}_\alpha(x)
   \subseteq\cC^{\rm del}_\alpha(x),
   \label{eq:fixed-set-hierarchy}
 \end{equation}
 and each proposed set has coverage at least $k/(n+1)\ge1-\alpha$.
 \item Fix decoded reports $V$ and $w(x,y)$. By \cref{eq:directional-quantizer}, an honest
 clean value lies at most $\rho_r^+$ below and $\rho_r^-$ above its report on channel $r$.
 An \emph{entrywise compatible completion} is a set $H\in\mathfrak H_A$ with clean values
 in these intervals at every calibration and query entry of its nodes, and nothing else
 imposed. This is the \emph{interval model}: it assumes no common score function or
 cross-example law. The clean rank event holds for $y$ in a completion if $H$'s mean clean
 query value for $y$ is at most the $k$th smallest of $H$'s $n$ mean clean calibration
 values. Then
 \begin{equation}
  y\in\cC^{\rm fs}_\alpha(x)
  \iff
  \text{the clean rank event holds for $y$ in some such completion},
  \label{eq:fixed-set-identified}
 \end{equation}
 so $\cC^{\rm fs}$ is pointwise smallest among report-only rules that must keep every
 such $y$.
\end{enumerate}

\end{theorem}

Fixed membership therefore removes every candidate that no single honest group can
explain, and nothing more. A retained candidate must admit one coherent explanation: one possible honest group whose
clean scores, within their error intervals, place it under that group's cutoff. When
none exists, fixed-set drops the candidate even when separate optimizations would keep it,
at no cost to oracle inclusion or coverage. Part 2 shows that this test is exact, with
$g_\rho$ tight, so no report-only rule with the clean-score integrity of
\cref{sec:setup} can discard more. This minimality is relative to the interval
model: a hub modeling how honest scores relate across examples could rule out
completions. Each inclusion in \cref{eq:fixed-set-hierarchy} can be strict (\cref{app:proof-fixed-set}).
\rev{Integrity also fixes where errors can fall. Within the budget, fixed-set misses the correct
answer only where the clean oracle misses it, so corrupt nodes can enlarge the set but cannot
create an error or move errors onto questions they choose (\cref{prop:integrity-transfers}).}

\subsection{Report spread limits the extra answers}

\begin{theorem}[Pathwise proposed-set sandwich]
\label{thm:set-sandwich}
Fix an honest set $\cH$ with $a=K-|\cH|\le A<K$, assume \cref{eq:directional-quantizer},
and suppose all decoded reports on channel $r\in\{\calib,\qry\}$ lie in a common
interval of width $D_r$. With the membership term $\mu_r=\min\{1,A/(K-A)\}D_r$ and
$W_{\rm del}=\sum_{r}(\mu_r+\rho_r^-+\rho_r^+)$, for $k\le n$, \cref{eq:fixed-set-hierarchy}
extends pathwise to $\cC^{\rm del}_\alpha(x)\subseteq\{y:s_{\cH}(x,y)\le R_{(k)}+W_{\rm del}\}$.
\end{theorem}

Robustness therefore costs only candidates whose clean scores lie near the oracle cutoff:
every extra candidate has a clean score within $W_{\rm del}$ above it, and no score
smoothness is needed.
The width combines uncertainty about which nodes are honest with reconstruction error,
so extra bits shrink the second part but never the first. How many extra candidates
appear depends on how many candidate scores lie near the cutoff
(\cref{cor:size-modulus,rem:size-density}).

The direction of an attack decides how many extra answers remain. Consider reports in a
range $[\rlo,\rhi]$ ($[0,1]$ here). For any number of corrupt
nodes, \emph{high--low} ($\rhi$ at calibration, $\rlo$ at query) yields the largest
proposed sets and \emph{low--high} ($\rlo$ then $\rhi$) the smallest, so they test
worst-case inflation and coverage. High--low and low--high are switching attacks: they change direction between phases. Low--high targets coverage, since its small calibration reports lower the cutoff and its large query reports raise every candidate's score. \emph{Stable maximum} sends $\rhi$ in both phases.
With $a=A$, low--high cancels the corrupt contribution: the honest group then gives both
the largest calibration threshold and the smallest query mean, so all three rules
coincide. Raising $A$ with reports fixed
never shrinks a set (\cref{app:endpoint-attacks}).

Two simpler baselines are also valid, but neither returns smaller sets than fixed-set
while keeping clean-score integrity. A \emph{guarded symmetric rule} discards the
$m\ge A$ largest and $m$ smallest reports, averages the rest, and adds a guard wide
enough to restore validity. With a guard computed from the full range of values a channel
can carry, the three proposed sets are contained in it pathwise (\cref{prop:dominance}).
The partial-conjunction $p$-value merger ranks each node's query score among that node's
own calibration scores (a conformal $p$-value). It then discards the $A$ smallest
$p$-values and thresholds a scaled order statistic of the rest. It is valid when every
honest node's $p$-value is, but it never estimates the honest mean, so it targets a
different event (\cref{prop:pmerge}). It can therefore return smaller sets than fixed-set
(\cref{sec:experiments}) without contradicting \cref{thm:fixed-set}.
\rev{In \cref{tab:consolidated}, however, it keeps more answers than fixed-set in every row, and
with strongly peaked scorers nearly every answer at eight or 16 bits (\cref{app:revision-resolution}).
It can also drop answers that the honest-mean oracle
keeps, which the proposed rules never do within the budget: with Llama-3.2-3B judges on
OpenBookQA it does so on 409--434 of 464 replicates (\cref{tab:integrity-ledger}).}
{}
\section{Simulations and experiments}
\label{sec:experiments}

We test the rules in three settings of increasing realism. The first uses simulated
scores, where every factor can be varied. The second uses four multiple-choice
question-answering benchmarks, three of them
medical~\citep{mihaylov2018openbookqa,jin2020medqa,pal2022medmcqa,hendrycks2021mmlu},
scored by a small language model~\citep[GPT-2;][]{radford2019gpt2} over 16 document
shards. The third uses 16 copies of a stronger
model~\citep[Llama-3.2-3B;][]{grattafiori2024llama3} as nodes, some hijacked by prompt
injection\rev{, an injected instruction that replaces the one a judge receives (text planted in
retrieved passages is studied in \cref{app:revision-passages})}. Unless stated otherwise, there are 16
nodes, at most two Byzantine, and the target coverage is 90\%. Every method sees the same
data and the same corrupted nodes, whose identities are hidden. We compare the three
proposed rules with the honest-mean oracle, two other rules with guarantees, and seven
common robust summaries without one, such as the median and the plain average. The full
design, all comparators, and every number are in
\cref{app:experiments,app:results,app:judge-panel}.

Fixed-set is closest to the oracle, and its advantage grows with the declared budget.
Call the answers a rule keeps beyond the honest-mean oracle its \emph{excess}. In
simulations where the corrupt nodes send the largest possible score in both phases,
fixed-set removes about three quarters of deletion's excess. Declaring a budget of six
instead of two costs deletion about one extra answer per question, but fixed-set only
about one fifth (\cref{fig:synthetic-sweeps}a). A larger budget admits more possible honest groups; deletion can take its cutoff from one group and its query score from another, while fixed-set must find a single group that supports the candidate. A cautious budget is therefore cheap.
Sending scores with eight bits instead of exact values costs almost nothing
(\cref{fig:synthetic-sweeps}b). Extra bits shrink only the rounding part of the width in \cref{thm:set-sandwich}, not the part due to not knowing which nodes are honest, so set sizes level off.

\begin{figure}[tb]
\centering
\includegraphics[width=\linewidth]{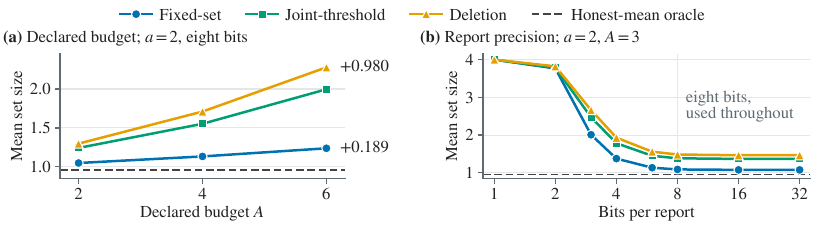}
\caption{Fixed-set stays close to the oracle as the declared budget grows (a), and
eight-bit reports lose almost nothing (b). Simulated scores; corrupt nodes send the
largest score in both phases; dashed line: honest-mean oracle; labels in (a): rise from
$A=2$ to 6. Details in \cref{tab:synthetic-budget,tab:synthetic-resolution}.}
\label{fig:synthetic-sweeps}
\end{figure}

Coverage holds whenever the assumptions do. In every simulated setting that meets the assumptions, the proposed rules reach the
target coverage (\cref{tab:synthetic-sensitivity}). They also reach it when nodes fail
at random, as in the language-model experiment below, where some runs have more
failures than declared (\cref{tab:judge-coverage}). \Cref{thm:coverage,thm:fixed-set} guarantee this within the budget; with random failures, \cref{cor:random-membership} guarantees only a lower floor, which the rules exceed.

Summaries without a guarantee can fail. When corrupt nodes report low scores during calibration and high scores at the query,
five of the seven common summaries miss the 90\% target, one by more than 30 points,
while the proposed rules do not (\cref{fig:unguarded-coverage}a). The plain average fails
the same way on the real benchmarks (\cref{tab:real-comparators}). An attack that behaves the same in both phases cannot cause this, because a summary's calibration and query scores then stay exchangeable (\cref{lem:phase-blind}). The proposed rules hold because their envelopes bound the honest mean in each phase, whatever the reports (\cref{thm:coverage}). Realistic failures are
milder, but hijacked language models that mislead only the calibration step still pull
the plain average below the target (\cref{fig:unguarded-coverage}b). They lower the high calibration scores that set its cutoff, so it drops the correct answer more often (\cref{app:judge-panel}).

\begin{figure}[tb]
\centering
\includegraphics[width=\linewidth]{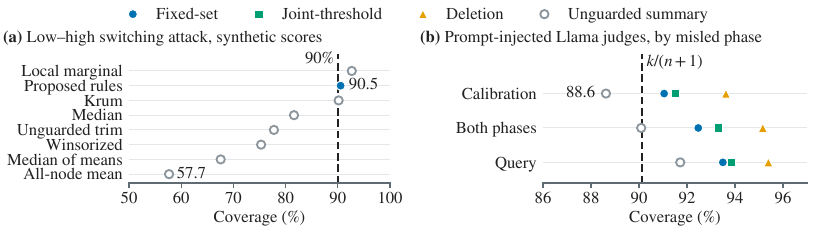}
\caption{Summaries without a guarantee can miss the coverage target where the proposed
rules hold. (a) Simulated scores; corrupt nodes report low scores at calibration and high
scores at the query. (b) Llama-3.2-3B judges, some hijacked by prompt injection \rev{(an
injected instruction)}, by the phase they mislead. Comparators are defined in \cref{app:baselines}.}
\label{fig:unguarded-coverage}
\end{figure}

\begin{figure}[tb]
\centering
\includegraphics[width=\linewidth]{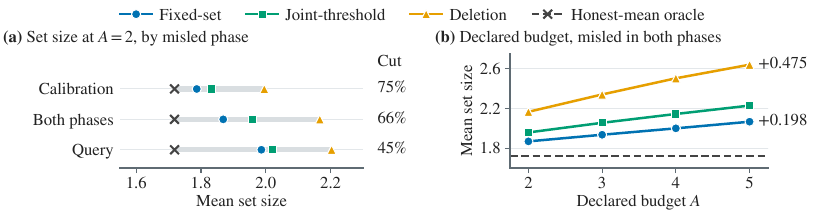}
\caption{With 16 Llama-3.2-3B judges, some hijacked by prompt injection \rev{(an injected
instruction)}, fixed-set removes
45--75\% of deletion's excess over the honest-mean oracle (a), and its advantage grows
with the declared budget (b). (a) Rows: phases in which judges are misled; bars span the
oracle to deletion, and Cut is the share of that excess fixed-set removes. (b) Judges
misled in both phases (\cref{tab:judge-sizes}).}
\label{fig:judge-panel}
\end{figure}

The savings carry over to real scorers. With the small GPT-2 scorer, fixed-set covers at least 90\% in every real-data condition within the declared budget
and, under stable maximum, removes 34--45\% of deletion's excess on every benchmark (\cref{tab:real-main}).
Because this scorer is close to chance, these numbers measure the cost of robustness,
not answer quality (\cref{app:score-quality}). With the stronger Llama judges, all
proposed rules still reach the target, and fixed-set removes 45--75\% of deletion's
excess. Its saving roughly doubles as the declared budget grows from two to five
(\cref{fig:judge-panel}), as in the simulations. Like wood in \cref{sec:deletion-method}, the removed answers are ones that no possible honest group supports. A tempting shortcut, dropping the two
judges that erred most on held-out questions, shrinks the sets further. But in two
thirds of the runs it also drops answers the honest judges support, which fixed-set never
does (\cref{tab:judge-trust}).

\begin{table}[tbp]
\rev{%
\centering\footnotesize
\setlength{\tabcolsep}{2.6pt}
\caption{Stronger scorers make the sets small at the same guarantee. Two of sixteen nodes are corrupt, declared budget $A=2$, target 90\%; means over 100 calibration splits (500 on OpenBookQA except for GPT-2), with standard errors of at most 0.4 points of coverage. Under stable maximum, fixed-set's sets stay within 0.275 answers of the honest-mean oracle; under low--high reporting, the all-node mean's mean coverage falls to 54.59\%, while fixed-set's is never more than 1.6 standard errors below 90\%. Acc.: top-answer accuracy of the sixteen-node mean on unrounded scores (\%); Cov.\ and Single: fixed-set's coverage and share of one-answer sets (\%); Bits is the report resolution: the Qwen rows are at 24 bits, the finest analyzed for each. BioASQ-Y/N and PubMedQA* (two and three options): the disjoint federation of \cref{tab:literature}. GPT-2 rows: the hash-partition cells of \cref{tab:real-main,tab:real-all}; GPT-2's trimmed-mean ranker is 24.1--28.1\% accurate (\cref{tab:accuracy-anchors}).}
\label{tab:consolidated}
\scriptsize
\begin{tabular}{@{}llrrrrrrrrrr@{}}
\toprule
 & & & & \multicolumn{6}{c}{Stable maximum} & \multicolumn{2}{c}{Low--high}\\
\cmidrule(lr){5-10}\cmidrule(l){11-12}
Scorer & Benchmark & Bits & Acc. & Oracle & Fixed & Cov. & Del. & Merger & Single & Fixed & Mean\\
\midrule
GPT-2 & OpenBookQA & 8 & -- & 3.525 & 3.654 & 92.76 & 3.721 & 4.000 & 2.4 & 91.49 & 54.59\\
 & MedQA-US & 8 & -- & 3.557 & 3.659 & 92.28 & 3.736 & 3.929 & 0.1 & 91.01 & 58.65\\
 & MedMCQA & 8 & -- & 3.540 & 3.647 & 92.42 & 3.721 & 3.964 & 0.8 & 91.22 & 59.89\\
 & MMLU-Med & 8 & -- & 3.519 & 3.596 & 91.56 & 3.653 & 3.895 & 0.9 & 90.53 & 60.94\\
\midrule
Llama-3.2-3B & OpenBookQA & 8 & 73.0 & 1.721 & 1.865 & 91.99 & 2.104 & 2.857 & 43.7 & 90.41 & 82.08\\
 & MedQA-US & 8 & 55.3 & 2.625 & 2.802 & 92.63 & 2.979 & 3.483 & 17.4 & 91.09 & 70.09\\
 & MedMCQA & 8 & 47.5 & 2.734 & 2.905 & 92.39 & 3.114 & 3.581 & 14.2 & 90.66 & 65.30\\
 & MMLU-Med & 8 & 66.9 & 1.948 & 2.114 & 92.28 & 2.322 & 2.866 & 36.6 & 90.34 & 77.76\\
\midrule
Qwen2.5-7B & OpenBookQA & 24 & 78.1 & 1.419 & 1.548 & 92.04 & 1.903 & 2.971 & 57.9 & 90.09 & 83.64\\
 & MedQA-US & 24 & 55.6 & 2.549 & 2.673 & 91.12 & 2.930 & 3.592 & 20.2 & 89.80 & 70.52\\
 & MMLU-Med & 24 & 73.3 & 1.742 & 1.865 & 91.44 & 2.152 & 2.827 & 45.4 & 89.96 & 82.84\\
\midrule
Qwen2.5-14B & OpenBookQA & 24 & 86.1 & 1.115 & 1.183 & 91.94 & 1.282 & 2.442 & 83.1 & 90.16 & 86.40\\
 & MedQA-US & 24 & 64.1 & 2.179 & 2.309 & 91.29 & 2.628 & 4.000 & 29.8 & 89.93 & 76.93\\
 & MedMCQA & 24 & 54.2 & 2.665 & 2.774 & 91.25 & 3.080 & 4.000 & 20.2 & 89.78 & 69.54\\
 & MMLU-Med & 24 & 78.7 & 1.394 & 1.486 & 91.45 & 1.745 & 3.634 & 63.4 & 89.92 & 85.05\\
\midrule
Llama-3.2-3B & BioASQ-Y/N & 8 & 71.0 & 1.417 & 1.512 & 92.84 & 1.658 & 2.000 & 48.8 & 90.14 & 82.54\\
 & PubMedQA* & 8 & 52.5 & 2.028 & 2.303 & 92.67 & 2.494 & 3.000 & 3.7 & 90.57 & 79.58\\
Qwen2.5-14B & BioASQ-Y/N & 24 & 73.9 & 1.366 & 1.457 & 93.02 & 1.728 & 2.000 & 54.3 & 90.12 & 82.30\\
 & PubMedQA* & 24 & 39.2 & 2.373 & 2.635 & 95.29 & 2.970 & 3.000 & 4.3 & 91.07 & 63.31\\
\bottomrule
\end{tabular}
\par}%
\end{table}
{}

\rev{Stronger scorers make the sets small at the same guarantee (\cref{tab:consolidated}). With
Qwen2.5-14B~\citep{qwen2024qwen25} as every node's scorer, fixed-set keeps 1.49 answers on
MMLU-Med and 1.18 on OpenBookQA, a single answer for 63\% and 83\% of questions
(\cref{tab:ladder} gives its saving over deletion). Set size follows the scorer's accuracy: on
MedQA-US and MedMCQA, where that model is right on 64\% and 54\% of questions, fixed-set keeps
2.3--2.8 answers. For the sharply peaked Qwen scorers, \cref{tab:consolidated} uses 24-bit
reports, which leave every guarantee unchanged (\cref{app:revision-resolution}). On the exam benchmarks,
under low--high reporting, the plain average misses the target by 4--25 points with each of these scorers,
and trimming, winsorizing, or the median recovers at most part of that loss. On the two literature
tasks, BioASQ-Y/N and PubMedQA*, the 16-node mean is 16.5--29.7 points more accurate than the model
without retrieval, so the federation adds accuracy as well as validity under attack
(\cref{app:revision-literature}).}

The comparisons have limits, and fixed-set is cheap to run. One rule with a guarantee, the $p$-value merger, can return smaller sets under some
attacks, because it protects a different target (\cref{prop:pmerge};
\cref{tab:synthetic-connected}). After calibration, fixed-set answers a batch of
667 questions in about 6\,ms of hub computation (\cref{tab:computation-benchmark}).

\rev{Attacks through retrieved text and changing identities behave as the theory predicts
(\cref{app:revision-passages,app:revision-identities}). A false claim added to a node's shard
makes that node pick the planted wrong answer on 75--94\% of OpenBookQA questions, yet with
the claim present in both phases every proposed rule keeps its coverage
(\cref{cor:phase-blind}); a source that is truthful during calibration and false afterward
lowers the plain average to 0.78--0.86, while fixed-set keeps 0.90--0.91. When the two corrupt
nodes are redrawn for every question and switch direction between phases, deletion stays at the
per-question oracle's coverage on every panel, as \cref{thm:rotating} guarantees, while fixed-set
and joint-threshold, which assume a fixed honest group, can fall below it, as \cref{prop:two-node}
allows.}
{}
\section{Discussion and conclusion}
\label{sec:discussion}

Robustness becomes affordable once the hub remembers that the same nodes are honest throughout:
fixed-set inference keeps every candidate for which the clean rank event can
still hold under one possible honest group and, in the interval model, no
others. Its payoff grows with the declared budget. Natural next steps are faster subset
search for large federations, open-ended generation, and partial alignment, in which
nodes share only some calibration questions.
\rev{The guarantees cover any finite candidate list fixed before the nodes score, including
lists generated at the hub with a ``none of these'' entry (\cref{prop:generated-lists}), not
free-form answers. Under high--low reporting, every rule returns nearly every answer while the
target is well above the scorer's accuracy, and the proposed rules' sets shrink fast once it is
not (\cref{tab:target-sweep}). Two exact certificates settle 94.3--100\% of fixed-set's
decisions without enumeration (\cref{prop:certificates,tab:enumeration-cost}), and deletion needs none.}
\label{sec:iclr-main-end}
{}

\clearpage
\section*{Reproducibility statement}
Proofs and the full experimental design are in the appendix (see the roadmap).
The supporting tables and figures reuse
verified numerical assets from archived outputs; \cref{app:results}
describes reaggregation, implementation checks, and the disclosed correction
to constant-column uncertainty calculations. The exact finite-alphabet verifier
in the supplement passes all 15 of its checks. The accompanying
supplementary archive contains implementation code, dependency
locks, verified numerical summaries, and a reproduction \texttt{README.md}.
It supports local verification and regeneration of the reported numerical
assets; bulk model scores and original execution archives are not included,
and the README lists what the archive omits. The archive also holds the
registered manifests and every generated table and figure. The code
accompanies the submission and will be archived with a digital object
identifier before publication.

\section*{Use of generative AI and AI-assisted technologies}
We used generative AI tools (ChatGPT, Google Gemini, and Claude) for language
editing and code formatting support. All data, results, and mathematical
derivations are the authors' own work.

{}
\clearpage
\appendix
\crefalias{section}{appendix}
\crefalias{subsection}{appendix}
\raggedbottom
\section*{Appendix roadmap}
\Crefrange{sec:related}{sec:bit-audit} give related work, the protocol, and
additional formal results. \Crefrange{app:notation}{app:experiments} give
notation, proofs, implementation details, and the experimental design. \Cref{app:results}
gives the supporting empirical results. \Cref{app:checklist} lists every
assumption the coverage and minimality theorems use. \Cref{app:judge-panel}
covers restricted subset families and the \texttt{Llama-3.2-3B-Instruct} judge
panel. \rev{\Cref{app:revision-theory,app:revision-experiments} hold the results added
during the discussion: changing identities, planted content, generated candidate lists, and
exact certificates; and the experiments with stronger scorers.} Throughout, \emph{registered} means fixed and recorded before any evaluation data
were scored, and the \emph{operational} configuration is the default of
\cref{sec:experiments} ($K=16$, $a=A=2$, eight-bit reports; \cref{app:experiments}).

\section{Extended related work and novelty boundary}
\label{sec:related}

\paragraph{Conformal RAG.}
TRAQ composes retrieval and generation uncertainty into valid answer sets in a
centralized pipeline~\citep{li2024traq}, and C-RAG certifies bounds on
generation risk, including under distribution shift~\citep{kang2024crag}.
Our finite-label formulation is closest to conformal sets over a language
model's multiple-choice answers~\citep{kumar2023conformalmcqa}, with retrieval
and scoring folded into the frozen score $s_i$ of \cref{sec:setup}.
Conformal factuality~\citep{mohri2024conformal} and group-conditional subclaim
assessment~\citep{feng2025conditionalrag} are single-site;
\citet{campos2024cpnlp} survey conformal prediction for language tasks. None
treats nodes that can alter query-time messages.

\paragraph{Federated conformal prediction.}
Federated conformal prediction (FCP) either assumes honest clients or defends
calibration only. Honest-client methods target a client mixture under partial
exchangeability~\citep{lu2023federated}, correct label
shift~\citep{plassier2023labelshift}, calibrate prescribed group mixtures with
mergeable coresets~\citep{wen2026groupconditional}, handle joint shift with a
score-distribution sketch~\citep{shi2026jointshift}, calibrate in one
round~\citep{humbert2023oneshot}, or correct wireless channel
effects~\citep{zhu2024wfcp}. Rob-FCP~\citep{kang2024robfcp} and
PRISM-FCP~\citep{lari2026prism,lari2026partial} filter Byzantine calibration
summaries but leave query-time scores unprotected. ACon$^2$ forms an online
consensus over Byzantine data sources~\citep{park2023acon2} rather than
calibrating an honest mean, and contaminated split conformal prediction studies
corrupted examples rather than corrupted nodes~\citep{clarkson2024contamination}.
FC-RAG~\citep{dubeyhuo2026bandwidth} calibrates from examples partitioned across
nodes and assumes those summaries approximate the quantile of the same-example
mean; by \cref{prop:nonidentification}, node-wise marginals alone do not imply
this. Anytime-FC-RAG gives sequential validity for honest
swarms~\citep{dubeyhuo2026anytime}.

\paragraph{Certified robust conformal prediction and $p$-value merging.}
Randomly smoothed conformal prediction (RSCP) and its improved variant RSCP+
inflate the threshold with a randomized-smoothing
envelope~\citep{gendler2022adversarial,yan2024rscpplus}. Verifiably robust
conformal prediction (VRCP) obtains score
bounds by neural-network verification~\citep{jeary2024verifiably}. Reliable
prediction sets (RPS) and CDF-aware sets (CAS)
protect sets against training or calibration poisoning and test-time
evasion~\citep{scholten2025reliable,zargarbashi2024robust}. Our coverage proof
shares their mechanism, certified score bounds that preserve the clean rank
event. Here, however, the perturbations are coordinated, candidate-specific messages
from unknown nodes in both phases, and coupling one honest set across all of
them yields the fixed-set rule and its minimality (\cref{thm:fixed-set}). Our
$p$-merger baseline is the partial-conjunction
construction~\citep{benjamini2008partial,wang2019partial} built on
order-statistic merging~\citep{ruger1978maximale,vovkwang2020combining},
specialized to at most $A$ arbitrary reports; majority vote similarly merges
dependent conformal sets~\citep{gasparin2024merging}. We give a self-contained
proof and sharp constant, not a new merger, and the merger does not target the
honest mean.

\paragraph{Robust aggregation and communication.}
Coordinatewise medians, trimmed means, the median of means and geometric
medians are standard Byzantine-robust
tools~\citep{yin2018byzantine,chen2017byzantine,pillutla2022robust,guerraoui2024primer},
and estimating an honest mean under heterogeneity costs an explicit spread
term~\citep{karimireddy2022bucketing}. Work on quantized and compressed robust learning
studies related communication
tradeoffs~\citep{suresh2017distributed,rammal2024compression}. Krum's
guarantees rest on SGD assumptions that do not transfer to candidate
scores~\citep{blanchard2017krum}, so it is only a secondary baseline. Our
one-sided deletion means solve an extremal feasible-subset problem, and the
envelopes use deterministic error bounds (\cref{lem:order-lipschitz}) rather
than assuming quantization stays unbiased after sorting.

\paragraph{Federated RAG systems and attacks.}
\citet{chakraborty2025fedragmap} map federated RAG; routing and personalized
retrieval reduce cost~\citep{dhasade2026ragroute,he2025pfedrag}. SecureCollaRAG
filters documents under an independent malicious-source
model~\citep{wang2026securecollarag}, routing hijacking attacks a different
surface~\citep{mu2026routinghijacking}, corpus poisoning corrupts the honest
score functions themselves~\citep{zhong2023poisoning,zou2025poisonedrag}, and
RobustRAG certifies answers against malicious passages within one
pipeline~\citep{xiang2026robustrag}. None provides conformal coverage against
arbitrary node messages.

\FloatBarrier

\section{Why local score distributions cannot calibrate the honest mean}
\label{app:alignment-necessity}

Exact local score distributions omit the dependence needed to calibrate a
same-example mean. A cumulative distribution function (CDF) gives the
probability that a score is at or below each threshold; write
$\|F-G\|_\infty=\sup_t|F(t)-G(t)|$ for the largest CDF discrepancy.
The following specialization of classical fixed-marginal dependence facts~\citep{makarov1982fixedmarginals} gives an explicit calibration obstruction.

\begin{proposition}[Marginals do not identify the honest-average law]
\label{prop:nonidentification}
Even with two honest nodes, no Byzantine reports, infinite calibration data,
and continuous scores in $[0,1]$, the two node-wise marginal score laws do not
identify the law of their average.  There exist two worlds with identical
$\Unif[0,1]$ node marginals whose average-score CDFs $F_0,F_1$ satisfy
\[
       \|F_0-F_1\|_\infty=\frac14.
\]
Consequently, every possibly randomized estimator $\widehat F$ whose population
input consists only of the local marginal CDFs obeys the following lower
bound, where $\E_r$ averages over the estimator's randomness in world $r$:
\[
 \max_{r\in\{0,1\}}
 \E_r\|\widehat F-F_r\|_\infty\ge\frac18.
\]
At target coverage $3/4$, any common deterministic threshold valid in both
worlds has at least $1/4$ overcoverage in one of them.
\end{proposition}

With $U\sim\Unif[0,1]$, the score pairs are $(U,U)$ in World 0, with
average-score CDF $F_0$, and $(U,T(U))$ in World 1, with CDF $F_1$, where
$T(u)=2u$ for $u\le1/2$ and $T(u)=2-2u$ otherwise; both node marginals are
uniform. More precise local summaries cannot recover the missing dependence;
alignment supplies it by keeping together the reports produced for the same
example. The conclusion concerns calibration of the honest arithmetic mean,
not other targets. \Cref{app:proof-impossibility} proves the result.

\section{Detailed protocol and communication model}
\label{sec:protocol-details}

This protocol keeps three sources of uncertainty separate: unknown honest
identities, reconstruction error from finite-bit reports, and missing honest
responses. Bounded scores make communication costs and efficiency bounds
explicit. At a fixed input and candidate, disagreement among honest nodes is
measured by the clean pointwise range
\begin{equation*}
 R_{\cH}(x,y)=\max_{i\in\cH}s_i(x,y)-
                 \min_{i\in\cH}s_i(x,y).
\end{equation*}
For the explicit finite-bit protocol and all experiments, fix
$0<\Smax<\infty$ and include deterministic clipping in the frozen score so
that $s_i\in[0,\Smax]$.  Then $R_{\cH}\le\Smax$.

\subsection{Threat and communication models}

\paragraph{Adversarial reports and missing messages.}
Authentication gives each registered node one report slot; the primary
guarantee requires
\begin{equation*}
       0\le a\le A<K.
\end{equation*}
Byzantine reports are unrestricted as in \cref{sec:setup}.  The hub's input gateway authenticates one
slot per registered node, rejects malformed contents, and fills every rejected
or missing authenticated-node slot with a registered finite sentinel, a
preassigned replacement report.  Such a
slot-level absence or malformed signed packet is handled by the protocol's
adversarial/dropout accounting, not as an infrastructure failure. Transport,
scheduler, storage, or service failures are
retried under the registered retry policy, fixed before evaluation. An unresolved infrastructure
failure aborts the run rather than being relabeled as a sentinel
report or deterministic model fallback (the frozen score map's registered
output for a retrieval or model failure; \cref{app:protocol}). An aborted run
carries no coverage guarantee.
The honest-dropout budget $D$ is the number of honest nodes allowed to miss a
report (unrelated to the report width $D_r$); the primary protocol sets $D=0$: honest nodes
respond, while a Byzantine nonresponse is simply one of the at most $A$
arbitrary reports.

\paragraph{Information available to the hub and nodes.}
The default protocol broadcasts shared calibration inputs and candidates, so
it is not a private-input protocol.  It does not pool private node corpora.
The vector form, in which every node scores all $M$ candidates of each
calibration example, does not transmit the
calibration label to nodes, but it assumes the calibration inputs themselves may
be shared.  Sensitive inputs require consent, secure execution, or a separate
privacy mechanism, none of which is analyzed here.  Secure aggregation can
hide individual submissions but cannot make their contents truthful; input
validation is complementary~\citep{bonawitz2017secure,bell2023acorn}.
Moreover, an ordinary sum-only secure-aggregation interface does not provide
the identity-resolved report array needed to enumerate feasible honest sets,
delete selected identities, or validate per-node packets.  Implementing those
rules while hiding reports would require a richer secure-computation and
authenticated-input-validation protocol, which we do not analyze.  Not
transmitting the recorded label is an interface restriction, not a
confidentiality guarantee: a capable node may infer the answer from the input.

\paragraph{Reconstructing finite-bit scores.}
Quantization replaces each real score by one of finitely many values that the
hub can decode. For integer $b\ge1$, let $Q_b$ be the nearest-neighbor scalar quantizer with $2^b$
equally spaced reconstruction levels including $0$ and $\Smax$:
\begin{equation}
 Q_b(u)=\frac{\Smax}{2^b-1}
 \operatorname{round}\!\left(\frac{(2^b-1)u}{\Smax}\right),
 \qquad
 \rho(b)=\frac{\Smax}{2(2^b-1)}.
 \label{eq:quantizer}
\end{equation}
For $u\in[0,\Smax]$, at an exact cell midpoint, \textup{round} selects the
smaller reconstruction level, making $Q_b$ deterministic.  Thus
$|Q_b(u)-u|\le\rho(b)$ on the declared score domain.
Heterogeneous registered maps $Q_{i,r}$ satisfying
\cref{eq:directional-quantizer} are allowed; when $Q_{i,r}=Q_{b_i^r}$ with
depths $b_i^r\ge1$, set $b_r=\min_i b_i^r$ and
$\rho_r^-=\rho_r^+=\rho(b_r)$.

\paragraph{Aggregation notation.}
For real reports $v_1,\ldots,v_K$, let $v_{(1)}\le\cdots\le v_{(K)}$ be their
order statistics. For a symmetric comparator, choose a registered trim $m$ satisfying
$A\le m$ and $2m<K$, and define
\begin{equation*}
 \TM_m(v_1,\ldots,v_K)
 =\frac1{K-2m}\sum_{\ell=m+1}^{K-m}v_{(\ell)}.
\end{equation*}
The primary symmetric comparator sets $m=A$; over-trimming is an efficiency
diagnostic.

\paragraph{Label-aware scalar variant.}
The vector protocol does not transmit the label to nodes and makes calibration
information operationally match deployment.  If label disclosure is acceptable,
each honest node may instead transmit only the true-label score.  The guarded
arbitrary-message results remain unchanged and calibration communication falls
by a factor of $M$.  The unguarded matched-pipeline result, which relies on
identical scoring behavior in calibration and prediction
(\cref{lem:phase-blind}), then requires a separate exchangeability
justification, because Byzantine nodes see the label during calibration but
not deployment.

\section{Computational cost and exact decisions}
\label{sec:computation-details}
Fixed-set enumerates every feasible honest set, obtained by deleting at most $A$
identities; their number is
\begin{equation*}
  N_A(K)=\sum_{d=0}^{A}\binom{K}{d}.
\end{equation*}
Thus
$N_2(16)=137$, $N_3(16)=697$, and $N_3(64)=43{,}745$.  When $k\le n$,
complement sums (subtracting deleted-node sums from the total) and
linear-time order-statistic selection give calibration cost
$O\{nK+n\sum_{d=0}^A\binom{K}{d}(d+1)\}$ and, after storing the $T_k(H;V)$
values, fixed-set query cost
$O\{MK+M\sum_{d=0}^A\binom{K}{d}(d+1)\}$.  Ordinary sorting adds the usual
logarithmic factor.
The registered quantized implementation uses integer code sums and exact
rational comparisons to preserve boundary ties and the set hierarchy;
\cref{app:exact-grid-decisions} gives the arithmetic contract.

\paragraph{Measured CPU cost.}
Exact fixed-set decisions complete in milliseconds for a batch of queries
at the operational budget in a controlled local benchmark.
We use the retained OpenBookQA score cache from both corpus partitions,
with $K=16$, $M=4$, eight-bit clean reports, nominal 90\% coverage, and
declared budgets $A\in\{2,3\}$. The first 333 cache-order examples calibrate
the rule and the remaining 667 form one query batch. This fixed split
measures computation; it adds no coverage estimate to the registered study.
All methods receive the same quantized arrays. Measurements use an Intel
Core i7-1068NG7 CPU at a nominal 2.30\,GHz, 32\,GiB RAM, Python 3.12.7,
and NumPy 1.26.4, with one process and native thread pools limited to one.
Each stage receives three warm-ups and 21 timed repetitions with rotating
execution order; the code archive retains every timing and its quartiles.

The benchmark separates reusable calibration state from query work while
preserving the production arithmetic and diagnostic outputs. Here
``production API'' is the released library function, which performs
calibration and query in one call; the benchmark's ``staged'' code splits
the two phases so that cached calibration state can be timed. Across both
partitions, budgets, and three methods, all 12 staged-versus-unstaged output
comparisons match the unchanged grid implementation (the exact
integer-arithmetic implementation of \cref{app:exact-grid-decisions}) exactly,
and all prediction sets match the public production
API. \Cref{tab:computation-benchmark} also times the unchanged combined grid
call. Timings exclude retrieval, model scoring, communication, input
loading, quantization, grid validation, and true-label score extraction.

\begingroup
\let\resulttablefloat\table
\renewcommand{\table}[1][]{\resulttablefloat[H]}%
\begin{table}[H]
\centering
\small
\setlength{\tabcolsep}{3pt}
\caption{Cached fixed-set inference processes 667-query batches in milliseconds. Times are medians of 21 measurements, with ranges across the two retained partitions; $K=16$, $n=333$, $M=4$, eight-bit reports, and nominal 90\% coverage. Query uses staged cached kernels; Combined uses the production function.}
\label{tab:computation-benchmark}
\begin{tabular}{@{}clrrrrr@{}}
\toprule
$A$ & Rule & Cal. & Query batch & Combined & Cache & Query peak\\
 & & ms & ms & ms & KiB & MiB\\
\midrule
2 & Fixed-set & 2.96--3.19 & 6.27--6.47 & 10.04--10.92 & 21.16 & 6.33\\
2 & Joint-threshold & 3.22--3.28 & 0.68--0.76 & 4.13--4.14 & 0.33 & 0.41\\
2 & Deletion & 0.35--0.40 & 0.57--0.60 & 0.92--1.00 & 0.33 & 0.41\\
\midrule
3 & Fixed-set & 13.44--14.42 & 21.13--22.54 & 35.18--38.15 & 104.28 & 30.55\\
3 & Joint-threshold & 13.82--14.58 & 0.77--0.82 & 15.54--15.70 & 0.33 & 0.41\\
3 & Deletion & 0.33--0.82 & 0.59--0.69 & 1.04--1.14 & 0.33 & 0.41\\
\bottomrule
\end{tabular}
\end{table}
\endgroup{}

For fixed-set at $A=2$, dividing the batch time by 667 gives
9.39--9.70\,$\mu$s per query (amortized, not isolated latency).

The memory columns distinguish retained state from allocations during
query processing. Cached state counts reachable Python objects and owned
NumPy storage; query peak is the maximum of three separate
\texttt{tracemalloc} measurements per partition, excluding preallocated
inputs and cache. It includes traced NumPy buffers but is not total process
memory. \Cref{tab:computation-benchmark} reports ranges of partition-specific medians
for timing, not confidence intervals. Processor scheduling and frequency
are uncontrolled; the benchmark is a local measurement of these kernels,
not a claim about end-to-end deployment latency or scaling to larger budgets.

\section{Endpoint attacks and conservative budgets}
\label{app:endpoint-attacks}

\Cref{tab:attack-directions} lists the three endpoint attacks of
\cref{sec:theory}; the corollary below makes their roles precise.

\begin{table}[H]
\centering
\small
\caption{Each endpoint attack sends one end of the report range in each phase:
high--low yields the largest proposed sets, and low--high at $a=A$ cancels the
corrupt contribution. All three rules retain their
coverage guarantee under the stated assumptions; efficiency depends on
the direction of the reports.}
\label{tab:attack-directions}
\begin{tabularx}{\linewidth}{@{}lccX@{}}
\toprule
Attack & Calibration & Query & Role in the comparison\\
\midrule
Stable maximum & $\rhi$ & $\rhi$ & Measures gains from fixed membership across phases\\
High--low & $\rhi$ & $\rlo$ & Maximizes proposed-set inclusion for a fixed honest transcript\\
Low--high & $\rlo$ & $\rhi$ & At $a=A$, recovers the padded honest decoded comparison\\
\bottomrule
\end{tabularx}
\end{table}

\begin{corollary}[Endpoint attacks and conservative budgets]
\label{cor:endpoints-budget}
If $k=n+1$, all set conclusions below are immediate, since every rule then
returns $\cY$.  Suppose $k\le n$.
Fix the honest decoded reports and suppose Byzantine reports must lie in
$[\rlo,\rhi]$.  For each of $\cC^{\rm fs}$, $\cC^{\rm jt}$, and $\cC^{\rm del}$,
setting every Byzantine calibration report to $\rhi$ and every Byzantine query
report to $\rlo$ produces a prediction set containing the set under every other
Byzantine transcript.  Thus the high--low endpoint attack is pathwise binding
for set inclusion, and hence for set size and full-set status, within this
bounded-message class.  It also maximizes proposed-set coverage for fixed
clean data and is therefore not an undercoverage stress.  Symmetrically,
setting every Byzantine calibration report to $\rlo$ and every Byzantine query
report to $\rhi$ produces a set contained in the set under every other Byzantine
transcript in $[\rlo,\rhi]$, whatever the number $a$ of Byzantine nodes. Low--high therefore
minimizes proposed-set coverage for fixed clean data within this class.
Neither extremal result extends to the $p$-value merger or to point accuracy (whether the
top-ranked answer is correct).

Conversely, suppose $a=A$ and low--high sends calibration values no larger
than every honest decoded value and query values no smaller than every honest
decoded value.  Define the honest-decoded padded set
\begin{equation*}
 \cC^{\rm dec,\cH}_\alpha(x)=
 \left\{y:\frac1h\sum_{i\in\cH}w_i(x,y)
 \le T_k(\cH;V)+g_\rho\right\}.
\end{equation*}
Then all three proposed rules equal $\cC^{\rm dec,\cH}_\alpha(x)$ pathwise.
If, in addition, the registered reconstruction errors are zero, then
\begin{equation*}
 \cC^{\rm fs}_\alpha(x)=\cC^{\rm jt}_\alpha(x)
 =\cC^{\rm del}_\alpha(x)=\cC^\circ_\alpha(x)
\end{equation*}
pathwise.  If the $n+1$ clean true-label scores are exchangeable and almost
surely distinct, their common coverage is exactly $k/(n+1)$ for $k\le n$.

Finally, for a fixed transcript, fixed directional padding, and
$0\le A_1\le A_2<K$, each of the three proposed sets at budget $A_1$ is
contained in its counterpart at budget $A_2$.
\end{corollary}

\paragraph{Interpretation.}
High coverage under high--low can coexist with an uninformative
all-answer set: this attack tests efficiency, not resistance to undercoverage.

The final monotonicity makes the tradeoff in choosing $A$ explicit.  In the
primary no-dropout protocol the additive padding $g_\rho$ is independent of
$A$, but conservative budgets still enlarge the feasible-set family and widen
the deletion aggregates.  Underestimating $A$ invalidates all three guarantees
when $a>A$; no universal fragility ordering against the symmetric rule or
$p$-merger follows without a distribution and attack model.

\section{Symmetric comparison and its limitations}
\label{sec:symmetric-details}

We first define the guarded symmetric comparator, whose guard is an additive
threshold adjustment,
then state precisely when all proposed sets are contained in it.
Take $A\le m$, $2m<K$, symmetric reconstruction error
$\rho_r=\rho_r^-=\rho_r^+$, and
define
\begin{align*}
 \widehat R_j^{\rm sym,m}&=\TM_m(\widetilde s^{\calib}_{1j}(Y_j),\ldots,
                                  \widetilde s^{\calib}_{Kj}(Y_j)),\\
 \widehat s^{\rm sym,m}(x,y)&=\TM_m(\widetilde s^{\qry}_1(x,y),\ldots,
                                  \widetilde s^{\qry}_K(x,y)).
\end{align*}
For $k\le n$, let $\widehat R^{\rm sym,m}_{(k)}$ be the $k$th order statistic of
$\widehat R^{\rm sym,m}_1,\ldots,\widehat R^{\rm sym,m}_n$.  For this comparison,
take $D_r$ to be the full decoded-alphabet width, the width of the range of
values a report can decode to, and set
$g_r(m)=mD_r/(K-A)$.  For $k\le n$, the comparator uses threshold
$\widehat R^{\rm sym,m}_{(k)}+g_{\calib}(m)+g_{\qry}(m)
 +\rho_{\calib}+\rho_{\qry}$ and prediction set
\[
 \cC^{\rm sym,m}_\alpha(x)=
 \left\{y:\widehat s^{\rm sym,m}(x,y)
 \le \widehat R^{\rm sym,m}_{(k)}+g_{\calib}(m)+g_{\qry}(m)
 +\rho_{\calib}+\rho_{\qry}\right\}.
\]
When $k=n+1$, the comparator, like the proposed rules, returns $\cY$.  Write
$\cC^{\rm sym}=\cC^{\rm sym,A}$ for the primary comparator.

\begin{proposition}[Dominance over every eligible full-alphabet symmetric guard]
\label{prop:dominance}
Suppose $A\le m$, $2m<K$, all methods use the same reports and symmetric
quantization bounds, and $g_r(m)=mD_r/(K-A)$ is computed from the full
decoded-alphabet width.  Then
\begin{equation*}
 \cC^{\rm fs}_\alpha(x)\subseteq\cC^{\rm jt}_\alpha(x)
 \subseteq\cC^{\rm del}_\alpha(x)
 \subseteq\cC^{\rm sym,m}_\alpha(x)
       \quad\text{for every }x
\end{equation*}
pathwise.  Over-trimming itself does not break the result when its guard grows
with $m$.  The claim can fail if an $m>A$ comparator retains the undersized
$AD_r/(K-A)$ guard or uses genuinely tighter certified honest-range
information; it is not an unconditional dominance statement.  For example,
with $K=5$, $A=1$, $m=2$, $n=k=1$ (so $\alpha\ge1/2$), zero reconstruction error, reports in
$[0,1]$, calibration reports $v=(0,0,0,1,1)$ and query reports
$w=(0,0,1,1,1)$, deletion keeps the candidate because
$\Dn_A(w)=\Up_A(v)=1/2$, whereas the comparator with the undersized guard
$2AD_r/(K-A)=1/2$ drops it, since $\TM_2(w)=1>\TM_2(v)+1/2=1/2$.
\end{proposition}

The guard is needed because trimming both tails moves the aggregate away from
the honest mean when honest nodes disagree, even without an attack. For exact
honest score values $x_i$, write
$R_{\cH}=\max_{i\in\cH}x_i-\min_{i\in\cH}x_i$ and
$\bar x_{\cH}=h^{-1}\sum_{i\in\cH}x_i$, and write $(x_{\cH},v_{\cA})$ for the
report vector with these honest entries and Byzantine entries $v_{\cA}$.

\begin{lemma}[Sharp contaminated symmetric trimmed mean]
\label{lem:sharp-tm}
If $a\le A\le m$ and $2m<K$, then for arbitrary finite Byzantine reports
\begin{equation*}
 \left|\TM_m(x_{\cH},v_{\cA})-\bar x_{\cH}\right|
 \le \frac{m}{h}R_{\cH}.
\end{equation*}
The attacked trimmed mean lies between the means of the smallest and largest
$h-m$ honest values, and the coefficient $m/h$ is worst-case sharp, even
when $a=0$.
Deterministic symmetric quantization adds at most $\rho_r$.
\end{lemma}

Thus, with a registered bound $\bar R_r$ on the honest score range, the
symmetric comparator is valid with per-channel guard $m\bar R_r/(K-A)+\rho_r$,
which $m=A$ minimizes; the $m>A$ sweep is a conservatism diagnostic.  Phase
switching makes some padding unavoidable for this rule:

\begin{proposition}[Low--high coverage lower bound]
\label{prop:low-high-lower}
Let $a=A$, $2A<K$, $h=K-A$, and suppose $k\le n$.  Let honest scores have the
location form
$s_i(Z_j)=S_j+\xi_i$: all honest nodes share the example-specific score $S_j$,
and $\xi_i$ is node $i$'s fixed offset. The offsets have range $R$ and the $S_j$ are
i.i.d.\ continuous. Let $S_{(k)}$ be the $k$th order statistic of
$S_1,\ldots,S_n$. Write $\xi_{(1)}\le\cdots\le\xi_{(h)}$ for the ordered
honest offsets. Assume exact honest reports and symmetric trim $m=A$.
Let Byzantine reports lie below every honest report in calibration and above
every honest report at prediction.  Define
\begin{align*}
 L_\xi&=\frac1{h-A}\sum_{i=1}^{h-A}\xi_{(i)},&
 U_\xi&=\frac1{h-A}\sum_{i=A+1}^{h}\xi_{(i)},&
 \Delta_\xi&=U_\xi-L_\xi.
\end{align*}
For a symmetric-trimmed threshold with total padding $g$, coverage equals
\begin{equation}
 \Pp\{S_{n+1}\le S_{(k)}-(\Delta_\xi-g)\}.
 \label{eq:low-high-coverage}
\end{equation}
For every $g<\Delta_\xi$, a bounded uniform law for $S$ with support width
less than $\Delta_\xi-g$ makes this probability zero.  Moreover,
\begin{equation}
 \sup_{\operatorname{range}(\xi)\le R}\Delta_\xi
 =R\min\left\{1,\frac{A}{K-2A}\right\}.
 \label{eq:low-high-minimax}
\end{equation}
Hence any fixed total-padding constant uniformly valid over this subclass must
be at least the right-hand side of \cref{eq:low-high-minimax}.  This lower bound
is generally smaller than the sufficient $2AR/(K-A)$ guard.
\end{proposition}

Switching from low calibration reports to high query reports changes which
honest offsets survive trimming, and a narrow clean-score law cannot absorb the
gap. The bound concerns this location family and this symmetric rule; it
imposes no padding on fixed-set inference.

\section{Matched pipelines and the robust \texorpdfstring{$p$}{p}-value alternative}
\label{sec:merger-details}

\begin{lemma}[A static phase-blind map preserves exchangeability]
\label{lem:phase-blind}
Let $\mathcal F$ contain all state frozen before calibration, and let
$\Xi_j$ represent any additional randomness used on example $j$. Conditional on
$\mathcal F$, suppose $Z_1,\ldots,Z_{n+1}$ are exchangeable,
$\Xi_1,\ldots,\Xi_{n+1}$ are i.i.d.\ and independent of the examples, and one
fixed measurable map satisfies
\begin{equation*}
       T_j=\phi(Z_j,\Xi_j;\mathcal F),\qquad j\in[n+1].
\end{equation*}
Then $T_1,\ldots,T_{n+1}$ are exchangeable conditional on $\mathcal F$.
Consequently, an unpadded split-conformal rule applied to the same attacked
aggregate in both phases has coverage at least $k/(n+1)$.
\end{lemma}

A sufficient condition for the lemma's single fixed map is static, example-independent node state; matched
candidates, information exposure, bit depths and reconstruction maps; identical
packet handling, tie rules and aggregation; and a fixed phase-blind Byzantine
map with only i.i.d.\ per-example randomness.  The lemma explains unguarded
coverage under stable attacks but says nothing about the honest-mean score and
does not cover switching attacks; deletion's validity comes from \cref{thm:coverage}.

An honest-mean-free alternative uses node-wise conformal $p$-values, which the
hub computes from the same aligned calibration reports:
\begin{equation}
 p_i(x,y)=\frac{1+\sum_{j=1}^n
   \ind\{\widetilde s^{\calib}_{ij}(Y_j)
          \ge\widetilde s^{\qry}_i(x,y)\}}{n+1}
 \label{eq:node-pvalue}
\end{equation}
This is the usual unsmoothed conformal $p$-value. A $p$-value is
\emph{superuniform} if
$\Pp\{p_i(X,Y)\le t\}\le t$ for every $t\in[0,1]$: an honest node reports a
small $p$-value for the true answer no more often than its nominal level.
For honest nodes whose scoring transformation and quantizer are matched
across phases and whose calibration/future pairs are exchangeable for the same
target law, applying the exchangeable-rank argument to the negated scores
proves marginal superuniformity, with the weak comparison making ties
conservative.
Matching the transformation alone does not establish this
common-target condition.  Shared calibration induces dependence among these
$p$-values, which the next result permits.

\begin{proposition}[Partial-conjunction/R\"uger merger with up to $A$ arbitrary reports]
\label{prop:pmerge}
Let at most $A$ of $K$ reported $p$-values be arbitrary, and suppose every
honest true-label value $p_i(X,Y)$ is marginally superuniform; no dependence
assumption is made.
For any prespecified $\ell\in\{1,\ldots,K-A\}$, sort all reports increasingly,
writing $p_{(r)}$ for the $r$th smallest reported value, and define
\begin{equation*}
 p^{\rm merge}_\ell
 =\min\left\{1,\frac{K-A}{\ell}p_{(A+\ell)}\right\}.
\end{equation*}
Then $p^{\rm merge}_\ell$ is superuniform, and the factor $(K-A)/\ell$ is
sharp under arbitrary honest dependence.  In particular, if $2A<K$, the
prespecified majority-style choice $\ell=A+1$ includes candidate $y$ exactly
when
\begin{equation}
 p_{(2A+1)}(x,y)>\frac{\alpha(A+1)}{K-A}.
 \label{eq:pmerge-cutoff}
\end{equation}
The denominator $K-A$ is sharp; replacing it by $K$, as in a direct
conservative counting bound, remains valid but lowers the inclusion cutoff
and can produce larger prediction sets.
\end{proposition}

Each honest $p$-value must be valid for the common target, which federated
conformal theory makes explicit through exchangeability or shift
assumptions~\citep{lu2023federated,plassier2023labelshift,barber2023beyond}.

\section{Bit budgets and certificate-width audit}
\label{sec:bit-audit}

Extra bits shrink the reconstruction terms of the padding and of the set-size
width geometrically, but never the term due to unknown membership.  Once
divided among examples and candidates, the bit budget controls reconstruction
error. This audit distinguishes the padding added by the
algorithm, denoted $G$, from the extra clean-score threshold width in the
set-size bound, denoted $W$; a large $W$ does not itself imply a large realized
prediction set.
Suppose node $i$ has $B_i^{\qry}$ bits per query and
$B_i^{\calib}$ bits for its complete $n$-example calibration transcript.  With
\begin{equation*}
 b_i^{\qry}=\left\lfloor\frac{B_i^{\qry}}{M}\right\rfloor,
 \quad
 b_i^{\calib}=\left\lfloor\frac{B_i^{\calib}}{nM}\right\rfloor,
 \quad
 b_r=\min_i b_i^r,
\end{equation*}
and $b_r\ge1$, the uniform quantizer yields
\begin{equation*}
 \rho_r=\frac{\Smax}{2(2^{b_r}-1)}.
\end{equation*}
Across $n$ examples, the vector protocol sends
$nM\sum_i b_i^{\calib}$ score bits, while one test query sends
$M\sum_i b_i^{\qry}$ score bits, excluding packet metadata; unused remainder
bits are allowed, so these payloads fit within the registered budgets.
The deletion threshold padding is $G_{\rm del}=\rho_{\calib}+\rho_{\qry}$, the
padding $g_\rho$ of \cref{sec:fixed-membership-method} with symmetric errors.
For a common decoded alphabet $[0,\Smax]$, its sandwich width, the outer-bound
width $W_{\rm del}$ of \cref{thm:set-sandwich} with $D_r=\Smax$, is
\begin{equation}
 W_{\rm del}=2\min\!\left\{1,\frac{A}{K-A}\right\}\Smax
              +2\rho_{\calib}+2\rho_{\qry}.
 \label{eq:explicit-deletion-width}
\end{equation}
The primary $m=A$ symmetric comparator instead has threshold padding
$G_{\rm sym}=2A\Smax/(K-A)+\rho_{\calib}+\rho_{\qry}$ and sandwich width
$W_{\rm sym}=2G_{\rm sym}$.

At the operational cell $(K,A,b_{\calib},b_{\qry})=(16,2,8,8)$ with
$\Smax=1$, the prespecified certificate-width screen, which compares each $G$
and $W$ with the score range $\Smax$, is
\begin{align*}
 G_{\rm del}&=0.00392,& W_{\rm del}&=0.29356,\\
 G_{\rm sym}&=0.28964,& W_{\rm sym}&=0.57927,
\end{align*}
and at $\alpha=0.1$ the $p$-merger cutoff is $0.1\times3/14=0.02143$.

At the prespecified four-bit stress cell
$(K,A,b_{\calib},b_{\qry})=(16,3,4,4)$ with $\Smax=1$
(defined in \cref{app:experiments}),
\begin{align*}
 G_{\rm del}&=0.06667,& W_{\rm del}&=0.59487,\\
 G_{\rm sym}&=0.52821,& W_{\rm sym}&=1.05641.
\end{align*}
The corresponding stress-cell $p$-merger cutoff is
$0.1\times4/13=0.03077$.
There the symmetric width exceeds the whole score range and its padding
consumes more than half of it, while the deletion certificate stays narrower.
These are design diagnostics, neither sufficient nor necessary for useful
realized sets, so every cell also reports realized set size and full-set rate
(the share of queries whose set contains all $M$ candidates).
The decoded-alphabet width in \cref{eq:explicit-deletion-width} is the price
of an upper efficiency certificate, not a validity guard.

\FloatBarrier
{}
\section{Notation and elementary order-statistic facts}
\label{app:notation}

This appendix fixes notation and elementary facts, and together with
\crefrange{app:proof-impossibility}{app:proof-coverage} and~\labelcref{app:random-membership}
it supplies complete proofs of every formal statement in the main paper; the
restricted-family results are proved in \cref{app:judge-panel}.  It also fixes the
tie and quantile conventions the proofs use.  The proof sequence follows the argument of the paper:
\cref{app:proof-impossibility} explains why node-wise marginals are insufficient,
\cref{app:proof-aggregation} derives the score envelopes, and
\cref{app:proof-coverage} turns those envelopes into coverage and efficiency
guarantees.  Readers interested in implementation can begin with
\cref{app:protocol}; the completed study's design and results appear in
\cref{app:experiments,app:results}.
\Cref{app:checklist} collects the assumptions needed for each guarantee, and
\cref{tab:notation} lists the principal symbols.

\begin{table}[H]
\centering
\footnotesize
\caption{Principal notation, roughly in the order the main text introduces it, followed by
symbols used only in the appendix.}
\label{tab:notation}
\begin{tabularx}{\linewidth}{@{}l>{\raggedright\arraybackslash}X@{\quad}l>{\raggedright\arraybackslash}X@{}}
\toprule
Symbol & Meaning & Symbol & Meaning\\
\midrule
$K$, $[K]$ & number of nodes; $\{1,\ldots,K\}$ & $V=(v_{ij})$ & decoded calibration reports on the correct answer\\
$\cA$, $\cH$ & Byzantine and honest node sets & $w(x,y)$ & decoded query reports for candidate $y$\\
$a=|\cA|$, $A$ & realized and declared corruption & $\mathfrak H_A$ & feasible honest subsets, $|H|\ge K-A$\\
$h=|\cH|$ & number of honest nodes & $c_j(H;V)$ & mean report of $H$ on example $j$\\
$\cY$, $M$ & answer space and its size & $T_k(H;V)$ & $k$th smallest of $c_1,\ldots,c_n$\\
$s_i$, $s_{\cH}$ & node score; honest-mean score & $\tau_{A,k}$ & largest $T_k(H;V)$ over $\mathfrak H_A$\\
$Q_{i,r}$ & quantizer on channel $r$ & $\Delta_{A,k}$ & fixed-set margin \eqref{eq:fixed-set-margin}\\
$\rho_r^-,\rho_r^+$ & directional error bounds & $\Up_A$, $\Dn_A$ & means of the largest / smallest $K-A$\\
$g_\rho$ & padding $\rho_\calib^-+\rho_\qry^+$ & $\cC^\circ,\cC^{\rm fs},\cC^{\rm jt},\cC^{\rm del}$ & oracle and proposed sets\\
$n$, $\alpha$, $k$ & calibration size, level, $\lceil(n+1)(1-\alpha)\rceil$ & $D_r$, $W_{\rm del}$ & report width; outer bound width\\
$R_j$, $R_{(k)}$ & clean calibration scores; $k$th smallest & $\varepsilon_A$ & $\Pp\{a>A\}$ under a failure law\\
$Z_j=(X_j,Y_j)$ & example $j$; $j=n+1$ is the future question & $\OS_k$ & $k$th smallest of $n$ values\\
$\cX$ & input space & $\widehat R_j$, $\widehat q^{\rm del}_\alpha$ & $\Up_A(v_{\cdot j})$; deletion threshold \eqref{eq:deletion-set}\\
$\bar v_H$ & mean of $v$ over the identities in $H$ & $\mu_r$ & membership term $\min\{1,A/(K-A)\}D_r$\\
$\rlo$, $\rhi$ & ends of the decodable report range & $m$, $\TM_m$ & symmetric trim count; mean after dropping $m$ per tail\\
\midrule
\multicolumn{4}{@{}l}{\emph{Appendix only}}\\
$\Smax$ & score bound, $s_i\in[0,\Smax]$ & $R_{\cH}$ & honest range, $\max_{i\in\cH}s_i-\min_{i\in\cH}s_i$\\
$D$ & honest-dropout budget (not $D_r$) & $\rho_r$ & symmetric error bound, $\rho_r^-=\rho_r^+$\\
\bottomrule
\end{tabularx}
\end{table}

\paragraph{Terms used throughout.}
Italicized terms are defined where first used in \crefrange{sec:introduction}{sec:experiments}.

Throughout, repeated values are counted separately in every multiset.  Thus
ties do not make the trimmed mean ambiguous: any permutation of equal values
has the same sum over the retained middle block.

For a vector $u=(u_1,\ldots,u_N)$, write
$u_{(1)}\le\cdots\le u_{(N)}$ for its order statistics.  We repeatedly use the
following elementary fact: moving every coordinate by at most $\eta$ moves
every ordered value, and hence every trimmed mean, by at most $\eta$.

\begin{lemma}[Sup-norm stability of order statistics]
\label{lem:order-lipschitz}
If $u,v\in\R^N$ and $\max_i|u_i-v_i|\le\eta$, then
\[
        |u_{(j)}-v_{(j)}|\le\eta,
        \qquad j=1,\ldots,N.
\]
Consequently, for $2m<N$,
\[
 |\TM_m(u)-\TM_m(v)|\le\eta.
\]
\end{lemma}

\begin{proof}
At least $j$ coordinates of $u$ are no larger than $u_{(j)}$.  The
corresponding coordinates of $v$ are no larger than $u_{(j)}+\eta$, so
$v_{(j)}\le u_{(j)}+\eta$.  Interchanging $u$ and $v$ gives
$u_{(j)}\le v_{(j)}+\eta$.  Averaging the inequalities over the retained
indices $j=m+1,\ldots,N-m$ proves the trimmed-mean statement.
\end{proof}

Quantiles remain well-defined when score distributions have ties. We use the
generalized inverse $F^{-1}(p)=\inf\{t:F(t)\ge p\}$, the smallest threshold
at which the CDF reaches $p$. All score
functions, corpora, model weights, prompts, and the honest set may be random
before calibration.  Every probability statement is conditional on the
information (formally, the sigma-field) that fixes them.  We suppress that
conditioning to keep notation readable.

\section{Proof of the nonidentification result}
\label{app:proof-impossibility}

The proof constructs two systems with identical local score distributions
but different distributions for their mean.  A procedure that sees only the
local distributions cannot distinguish the systems, which yields both the
distribution-estimation lower bound and the threshold example.

\begin{proof}[Proof of \cref{prop:nonidentification}]
Let $U\sim\Unif[0,1]$ and define the continuous tent map
\[
 T(u)=
 \begin{cases}
  2u,&0\le u\le1/2,\\
  2-2u,&1/2<u\le1.
 \end{cases}
\]
For $t\in[0,1]$,
\[
 \Pp\{T(U)\le t\}
 =\Pp\{U\le t/2\}+\Pp\{U\ge1-t/2\}
 =t.
\]
Thus $T(U)\sim\Unif[0,1]$.

In World 0 take $(s_1(U),s_2(U))=(U,U)$.  Its average $M_0=U$ has CDF
$F_0(t)=t$ on $[0,1]$.  In World 1 take
$(s_1(U),s_2(U))=(U,T(U))$.  Both local marginals are again uniform, while
\[
 M_1=\frac{U+T(U)}2
 =
 \begin{cases}
  3U/2,&U\le1/2,\\
  1-U/2,&U>1/2.
 \end{cases}
\]
For $0\le t<1/2$, only the first branch can be no larger than $t$, so
\[
  F_1(t)=\Pp\{U\le2t/3\}=\frac{2t}{3}.
\]
For $1/2\le t<3/4$, both branches contribute and their preimages are
disjoint:
\[
 \begin{aligned}
 F_1(t)
 &=\Pp\{U\le2t/3\}
   +\Pp\{U>1/2,\ 1-U/2\le t\}\\
 &=\frac{2t}{3}+\Pp\{U\ge2(1-t)\}\\
 &=\frac{8t}{3}-1.
 \end{aligned}
\]
Therefore
\[
F_1(t)=
\begin{cases}
0,&t<0,\\
2t/3,&0\le t<1/2,\\
8t/3-1,&1/2\le t<3/4,\\
1,&t\ge3/4.
\end{cases}
\]
On $[0,1/2)$ the absolute gap from $F_0$ is $t/3$, whose supremum is $1/6$.
On $[1/2,3/4)$ the signed gap is $F_1(t)-F_0(t)=5t/3-1$, increasing from
$-1/6$ to $1/4$.  On $[3/4,1]$ it is $1-t$, maximized at $1/4$.  Hence
$\|F_0-F_1\|_\infty=1/4$.

The population input to a marginal-only estimator is the same pair of uniform
CDFs in both worlds.  Consequently $\widehat F$ has the same distribution in
both.  For every realization, the triangle inequality gives
\[
 \frac14=\|F_0-F_1\|_\infty
 \le\|\widehat F-F_0\|_\infty+\|\widehat F-F_1\|_\infty.
\]
Take expectation with respect to the estimator's randomization and use
$\max(x,y)\ge(x+y)/2$ to obtain the lower bound $1/8$.

Finally, $F_0^{-1}(3/4)=3/4$.  Solving
$8t/3-1=3/4$ gives $F_1^{-1}(3/4)=21/32$.  Any single deterministic threshold
with coverage at least $3/4$ in World 0 must be at least $3/4$.  Since
$M_1\le3/4$ almost surely, that threshold has coverage one in World 1, which
is overcoverage $1/4$ relative to the target.  This completes the proof.
\end{proof}

\begin{remark}[What the proposition does and does not rule out]\label{rem:nonidentification-scope}
The proposition applies to protocols whose population information consists
only of the separate local CDFs, including arbitrarily fine histograms and
exact local quantiles.  It does not rule out a protocol that retains paired
score vectors on shared examples, transmits cross-moments under a correctly
specified copula model, or defines its deployment score to be a different
functional whose law is identified from the transmitted summaries.  The
aligned protocol takes the first route.
\end{remark}

\section{Proofs for deletion and symmetric aggregation}
\label{app:proof-aggregation}

The first proof finds the largest and smallest honest means compatible with
a reported score vector and the corruption budget.  These envelopes support
coverage even when Byzantine reports are unbounded.  The subsequent
bounded-report and symmetric-trimming results quantify how far an
aggregate can lie from the clean honest mean; their additional assumptions
are stated separately.

\begin{theorem}[Exact deletion envelopes and pointwise optimality]
\label{thm:extremal-envelopes}
For every $v\in\R^K$, write $\bar v_H=|H|^{-1}\sum_{i\in H}v_i$.  Then
\begin{align}
 \Up_A(v)
 &=\max_{H\in\mathfrak H_A}\bar v_H,
 &
 \Dn_A(v)
 &=\min_{H\in\mathfrak H_A}\bar v_H,
 \label{eq:extremal-identity}
\end{align}
so that, if at most $A$ coordinates are Byzantine, exact honest reports satisfy
\begin{equation}
       \Dn_A(v)\le \bar v_{\cH}\le \Up_A(v).
       \label{eq:exact-deletion-envelope}
\end{equation}
Moreover, $\Up_A$ is the pointwise smallest function that upper-bounds every
feasible subset mean in \cref{eq:extremal-identity}, and $\Dn_A$ is the
pointwise largest corresponding lower bound.  Writing $s_i(z)=s_i(x,y)$ and
$\widetilde s^r(z)=(\widetilde s_1^r(z),\ldots,\widetilde s_K^r(z))$ for
$z=(x,y)$, if honest reports on channel $r$ obey
$-\rho_r^-\le\widetilde s_i^r(z)-s_i(z)\le\rho_r^+$, then, uniformly over
Byzantine reports,
\begin{equation}
 s_{\cH}(z)\le \Up_A(\widetilde s^r(z))+\rho_r^-,
 \qquad
 \Dn_A(\widetilde s^r(z))\le s_{\cH}(z)+\rho_r^+.
 \label{eq:quantized-deletion-envelope}
\end{equation}
The two bounds are separately pointwise sharp when honest reconstructions are
constrained only by the common directional-error intervals.
\end{theorem}

Equality can fail for a particular fixed reconstruction map, which also
encodes the clipping range and half-open quantizer cells; at eight bits with
$\Smax=1$ the bound exceeds the value attainable under that map by at most about $0.002/|H|$ per clipped entry.

\begin{proof}
We first optimize over honest subset size, then account for reconstruction
error. The equality constructions also establish pointwise sharpness.
Fix a subset size $s$.  The largest mean over all $s$-subsets is the mean of
the largest $s$ coordinates, and the smallest is the mean of the smallest
$s$.  The top-$s$ mean is nonincreasing in $s$: adding a coordinate no larger
than the current top-$s$ mean cannot increase it.  The bottom-$s$ mean is
nondecreasing by the symmetric argument.  Hence both extrema over
$s\ge K-A$ occur at $s=K-A$, proving \cref{eq:extremal-identity}.

The true honest set has $h=K-a\ge K-A$ coordinates, so its reported mean is
one of the feasible subset means.  This proves
\cref{eq:exact-deletion-envelope}.  If a function $U(v)$ upper-bounds every
feasible subset mean, choose the $K-A$ indices of the largest reports to get
$U(v)\ge\Up_A(v)$.  Thus $\Up_A$ is pointwise smallest.  Choosing the bottom
$K-A$ indices proves the lower statement.

Let $q_i$ be the decoded honest report and $x_i$ its clean score, and write
$(q_{\cH},v_{\cA})$ for the full report vector.  From
$q_i-x_i\ge-\rho_r^-$,
\[
 \bar x_{\cH}\le\bar q_{\cH}+\rho_r^-
                 \le\Up_A(q_{\cH},v_{\cA})+\rho_r^-.
\]
Similarly, $q_i-x_i\le\rho_r^+$ gives
\[
 \Dn_A(q_{\cH},v_{\cA})\le\bar q_{\cH}
                 \le\bar x_{\cH}+\rho_r^+.
\]
These are \cref{eq:quantized-deletion-envelope}.  In the interval model, where
each honest clean score is known only to lie in its common directional-error
interval, equality in the upper robust envelope
is obtained by choosing the top
$K-A$ coordinates as honest and taking their clean values to be
$q_i+\rho_r^-$. The lower construction is symmetric.  This establishes the
stated pointwise sharpness in the interval model.  The construction need not be attainable by a particular fixed
reconstruction map, whose clipping range and half-open cells the interval model
ignores.
\end{proof}

The next elementary bound is used only for efficiency, not validity.

\begin{lemma}[Two-sided deviations under bounded reports]
\label{lem:bounded-deletion}
Let $x_i$ be honest node $i$'s clean score and $q_i$ its decoded report, write
$\bar x_{\cH}$ for the clean honest mean and $(q_{\cH},v_{\cA})$ for the report vector
with honest coordinates $q_i$ and arbitrary Byzantine coordinates $v_i$.  Suppose
all decoded reports lie in an interval of width $B$ (the $D_r$ of
\cref{thm:set-sandwich} on one channel), $a\le A$, and honest reconstruction errors
obey $-\rho^-\le q_i-x_i\le\rho^+$.  Then
\begin{align}
 -\rho^-&\le \Up_A(q_{\cH},v_{\cA})-\bar x_{\cH}
       \le \min\!\left\{1,\frac{A}{h}\right\}B+\rho^+,\label{eq:up-two-sided}\\
 -\min\!\left\{1,\frac{A}{h}\right\}B-\rho^-&\le \Dn_A(q_{\cH},v_{\cA})-\bar x_{\cH}
       \le\rho^+.\label{eq:dn-two-sided}
\end{align}
The coefficient $\min\{1,A/h\}$ is worst-case sharp.
\end{lemma}

\begin{proof}
The inner inequalities $\Up_A\ge\bar q_{\cH}$ and
$\Dn_A\le\bar q_{\cH}$ follow from \cref{thm:extremal-envelopes}, proving
the left side of \cref{eq:up-two-sided} and the right side of
\cref{eq:dn-two-sided}.  It remains to show
$\Up_A-\bar q_{\cH}\le \min\{1,A/h\}B$; the lower claim follows after negating all
reports.

If $B=0$, the report-deviation claim is immediate.  Assume $B>0$ and by
translation and scaling take the report interval to be $[0,1]$.  For any
fixed honest set, the function
$v\mapsto\Up_A(v)-\bar v_{\cH}$ is the maximum of finitely many linear
functions, so its maximum over the cube is attained at a vertex.  Suppose a
binary vertex has $u$ honest ones and $b\le a$ Byzantine ones.  Put
$\bar h=K-A$, the smallest possible honest count.  Then
\[
 \Up_A(v)-\bar v_{\cH}
 =\frac{\min\{u+b,\bar h\}}{\bar h}-\frac uh.
\]
If $u+b\ge\bar h$, the first term is one and
$u\ge\bar h-b$, so the difference is at most
$(h-\bar h+b)/h=(A-a+b)/h\le A/h$.  If $u+b<\bar h$, the expression is
$(u+b)/\bar h-u/h$.  It is nondecreasing in $u$ because $h\ge\bar h$, so it is bounded
above by its value at the adjacent boundary $u=\bar h-b$, which obeys the same
bound.
The difference between two means in $[0,1]$ is also at most one, proving the
minimum of the two bounds.  Rescaling restores $B$.  Finally, quantization can
move the honest reported
mean above the clean mean by at most $\rho^+$, proving the upper inequality.
If $h\ge A$, put $A$ honest reports at the lower endpoint, all other honest
reports at the upper endpoint, and all Byzantine reports at the upper
endpoint.  If $h<A$, put all honest reports at the lower endpoint and at least
$K-A$ Byzantine reports at the upper endpoint.  These arrays attain the
displayed coefficient; negation gives the lower equality.
\end{proof}

The symmetric comparator is controlled by how far extra reports can shift the
ranks of honest values. Let
$x^{\cH}_{(1)}\le\cdots\le x^{\cH}_{(h)}$ and define
\[
 L_{\cH,m}=\frac1{h-m}\sum_{j=1}^{h-m}x^{\cH}_{(j)},\qquad
 U_{\cH,m}=\frac1{h-m}\sum_{j=m+1}^{h}x^{\cH}_{(j)}.
\]

\begin{lemma}[Symmetric interlacing sandwich]
\label{lem:symmetric-interlace}
If $a\le A\le m$ and $2m<K$, then
\[
 L_{\cH,m}\le\TM_m(x_{\cH},v_{\cA})\le U_{\cH,m}.
\]
\end{lemma}

\begin{proof}
Let $w_{(1)}\le\cdots\le w_{(K)}$ order all reports.  Adding $a$ values to
the honest multiset gives
$x^{\cH}_{(j-a)}\le w_{(j)}\le x^{\cH}_{(j)}$ whenever the displayed honest
indices exist.  Put $L=K-2m$.  For retained index $j=m+r$,
$r=1,\ldots,L$, averaging the interlacing inequalities yields
\[
 \frac1L\sum_{j=m+1-a}^{h-m}x^{\cH}_{(j)}
 \le\TM_m(w)\le
 \frac1L\sum_{j=m+1}^{K-m}x^{\cH}_{(j)}.
\]
Since $h-m=L+(m-a)$, the left sum contains the largest $L$ values among the
first $h-m$ honest order statistics, so its mean is at least
$L_{\cH,m}$.  The upper sum contains the smallest $L$ values among the last
$h-m$ honest order statistics, so its mean is at most $U_{\cH,m}$.
\end{proof}

\begin{proof}[Proof of \cref{lem:sharp-tm}]
By \cref{lem:symmetric-interlace}, it suffices to bound the two honest
deletion means.  If $m=0$ the claim is immediate.  Let $B_m$ be the mean of
the $m$ smallest honest values.  Partitioning the honest sum gives
\[
 \bar x_{\cH}=\frac{h-m}{h}U_{\cH,m}+\frac mh B_m,
\]
and hence
$0\le U_{\cH,m}-\bar x_{\cH}\le (m/h)R_{\cH}$.  Negating the values gives
$0\le\bar x_{\cH}-L_{\cH,m}\le(m/h)R_{\cH}$.

For sharpness at any admissible $a$, take exactly $m$ honest values equal to
zero, all other honest values equal to one, and all Byzantine reports equal
to one.  The retained mean is one while the honest mean is $1-m/h$.
Thus the coefficient is attained for every admissible $a$, including $a=0$.
If honest reports are perturbed
by at most $\rho$, keep Byzantine values fixed, invoke
\cref{lem:order-lipschitz}, and add $\rho$ by the triangle inequality.
\end{proof}

\section{Proofs of validity, efficiency, and lower bounds}
\label{app:proof-coverage}

Exchangeability lower-bounds the probability of the clean rank event, and the
score envelopes show pathwise that this event implies inclusion.

\begin{lemma}[Exchangeable rank lemma]
\label{lem:rank}
Let $T_1,\ldots,T_{n+1}$ be exchangeable real random variables and
$k\in\{1,\ldots,n\}$.  If $T_{(k)}^{\calib}$ is the $k$th order statistic of
$T_1,\ldots,T_n$, then
\[
 \Pp\{T_{n+1}\le T_{(k)}^{\calib}\}\ge\frac{k}{n+1}.
\]
\end{lemma}

\begin{proof}
Let $\zeta_1,\ldots,\zeta_{n+1}$ be i.i.d.\ $\Unif[0,1]$ tie-breakers independent of the $T_j$ and
order $(T_j,\zeta_j)$ lexicographically.  The rank of pair $n+1$ is uniform.  A
rank at most $k$ implies $T_{n+1}\le T_{(k)}^{\calib}$ after discarding the
tie-breakers, proving the inequality.
\end{proof}

\begin{proof}[Proof of \cref{thm:coverage}]
If $k=n+1$, the returned set is $\cY$.  Suppose $k\le n$.  The calibration
part of \cref{eq:quantized-deletion-envelope} gives
$R_j\le\widehat R_j+\rho_{\calib}^-$ and therefore
\begin{equation*}
 R_{(k)}\le\widehat R_{(k)}+\rho_{\calib}^-.
\end{equation*}
On $\{R_{n+1}\le R_{(k)}\}$, the query envelope gives
\[
 \begin{aligned}
 \widehat s^{\qry}_{\rm del}(X_{n+1},Y_{n+1})
 &\le R_{n+1}+\rho_{\qry}^+\\
 &\le R_{(k)}+\rho_{\qry}^+\\
 &\le\widehat R_{(k)}+\rho_{\calib}^-+\rho_{\qry}^+
 =\widehat q^{\rm del}_\alpha.
 \end{aligned}
\]
Thus the clean rank event implies coverage pathwise.  Apply
\cref{lem:rank} to $R_1,\ldots,R_{n+1}$.  No property of the Byzantine
transcript beyond finite decoding was used.
\end{proof}

\subsection{Fixed-membership identified rules}
\label{app:proof-fixed-set}

We identify the largest compatible calibration threshold, then reuse one
feasible honest subset across calibration and query entries to obtain the
fixed-set rule, its nesting, coverage, and minimality among rules that
preserve the clean rank event in every entrywise compatible completion.

\begin{lemma}[Sharp joint calibration envelope]
\label{lem:sharp-joint-envelope}
Assume the conditions of \cref{thm:coverage} and let $k\le n$.  In the interval
model of \cref{thm:fixed-set},
\begin{equation}
 \sup_{\substack{H\in\mathfrak H_A,\ (x_{ij})\\
         -\rho_{\calib}^-\le v_{ij}-x_{ij}\le\rho_{\calib}^+\;
         (i\in H,\,j\in[n])}}
 \OS_k\left\{\frac1{|H|}\sum_{i\in H}x_{ij}:j\in[n]\right\}
 =\tau_{A,k}(V)+\rho_{\calib}^-.
 \label{eq:sharp-joint-envelope}
\end{equation}
Hence the right side upper-bounds the clean calibration $k$th order statistic
for every feasible fixed honest set and compatible clean-score array, and it is
the smallest such scalar bound based only on $V$, the declared budget, and the
reconstruction bounds.
\end{lemma}

\begin{proof}
Fix $k\le n$.  For any $H\in\mathfrak H_A$ and compatible clean calibration
array $(x_{ij})$, the directional interval gives
\[
 \frac1{|H|}\sum_{i\in H}x_{ij}
 \le c_j(H;V)+\rho_{\calib}^-
 \quad\text{for every }j.
\]
Order-statistic monotonicity therefore bounds the left side of
\cref{eq:sharp-joint-envelope} by
$T_k(H;V)+\rho_{\calib}^-\le
\tau_{A,k}(V)+\rho_{\calib}^-$.  Because $\mathfrak H_A$ is finite, choose a
maximizer $H^\star$ and set
$x_{ij}=v_{ij}+\rho_{\calib}^-$ for $i\in H^\star$.  These values satisfy the
abstract directional intervals and attain the upper bound.  This proves the
identity. Any scalar upper bound on the clean order statistic that is valid
in every entrywise compatible completion must also bound this maximizing
completion. Thus no smaller bound based only on $V$, the declared budget,
and the reconstruction bounds can have the required property.
\end{proof}

\begin{proof}[Proof of \cref{thm:fixed-set}]
Fix $k\le n$ and let $\cH$ be the actual honest set.  Its calibration values obey
\begin{equation*}
 T_k(\cH;V)\ge R_{(k)}-\rho_{\calib}^-,
 \qquad
 \frac1h\sum_{i\in\cH}w_i(x,y)
 \le s_{\cH}(x,y)+\rho_{\qry}^+.
\end{equation*}
Thus $s_{\cH}(x,y)\le R_{(k)}$ implies that the margin in
\cref{eq:fixed-set-margin} at $H=\cH$ is at most $g_\rho$, proving
$\cC^\circ\subseteq\cC^{\rm fs}$.

For every feasible $H$,
\[
 \frac1{|H|}\sum_{i\in H}w_i\ge\Dn_A(w),
 \qquad T_k(H;V)\le\tau_{A,k}(V).
\]
Consequently
\[
 \Delta_{A,k}(V,w)\ge\Dn_A(w)-\tau_{A,k}(V),
\]
and membership in $\cC^{\rm fs}$ implies membership in $\cC^{\rm jt}$.
For every $H$ and $j$, \cref{thm:extremal-envelopes} gives
$c_j(H;V)\le\Up_A(v_{\cdot j})$.  Applying order-statistic monotonicity and
then maximizing over $H$ yields
\begin{equation*}
 \tau_{A,k}(V)
 \le\OS_k\{\Up_A(v_{\cdot1}),\ldots,\Up_A(v_{\cdot n})\}
 =\widehat R_{(k)},
\end{equation*}
so $\cC^{\rm jt}\subseteq\cC^{\rm del}$.  This proves the hierarchy.  On the
clean true-label rank event, its leftmost inclusion gives membership in all
three deployable sets.  Applying \cref{lem:rank} proves their coverage.

It remains to prove the entrywise compatibility characterization.  Fix $H$.
Every compatible clean query array $(z_i)$ satisfies
\[
 \frac1{|H|}\sum_{i\in H}z_i
 \ge\frac1{|H|}\sum_{i\in H}w_i-\rho_{\qry}^+,
\]
whereas the compatible clean calibration order statistic is at most
$T_k(H;V)+\rho_{\calib}^-$.  Hence the clean rank event can hold only if
\begin{equation}
 \frac1{|H|}\sum_{i\in H}w_i-T_k(H;V)\le g_\rho.
 \label{eq:H-compatibility-condition}
\end{equation}
Conversely, when \cref{eq:H-compatibility-condition} holds, set
$x_{ij}=v_{ij}+\rho_{\calib}^-$ and
$z_i=w_i-\rho_{\qry}^+$ for every $i\in H$.  These values are compatible and
make the clean rank inequality exactly equivalent to
\cref{eq:H-compatibility-condition}.  Taking the union over
$H\in\mathfrak H_A$ proves \cref{eq:fixed-set-identified}.  A
$(V,w)$-measurable report-only rule that must preserve the clean rank event in
every entrywise compatible completion must include a candidate whenever any
such completion exists;
the fixed-set rule includes exactly those candidates.  This proves the stated
pointwise minimality.  The $k=n+1$ case follows from the common full-set
convention.
\end{proof}

\paragraph{Why the inclusions can be strict.}
The following examples isolate the information lost at each relaxation.
At zero padding, let $K=3,A=1,n=2,k=1$, which needs $\alpha\ge2/3$; these
examples isolate mechanics, not a practical level.  With rows indexed by nodes, the matrix
\[
 V=\begin{pmatrix}0&1\\1&0\\1&1\end{pmatrix}
\]
has $\tau_{A,k}=2/3$, whereas the coordinatewise deletion order statistic is
$1$; a constant query vector with entries $3/4$ makes
$\cC^{\rm jt}\subsetneq\cC^{\rm del}$.  For
\[
 V=\begin{pmatrix}0&0\\0&1\\1&0\end{pmatrix},\qquad w=(0,1,1),
\]
one has $\tau_{A,k}=\Dn_A(w)=1/2$, but every feasible same-set margin is
positive (its minimum is $1/3$), so
$\cC^{\rm fs}\subsetneq\cC^{\rm jt}$.  Finally, with $K=2,A=1,n=k=1$ (so $\alpha\ge1/2$), actual honest set $\{1\}$, calibration reports $(0,1)$,
and query reports $(1,0)$, the candidate is excluded by $\cC^\circ$ but
included by $\cC^{\rm fs}$ using feasible set $\{2\}$.  Thus the oracle
inclusion can also be strict.

\paragraph{Why every feasible subset size is retained.}
Enumerating more than the smallest feasible subsets protects the full honest
mean when the declared budget is conservative. Sets larger than $K-A$ in
\cref{eq:feasible-honest-sets} can be necessary for the stated honest-mean
identification and oracle containment. For
$K=3,A=1,n=2,k=1$ and
\[
 V=\begin{pmatrix}0&1\\1&0\\1&1\end{pmatrix},
\]
every two-node subset has calibration order statistic $1/2$, while $[3]$ has
$2/3$.  With no Byzantine nodes and $w=(3/5,3/5,3/5)$, the full honest-mean
oracle therefore includes the candidate, whereas fixed-set and joint-threshold
variants restricted to size-two subsets exclude it.
This is not a marginal-coverage counterexample.  Any fixed
$H_0\subseteq\cH$ of size $K-A$, chosen independently of the examples,
defines an exchangeable clean subset-mean score.  The exact-size fixed-set
union contains its directionally padded conformal set, and the exact-size
joint-threshold rule contains that union.  Both retain marginal coverage
at least $k/(n+1)$, but need not preserve the rank event of the full honest
mean when $a<A$.  Enumerating all feasible sizes supplies the stronger
identification and containment guarantees proved here.

\begin{proof}[Proof of \cref{cor:endpoints-budget}]
Each subset calibration mean, its order statistic, $\tau_{A,k}$, and
$\Up_A$ are coordinatewise nondecreasing in calibration reports.  Each subset
query mean and $\Dn_A$ are coordinatewise nondecreasing in query reports.
Therefore increasing Byzantine calibration coordinates and decreasing their
query coordinates can only decrease the fixed-set margin, increase the joint
threshold, increase the deletion threshold, or decrease either query
aggregate.  It can only add candidates to each proposed set.  The endpoints
$(\rhi,\rlo)$ attain the coordinatewise extremum.  Reversing both directions, the
endpoints $(\rlo,\rhi)$ attain the opposite extremum and can only remove candidates
relative to any other transcript in $[\rlo,\rhi]$.  Neither argument uses $a\le A$.

For low--high with $a=A$, the number of honest identities is $K-A$.  Because
the Byzantine calibration reports lie below the honest reports, the largest
$K-A$ reports are exactly the honest multiset; because the Byzantine query
reports lie above the honest reports, the smallest $K-A$ reports are exactly
the honest multiset. Thus, for each calibration column $j$,
$\Up_A(v_{\cdot j})=c_j(\cH;V)$ and
$\Dn_A(w)=h^{-1}\sum_{i\in\cH}w_i$.  More explicitly, any
$G\in\mathfrak H_A$ has size at least $h$.  The extremal identity in
\cref{eq:extremal-identity} therefore gives
$c_j(G;V)\le c_j(\cH;V)$ for every $j$ and, under the reversed endpoint
ordering at query time,
$|G|^{-1}\sum_{i\in G}w_i\ge h^{-1}\sum_{i\in\cH}w_i$.  Consequently,
\[
 \tau_{A,k}(V)=T_k(\cH;V),\qquad
 \Delta_{A,k}(V,w)=\frac1h\sum_{i\in\cH}w_i-T_k(\cH;V),
\]
and the deletion, joint-threshold, and fixed-set membership
criteria all equal the criterion defining
$\cC^{\rm dec,\cH}_\alpha(x)$, for arbitrary registered directional
padding.  With zero registered reconstruction error, decoded honest reports
equal their clean values, so $\cC^{\rm dec,\cH}=\cC^\circ$.  Under
exchangeability and no ties, the future clean score has a uniform rank among
the $n+1$ scores, so coverage is exactly $k/(n+1)$.

Finally, $\mathfrak H_{A_1}\subseteq\mathfrak H_{A_2}$ when $A_1\le A_2$.
Thus the fixed-set minimum can only decrease and the joint-threshold maximum
can only increase.  The mean of the largest $K-A$ reports is nondecreasing in
$A$, while the mean of the smallest $K-A$ reports is nonincreasing.  With
fixed padding, each change can only enlarge its corresponding prediction set.
\end{proof}

\begin{proof}[Proof of \cref{thm:set-sandwich}]
The $k=n+1$ case is immediate.  Let $k\le n$.  The first three inclusions are
\cref{eq:fixed-set-hierarchy}; it remains to prove the outer inclusion for
$\cC^{\rm del}$.

By \cref{lem:bounded-deletion} with $B=D_r$ on each channel and $h\ge K-A$, calibration scores obey
\[
 \widehat R_j\le R_j+\mu_{\calib}+\rho_{\calib}^+,
\]
so the same inequality holds for their $k$th order statistics.  The query
lower bound is
\[
 \widehat s^{\qry}_{\rm del}(x,y)
 \ge s_{\cH}(x,y)-\mu_{\qry}-\rho_{\qry}^-.
\]
If $y\in\cC^{\rm del}_\alpha(x)$, combine these inequalities with the
definition of its threshold to obtain
\[
 s_{\cH}(x,y)
 \le R_{(k)}+\mu_{\calib}+\mu_{\qry}
 +\rho_{\calib}^-+\rho_{\calib}^+
 +\rho_{\qry}^-+\rho_{\qry}^+
 =R_{(k)}+W_{\rm del}.
\]
This proves the pathwise sandwich.
\end{proof}

The outer width also controls the \emph{average} number of extra answers, once
the future question is assumed independent of the calibration split.

\begin{corollary}[Expected excess set size]
\label{cor:size-modulus}
Assume the conditions of \cref{thm:set-sandwich} and let
$N(t)=\E_X\sum_{y\in\cY}\ind\{s_{\cH}(X,y)\le t\}$ be the expected number of
candidates with clean score at most $t$, where $\E_X$ averages over a future
input and $\ind$ equals one when its condition holds and zero otherwise.  Write
$\omega_{\rm size}(u)=\sup_t\{N(t+u)-N(t)\}$ for $u\ge0$ and take the widths
$D_r$ and the reconstruction bounds to be fixed conditional on the
pre-calibration system.  If the future $X$ is independent of the calibration
split conditional on that system, then, for every
$\cC\in\{\cC^{\rm fs},\cC^{\rm jt},\cC^{\rm del}\}$,
\begin{equation}
 0\le\E\bigl[|\cC_\alpha(X)|-|\cC^\circ_\alpha(X)|\bigr]
 \le\omega_{\rm size}(W_{\rm del}).
\label{eq:size-modulus}
\end{equation}
If each candidate-score cumulative distribution function (CDF)
$G_y(t)=\Pp_X\{s_{\cH}(X,y)\le t\}$ satisfies
$|G_y(t)-G_y(t')|\le L_y|t-t'|$ for a nonnegative constant $L_y$ (the Lipschitz
condition), the upper bound is at most $\min\{M,W_{\rm del}\sum_y L_y\}$.
\end{corollary}

\begin{proof}
Conditional on the clean calibration examples, write $q=R_{(k)}$; no
conditioning on the adversarial transcript is needed. The expectation below
averages over the future input and any randomness in Byzantine reporting.
For any $\cC\in\{\cC^{\rm fs},\cC^{\rm jt},\cC^{\rm del}\}$,
independence of the future $X$ from calibration and the pathwise inclusions
give
\begin{align*}
 0&\le\E\bigl[|\cC_\alpha(X)|-|\cC^\circ_\alpha(X)|
       \mid Z_1,\ldots,Z_n\bigr]\\
  &\le N(q+W_{\rm del})-N(q)\le\omega_{\rm size}(W_{\rm del}).
\end{align*}
Averaging over calibration proves \cref{eq:size-modulus}.  If each $G_y$ is
$L_y$-Lipschitz, sum
$G_y(q+W_{\rm del})-G_y(q)\le L_y W_{\rm del}$ and cap the result by $M$.
\end{proof}

\begin{proof}[Proof of \cref{prop:dominance}]
If $k=n+1$, all rules return $\cY$.  Suppose $k\le n$.  Put $\bar h=K-A$ and
$T=\TM_m(v)$.  For ordered reports in an interval of width $B$, direct
decomposition gives
\begin{equation*}
 \Up_A(v)-\TM_m(v)\le\frac{m}{K-A}B,
 \qquad
 \TM_m(v)-\Dn_A(v)\le\frac{m}{K-A}B.
\end{equation*}
Indeed,
\begin{align*}
 \bar h\{\Up_A(v)-T\}
 &=\sum_{\ell=A+1}^{m}\{v_{(\ell)}-T\}
   +\sum_{\ell=K-m+1}^{K}\{v_{(\ell)}-T\}\le mB,\\
 \bar h\{T-\Dn_A(v)\}
 &=\sum_{\ell=1}^{m}\{T-v_{(\ell)}\}
   +\sum_{\ell=K-m+1}^{K-A}\{T-v_{(\ell)}\}\le mB.
\end{align*}
In the first line the first sum is nonpositive and the second has $m$ terms;
in the second line the second sum is nonpositive and the first has $m$ terms.
Empty sums are zero.  The arrays with $K-m$ zeros followed by $m$ copies of
$B$, and with $m$ zeros followed by $K-m$ copies of $B$, attain the two
respective constants.

Order-statistic monotonicity transfers the first inequality to calibration
$k$th order statistics.  If $y\in\cC^{\rm del}_\alpha(x)$, then
\[
 \begin{aligned}
 \widehat s^{\rm sym,m}(x,y)
 &\le\widehat s^{\qry}_{\rm del}(x,y)+g_{\qry}(m)\\
 &\le\widehat R_{(k)}+\rho_{\calib}+\rho_{\qry}+g_{\qry}(m)\\
 &\le\widehat R^{\rm sym,m}_{(k)}+g_{\calib}(m)+g_{\qry}(m)
       +\rho_{\calib}+\rho_{\qry},
 \end{aligned}
\]
which is exactly membership in $\cC^{\rm sym,m}$.  Prepending the hierarchy
from \cref{thm:fixed-set} proves the full claim.
\end{proof}

\begin{proof}[Proof of \cref{prop:low-high-lower}]
With $A$ Byzantine reports below the honest range, lower trimming removes the
Byzantine values and upper trimming removes the $A$ largest honest values.
Thus the calibration aggregate is $S_j+L_\xi$.  With Byzantine reports above
the honest range at prediction, it is $S_{n+1}+U_\xi$.  The threshold event is
therefore exactly \cref{eq:low-high-coverage}.  If
$d=\Delta_\xi-g>0$ and $S$ is uniform on an interval of width less than $d$, then
$S_{(k)}-d$ is below the lower endpoint for every calibration realization;
coverage is zero.

Put $r=h-A=K-2A$.  If $h\ge2A$, cancellation of the overlapping terms gives
\[
 \Delta_\xi
 =\frac{\sum_{i=h-A+1}^{h}\xi_{(i)}
        -\sum_{i=1}^{A}\xi_{(i)}}r
 \le\frac{A}{r}R.
\]
When $h<2A$, the lower and upper $r$-blocks are disjoint and their mean gap is
at most $R$.  Arrays whose offsets take only the values $\{0,R\}$ attain both
bounds, proving
\cref{eq:low-high-minimax}.
\end{proof}

\begin{proof}[Proof of \cref{lem:phase-blind}]
Conditional on $\mathcal F$, the pairs $(Z_j,\Xi_j)$ are exchangeable.
Applying the same measurable map $\phi$ to every pair preserves joint
permutation invariance, so the $T_j$ are exchangeable.  The final claim follows
from \cref{lem:rank}, with the full-set convention when $k=n+1$.
\end{proof}

\begin{proof}[Proof of \cref{prop:pmerge}]
The proof counts how many honest nodes must report small $p$-values before
the merger can reject, then constructs dependent honest values attaining
that probability bound.
Let the actual Byzantine count be $a\le A$ and, for $t\in[0,1]$, define
$N_{\cH}(t)=\sum_{i\in\cH}\ind\{p_i\le t\}$.  If
$p_{(A+\ell)}\le t$, at least $A+\ell-a$ honest reports are at most $t$.
Marginal superuniformity, linearity of expectation, and Markov's inequality
give
\[
 \Pp\{p_{(A+\ell)}\le t\}
 \le\frac{(K-a)t}{A+\ell-a}
 \le\frac{K-A}{\ell}t.
\]
For the last inequality write $d=A-a$; after cross-multiplication the
difference between right and left numerators is
$d\{K-A-\ell\}\ge0$.  For every $u<1$,
\[
 \{p^{\rm merge}_\ell\le u\}
 =\left\{p_{(A+\ell)}\le\frac{u\ell}{K-A}\right\},
\]
so the preceding display bounds its probability by $u$; the case $u=1$ is
trivial.  Taking $u=\alpha$ proves the prediction-set claim.

To show sharpness, take exactly $A$ Byzantine zeros, so that $h=K-A$.  Fix any
$t\in[0,\ell/h]$.  With probability $ht/\ell$, choose a uniformly
random $\ell$-subset of the honest nodes, set their $p$-values to $t$, and set
all other honest values to one; otherwise set every honest value to one.  For
each honest node and every $u<1$, the probability of being at most $u$ is zero
when $u<t$ and $t$ when $u\ge t$, hence the marginal is superuniform.  Nevertheless,
$p_{(A+\ell)}\le t$ with probability $ht/\ell$.  Therefore the multiplicative
constant cannot be reduced under arbitrary dependence.  Ties are counted with
multiplicity and candidate inclusion uses the strict comparison
$p^{\rm merge}_\ell>\alpha$, so exclusion is the event just bounded.
\end{proof}

\begin{remark}[Candidate regularity and an unconditional alternative]
\label{rem:size-density}
The true-label score may have a bounded density while incorrect-label scores
have atoms.  Hence that density alone does not control the uniform modulus
$\omega_{\rm size}$: an arbitrarily small threshold increment crossing such
an atom can add every incorrect label.  The bound uniform over calibration
thresholds therefore concerns $N$ or the candidate CDFs.

A bounded true-label density gives an unconditional excess-size bound even
when incorrect-label scores have atoms.  Under the
expected-size conditions of \cref{cor:size-modulus}, suppose each clean
calibration score $R_j$ has a density bounded by $L$ conditional on the frozen
system, and that $W_{\rm del}$ is fixed under that conditioning.
Then every proposed rule satisfies
\begin{equation}
 0\le\E\bigl[|\cC_\alpha(X)|-|\cC^\circ_\alpha(X)|\bigr]
 \le\min\{M,MnL W_{\rm del}\}.
 \label{eq:size-density-alternative}
\end{equation}
For $k\le n$, condition on $X$ and put $s_y=s_{\cH}(X,y)$.  The outer
sandwich can add candidate $y$ only if
$R_{(k)}\in[s_y-W_{\rm del},s_y)$.  Since an order statistic equals one of
its inputs, this implies that at least one calibration score $R_j$ lies in
that interval.  Independence from $X$, the marginal density bound, and a
union bound give probability at most $nL W_{\rm del}$.  Summing over the $M$
candidates and capping by $M$ proves the display; independence among the
calibration scores is not needed.  The excess is zero when $k=n+1$.
This bound permits candidate atoms but pays an explicit factor $n$.  The two
bounds use different regularity information; when both are available, their
minimum may be used.
\end{remark}

\section{Implementation and exact arithmetic}
\label{app:protocol}

This appendix gives the reference protocol for the three proposed rules and
the exact arithmetic that implements their inclusion tests.  All three use the
same aligned, identity-resolved reports and registered reconstruction bounds,
with their respective calibration and query comparisons from
\cref{thm:fixed-set}.

\paragraph{Freeze and aligned calibration.}
Before calibration the protocol fixes every model, retriever, corpus index, prompt,
candidate ordering rule, score transform, quantizer, membership list, retry
policy, corruption budget $A$, and optional honest-dropout budget $D$.  All
tuning uses data disjoint from calibration.  Infrastructure errors are handled
as in \cref{sec:protocol-details} and consume neither $A$ nor $D$; the
deterministic fallback is reserved for an input-dependent retrieval or model
failure selected by the frozen honest score map.  For each example the hub
sends every node the same input and ordered candidate list, and honest nodes
score, quantize, and sign a vector.  A missing or malformed authenticated
report is replaced by a fixed valid sentinel, never silently removed; the
experiments use the lowest codeword $0$ in calibration and the maximum
codeword $\Smax$ at query time.  The hub keeps node identities and forms $V$
from each vector's gold (correct-answer) coordinate.

\paragraph{Threshold and deployment.}
The index $k$ is computed by integer arithmetic; $k=n+1$ returns $\cY$, and
otherwise the three rules apply the comparisons of \cref{sec:method}, every
one non-strict.  More than $A+D$ absent or
malformed slots trigger an abort or recalibration.  The guarantees assume that
at most $D$ missing identities are honest, which absence alone cannot reveal,
so the primary protocol sets $D=0$.  The matched-pipeline diagnostic further
requires identical calibration and query maps, and Byzantine randomness that
does not depend on phase, index, history, or gold-label exposure.

\paragraph{Robust $p$-value option.}
The hub can compute the $p$-values \eqref{eq:node-pvalue} from the aligned score tensor, which
requires a complete matched calibration/query transcript for every honest
node.  If nodes instead transmit $p$-values, an exact rank costs
$\lceil\log_2(n+1)\rceil$ bits (nine when $n=499$).  Upward rounding preserves
validity; downward or nearest rounding need not.

\subsection{Exact decisions on the registered quantizer grid}
\label{app:exact-grid-decisions}

The registered implementation decides inclusion in exact integer arithmetic,
so a boundary tie falls on the side the theorems require.  Floating-point
evaluation of algebraically equal means can put the same boundary on opposite
sides of a non-strict comparison.  Write $L_\calib=2^{b_\calib}-1$ and
$L_\qry=2^{b_\qry}-1$, so each decoded report is $\Smax$ times an integer
index over $L_r$.  For a feasible $H$, let $C_H$ be the $k$th order statistic
of its calibration index sums and $\Sigma_H(y)$ its query index sum.  The
fixed-set test for this $H$ is exactly
\begin{equation*}
 2L_\calib \Sigma_H(y)
 \le 2L_\qry C_H+|H|(L_\calib+L_\qry),
\end{equation*}
obtained by multiplying the mean comparison and its two half-cell paddings by
$2|H|L_\calib L_\qry/\Smax$.  It preserves equality and allows unequal phase
precisions.  Joint-threshold, deletion, and guarded symmetric decisions use
the corresponding exact rational thresholds, and $\alpha$ is treated as a
rational number.  The common point ranker of \cref{eq:common-point-ranker}
also compares integer code sums, so exact ties reach the fixed candidate
order (\cref{app:rag-protocol}).  Every pathwise assertion implied by the
theory is checked before summarization; a failure would be an implementation
or assumption error, not Monte Carlo uncertainty.

\section{Random Byzantine membership}
\label{app:random-membership}

This appendix proves \cref{cor:random-membership}, in which the Byzantine set $\cA$ is
drawn from a probability law instead of being fixed, and gives the tail probability
$\varepsilon_A$ for independent failures. The deployed rule does not change: the hub
still uses the declared budget $A$ and never estimates $\cA$ or its law.

\paragraph{Failure law.}
Let $\cA\subseteq[K]$ be a random variable on the same probability space, with
$\cH=[K]\setminus\cA$ and $a=|\cA|$, satisfying the following.
\begin{enumerate}
 \item $\cA$ is drawn once before calibration and held fixed for the entire
 calibration-to-query comparison.
 \item For every $\mathcal A_0\subseteq[K]$ with $\Pp\{\cA=\mathcal A_0\}>0$,
 the clean examples $Z_1,\ldots,Z_{n+1}$ are exchangeable conditional on
 $\{\cA=\mathcal A_0\}$ and the frozen honest score functions.
 \item The registered directional bounds \eqref{eq:directional-quantizer} hold
 for every node that can be honest (they are properties of the registered
 quantizers), and the declared budget $A<K$ is a deterministic constant, fixed
 before any data are seen.
\end{enumerate}
Drawing $\cA$ independently of the clean examples is the simplest way to meet
the second condition, but independence is not required.  What the condition
excludes is membership selected after inspecting the clean calibration
scores, since conditioning on such a choice need not leave the examples
exchangeable.  Write
\begin{equation*}
 \varepsilon_A=\Pp\{a>A\}
\end{equation*}
for the probability that the realized corruption exceeds the declared budget.

Conditions~1--3 restate the hypotheses of \cref{cor:random-membership}.

\begin{proof}[Proof of \cref{cor:random-membership}]
For $k=n+1$ all three rules return $\cY$ and coverage is one, so let $k\le n$.
Fix $\mathcal A_0$ with $\Pp\{\cA=\mathcal A_0\}>0$ and $|\mathcal A_0|\le A$.
Conditional on $\{\cA=\mathcal A_0\}$ the honest set is the deterministic set
$\mathcal H_0=[K]\setminus\mathcal A_0$, the clean examples are exchangeable
by the second condition, and
$|\mathcal H_0|=K-|\mathcal A_0|\ge K-A$, so $\mathcal H_0\in\mathfrak H_A$.
Every hypothesis of \cref{thm:coverage} therefore holds in this conditional
model, and part~1 of \cref{thm:fixed-set} gives the same bound for the two
smaller rules, since each contains the oracle set, giving
\[
 \Pp\{Y_{n+1}\in\cC_\alpha(X_{n+1})\mid\cA=\mathcal A_0\}\ge\frac{k}{n+1}.
\]
The bound holds for every such $\mathcal A_0$ and the conditional probability
is nonnegative on the remaining event, so averaging over $\cA$ and discarding
$\{a>A\}$ gives
\[
 \Pp\{Y_{n+1}\in\cC_\alpha(X_{n+1})\}
 =\E\bigl[\Pp\{Y_{n+1}\in\cC_\alpha(X_{n+1})\mid\cA\}\bigr]
 \ge\frac{k}{n+1}\,\Pp\{a\le A\}.
\]
This is the first inequality of \cref{eq:random-membership-coverage}.  The
second uses $k/(n+1)\ge1-\alpha$ and
$(1-\alpha)(1-\varepsilon_A)=1-\alpha-\varepsilon_A+\alpha\varepsilon_A\ge1-\alpha-\varepsilon_A$.
\end{proof}

\paragraph{Independent per-node failure.}
The corollary needs only the single number $\varepsilon_A$, which a per-node
failure model supplies in closed form.  Suppose node $i$ is Byzantine
independently with probability $p_i$.  Then $a$ has the Poisson--binomial
distribution, the law of a sum of independent but not identically distributed
indicators, and $\varepsilon_A$ is its upper tail.  In the two-class case of
interest, where
$K_1$ ordinary nodes each fail with probability $p_1$ and $K_2=K-K_1$
suspect nodes each fail with the larger probability $p_2$ (the subscripts of $p_1$
and $p_2$ index classes, not nodes),
\[
 \varepsilon_A=\sum_{m=A+1}^{K}\ \ \sum_{\substack{m_1+m_2=m\\ 0\le m_\ell\le K_\ell}}\ \
 \prod_{\ell=1}^{2}\binom{K_\ell}{m_\ell}p_\ell^{m_\ell}(1-p_\ell)^{K_\ell-m_\ell}.
\]
With a single class this is $\Pp\{\mathrm{Bin}(K,p)>A\}$.  The model describes
which nodes are corrupt, not how a corrupt node reports; a node counted in
$\cA$ still reports arbitrarily.

Upper bounds on the per-node failure probabilities give a valid upper bound
on $\varepsilon_A$.

\begin{lemma}[Monotone tails and the corner rule]
\label{lem:membership-monotone}
Let $p=(p_i)_{i\in[K]}$ and $q=(q_i)_{i\in[K]}$ be per-node failure
probabilities for independent membership with $p_i\le q_i$ for every $i$.  Then
$\varepsilon_A(p)\le\varepsilon_A(q)$ for every $A$.  Consequently, for a
componentwise box $\mathcal U=\{p:p_i\le\bar p_i\text{ for all }i\}$,
\[
 \sup_{p\in\mathcal U}\varepsilon_A(p)=\varepsilon_A(\bar p),
\]
so the robust declared budget
\[
 A^\star(\delta,\mathcal U)
 =\min\Bigl\{A'\in\{0,\ldots,K-1\}:\sup_{p\in\mathcal U}\Pp_p\{a>A'\}\le\delta\Bigr\}
\]
follows from one exact tail evaluation at the corner $\bar p$ of $\mathcal U$, with no search over
$\mathcal U$.
\end{lemma}

\begin{proof}
Let $U_1,\ldots,U_K$ be i.i.d.\ $\Unif[0,1]$ and set
$a_p=\sum_i\ind\{U_i\le p_i\}$ and $a_q=\sum_i\ind\{U_i\le q_i\}$, which have
the required marginals.  For each $i$, $p_i\le q_i$ gives
$\ind\{U_i\le p_i\}\le\ind\{U_i\le q_i\}$, so $a_p\le a_q$ pointwise and
$\Pp\{a_p>A\}\le\Pp\{a_q>A\}$ for every $A$.  The supremum over $\mathcal U$
is then attained at $\bar p$, which lies in $\mathcal U$.
\end{proof}

\section{Empirical design and reproducible methods}
\label{app:experiments}

The full registered manifests, configurations, and implementation details are
in the supplementary code's README.  This section specifies what the
completed synthetic and four-task RAG studies ran; findings and the
independent output audit appear in \cref{app:results}.  Synthetic cells
calibrate on $n=499$ examples unless stated otherwise, and
\cref{tab:real-pools} gives each RAG task's calibration size.  A clean tensor
stores scores before offline attacks and quantization; an analysis cell
applies one configuration to those scores.  Several cells can reuse a tensor,
so cell counts are not counts of independent datasets.

\subsection{Registered synthetic models}

\paragraph{Smooth copula model.}
The smooth model separates dependence between nodes from variation in their
score levels.  Here $i$ indexes a node, $j$ an example, and $y$ a candidate.
For each replicate, draw fixed node effects $\delta_i$ by randomly permuting
evenly spaced points in $[-1,1]$, so their mean is zero, and draw $Y_j$
uniformly from $\{1,\ldots,M\}$.  With $\Phi$ the standard-normal CDF, let
$G_{jy}$ and $E_{ijy}$ be independent arrays of i.i.d.\ standard Gaussians and,
for dependence parameter $\gamma\in[0,1]$, form
\[
 U_{ijy}=\Phi\{\sqrt\gamma\,G_{jy}+\sqrt{1-\gamma}\,E_{ijy}\}.
\]
The shared $G_{jy}$ couples the nodes; $E_{ijy}$ supplies node-specific
variation.  Let $F^{-1}_{p,q}$ be the $\operatorname{Beta}(p,q)$ quantile
function and $\eta\ge0$ the magnitude of node effects.  For $y=Y_j$, set
\[
 s^{\rm syn}_{ijy}=\operatorname{logit}^{-1}
 \left[
  \operatorname{logit}\{F^{-1}_{2,5}(U_{ijy})\}+\eta\delta_i
 \right],
\]
with $U_{ijy}$ clipped to $[10^{-9},1-10^{-9}]$, and use $F^{-1}_{5,2}$ when
$y\ne Y_j$.  Without node effects ($\eta=0$) a correct answer's score has mean
$2/7$ and an incorrect one's $5/7$, so smaller scores favor inclusion.
Conditional on the node effects, examples are independent, while $\gamma$
controls within-example node dependence and $\eta$ node heterogeneity.  The
registered values are $\gamma\in\{0,0.5,0.9,0.99\}$ and
$\eta\in\{0,0.5,1,2\}$.  Randomized stable attacks draw their randomness from
an independent per-example key, never from an example's index or phase.  The
label generates the joint law but is never an argument of the deployed score
map or a phase-blind attack.  Clipping and finite-precision storage can
create atoms, so coverage uses the tie-safe rank inequality.

\paragraph{Atomic model.}
Discrete score laws test the conservative treatment of tied ranks.  The
two-point true-label law assigns probabilities $(0.6,0.4)$ to $(0.2,0.5)$; an
incorrect label assigns $(0.2,0.4,0.4)$ to $(0.2,0.5,0.8)$.  The four-point
variant uses supports $(0.1,0.3,0.5,0.7)$ and $(0.3,0.5,0.7,0.9)$ with equal
masses.  Node $i$ uses a shared draw with probability $\sqrt\gamma$, so two
nodes share it with probability $\gamma$.  The diagnostic fixes
$(K,A,a,m,n,M,\alpha,\gamma)=(16,2,2,2,499,4,0.10,0.99)$, with 200
replicates, 2,000 test examples, bit depths $\{1,2,4,8\}$, and the \texttt{max}
and \texttt{low\_high} attacks.

\paragraph{Identification model.}
This experiment compares separate local distributions with aligned node
averages in the two worlds of \cref{prop:nonidentification}, with $K=2$,
$A=m=0$, and $\alpha=0.25$.  Each of 200 replicates uses nested sample sizes
$n\in\{99,999,9{,}999\}$ and 5,000 test examples.  The local arm thresholds
the average of the two empirical node CDFs, $G_n(t)$, at the smallest
pooled-grid $t$ with $G_n(t)\ge k/(n+1)$; the aligned arm uses the $k$th order
statistic of the paired node averages.  We record sup-norm CDF error,
threshold error against the exact $(1-\alpha)$ quantile, and true-label
coverage; set size is undefined because no false-label law is specified.

\paragraph{Shift stress model.}
The shift control changes the future true-label law without Byzantine
corruption, so it lies outside the exchangeability guarantee.  At the
operational factors with $a=0$ and no attack, it reuses 200 calibration
replicates of the operational clean tensor and changes only the correct-label
test law from $\operatorname{Beta}(2,5)$ to $\operatorname{Beta}(3,4)$, with
5,000 test scores per node.  Only true-label coverage is reported, and these
runs are not pooled with theorem-confirming experiments.

\paragraph{Narrow-support low--high model.}
This model checks the zero-coverage failure of symmetric trimming with
insufficient padding.  Fix $K=16$, $A=a=m=2$, $n=499$, $\alpha=0.10$, 200
replicates, and 5,000 test examples.  The 14 honest offsets are
$(-0.15,-0.15,0,\ldots,0,0.15,0.15)$ with ten zeros, $S\sim\Unif[0.45,0.46]$,
honest reports are exact, and the two Byzantine nodes report low--high.  With
$L_\xi=-0.025$, $U_\xi=0.025$, and $\Delta_\xi=0.05$, the exact population
coverage at total paddings $g\in\{0,0.025,0.05,0.10\}$ is $(0,0,450/500,1)$,
against which the Monte Carlo estimates are compared.  The sharp $p$-merger
construction at $K-A=14$ and $\ell=3$ is checked by exact enumeration of its
$\binom{14}{3}=364$ subset states plus the null state.  With
$t=\alpha\ell/(K-A)=3/140$ and both Byzantine $p$-values zero, the all-one
honest state has probability $0.9$ and each three-subset state, which sets
those three honest $p$-values to $t$, has probability $0.1/364$.  Each honest
$p$-value is then superuniform, whereas $\Pp\{p_{(A+\ell)}\le t\}=0.1$
exactly.

\begin{table}[H]
\centering
\small
\caption{Registered synthetic factor levels, not the Cartesian product of all
experiments.  The deployable trim is chosen from the registered budget $A$,
never from the unobserved realized count $a$.}
\label{tab:synthetic-grid}
\begin{tabularx}{\textwidth}{@{}p{0.25\textwidth}X@{}}
\toprule
Factor & Values\\
\midrule
Nodes $K$ & $8,16,32,64$\\
Declared budget $A$ & stochastic grid at $3$; operational recombinations at $2$;
offline sensitivity at $2,4,6$; analytic $G/W$ (padding and width, \cref{sec:bit-audit}) phase diagram at
$0,1,2,3,4,6$\\
Actual Byzantine count $a$ & $0,1,2,3,4,6$ (report the induced $a/K$)\\
Symmetric trim $m$ & $\max(0,A-2),A-1,A,A+\lceil0.05K\rceil,
            \lfloor(K-1)/2\rfloor$\\
Calibration size $n$ & $49,99,249,499,999,2499$\\
Candidate count $M$ & $2,4,8,16$\\
Bits per scalar & $1,2,3,4,6,8,16,32$\\
Miscoverage $\alpha$ & $0.05,0.10,0.20$\\
Node dependence $\gamma$ & $0,0.5,0.9,0.99$\\
Heterogeneity $\eta$ & $0,0.5,1,2$\\
\bottomrule
\end{tabularx}
\end{table}

\paragraph{Primary configuration and offline reuse.}
The operational primary, the default for the synthetic results of
\cref{sec:experiments}, is
\[
(K,a,A,m,n,M,b_{\calib},b_{\qry},\alpha,\gamma,\eta)
=(16,2,2,2,499,4,8,8,0.1,0.5,0.5).
\]
The retained $A=3$ block is anchored at the registered baseline
$(16,2,3,3,499,4,4,4,0.1,0.5,0.5)$ (cell 000), from which the one-factor sweeps
vary one factor at a time over \cref{tab:synthetic-grid}.  Its variant with
$a=A=m=3$ and four bits is the \emph{four-bit stress cell}; stable maximum
(cell 006) and high--low (cell 102) there form the four-bit stress inference
families.  \emph{Offline reuse} means applying $A$, $m$, and bit depth to a
cached full-precision clean tensor instead of generating new scores; the
operational cells and cell 102 reuse cell 000's tensor.  Four bit-mismatch
summaries, $(b_{\calib},b_{\qry})\in\{(2,8),(4,8),(8,4),(8,2)\}$ at the
operational factors, test directional padding but are ineligible for the
matched-pipeline and per-node $p$-value assertions.  The reference attack is
stable maximum, replaced only in the attack sweep.  At the operational anchor
it is the primary efficiency attack, high--low the worst-case-inflation
endpoint, and low--high the exact-cancellation endpoint.  The primary
theorem-valid trim is $m=A$; $m<A$ is an assumption violation and $m>A$ an
over-trimming diagnostic.  The two interaction blocks cross
$a\in\{0,2,3,4\}$, $m\in\{2,3,4\}$, $b\in\{2,4,8\}$ and
$K\in\{8,16,32\}$, $\gamma\in\{0,0.9,0.99\}$, $\eta\in\{0,0.5,2\}$.

\paragraph{Analysis scope and secondary comparators.}
The retained manifest has 94 analysis cells with $A=3$.  Adding the operational
primary, the four bit-mismatch summaries, and 11 targeted offline summaries
gives 110 stochastic summaries on 88 distinct clean-score cells.  The targeted
summaries are the two operational switching attacks, the $A=2,a=3$ low--high
violation, four-bit stress high--low, and seven budget cells at
$A=m\in\{2,4,6\}$.  Each has 500 calibration replicates and 5,000
test examples.  Only 23 tensors carry every candidate coordinate, so only they
support set size, full-set rate, and accuracy; the others support coverage.
The secondary synthetic comparators (\cref{app:baselines}) run on the 45
full-candidate smooth cells backed by those 23 tensors, all with $K=16$, and
enter no primary simultaneous-inference family.

\subsection{Attack catalog}

Each attack fixes what the corrupted nodes report on each channel, with
scores in $[0,\Smax]$, $\Smax=1$, and minimum and maximum codewords
$\rlo=0$ and $\rhi=\Smax$ (\cref{app:endpoint-attacks}).  The seven primary
synthetic attacks transform a cached potential clean score
$s_i^{\rm pot}(x,y)$ through the node's registered quantizer $Q_{i,r}$ on
channel $r$.  \texttt{min} and \texttt{max} (stable maximum) report the
minimum or maximum codeword on both channels; \texttt{reflect} reports
$Q_{i,r}(\Smax-s_i^{\rm pot})$; \texttt{shift\_plus} adds $0.25\Smax$ and
clips; \texttt{uniform} reports a uniformly drawn codeword per key,
candidate, and node without a phase flag.  These five \emph{stable} attacks
use the same transformation in both phases.  The two \emph{switching}
attacks, \texttt{low\_high} (low--high) and \texttt{high\_low} (high--low),
report the minimum in one phase and the maximum in the other, low--high
reporting the minimum in calibration.  The real-data manifest adds
\texttt{cal\_only\_max} and \texttt{query\_only\_max} (cal.\ max and query
max), which report the maximum codeword on the named channel and the normally
quantized score on the other, plus secondary negative-shift and bit-flip
faults.

For each base clean cell and replicate, one salted uniform permutation of the
$K$ identities is drawn and its first $a$ identities are corrupted, so
corrupted sets are nested.  The permutation is shared by every method,
attack, $A$, $m$, and bit depth, and the attack interface never receives the
gold label.  In matched-pipeline experiments a stable attack is a fixed,
phase-blind transformation that cannot inspect the example index, phase,
calibration history, or label availability, which tests
\cref{lem:phase-blind}.  Switching attacks test the deletion coverage theorem
and the symmetric low--high lower bound.  Cells with $a>A$ test
misspecification; theorem assertions are disabled in such deliberately
ineligible cells.

\subsection{Executed finite-answer RAG protocol}
\label{app:rag-protocol}

The executed study comprises OpenBookQA, MedQA-US, MedMCQA, and
MMLU-Med~\citep{mihaylov2018openbookqa,jin2020medqa,pal2022medmcqa,hendrycks2021mmlu}.
The three medical pools use the representation of the MIRAGE (Medical
Information Retrieval-Augmented Generation Evaluation)
benchmark~\citep{xiong2024benchmarking} with its complete Textbooks corpus.
The results do not represent the complete MIRAGE suite.

\paragraph{Corpora, virtual nodes, and retrieval.}
Each corpus is split two ways into $K=16$ virtual nodes, so these are
controlled public-data federations rather than independent institutional
deployments.  OpenBookQA uses 1,326 distinct official science facts, and
the medical tasks use all 125,847 passages from 18 Textbooks sources.  The hash
partition (H) assigns documents or passages by a frozen salted hash.  The
topic/source partition (T) uses frozen topic assignments from
\texttt{all-MiniLM-L6-v2} embeddings~\citep{reimers2019sbert} for OpenBookQA
and keeps whole books together for the medical tasks, assigning each book greedily to
the least-loaded node.  Each node retrieves four passages from its own shard
by BM25 keyword ranking~\citep{robertson2009bm25} ($k_1=1.2$,
$b_{\mathrm{BM25}}=0.75$) using the question alone, so candidate answers and
labels cannot alter the evidence.

\paragraph{Frozen answer scoring.}
All answer scores use GPT-2 small (124 million
parameters)~\citep{radford2019gpt2} with the prompt
\begin{quote}\ttfamily
Question: [question]\\
Context: [space-joined retained passages]\\
Answer:
\end{quote}
Each complete answer text $y$ is scored in a separate teacher-forced pass as
the continuation of one space followed by its tokens $o_{y1},\ldots,o_{yT_y}$.
The candidate weight uses the average log probability per token, avoiding a
direct penalty for longer answers:
\[
 \ell_i(y)=\frac1{T_y}\sum_{t=1}^{T_y}
 \log p_{\rm GPT2}(o_{yt}\mid \text{prompt}_i,o_{y,<t}),\qquad
 \pi_i(y)=\frac{\exp\{\ell_i(y)\}}
                  {\sum_{z\in\cY}\exp\{\ell_i(z)\}}.
\]
These weights are scoring devices, not calibrated posterior probabilities.
Passages are shortened to fit the 1,024-token limit, while question and answer
texts are never truncated; a question/candidate-only overflow yields the
registered all-ones score vector, kept in every denominator.  Candidate order
is fixed by a salted public-content hash.  With $z_i(x)$ the context node $i$
retrieved, the primary score is
\begin{equation*}
 s_i(x,y)=1-\pi_i\{y\mid x,z_i(x)\}\in[0,1],
\end{equation*}
so smaller scores favor inclusion.  The secondary NLL score is the clipped
negative log likelihood
\[
 s_i^{\rm NLL}(x,y)=
 \min\{-\log(\max[\pi_i(y\mid x,z_i(x)),10^{-12}]),12\}/12,
\]
applied before quantization or attacks and separately calibrated, with the
primary splits and corruption randomness.  The no-RAG and centralized-union
anchors (\cref{app:baselines}) use the same scoring with empty context and
with four passages retrieved from the whole corpus.

\begin{table}[H]
\centering\small
\caption{Executed pool sizes and the split used within each repeat; pilot
records are excluded from every held-out pool.  Every task has four answer
candidates, and $k=\lceil0.9(n+1)\rceil$ for nominal coverage 90\%.}
\label{tab:real-pools}
\begin{tabular}{@{}lrrrrr@{}}
\toprule
Task & Pilot & Held-out $N$ & Calibration $n$ & Evaluation & $k$\\
\midrule
OpenBookQA & 500 & 1,000 & 333 & 667 & 301\\
MedQA-US & 100 & 1,173 & 391 & 782 & 353\\
MedMCQA & 100 & 4,083 & 499 & 3,584 & 450\\
MMLU-Med & 100 & 989 & 329 & 660 & 297\\
\bottomrule
\end{tabular}
\end{table}

\paragraph{Pools and repeated splits.}
Scoring decisions were made on pilot records disjoint from the held-out pools
(\cref{tab:real-pools}).  Each of 100 repeats permutes the fixed held-out pool
without consulting labels or scores, using the first $n$ records for
calibration and the remaining $N-n$ for evaluation.  Corrupt identities and
attack randomness are paired across methods and partitions.  This supports
exchangeability for a uniformly selected held-out record, not independence
between that record and its calibration subset, so efficiency is measured
directly and the bounds in \cref{eq:size-modulus,eq:size-density-alternative}
are not invoked.

\paragraph{Threats and cells.}
Primary cells use $(K,A,m,b,\alpha)=(16,2,2,8,0.10)$ and seven threats: clean,
stable maximum with $a=1,2$, and, with $a=2$, calibration-only maximum,
query-only maximum, high--low, and low--high.  Across four tasks and two
partitions this gives 56 primary cells.  Secondary cells add $a=3>A$ bound
violations, a four-bit stress block $(A,m,b)=(3,3,4)$, the fault attacks,
$\alpha\in\{0.05,0.20\}$, secondary aggregators, and NLL scores, for 408 cells
and 264,800 repeat-level rows in total.

\paragraph{Common point ranker and raw sets.}
Point prediction is separate from set construction.  Every non-oracle
federated set method shares one predictor, which ranks candidates by $\TM_A$,
the mean of the $K$ query reports after the $A$ largest and $A$ smallest are
discarded:
\begin{equation}
 \widehat y_{\rm common}(x)=
 \arg\min_{y\in\cY}\TM_A\{w_1(x,y),\ldots,w_K(x,y)\}.
 \label{eq:common-point-ranker}
\end{equation}
Ties use the fixed candidate order.  Its accuracy is reported once per
attacked tensor, not as a gain attributable to set construction.  All reported
set sizes and coverage use raw sets, including empty ones.  The optional
display augmentation
\begin{equation}
 \cC^+(x)=\begin{cases}\cC(x),&\cC(x)\ne\varnothing,\\
 \{\widehat y_{\rm common}(x)\},&\cC(x)=\varnothing
 \end{cases}
 \label{eq:empty-augmentation}
\end{equation}
cannot lower coverage, but it changes set utility, does not certify the forced
answer, and never replaces the raw results.

\subsection{Baseline definitions}
\label{app:baselines}

Only the proposed rules, guarded symmetric trimming, and the $p$-value merger
carry coverage guarantees against arbitrary reports in both phases, the last
two under the conditions in their entries.  The unguarded summaries and the
Rob-FCP adaptation have no such guarantee.  All non-oracle baselines receive
the same quantized node tensor and candidate ordering, and no deployable
baseline receives honest identities or realized honest ranges.

\begin{enumerate}
 \item The \textbf{honest-mean oracle} (not deployable) and the three proposed
 rules are as defined in \cref{sec:setup,sec:method}.
 \item \textbf{Guarded symmetric trimming} averages the reports left after
 discarding the $m=A$ largest and $m$ smallest ($\TM_m$,
 \cref{sec:protocol-details}) and adds a guard of $A\Smax/(K-A)+\rho_r$ per
 channel (\cref{sec:symmetric-details}).  It shows, in realized sets, the vacuity
 that \cref{sec:bit-audit} diagnoses at the four-bit stress cell.  \textbf{Unguarded trimming} is the same trimmed
 mean with no guard, a matched-pipeline diagnostic.
 \item The \textbf{$p$-value merger} (partial-conjunction/R\"uger merger) keeps
 a candidate when the $(2A+1)$th smallest of the $K$ per-node conformal
 $p$-values \eqref{eq:node-pvalue} exceeds the cutoff \eqref{eq:pmerge-cutoff},
 so the $A$ smallest cannot exclude a candidate alone (\cref{prop:pmerge}).  It
 requires matched quantizers and superuniform honest ranks for the same future
 target.
 \item The \textbf{all-node mean} averages all $K$ reports; the
 \textbf{median} averages the two central order statistics when $K$ is even; the \textbf{$m$-Winsorized
 mean} replaces the bottom and top $m=A$ values by $v_{(m+1)}$ and
 $v_{(K-m)}$ before averaging.  All three use aligned calibration with no
 guard.
 \item \textbf{Vector-level Krum} selects, for each calibration example and
 query, the node's $M$-score report minimizing the summed squared Euclidean
 distance to its $K-A-2$ nearest other reports, with salted-identity tie
 breaks.  It is run only when $K>2A+2$ and does not inherit the convergence
 theory of \citet{blanchard2017krum}.
 \item The \textbf{median of means} averages within $2A+1$ fixed, publicly
 seeded node groups and takes the median group mean; it is run only when
 $2A+1\le K$.
 \item The \textbf{local-marginal comparator} discards alignment
 (\cref{sec:alignment}).  It forms each node's empirical CDF
 $\widehat F_i$ of its own correct-answer calibration reports, sets
 $q_{\rm LM}$ to the smallest grid value with
 $\TM_A(\widehat F_1(t),\ldots,\widehat F_K(t))\ge k/(n+1)$, and includes a
 query candidate when $\TM_A$ of its $K$ reports is at most $q_{\rm LM}$.
 \item The \textbf{Rob-FCP adaptation} of \citet{kang2024robfcp} summarizes
 each node's correct-answer calibration scores as a histogram, keeps the
 $K-A$ nodes whose histograms lie closest to their nearest neighbors, and
 sets the threshold from their pooled scores.  Queries use the unprotected
 all-node mean.  Defaults were fixed before pilot outcomes and no filtering
 threshold was tuned; the original coverage theorem is not claimed here.
 \item The \textbf{no-RAG} and \textbf{centralized-union anchors} score with
 empty context or with retrieval from the union corpus.  They report top-one
 accuracy only and are not federated coverage baselines.
\end{enumerate}

\subsection{Metrics and uncertainty}

Uncertainty is computed across complete calibration repeats (replicates),
because queries within a repeat share the same calibrated rule, node effects,
corrupt identities, and attack realization.

\paragraph{Repeat-level summaries.}
Each replicate $r$ is first reduced to its within-replicate statistic $T_r$,
such as coverage, mean set size, full-set rate, or accuracy.  Across
$R_{\mathrm{rep}}$ replicates, the estimate is the mean $\bar T$, and its
Monte Carlo standard error is
\[
 \operatorname{se}(\bar T)=
 \frac{\operatorname{sd}(T_1,\ldots,T_{R_{\mathrm{rep}}})}{\sqrt{R_{\mathrm{rep}}}},
\]
with $\operatorname{sd}$ the sample standard deviation.  A method contrast is
first reduced to its within-repeat paired difference and treated identically.

\paragraph{Finite-pool estimand.}
The real-data estimand averages a metric over random splits, corrupt
identities, and attack draws within the observed pool.  Fix a held-out pool
$\mathcal P$ of size $N$, let $C$ be uniform over its $n$-subsets, let $\Pi$
be an independent uniform permutation of the $K$ identities whose first $a$
entries are corrupted, and let $\Xi$ be any independent attack randomness.
For a metric loss $L$,
\[
 \theta=\E_{C,\Pi,\Xi}\left[
  \frac1{N-n}\sum_{x\in\mathcal P\setminus C}L(C,x;\Pi,\Xi)
 \right].
\]
We use 100 paired repeats per task, so the Monte Carlo standard error is the
sample standard deviation divided by 10.  These intervals quantify
finite-pool, corrupt-identity, and attack randomization only.  We do not use a
question bootstrap, do not claim a superpopulation interval, and do not
require a 95\% lower bound to exceed $1-\alpha$ as a test of the theorem.

\paragraph{Simultaneous intervals.}
A contrast is a metric difference, and a family is the group of contrasts
covered by one simultaneous statement.  Within each family, limited to
coverage, mean size, and full-set rate, intervals use a fixed-seed,
$10{,}000$-draw Rademacher multiplier procedure.  For replicate-level
contrasts $d_{rc}$ (signed coverage relative to $1-\alpha$, or a paired method
difference) with repeat means $\bar d_c$, each draw $b$ uses independent
equiprobable signs $\xi_{rb}\in\{-1,1\}$ to form
\[
 T_b=\max_{c:\,\widehat{\operatorname{se}}(\bar d_c)>0}\left|
 \frac{R_{\mathrm{rep}}^{-1}\sum_r\xi_{rb}(d_{rc}-\bar d_c)}
      {\widehat{\operatorname{se}}(\bar d_c)}
 \right|.
\]
The reported approximate 95\% simultaneous intervals are
$\bar d_c\pm q_{.95}\widehat{\operatorname{se}}(\bar d_c)$, with $q_{.95}$ the
9,500th smallest of the 10,000 draws; they are asymptotic, not exact.  An
endpoint the design proves deterministic receives a point interval at its
known constant.  An endpoint with merely zero observed standard deviation is
reported with its estimate, and its adjusted interval is marked unavailable.
A non-finite contrast value aborts the family.

\paragraph{Families and diagnostics.}
The executed analysis has 48 families with 1,978 contrasts: 20 coverage
families (860 contrasts) and 28 efficiency families (1,118 paired contrasts),
including one 70-contrast family per real task for each of coverage, mean
size, and full-set rate.  Separate 95\% levels do not imply simultaneous
coverage across families.  Fixed-set, joint-threshold, and deletion are
contrasted with the honest-mean oracle; guarded symmetric trimming and the
$p$-merger with deletion; secondary comparators enter no family.  Coverage
families use every true-label tensor, while size and full-set families use
only the full-candidate tensors.  Every deterministic theorem inequality is
checked before aggregation in exact arithmetic on the registered grid, so it
needs no confidence interval.  Undercoverage is
$[1-\alpha-\widehat{\mathrm{coverage}}]_{+}$ with $[u]_+=\max\{u,0\}$, and
retrieval recall is not reported.
{}
\FloatBarrier
\section{Supporting empirical results and output audit}
\label{app:results}

The supporting results explain where the observed set-size gains come from
and how their uncertainty was checked. We describe the numerical audit, then the
complete synthetic reference conditions and sensitivity studies, then the complete
real-data comparisons and the score-quality context they sit in. Further
supporting displays (per-task threat plots, the NLL-score and real-data
$\alpha$ sensitivities, and the discrete-score and shifted-query controls) are in the
supplement. Settings, cell IDs and attack names follow
\cref{app:experiments}. A cell is \emph{eligible} for a method when it meets the
assumptions of that method's guarantee; deliberately invalid cells, such as
$a>A$, are kept as negative controls. Set sizes are \emph{raw}: an empty set
counts as zero and is never replaced by a forced answer
(\cref{eq:empty-augmentation}).

\subsection{Verification and corrected inference}

Independent reaggregation reproduces the reported means, set-status
identities, and paired method differences, checks every output against its
recorded inputs, and reproduces all 48 original family critical values. It
covers all 77,500 synthetic result files (510,000 method rows) and 72
real-data analysis bundles (264,800 repeat rows). The original files remain
immutable.

\paragraph{The constant-column correction.}
A metric or paired difference that takes the same value in every observed
repeat does not justify a zero-width interval, so the registered rule reports
its estimate and zero standard error but leaves its interval unavailable
unless the design proves it deterministic. Floating-point roundoff in the
original aggregation had given some such columns spurious positive standard
errors; the correction detects exact equality and otherwise centers with
accurate \texttt{fsum} summation. When the same 10,000 sign draws per
family are replayed, 141 synthetic and 94 real endpoints become unavailable (149 and 114 in
total), and some family critical values change. The largest change to a
retained bound is $1.12\times10^{-5}$ (synthetic) and $2.45\times10^{-4}$
(real). No synthetic interval changes sign; three real intervals (fixed-set,
joint-threshold and deletion coverage-minus-$0.9$ for MMLU-Med/hash/low--high)
move from including zero to a lower endpoint of about $6.41\times10^{-5}$, a
borderline change that is disclosed rather than used as a claim. Scores,
splits, attacks and point estimates are unchanged.

\paragraph{Theorem diagnostics.}
No proposed-envelope containment failure occurs in a theorem-eligible
primary cell. Positive failure counters in the retained outputs occur in
deliberate $a>A$ tests. Some counters are stored on oracle rows, which are
always flagged eligible; the audit instead judges each envelope check by
whether $a\le A$ in its cell. Grid membership, hierarchy, and
guarded-symmetric containment use exact integer/rational decisions.
Diagnostics for outer score bounds and the decoded-mean cancellation reference
(the honest-decoded padded set $\cC^{\rm dec,\cH}_\alpha$ of
\cref{cor:endpoints-budget}, which all three rules equal under exact-budget
low--high) use the registered $10^{-12}$ floating tolerance. Across 500
synthetic replicates, 57 checked combinations of audit and method yield 28,500
repeat-level checks and 570 million
coordinate comparisons of budget nesting and high--low extremality, with zero
violations. The audit matched the aggregate check counts to their source
records; it did not regenerate every stored record of which candidates each
set contains, nor the underlying model scores.
\FloatBarrier

\subsection{Complete synthetic reference conditions and comparison intervals}
\label{app:synthetic-reference}

Fixed-set is never larger than either relaxation and covers at least 90\% in
every reference condition within the declared budget; attack direction sets
how much inflation remains. At the low--high endpoint with $a=A$, all three
proposed rules coincide with the honest nodes' own comparison plus
reconstruction padding; their common mean size is \SynLowHighSize{}, with
\SynLowHighEmptyPercent\% empty sets, so a mean size below one does not mean
that every query receives one correct answer. Exceeding the budget,
$a=3>A=2$, lowers proposed coverage to \SynViolatedCoverage\%.
\Cref{tab:synthetic-main} reports all six reference
conditions, including the four-bit stress cell (four-bit reports with
$a=A=m=3$) and the deliberate corruption-budget violation, so the favorable
operational result (eight bits, $a=A=2$) can be read alongside its limits.

\begingroup
\let\resulttablefloat\table
\renewcommand{\table}[1][]{\resulttablefloat[H]}%
\begin{table}[tbp]
\centering
\footnotesize
\setlength{\tabcolsep}{3pt}
\caption{Fixed-set reduces deletion's set size under stable maximum; the $p$-merger is smaller under
high--low. Entries are coverage (\%) above raw mean size, including empty sets, over 500
repeats (nominal 90\%). Labels give cell ID and $(a,A,b)$: actual corruption, declared
budget, and common report depth. The final $a>A$ row lies outside the guarantees, which
require $a\le A$. The oracle uses honest identities and full precision; rounded 100.00
values need not imply perfect coverage.}
\label{tab:synthetic-main}
\begin{tabularx}{\linewidth}{@{}l*{6}{>{\centering\arraybackslash}X}@{}}
\toprule
Condition; ID; $(a,A,b)$ & Fixed-set & Joint & Deletion & Guarded sym. & p-merger & Oracle \\
\midrule
\shortstack[l]{Clean\\103; $(0,2,8)$} & \shortstack{93.77\\1.036} & \shortstack{95.36\\1.088} & \shortstack{96.25\\1.128} & \shortstack{99.97\\2.530} & \shortstack{98.74\\1.404} & \shortstack{90.02\\0.953} \\[2pt]
\shortstack[l]{Stable max\\094; $(2,2,8)$} & \shortstack{94.11\\1.049} & \shortstack{97.84\\1.243} & \shortstack{98.31\\1.298} & \shortstack{99.98\\2.648} & \shortstack{99.03\\1.487} & \shortstack{90.04\\0.955} \\[2pt]
\shortstack[l]{Low--high\\099; $(2,2,8)$} & \shortstack{90.54\\0.965} & \shortstack{90.54\\0.965} & \shortstack{90.54\\0.965} & \shortstack{99.82\\1.946} & \shortstack{97.05\\1.222} & \shortstack{90.04\\0.955} \\[2pt]
\shortstack[l]{High--low\\100; $(2,2,8)$} & \shortstack{99.73\\1.836} & \shortstack{99.84\\1.995} & \shortstack{99.89\\2.126} & \shortstack{100.00\\3.214} & \shortstack{99.03\\1.487} & \shortstack{90.04\\0.955} \\[2pt]
\shortstack[l]{Four-bit high--low\\102; $(3,3,4)$} & \shortstack{100.00\\3.793} & \shortstack{100.00\\3.868} & \shortstack{100.00\\3.934} & \shortstack{100.00\\4.000} & \shortstack{99.72\\1.884} & \shortstack{90.01\\0.956} \\[2pt]
\shortstack[l]{Low--high violation\\101; $(3,2,8)$} & \shortstack{77.13\\0.784} & \shortstack{77.13\\0.784} & \shortstack{77.13\\0.784} & \shortstack{99.09\\1.465} & \shortstack{94.57\\1.110} & \shortstack{90.01\\0.956} \\[2pt]
\bottomrule
\end{tabularx}
\end{table}
\endgroup{}

Fixed membership removes a substantial part of the extra set size introduced
by the relaxations: under operational stable maximum, fixed-set removes about
three quarters of deletion's excess over the oracle, and its mean size is
\SynStableFsReductionPercent\% below deletion's.
The high--low and four-bit rows also show that fixed membership does
not make fixed-set smaller than the $p$-merger under every attack.
\FloatBarrier

The registered uncertainty calculations confirm the reversal between
fixed-set and the $p$-merger. Under operational high--low, their mean sizes
are \SynHighLowFsSize{} and \SynHighLowPmergeSize{}, respectively.
\Cref{tab:synthetic-connected} reports the size differences together with
conservative intervals obtained from the existing simultaneous comparisons.

\begingroup
\let\resulttablefloat\table
\renewcommand{\table}[1][]{\resulttablefloat[H]}%
\begin{table}[tbp]
\centering
\footnotesize
\setlength{\tabcolsep}{3pt}
\caption{The fixed-set--$p$-merger size advantage reverses with the attack. Entries are mean-size
differences (fixed-set minus $p$-merger); negative values favor fixed-set. Conservative
intervals are linear combinations of corrected approximate 95\% simultaneous bounds. The
first three rows use the operational-attacks family; the final row uses the four-bit-stress
family. Simultaneity applies within each family, not across both.}
\label{tab:synthetic-connected}
\begin{tabularx}{\linewidth}{@{}ll*{3}{>{\centering\arraybackslash}X}@{}}
\toprule
Condition; ID & Original family & Mean difference & Lower & Upper \\
\midrule
Stable max; 094 & operational attacks & -0.4381 & -0.4467 & -0.4295 \\
Low--high; 099 & operational attacks & -0.2572 & -0.2612 & -0.2531 \\
High--low; 100 & operational attacks & +0.3491 & +0.3267 & +0.3716 \\
Four-bit high--low; 102 & four-bit stress & +1.9089 & +1.8940 & +1.9238 \\
\bottomrule
\end{tabularx}
\end{table}
\endgroup{}

The operational-attacks family is the family of mean-size contrasts in the
operational attack cells, including cells 094, 099 and 100.
\FloatBarrier

\subsection{Identification and padding diagnostics}

\Cref{tab:synthetic-diagnostics} tests alignment in the two worlds of
\cref{prop:nonidentification} (upper block) and the padding requirement of
\cref{prop:low-high-lower} (lower block); designs are in \cref{app:experiments}.

\begingroup
\let\resulttablefloat\table
\renewcommand{\table}[1][]{\resulttablefloat[H]}%
\begin{table}[tbp]
\centering
\footnotesize
\setlength{\tabcolsep}{3pt}
\caption{Alignment tracks the honest-mean distribution in both worlds; local marginals overcover in
World 1. Insufficient padding gives zero coverage in the separate narrow-support
construction. Entries average 200 repeats. Identification uses $n=9{,}999$ and nominal 75\%;
padding uses nominal 90\%. CDF (cumulative distribution function) error is the maximum
absolute difference from truth; threshold error is signed. Padding's population coverage is
known analytically. Means are descriptive.}
\label{tab:synthetic-diagnostics}
\begin{tabularx}{\linewidth}{@{}ll*{3}{>{\centering\arraybackslash}X}@{}}
\toprule
World & Calibration summary & CDF error & Threshold error & Coverage (\%) \\
\midrule
0 & Aligned & 0.00855 & 0.00000 & 75.064 \\
0 & Local marginal & 0.00627 & +0.00013 & 75.080 \\
1 & Aligned & 0.00858 & -0.00002 & 75.063 \\
1 & Local marginal & 0.24985 & +0.09357 & 99.682 \\
\midrule
\multicolumn{2}{l}{Narrow-support padding} & \multicolumn{2}{c}{Observed coverage (\%)} & Exact population (\%) \\
\multicolumn{2}{l}{0.0} & \multicolumn{2}{c}{0.000} & 0.0 \\
\multicolumn{2}{l}{0.025} & \multicolumn{2}{c}{0.000} & 0.0 \\
\multicolumn{2}{l}{0.05} & \multicolumn{2}{c}{90.188} & 90.0 \\
\multicolumn{2}{l}{0.1} & \multicolumn{2}{c}{100.000} & 100.0 \\
\bottomrule
\end{tabularx}
\end{table}
\endgroup{}

The identification rows show why preserving paired examples matters: in
World 1 the local-marginal rule's discrepancy is excessive conservatism,
while World 0 shows that local summaries can work in one dependence structure
without identifying the relevant law in both. The padding rows illustrate the
proved padding requirement and the cost of excessive guarding.
\FloatBarrier

\subsection{Complete communication and declared-budget sweeps}
\label{app:synthetic-design}

The synthetic resolution sweep isolates report precision, which the real-data
four-bit stress block (supplement) confounds with a larger budget.
\Cref{tab:synthetic-resolution} gives the complete equal-depth sweep, with the
same bit depth in both phases, and the four conditions with unequal calibration
and query depths.

\begingroup
\let\resulttablefloat\table
\renewcommand{\table}[1][]{\resulttablefloat[H]}%
\begin{table}[tbp]
\centering
\footnotesize
\setlength{\tabcolsep}{3pt}
\caption{Finer reports reduce fixed-set size, with little change beyond eight bits in this sweep.
Entries are 500-repeat descriptive means; $b_c/b_q$ denotes calibration/query bits. Stable
maximum uses $a=2$, with $A=m=3$ in the equal-depth block and $A=m=2$ otherwise. Equal-depth
cells are independently generated. Daggers mark unmatched-map $p$-merger results outside its
guarantee; directional padding remains eligible.}
\label{tab:synthetic-resolution}
\begin{tabularx}{\linewidth}{@{}lcc*{6}{>{\centering\arraybackslash}X}@{}}
\toprule
ID & $A$ & $b_c/b_q$ & FS cov. (\%) & FS size & Joint size & Del. size & Sym. size & $p$ size \\
\midrule
021 & 3 & 1/1 & 100.00 & 4.000 & 4.000 & 4.000 & 4.000 & 4.000 \\
022 & 3 & 2/2 & 100.00 & 3.768 & 3.788 & 3.823 & 4.000 & 3.242 \\
023 & 3 & 3/3 & 99.84 & 2.010 & 2.473 & 2.664 & 4.000 & 2.149 \\
000 & 3 & 4/4 & 98.71 & 1.377 & 1.785 & 1.932 & 4.000 & 1.796 \\
024 & 3 & 6/6 & 96.25 & 1.136 & 1.452 & 1.561 & 3.978 & 1.628 \\
025 & 3 & 8/8 & 95.26 & 1.090 & 1.385 & 1.485 & 3.957 & 1.588 \\
026 & 3 & 16/16 & 94.93 & 1.079 & 1.369 & 1.467 & 3.950 & 1.578 \\
027 & 3 & 32/32 & 94.95 & 1.078 & 1.369 & 1.467 & 3.950 & 1.582 \\
\midrule
094 & 2 & 8/8 & 94.11 & 1.049 & 1.243 & 1.298 & 2.648 & 1.487 \\
095 & 2 & 2/8 & 99.90 & 2.137 & 2.425 & 2.502 & 3.939 & 1.467$^{\dagger}$ \\
096 & 2 & 4/8 & 96.66 & 1.156 & 1.382 & 1.452 & 2.982 & 1.484$^{\dagger}$ \\
097 & 2 & 8/4 & 96.64 & 1.156 & 1.382 & 1.453 & 2.981 & 1.490$^{\dagger}$ \\
098 & 2 & 8/2 & 99.89 & 2.209 & 2.452 & 2.574 & 3.727 & 1.457$^{\dagger}$ \\
\bottomrule
\end{tabularx}
\end{table}
\endgroup{}

The equal-depth means are descriptive; they do not identify an optimal bit depth. At one bit,
every displayed method returns all four candidates: coverage remains
conservative while answer selection disappears. In the unequal-depth rows the
$p$-merger loses its guarantee, which needs matched scoring and quantization
in both phases (\cref{sec:merger-details}).
\FloatBarrier

Fixed membership also limits the cost of declaring a conservative corruption
budget. \Cref{tab:synthetic-budget} holds clean scores, reports, and actual
corrupt identities fixed while increasing $A$, so these comparisons isolate
the declared-budget effect. It covers clean, stable-maximum, and high--low
conditions at every displayed budget.

\begingroup
\let\resulttablefloat\table
\renewcommand{\table}[1][]{\resulttablefloat[H]}%
\begin{table}[tbp]
\centering
\footnotesize
\setlength{\tabcolsep}{3pt}
\caption{Fixed membership limits inflation as the declared budget grows, under clean and
stable-maximum reports. Entries are 500-repeat descriptive raw mean sizes, including empty
sets; FS denotes fixed-set, with coverage in percent (nominal 90\%). Reports use eight bits.
Within each attack block, only budget/trim ($m=A$) changes; clean scores and attack
realizations are paired.}
\label{tab:synthetic-budget}
\begin{tabularx}{\linewidth}{@{}lcc*{5}{>{\centering\arraybackslash}X}r@{}}
\toprule
Attack; ID & $a$ & $A$ & Fixed-set & Joint & Deletion & Sym. & $p$-merger & FS cov. \\
\midrule
Clean; 103 & 0 & 2 & 1.036 & 1.088 & 1.128 & 2.530 & 1.404 & 93.77 \\
Clean; 104 & 0 & 4 & 1.106 & 1.234 & 1.340 & 4.000 & 1.569 & 95.75 \\
Clean; 107 & 0 & 6 & 1.189 & 1.431 & 1.646 & 4.000 & 1.845 & 97.13 \\
\midrule
Stable max; 094 & 2 & 2 & 1.049 & 1.243 & 1.298 & 2.648 & 1.487 & 94.11 \\
Stable max; 105 & 2 & 4 & 1.134 & 1.553 & 1.710 & 4.000 & 1.737 & 96.19 \\
Stable max; 108 & 2 & 6 & 1.238 & 1.996 & 2.278 & 4.000 & 2.320 & 97.61 \\
\midrule
High--low; 100 & 2 & 2 & 1.836 & 1.995 & 2.126 & 3.214 & 1.487 & 99.73 \\
High--low; 106 & 2 & 4 & 2.351 & 2.701 & 2.964 & 4.000 & 1.737 & 99.94 \\
High--low; 109 & 2 & 6 & 3.050 & 3.468 & 3.708 & 4.000 & 2.320 & 99.99 \\
\bottomrule
\end{tabularx}
\end{table}
\endgroup{}

A wider budget permits more possible honest groups, but fixed membership
avoids much of the additional set
inflation that deletion incurs. The clean block shows that a conservative
budget also costs efficiency when no node is actually corrupt, while the
high--low block shows the larger sets caused by switching attacks.
\FloatBarrier

\subsection{Broad synthetic sensitivity and assumption controls}

Coverage holds well beyond the reference settings of
\cref{app:synthetic-reference}: every eligible cell of the broad sensitivity
study has mean coverage at or above nominal. \Cref{tab:synthetic-sensitivity} summarizes
coverage across changes in node count, actual corruption, trimming,
calibration size, candidate count, dependence, heterogeneity, nominal level,
and attack type, together
with the two registered interaction blocks, which vary $(a,m,b)$ and
$(K,\gamma,\eta)$ jointly (\cref{app:experiments}).

\begingroup
\let\resulttablefloat\table
\renewcommand{\table}[1][]{\resulttablefloat[H]}%
\begin{table}[tbp]
\centering
\footnotesize
\setlength{\tabcolsep}{3pt}
\caption{Every eligible cell has mean coverage at least nominal. Ranges span 500-repeat cell means,
expressed as excess over each cell's nominal coverage in percentage points (pp). Counts
follow fixed-set (FS) / symmetric / $p$-merger order. Each method's range is descriptive,
not a confidence interval, and covers only its eligible cells. Corruption and trim
violations remain in the complete numeric summaries; \cref{tab:synthetic-main} includes the
deliberate budget violation.}
\label{tab:synthetic-sensitivity}
\begin{tabularx}{\linewidth}{@{}l r *{3}{>{\centering\arraybackslash}X}@{}}
\toprule
Registered block & Eligible cells & FS excess (pp) & Sym. excess (pp) & $p$ excess (pp) \\
\midrule
Node count $K$ & 3/3/3 & 7.20--9.75 & 9.40--10.00 & 9.46--10.00 \\
Actual corruption $a$ & 3/3/3 & 8.22--8.95 & 10.00--10.00 & 9.50--9.73 \\
Symmetric trim $m$ & 4/2/4 & 8.69--8.72 & 10.00--10.00 & 9.65--9.66 \\
Calibration size $n$ & 5/5/5 & 8.68--8.86 & 10.00--10.00 & 9.65--9.87 \\
Candidate count $M$ & 3/3/3 & 8.70--8.72 & 10.00--10.00 & 9.65--9.66 \\
Nominal $\alpha$ & 2/2/2 & 4.54--16.36 & 5.00--20.00 & 4.92--18.75 \\
Dependence $\gamma$ & 3/3/3 & 6.13--10.00 & 10.00--10.00 & 8.01--10.00 \\
Heterogeneity $\eta$ & 3/3/3 & 8.67--9.26 & 10.00--10.00 & 9.64--9.78 \\
Stable/switching attacks & 6/6/6 & 7.27--9.99 & 10.00--10.00 & 9.01--9.65 \\
$a,m,b$ interaction & 20/12/20 & 4.83--10.00 & 10.00--10.00 & 9.01--10.00 \\
$K,\gamma,\eta$ interaction & 24/24/24 & 5.40--10.00 & 9.77--10.00 & 7.93--10.00 \\
\bottomrule
\end{tabularx}
\end{table}
\endgroup{}

Two further
controls (supplement) separate tie handling from exchangeability: on discrete-score cells
every method, including the oracle, covers conservatively, while under the
shifted-query law, which breaks clean-example exchangeability, all three
proposed methods and the $p$-merger ($87.16\%$) fall below nominal.
\FloatBarrier

The secondary synthetic aggregators test whether a conventional robust
summary is sufficient under changes in reporting behavior. These methods
use the same scores as the primary comparisons but do not have the proposed
guard for the clean honest-mean comparison.

\begingroup
\let\resulttablefloat\table
\renewcommand{\table}[1][]{\resulttablefloat[H]}%
\begin{table}[tbp]
\centering
\footnotesize
\setlength{\tabcolsep}{3pt}
\caption{Low--high switching causes undercoverage for several aggregators; Krum and local-marginal
calibration remain above nominal coverage. Entries are coverage (\%) above raw mean size,
including empty sets, over 500 paired repeats; settings match \cref{tab:synthetic-main}. MoM
denotes median of means. These descriptive means lie outside the primary simultaneous
families. Switching need not preserve exchangeability; no primary Byzantine-envelope
guarantee is asserted for these rows.}
\label{tab:synthetic-comparators}
\begin{tabularx}{\linewidth}{@{}l*{7}{>{\centering\arraybackslash}X}@{}}
\toprule
Condition; source ID & \shortstack{All-node\\mean} & Median & Winsorized & \shortstack{Unguarded\\trim} & \shortstack{Local\\marginal} & MoM & Krum \\
\midrule
\shortstack[l]{Stable max\\094} & \shortstack{90.05\\0.956} & \shortstack{90.10\\0.968} & \shortstack{90.00\\0.956} & \shortstack{90.01\\0.957} & \shortstack{95.80\\1.110} & \shortstack{90.03\\0.975} & \shortstack{90.17\\1.012} \\[2pt]
\shortstack[l]{Low--high\\099} & \shortstack{57.67\\0.579} & \shortstack{81.61\\0.841} & \shortstack{75.30\\0.764} & \shortstack{77.76\\0.791} & \shortstack{92.68\\1.011} & \shortstack{67.57\\0.685} & \shortstack{90.16\\1.012} \\[2pt]
\shortstack[l]{High--low\\100} & \shortstack{99.05\\1.440} & \shortstack{95.32\\1.112} & \shortstack{97.00\\1.177} & \shortstack{96.43\\1.143} & \shortstack{98.78\\1.385} & \shortstack{98.06\\1.383} & \shortstack{90.18\\1.013} \\[2pt]
\shortstack[l]{Four-bit high--low\\102} & \shortstack{99.88\\2.094} & \shortstack{97.74\\1.288} & \shortstack{98.63\\1.372} & \shortstack{98.17\\1.299} & \shortstack{99.06\\1.487} & \shortstack{99.20\\1.798} & \shortstack{93.50\\1.113} \\[2pt]
\bottomrule
\end{tabularx}
\end{table}
\endgroup{}

\Cref{tab:synthetic-comparators} shows that several unguarded methods give
smaller sets and near-nominal coverage under stable maximum, then lose
coverage under low--high switching: failures of particular aggregation rules
under a specified attack, not a universal failure of robust aggregation.
Under high--low and in the four-bit stress cell, the same methods remain
competitive.
\FloatBarrier

\subsection{Complete real-data coverage and paired uncertainty}
\label{app:real-threats}

\Cref{tab:real-main} summarizes the eight task and partition pairs under stable maximum
at $a=A=2$: fixed-set saves
\RealStableMinSizeSaving--\RealStableMaxSizeSaving{} answers per query (paired standard
error at most 0.003), removing 34--45\% of deletion's excess over the oracle. The oracle
keeps \RealOracleMinSize--\RealOracleMaxSize{} of four answers, and the hub's top-ranked
answer is \RealCommonMinAccuracy--\RealCommonMaxAccuracy\% accurate against 25\% chance
(\cref{app:score-quality}).

\begingroup
\let\resulttablefloat\table
\renewcommand{\table}[1][]{\resulttablefloat[H]}%
\begin{table}[tbp]
\centering\footnotesize
\setlength{\tabcolsep}{2.6pt}
\caption{Fixed membership removes 33.9--44.7\% of the extra answers deletion retains above the
honest-mean oracle, in all eight comparisons, with fixed-set coverage above 90\%. GPT-2
scores under stable maximum, $a=A=2$, eight bits, 100 repeats; coverage $\pm$ one Monte
Carlo standard error. \Cref{fig:real-savings} gives all 56 cells.}
\label{tab:real-main}
\begin{tabular}{@{}llrrrrrrrrr@{}}
\toprule
Task & Part. & Fixed cov. & Oracle & Fixed & Joint & Del. & Saved & Cut & Merger & Rob-FCP$^\dagger$\\
\midrule
OpenBookQA & H & $92.76\pm 0.21$ & 3.525 & 3.654 & 3.673 & 3.721 & \textbf{0.066} & \textbf{33.9\%} & 4.000 & 91.54/3.591\\
OpenBookQA & T & $92.84\pm 0.20$ & 3.540 & 3.674 & 3.691 & 3.746 & \textbf{0.071} & \textbf{34.7\%} & 4.000 & 91.70/3.612\\
MedQA-US & H & $92.28\pm 0.19$ & 3.557 & 3.659 & 3.690 & 3.736 & \textbf{0.077} & \textbf{43.3\%} & 3.929 & 90.96/3.604\\
MedQA-US & T & $92.38\pm 0.15$ & 3.567 & 3.667 & 3.696 & 3.737 & \textbf{0.070} & \textbf{41.4\%} & 3.931 & 91.33/3.619\\
MedMCQA & H & $92.42\pm 0.15$ & 3.540 & 3.647 & 3.674 & 3.721 & \textbf{0.073} & \textbf{40.5\%} & 3.964 & 91.04/3.589\\
MedMCQA & T & $92.47\pm 0.14$ & 3.544 & 3.654 & 3.678 & 3.727 & \textbf{0.073} & \textbf{39.9\%} & 3.970 & 91.04/3.594\\
MMLU-Med & H & $91.56\pm 0.21$ & 3.519 & 3.596 & 3.630 & 3.653 & \textbf{0.057} & \textbf{42.4\%} & 3.895 & 90.06/3.529\\
MMLU-Med & T & $91.60\pm 0.19$ & 3.509 & 3.590 & 3.625 & 3.656 & \textbf{0.065} & \textbf{44.7\%} & 3.895 & 90.40/3.537\\
\bottomrule\end{tabular}
\par\smallskip\raggedright\footnotesize
H/T: hash/topic-source partition. Sizes are mean answers retained; Saved is
$\text{Del.}-\text{Fixed}$ and Cut is that saving as a share of deletion's excess over
the oracle, both before rounding. Merger: partial-conjunction $p$-value merger; guarded
symmetric trimming returns all four answers and is omitted. $^\dagger$Coverage/size of
the Rob-FCP adaptation, which defends calibration only, so it lacks a guarantee here; no
tested attack targets it.
\end{table}
\endgroup{}

\Cref{fig:real-savings} gives the paired saving from fixed membership in every
primary cell, the complete version of the comparison \cref{sec:experiments}
summarizes. Fixed-set is smaller than deletion in 40 of the 56 cells, in all 100
repeats of each, and equal in the remaining 16, never larger, as the proved
hierarchy requires. The 16 ties are exactly the two switching columns: low--high,
where the rules coincide (\cref{app:synthetic-reference}), and high--low, which
drives all of them to the full answer set.

\begin{figure}[H]
 \centering
 \includegraphics[width=\linewidth]{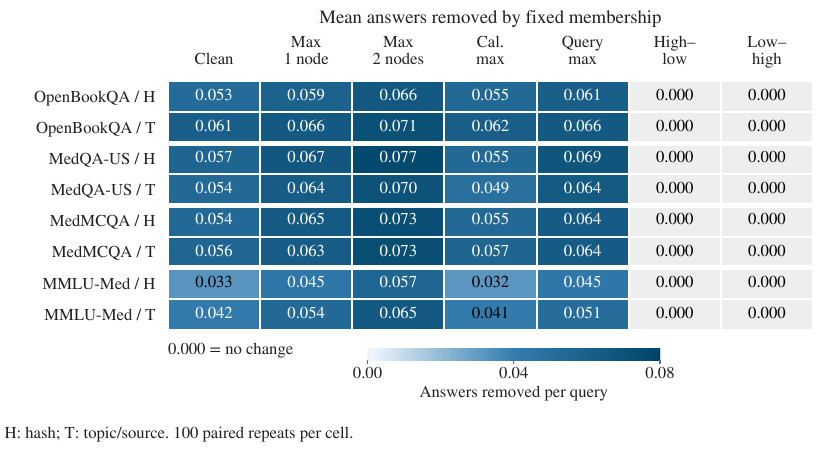}
 \caption{Fixed membership removes answers in every cell except the two switching
 columns, where the rules coincide or saturate.
 Deletion minus fixed-set mean raw set size over 100 paired repeats for each of the 56
 primary real-data cells. Rows are task and partition; columns are the seven threats, with
 \emph{Max, $j$ node(s)} denoting stable maximum at $j$ corrupt nodes, and \emph{Cal.\ max}
 and \emph{Query max} denoting the maximum sent on one channel only. Coverage for every cell is
 in \cref{tab:real-all}.}
 \label{fig:real-savings}
\end{figure}
\FloatBarrier

The full primary table (\cref{tab:real-all}) gives coverage and size for all
tasks, partitions, and threats at $K=16$, $A=2$, eight-bit reports, and
nominal 90\% coverage. Every task shows the same attack-dependent pattern.
Fixed-set is smaller than deletion under stable maximum. High--low drives them to all four
answers while the $p$-merger stays smaller on the medical tasks, and guarded
symmetric trimming returns all four candidates throughout. The consistency
across partitions concerns these fixed question pools and score functions; it
is not evidence of useful clinical answer selection.

The Rob-FCP adaptation remains empirically competitive across the complete
primary design. It is smaller than fixed-set in \RealRobSizeLowerCells{}
of 56 cells and has lower coverage in \RealRobCoverageLowerCells{}.
Secondary comparator ranges appear in \cref{tab:real-comparators}; threat
labels in \cref{tab:real-all,tab:real-comparators} follow \cref{fig:real-savings}.
\FloatBarrier

\begingroup
\let\resulttablefloat\table
\renewcommand{\table}[1][]{\resulttablefloat[H]}%
\begingroup\scriptsize
\setlength{\tabcolsep}{2.8pt}
\begin{longtable}{@{}lllrrrrrr@{}}
\caption{Fixed-set mean coverage exceeds 90\% in every primary cell, with smaller sets than deletion
except where they coincide. Entries are coverage (\%) / raw mean set size over 100 repeats
on fixed question pools, including empty sets. H/T denote hash/topic-source partitions.
These are repeat means; registered stable-maximum uncertainty appears in
\cref{tab:real-adjusted}.}\label{tab:real-all}\\
\toprule
Task & Part. & Threat & Oracle & Fixed & Joint & Deletion & Sym. & Merger\\\midrule
\endfirsthead
\toprule Task & Part. & Threat & Oracle & Fixed & Joint & Deletion & Sym. & Merger\\\midrule
\endhead
\bottomrule\endfoot
OpenBookQA & H & Clean & 90.10/3.526 & 92.50/3.638 & 92.70/3.647 & 93.56/3.691 & 100.00/4.000 & 100.00/4.000\\
OpenBookQA & H & Max, 1 & 90.06/3.525 & 92.62/3.647 & 92.93/3.661 & 93.88/3.706 & 100.00/4.000 & 100.00/4.000\\
OpenBookQA & H & Max, 2 & 90.07/3.525 & 92.76/3.654 & 93.17/3.673 & 94.11/3.721 & 100.00/4.000 & 100.00/4.000\\
OpenBookQA & H & Cal. max & 90.07/3.525 & 93.35/3.680 & 93.52/3.690 & 94.42/3.735 & 100.00/4.000 & 100.00/4.000\\
OpenBookQA & H & Query max & 90.07/3.525 & 91.93/3.614 & 92.30/3.630 & 93.22/3.675 & 100.00/4.000 & 100.00/4.000\\
OpenBookQA & H & High--low & 90.07/3.525 & 100.00/4.000 & 100.00/4.000 & 100.00/4.000 & 100.00/4.000 & 100.00/4.000\\
OpenBookQA & H & Low--high & 90.07/3.525 & 91.49/3.593 & 91.49/3.593 & 91.49/3.593 & 100.00/4.000 & 100.00/4.000\\
OpenBookQA & T & Clean & 90.01/3.541 & 92.38/3.656 & 92.58/3.667 & 93.64/3.717 & 100.00/4.000 & 100.00/4.000\\
OpenBookQA & T & Max, 1 & 89.98/3.540 & 92.57/3.664 & 92.88/3.678 & 93.90/3.730 & 100.00/4.000 & 100.00/4.000\\
OpenBookQA & T & Max, 2 & 89.96/3.540 & 92.84/3.674 & 93.20/3.691 & 94.22/3.746 & 100.00/4.000 & 100.00/4.000\\
OpenBookQA & T & Cal. max & 89.96/3.540 & 93.35/3.700 & 93.56/3.711 & 94.55/3.762 & 100.00/4.000 & 100.00/4.000\\
OpenBookQA & T & Query max & 89.96/3.540 & 91.93/3.632 & 92.18/3.646 & 93.27/3.698 & 100.00/4.000 & 100.00/4.000\\
OpenBookQA & T & High--low & 89.96/3.540 & 100.00/4.000 & 100.00/4.000 & 100.00/4.000 & 100.00/4.000 & 100.00/4.000\\
OpenBookQA & T & Low--high & 89.96/3.540 & 91.45/3.609 & 91.45/3.609 & 91.45/3.609 & 100.00/4.000 & 100.00/4.000\\
MedQA-US & H & Clean & 90.02/3.557 & 91.96/3.645 & 92.19/3.656 & 93.33/3.702 & 100.00/4.000 & 98.60/3.924\\
MedQA-US & H & Max, 1 & 90.01/3.556 & 92.10/3.652 & 92.62/3.674 & 93.72/3.719 & 100.00/4.000 & 98.62/3.926\\
MedQA-US & H & Max, 2 & 89.99/3.557 & 92.28/3.659 & 93.05/3.690 & 94.08/3.736 & 100.00/4.000 & 98.64/3.929\\
MedQA-US & H & Cal. max & 89.99/3.557 & 93.27/3.699 & 93.52/3.710 & 94.46/3.753 & 100.00/4.000 & 98.64/3.929\\
MedQA-US & H & Query max & 89.99/3.557 & 91.21/3.614 & 91.66/3.634 & 92.84/3.683 & 100.00/4.000 & 98.40/3.914\\
MedQA-US & H & High--low & 89.99/3.557 & 100.00/4.000 & 100.00/4.000 & 100.00/4.000 & 100.00/4.000 & 98.64/3.929\\
MedQA-US & H & Low--high & 89.99/3.557 & 91.01/3.606 & 91.01/3.606 & 91.01/3.606 & 100.00/4.000 & 98.38/3.913\\
MedQA-US & T & Clean & 89.96/3.565 & 92.16/3.656 & 92.47/3.669 & 93.46/3.711 & 100.00/4.000 & 98.49/3.926\\
MedQA-US & T & Max, 1 & 89.96/3.567 & 92.26/3.661 & 92.77/3.682 & 93.81/3.724 & 100.00/4.000 & 98.55/3.928\\
MedQA-US & T & Max, 2 & 90.01/3.567 & 92.38/3.667 & 93.12/3.696 & 94.06/3.737 & 100.00/4.000 & 98.62/3.931\\
MedQA-US & T & Cal. max & 90.01/3.567 & 93.36/3.705 & 93.59/3.716 & 94.50/3.754 & 100.00/4.000 & 98.62/3.931\\
MedQA-US & T & Query max & 90.01/3.567 & 91.45/3.625 & 91.87/3.645 & 92.94/3.689 & 100.00/4.000 & 98.17/3.913\\
MedQA-US & T & High--low & 90.01/3.567 & 100.00/4.000 & 100.00/4.000 & 100.00/4.000 & 100.00/4.000 & 98.62/3.931\\
MedQA-US & T & Low--high & 90.01/3.567 & 91.25/3.617 & 91.25/3.617 & 91.25/3.617 & 100.00/4.000 & 98.10/3.910\\
MedMCQA & H & Clean & 89.86/3.538 & 92.16/3.635 & 92.39/3.645 & 93.31/3.689 & 100.00/4.000 & 99.13/3.960\\
MedMCQA & H & Max, 1 & 89.87/3.538 & 92.27/3.640 & 92.70/3.659 & 93.65/3.705 & 100.00/4.000 & 99.17/3.962\\
MedMCQA & H & Max, 2 & 89.89/3.540 & 92.42/3.647 & 92.99/3.674 & 93.94/3.721 & 100.00/4.000 & 99.21/3.964\\
MedMCQA & H & Cal. max & 89.89/3.540 & 93.22/3.684 & 93.42/3.693 & 94.30/3.738 & 100.00/4.000 & 99.21/3.964\\
MedMCQA & H & Query max & 89.89/3.540 & 91.43/3.605 & 91.87/3.623 & 92.88/3.669 & 100.00/4.000 & 99.06/3.956\\
MedMCQA & H & High--low & 89.89/3.540 & 100.00/4.000 & 100.00/4.000 & 100.00/4.000 & 100.00/4.000 & 99.21/3.964\\
MedMCQA & H & Low--high & 89.89/3.540 & 91.22/3.596 & 91.22/3.596 & 91.22/3.596 & 100.00/4.000 & 98.86/3.946\\
MedMCQA & T & Clean & 89.89/3.543 & 92.14/3.640 & 92.41/3.651 & 93.40/3.696 & 100.00/4.000 & 99.28/3.966\\
MedMCQA & T & Max, 1 & 89.88/3.543 & 92.30/3.646 & 92.70/3.665 & 93.69/3.710 & 100.00/4.000 & 99.30/3.968\\
MedMCQA & T & Max, 2 & 89.89/3.544 & 92.47/3.654 & 92.99/3.678 & 94.04/3.727 & 100.00/4.000 & 99.35/3.970\\
MedMCQA & T & Cal. max & 89.89/3.544 & 93.22/3.688 & 93.44/3.698 & 94.43/3.745 & 100.00/4.000 & 99.35/3.970\\
MedMCQA & T & Query max & 89.89/3.544 & 91.48/3.612 & 91.88/3.629 & 92.95/3.676 & 100.00/4.000 & 99.20/3.963\\
MedMCQA & T & High--low & 89.89/3.544 & 100.00/4.000 & 100.00/4.000 & 100.00/4.000 & 100.00/4.000 & 99.35/3.970\\
MedMCQA & T & Low--high & 89.89/3.544 & 91.24/3.602 & 91.24/3.602 & 91.24/3.602 & 100.00/4.000 & 98.95/3.950\\
MMLU-Med & H & Clean & 89.81/3.518 & 91.40/3.589 & 91.66/3.600 & 92.17/3.622 & 100.00/4.000 & 97.99/3.891\\
MMLU-Med & H & Max, 1 & 89.80/3.519 & 91.50/3.593 & 91.99/3.615 & 92.59/3.639 & 100.00/4.000 & 98.03/3.893\\
MMLU-Med & H & Max, 2 & 89.82/3.519 & 91.56/3.596 & 92.36/3.630 & 92.98/3.653 & 100.00/4.000 & 98.06/3.895\\
MMLU-Med & H & Cal. max & 89.82/3.519 & 92.62/3.640 & 92.89/3.649 & 93.45/3.672 & 100.00/4.000 & 98.06/3.895\\
MMLU-Med & H & Query max & 89.82/3.519 & 90.69/3.558 & 91.19/3.579 & 91.71/3.603 & 100.00/4.000 & 97.67/3.874\\
MMLU-Med & H & High--low & 89.82/3.519 & 100.00/4.000 & 100.00/4.000 & 100.00/4.000 & 100.00/4.000 & 98.06/3.895\\
MMLU-Med & H & Low--high & 89.82/3.519 & 90.53/3.551 & 90.53/3.551 & 90.53/3.551 & 100.00/4.000 & 97.66/3.873\\
MMLU-Med & T & Clean & 89.74/3.510 & 91.42/3.583 & 91.72/3.595 & 92.38/3.625 & 100.00/4.000 & 97.89/3.888\\
MMLU-Med & T & Max, 1 & 89.75/3.510 & 91.54/3.587 & 92.08/3.611 & 92.75/3.641 & 100.00/4.000 & 97.95/3.891\\
MMLU-Med & T & Max, 2 & 89.73/3.509 & 91.60/3.590 & 92.37/3.625 & 93.11/3.656 & 100.00/4.000 & 98.01/3.895\\
MMLU-Med & T & Cal. max & 89.73/3.509 & 92.63/3.635 & 92.95/3.648 & 93.66/3.677 & 100.00/4.000 & 98.01/3.895\\
MMLU-Med & T & Query max & 89.73/3.509 & 90.61/3.549 & 91.12/3.571 & 91.80/3.600 & 100.00/4.000 & 97.57/3.870\\
MMLU-Med & T & High--low & 89.73/3.509 & 100.00/4.000 & 100.00/4.000 & 100.00/4.000 & 100.00/4.000 & 98.01/3.895\\
MMLU-Med & T & Low--high & 89.73/3.509 & 90.43/3.542 & 90.43/3.542 & 90.43/3.542 & 100.00/4.000 & 97.56/3.869\\
\end{longtable}
\endgroup
\endgroup{}
\FloatBarrier

The registered real-data contrasts quantify remaining robustness overhead
and the cost of alternative constructions. In \cref{tab:real-adjusted},
fixed-set minus oracle measures that overhead; symmetric or merger minus
deletion compares those alternatives with the scalable deletion rule.
The corrected intervals are simultaneous within each registered task/metric
family, without adjustment across all tasks or all three metrics. ``Full''
rows report the full-set rate, the share of queries whose set
contains all $M$ candidates.

\begingroup
\let\resulttablefloat\table
\renewcommand{\table}[1][]{\resulttablefloat[H]}%
\begin{table}[tbp]\centering\scriptsize
\setlength{\tabcolsep}{3pt}
\caption{Under stable maximum ($a=A=2$), all displayed contrasts indicate excess conservatism.
Entries are paired mean [lower, upper] with approximate 95\% simultaneous Monte Carlo
intervals within each 70-contrast task/metric family, for fixed question pools. Size is in
candidates; full-set rate is in percentage points (pp). H/T denote hash/topic-source
partitions. Positive values mean the first method exceeds its reference.}
\label{tab:real-adjusted}
\begin{tabular}{@{}lllrrr@{}}\toprule
Task & Part. & Metric & Fixed $-$ oracle & Sym. $-$ del. & Merger $-$ del.\\\midrule
OpenBookQA & H & Size & 0.129 [0.123, 0.136] & 0.279 [0.257, 0.301] & 0.279 [0.257, 0.301]\\
OpenBookQA & H & Full (pp) & 8.180 [7.762, 8.598] & 19.135 [17.778, 20.492] & 19.135 [17.778, 20.492]\\
OpenBookQA & T & Size & 0.134 [0.127, 0.141] & 0.254 [0.230, 0.278] & 0.254 [0.230, 0.278]\\
OpenBookQA & T & Full (pp) & 8.168 [7.709, 8.627] & 17.577 [16.031, 19.123] & 17.577 [16.031, 19.123]\\
MedQA-US & H & Size & 0.101 [0.096, 0.106] & 0.264 [0.250, 0.278] & 0.193 [0.180, 0.205]\\
MedQA-US & H & Full (pp) & 7.231 [6.863, 7.600] & 23.331 [22.220, 24.442] & 16.446 [15.447, 17.445]\\
MedQA-US & T & Size & 0.100 [0.095, 0.104] & 0.263 [0.250, 0.276] & 0.194 [0.181, 0.206]\\
MedQA-US & T & Full (pp) & 7.522 [7.160, 7.884] & 23.347 [22.312, 24.381] & 16.673 [15.723, 17.622]\\
MedMCQA & H & Size & 0.108 [0.104, 0.111] & 0.279 [0.263, 0.296] & 0.243 [0.225, 0.261]\\
MedMCQA & H & Full (pp) & 7.346 [7.090, 7.602] & 23.052 [21.840, 24.264] & 19.855 [18.401, 21.309]\\
MedMCQA & T & Size & 0.110 [0.106, 0.114] & 0.273 [0.258, 0.289] & 0.243 [0.226, 0.260]\\
MedMCQA & T & Full (pp) & 7.449 [7.174, 7.724] & 22.524 [21.349, 23.699] & 19.828 [18.451, 21.205]\\
MMLU-Med & H & Size & 0.077 [0.073, 0.082] & 0.347 [0.327, 0.368] & 0.242 [0.224, 0.261]\\
MMLU-Med & H & Full (pp) & 5.177 [4.847, 5.508] & 27.764 [26.243, 29.284] & 18.835 [17.412, 20.258]\\
MMLU-Med & T & Size & 0.081 [0.077, 0.085] & 0.344 [0.325, 0.363] & 0.239 [0.223, 0.255]\\
MMLU-Med & T & Full (pp) & 5.468 [5.164, 5.773] & 27.668 [26.279, 29.057] & 18.774 [17.518, 20.030]\\
\bottomrule\end{tabular}\end{table}
\endgroup{}

These contrasts are not the direct fixed-set-minus-deletion saving, which
remains a descriptive paired comparison, not a registered interval contrast;
its paired estimate and Monte Carlo standard error are retained as
machine-readable artifacts. To express the saving relative to deletion's excess size, let $\bar S_u$ be method $u$'s
repeat-mean raw set size. The reported ratio, the Cut column of \cref{tab:real-main}, is
$(\bar S_{\rm del}-\bar S_{\rm fs})/
(\bar S_{\rm del}-\bar S_{\rm oracle})$, computed separately for each
stable-maximum task/partition cell. This ratio of repeat means is neither
a mean of repeat-level ratios nor an additional registered confidence contrast.
\FloatBarrier

\subsection{Score-quality anchors and secondary calibration}
\label{app:score-quality}

The real-data gains reduce robustness overhead in a weak-score regime.
\Cref{tab:accuracy-anchors} compares the common ranker
(\cref{eq:common-point-ranker}) with no-retrieval and whole-corpus anchors.
With broad oracle sets and near-chance point accuracy, a smaller robust set
here is a reduction in overhead, not reliable forced-answer selection.

\begingroup
\let\resulttablefloat\table
\renewcommand{\table}[1][]{\resulttablefloat[H]}%
\begin{table}[tbp]\centering\small
\caption{Whole-corpus retrieval has the highest observed accuracy, but all scorers remain near 25\%
chance. Entries are percentage means $\pm$ one Monte Carlo standard error over 100 splits of
fixed question pools. The common ranker is shared by federated methods; No RAG omits
retrieval, and Union RAG retrieves four whole-corpus passages. No RAG and Union RAG scores
are reused across partitions.}
\label{tab:accuracy-anchors}
\begin{tabular}{@{}llrrr@{}}\toprule
Task & Partition & Common ranker & No RAG & Union RAG\\\midrule
OpenBookQA & Hash & $24.62\pm 0.09$ & $25.00\pm 0.09$ & $27.38\pm 0.10$\\
OpenBookQA & Topic/source & $24.59\pm 0.09$ & $25.00\pm 0.09$ & $27.38\pm 0.10$\\
MedQA-US & Hash & $24.10\pm 0.09$ & $24.57\pm 0.08$ & $26.03\pm 0.09$\\
MedQA-US & Topic/source & $24.71\pm 0.08$ & $24.57\pm 0.08$ & $26.03\pm 0.09$\\
MedMCQA & Hash & $28.12\pm 0.03$ & $26.89\pm 0.02$ & $29.72\pm 0.03$\\
MedMCQA & Topic/source & $27.52\pm 0.02$ & $26.89\pm 0.02$ & $29.72\pm 0.03$\\
MMLU-Med & Hash & $27.21\pm 0.10$ & $27.40\pm 0.10$ & $28.83\pm 0.11$\\
MMLU-Med & Topic/source & $27.17\pm 0.11$ & $27.40\pm 0.10$ & $28.83\pm 0.11$\\
\bottomrule\end{tabular}\end{table}
\endgroup{}
\FloatBarrier

The registered secondary analyses (supplement) show that the full sets are
partly a matter of score design. Under the separately calibrated NLL score of
\cref{app:rag-protocol}, no rule returns all four candidates for every query in
every repeat of any cell; across its 56 cells fixed-set mean size is 3.53--3.92
and mean coverage 90.20--98.58\%. It is a secondary analysis, not an
outcome-selected replacement for the primary score, and leaves the weak-model
limitation intact. The real-data $\alpha$ sensitivity shows that lower nominal
coverage permits smaller sets.
\FloatBarrier

\subsection{Real-data boundary controls and secondary comparators}

Exceeding the declared budget can break coverage. The real-data violation
repeats the low--high test with $a=3>A=2$, outside the eligible primary cells;
real fixed-set coverage falls to
\RealViolationMinCoverage--\RealViolationMaxCoverage\%.

\begingroup
\let\resulttablefloat\table
\renewcommand{\table}[1][]{\resulttablefloat[H]}%
\begin{table}[tbp]\centering\small
\caption{Underestimating corruption lowers all three proposed rules' coverage to 68.74--75.92\%,
below nominal 90\%. The deliberate violation uses $a=3>A=2$, low--high reporting, and
eight-bit reports. Entries are descriptive mean coverage percentages on fixed question
pools; the guarantees, which require $a\le A$, do not apply.}
\label{tab:real-bound-violation}
\begin{tabular}{@{}llrrrrr@{}}\toprule
Task & Partition & Fixed & Joint & Deletion & Symmetric & Merger\\\midrule
OpenBookQA & Hash & 68.74 & 68.74 & 68.74 & 100.00 & 100.00\\
OpenBookQA & Topic/source & 69.38 & 69.38 & 69.38 & 100.00 & 100.00\\
MedQA-US & Hash & 72.65 & 72.65 & 72.65 & 100.00 & 98.06\\
MedQA-US & Topic/source & 72.51 & 72.51 & 72.51 & 100.00 & 97.76\\
MedMCQA & Hash & 72.90 & 72.90 & 72.90 & 100.00 & 98.52\\
MedMCQA & Topic/source & 72.08 & 72.08 & 72.08 & 100.00 & 98.46\\
MMLU-Med & Hash & 75.92 & 75.92 & 75.92 & 100.00 & 97.35\\
MMLU-Med & Topic/source & 74.94 & 74.94 & 74.94 & 100.00 & 97.25\\
\bottomrule\end{tabular}\end{table}
\endgroup{}

The symmetric and merger results remain at or above nominal in these cells,
but that does not restore a guarantee whose declared budget has been
exceeded.
\FloatBarrier

The secondary real-data comparators show both competitive outcomes and
attack-sensitive failures. \Cref{tab:real-comparators} shows ranges over
all four tasks and both partitions at four reference threats; the complete
numerical artifact retains all seven threats and individual standard errors.

\begingroup
\let\resulttablefloat\table
\renewcommand{\table}[1][]{\resulttablefloat[H]}%
\begingroup\scriptsize
\setlength{\tabcolsep}{3pt}
\begin{longtable}{@{}llrrr@{}}
\caption{Low--high reporting drives several comparators below 90\% coverage; Krum and the Rob-FCP
adaptation remain above it in these cells. Entries span repeat means across four tasks and
both partitions (eight fixed-pool cells), not confidence intervals. Coverage and full-set
rate are percentages; size counts retained answers. All seven threats and individual Monte
Carlo standard errors remain in the numeric artifact.}\label{tab:real-comparators}\\
\toprule Method & Threat & Coverage range & Size range & Full-set range\\\midrule
\endfirsthead
\toprule Method & Threat & Coverage range & Size range & Full-set range\\\midrule
\endhead\bottomrule\endfoot
All-node mean & Clean & 89.74--90.15 & 3.509--3.565 & 62.11--70.00\\
All-node mean & Max, 2 & 89.70--90.10 & 3.508--3.568 & 62.02--69.94\\
All-node mean & High--low & 100.00--100.00 & 4.000--4.000 & 100.00--100.00\\
All-node mean & Low--high & 54.57--60.94 & 2.126--2.329 & 5.83--8.34\\
Median & Clean & 90.08--90.72 & 3.523--3.602 & 63.45--72.38\\
Median & Max, 2 & 90.04--90.79 & 3.520--3.603 & 63.22--72.50\\
Median & High--low & 91.13--91.88 & 3.571--3.644 & 66.61--75.21\\
Median & Low--high & 88.64--89.30 & 3.460--3.548 & 59.38--68.80\\
Winsorized & Clean & 89.77--90.04 & 3.507--3.582 & 62.23--70.60\\
Winsorized & Max, 2 & 89.80--90.10 & 3.508--3.582 & 62.41--70.49\\
Winsorized & High--low & 91.88--92.96 & 3.609--3.685 & 68.94--76.43\\
Winsorized & Low--high & 86.39--87.28 & 3.384--3.444 & 54.49--62.74\\
Unguarded trim & Clean & 89.85--90.07 & 3.510--3.581 & 62.50--70.73\\
Unguarded trim & Max, 2 & 89.84--90.08 & 3.512--3.582 & 62.74--70.56\\
Unguarded trim & High--low & 91.55--92.50 & 3.597--3.671 & 68.16--75.69\\
Unguarded trim & Low--high & 86.99--87.90 & 3.407--3.465 & 55.97--63.99\\
Median of means & Clean & 89.82--90.20 & 3.511--3.585 & 62.44--70.69\\
Median of means & Max, 2 & 89.85--90.19 & 3.513--3.586 & 62.66--70.77\\
Median of means & High--low & 92.32--93.01 & 3.626--3.694 & 70.14--76.81\\
Median of means & Low--high & 85.95--87.15 & 3.360--3.437 & 53.02--62.92\\
Krum & Clean & 90.16--90.93 & 3.524--3.614 & 63.53--72.63\\
Krum & Max, 2 & 90.18--90.90 & 3.522--3.615 & 63.02--72.47\\
Krum & High--low & 90.19--90.90 & 3.521--3.615 & 63.00--72.40\\
Krum & Low--high & 90.18--90.95 & 3.522--3.613 & 63.02--72.60\\
Local marginal & Clean & 90.16--91.27 & 3.533--3.620 & 64.10--73.81\\
Local marginal & Max, 2 & 89.85--90.89 & 3.515--3.599 & 63.00--72.76\\
Local marginal & High--low & 91.55--92.84 & 3.596--3.685 & 68.10--77.66\\
Local marginal & Low--high & 89.42--90.40 & 3.496--3.580 & 61.74--71.54\\
Rob-FCP adaptation & Clean & 91.05--92.29 & 3.574--3.661 & 66.67--75.85\\
Rob-FCP adaptation & Max, 2 & 90.06--91.70 & 3.529--3.619 & 63.65--74.29\\
Rob-FCP adaptation & High--low & 100.00--100.00 & 4.000--4.000 & 100.00--100.00\\
Rob-FCP adaptation & Low--high & 90.06--91.70 & 3.529--3.619 & 63.65--74.29\\
\end{longtable}
\endgroup
\endgroup{}

Several robust summaries remain near nominal under stable maximum but
undercover under low--high in \cref{tab:real-comparators}. Krum and the
Rob-FCP adaptation remain above nominal throughout the displayed low--high
cells, while local-marginal coverage spans both sides of nominal. Thus an
aggregator can be competitive on these fixed tasks without inheriting the
proposed guarantee against arbitrary reports in both phases.

The scoring audit confirms complete model outputs for all 246,330 held-out score vectors,
with zero fallbacks caused by a question or answer exceeding the token limit.
\FloatBarrier
{}
\FloatBarrier
\section{Assumption and claim checklist}
\label{app:checklist}

The coverage and minimality guarantees (\cref{thm:coverage,thm:fixed-set})
rest on the five conditions below and on nothing stronger hidden in the proofs;
items 2 and 3 also record what the baselines and the efficiency bounds add.

\begin{enumerate}
 \item Clean calibration and future examples are exchangeable (their joint law
 is unchanged by reordering) conditional on the scoring system frozen before
 calibration.
 \item The honest node set is fixed, authenticated, and contains at least
 $K-A$ nodes, with $A<K$ (baseline requirements:
 \cref{sec:symmetric-details,prop:pmerge}).
 \item Honest reports satisfy registered directional reconstruction-error
 bounds.  Neither a bounded decoded-report alphabet nor an honest-range
 certificate (a registered bound on the range of the honest scores) is required
 for the proposed rules' coverage when every honest node responds.  A
 bounded decoded-report alphabet is required for the pathwise upper efficiency
 sandwich; the expected-size bounds additionally require fixed certificate
 widths and the future $X$ to be independent of the clean calibration split
 conditional on the frozen system, as stated in \cref{cor:size-modulus}.
 The alternative bound \eqref{eq:size-density-alternative} also assumes a bounded
 true-label density.
 \item Honest nodes evaluate every shared calibration candidate and every
 future/query candidate with the same frozen score functions $s_i$, whose honest
 average $s_{\cH}$ (\cref{eq:honest-mean}) is the clean target.
 \item The finite candidate space, candidate ordering, and coordinate
 convention are frozen across phases.  Honest nodes evaluate every candidate;
 during calibration the hub selects the correct-answer (gold) coordinate only after receiving
 the complete score vectors.
\end{enumerate}

Beyond \cref{thm:coverage}'s stated exclusions, the theorems require no
i.i.d.\ nodes, unbiased noise, symmetric quantization or
Dvoretzky--Kiefer--Wolfowitz (DKW) approximation; the $p$-merger's extra
assumption is in \cref{prop:pmerge}.
{}
\FloatBarrier
\section{Restricted subset families and a prompt-induced panel failure}
\label{app:judge-panel}

This section makes two points. First, the fixed-set rule of \cref{sec:method}
can search any family of subsets chosen from data instead of $\mathfrak H_A$.
Coverage then costs one tail probability, but clean-score integrity holds only
while the true honest set is still searched (\cref{prop:restricted-family}).
Second, when bad reports are \emph{induced} rather than constructed, the
proposed rules hold coverage in every regime, while the unguarded all-node mean
does not when only calibration is misled.

A misled judge emits a finite report, so \cref{thm:coverage} already covers it;
the measurements locate ordinary misbehavior inside the proved envelope rather
than test the worst case. The clean reference \eqref{eq:honest-mean} is the
unquantized faithful-instruction mean over the judges that drew the faithful
instruction.

\paragraph{The panel.}
Sixteen nodes each hold a disjoint shard of the $1{,}326$ OpenBookQA evidence
facts. Each retrieves four passages per question by
BM25~\citep{robertson2009bm25} on its own shard, querying with the question and
its four options. Each then scores all four options of $1{,}000$ held-out
questions with \texttt{Llama-3.2-3B-Instruct}~\citep{grattafiori2024llama3}. The
score of an option is one minus the model's probability of its letter after
``Answer:'', renormalized over the four letters, so smaller scores favor
inclusion as in \cref{sec:setup}. Every question is scored under a faithful
instruction, to answer from the retrieved evidence, and an injected one, to
answer with the first listed option whatever the evidence says.

Each of 500 replicates permutes the questions into an audit pool of 100, a
calibration split of $n=333$ and 567 evaluation questions. It draws the misled
judges independently, fourteen with probability $0.02$ and two with probability
$0.5$, and reports eight-bit scores at $K=16$ and budgets $A=2,\ldots,5$. Only
the two trust-based comparators read the audit pool. A judge's \emph{audit
error} is the fraction of audit questions on which its smallest reported score
is not the correct answer's alone (ties count as errors). Over all judge--question
pairs, the audit error is $0.2990$ under the faithful instruction and $0.5874$
under the injected one. The injection thus makes a judge content-blind in one
direction rather than noisy. The injected audit error stays well below one partly because the correct
answer is listed first on $26.4\%$ of questions.

\subsection{A subset family chosen from data}

For a family $\mathcal G$ of nonempty subsets of $[K]$, define the restricted
margin and rule by
\begin{align*}
 \Delta_{\mathcal G,k}(V,w)
 &=\min_{H\in\mathcal G}
   \left\{\frac1{|H|}\sum_{i\in H}w_i-T_k(H;V)\right\},\\
 \cC^{\mathcal G}_\alpha(x)
 &=\left\{y:\Delta_{\mathcal G,k}\{V,w(x,y)\}\le g_\rho\right\},
\end{align*}
so that $\cC^{\mathfrak H_A}_\alpha$ is the fixed-set rule
\eqref{eq:fixed-set-rule}. Extending \cref{eq:honest-mean}, write
$s_H(x,y)=|H|^{-1}\sum_{i\in H}s_i(x,y)$ for the clean mean over a nonempty
$H\subseteq[K]$, $R^H_j=s_H(X_j,Y_j)$ for its clean calibration scores, and
$R^H_{(k)}$ for the $k$th smallest of $R^H_1,\ldots,R^H_n$.

\begin{proposition}[Restricted subset families]
\label{prop:restricted-family}
Assume $k\le n$ and that the registered directional bounds
\eqref{eq:directional-quantizer} hold for the honest nodes. Assume also that $\cA$ is
drawn once before calibration and held fixed for the whole calibration-to-query
comparison; this is the first condition of the failure law in
\cref{app:random-membership} and has a deterministic honest set as its
degenerate case. Let $\mathcal G$ be a random family of nonempty subsets of
$[K]$, chosen as a measurable function of audit data $\mathcal D$, and assume
the clean examples $Z_1,\ldots,Z_{n+1}$ are exchangeable conditional on
$(\mathcal D,\cA)$. Write
\begin{equation*}
 \varepsilon_{\mathcal G}
 =\Pp\{\,\text{no }H\in\mathcal G\text{ satisfies }H\subseteq\cH\,\}.
\end{equation*}
Then, uniformly over all finite Byzantine calibration and query reports,
\begin{equation}
 \Pp\{Y_{n+1}\in\cC^{\mathcal G}_\alpha(X_{n+1})\}
 \ \ge\ \frac{k}{n+1}\left(1-\varepsilon_{\mathcal G}\right).
 \label{eq:restricted-coverage}
\end{equation}
Moreover, clean-score integrity for $s_{\cH}$, that is
$\cC^\circ_\alpha(x)\subseteq\cC^{\mathcal G}_\alpha(x)$ for every query and
admissible transcript, holds whenever $\cH\in\mathcal G$, and it can fail when
$\cH\notin\mathcal G$. For $k=n+1$, $\cC^{\mathcal G}_\alpha$ returns $\cY$ and
coverage is one.
\end{proposition}

\begin{proof}
Let $E_{\mathcal G}$ be the event that some $H\in\mathcal G$ satisfies
$H\subseteq\cH$, so $\Pp(E_{\mathcal G})=1-\varepsilon_{\mathcal G}$.

\emph{Step 1: one all-honest subset suffices, pathwise.}
Fix any $H_0\in\mathcal G$ with $H_0\subseteq\cH$; it is nonempty. Averaging
the directional bounds \eqref{eq:directional-quantizer} over $i\in H_0$ gives,
for each calibration example $j$,
\[
 R^{H_0}_j-\rho_{\calib}^-\ \le\ c_j(H_0;V)\ \le\ R^{H_0}_j+\rho_{\calib}^+,
\]
and by monotonicity of order statistics
\[
 R^{H_0}_{(k)}-\rho_{\calib}^-\ \le\ T_k(H_0;V)
 \ \le\ R^{H_0}_{(k)}+\rho_{\calib}^+.
\]
At the query, on the true label, honesty of $H_0$ gives
\[
 \frac1{|H_0|}\sum_{i\in H_0}w_i(X_{n+1},Y_{n+1})
 \ \le\ R^{H_0}_{n+1}+\rho_{\qry}^+.
\]
The minimum defining $\Delta_{\mathcal G,k}$ runs over a family containing
$H_0$, so
\[
 \Delta_{\mathcal G,k}
 \ \le\ R^{H_0}_{n+1}+\rho_{\qry}^+-R^{H_0}_{(k)}+\rho_{\calib}^-.
\]
Hence on the clean rank event $\{R^{H_0}_{n+1}\le R^{H_0}_{(k)}\}$ for the
$H_0$-mean, $\Delta_{\mathcal G,k}\le\rho_{\qry}^++\rho_{\calib}^-=g_\rho$ and
$Y_{n+1}\in\cC^{\mathcal G}_\alpha(X_{n+1})$.

\emph{Step 2: the choice of $H_0$ does not break exchangeability.}
Condition on $(\mathcal D,\cA)$. On $E_{\mathcal G}$, let $H_0$ be the
lexicographically least $H\in\mathcal G$ with $H\subseteq\cH$. Since
$\mathcal G$ is a function of $\mathcal D$ and $\cH=[K]\setminus\cA$, $H_0$ is
deterministic given $(\mathcal D,\cA)$ and does not depend on the calibration
scores. The clean examples are exchangeable given $(\mathcal D,\cA)$, and
$R^{H_0}_j$ applies one fixed function to each, so \cref{lem:rank} gives
\[
 \Pp\{R^{H_0}_{n+1}\le R^{H_0}_{(k)}\mid\mathcal D,\cA\}\ \ge\ \frac{k}{n+1}.
\]

\emph{Step 3: average.}
The event $E_{\mathcal G}$ is $(\mathcal D,\cA)$-measurable, so combining
Steps~1 and~2 and discarding the complementary event,
\[
 \Pp\{Y_{n+1}\in\cC^{\mathcal G}_\alpha(X_{n+1})\}
 \ \ge\ \E\left[\ind_{E_{\mathcal G}}\,\frac{k}{n+1}\right]
 \ =\ \frac{k}{n+1}\left(1-\varepsilon_{\mathcal G}\right),
\]
which is \cref{eq:restricted-coverage}.

\emph{Step 4: integrity when $\cH\in\mathcal G$.}
Apply Step~1 with $H_0=\cH$. The clean rank event for the $\cH$-mean is
$\{Y_{n+1}\in\cC^\circ_\alpha(X_{n+1})\}$ by \cref{eq:clean-rank-event}, and the
same pathwise chain for an arbitrary candidate $y$ gives
$\cC^\circ_\alpha(x)\subseteq\cC^{\mathcal G}_\alpha(x)$. This is the
containment argument of \cref{thm:fixed-set} with $\mathfrak H_A$ replaced by
$\mathcal G$.

\emph{Step 5: integrity can fail when $\cH\notin\mathcal G$.}
Take $K=2$, $a=0$, so $\cH=\{1,2\}$, and $\mathcal G=\{\{1\}\}$. Let the
registered reconstruction errors be zero, so $g_\rho=0$, and take $n=1$ and
$\alpha=1/2$, so $k=1$. Let both nodes report $0$ at calibration, so
$T_1(\{1\};V)=0$ and $R_{(1)}=0$. For a candidate $y$, let node $1$ report $+1$
and node $2$ report $-1$. The oracle's query score is $0\le R_{(1)}$, so it
admits $y$; but $\Delta_{\mathcal G,1}=1>0=g_\rho$, so
$\cC^{\mathcal G}_\alpha$ excludes it.

Finally, for $k=n+1$ the full-set convention of \cref{sec:setup} returns $\cY$.
\end{proof}

An audit run on its own held-out block of questions, as here, meets the
exchangeability hypothesis; what the hypothesis excludes is an audit that inspects the clean
calibration scores. Neither $\varepsilon_{\mathcal G}$ nor whether
$\cH\in\mathcal G$ is observable at deployment.

\subsection{The two comparators as corollaries}

Both comparators act on the judges' audit errors. \emph{V1, trust pruning},
removes the two judges with the highest audit error and runs the fixed-set rule
on the other fourteen at budget $\max(0,A-2)$. \emph{V2, an audited budget},
flags the judges whose audit error exceeds $0.4$ and runs the fixed-set rule on
all sixteen at the audited budget $A'$. This is the smallest budget whose tail
under the corner rule of \cref{lem:membership-monotone}, with flagged judges
failing with probability $\bar p_2=0.5$ and the rest with $\bar p_1=0.02$, is at
most $\delta=0.01$.

\begin{corollary}[An audited budget]
\label{cor:judge-lowered-budget}
Let $A'<K$ be an audited budget, chosen from the audit data as in
\cref{prop:restricted-family}, and take $\mathcal G=\mathfrak H_{A'}$. Then
$\cH\in\mathcal G$ if and only if $a\le A'$, so
$\varepsilon_{\mathcal G}=\Pp\{a>A'\}$, and integrity holds on $\{a\le A'\}$.
If $A'\ge A_0$ on every realization for a constant $A_0<K$, then
$\varepsilon_{\mathcal G}\le\varepsilon_{A_0}$, the failure law's own tail.
If the audit returns a deterministic $A'$, the proposition reduces to
\cref{cor:random-membership} at $A'$.
\end{corollary}

\begin{proof}
Since $A'<K$, every member of $\mathfrak H_{A'}$ is nonempty, and a set
$H\subseteq\cH$ with $|H|\ge K-A'$ exists if and only if $a\le A'$, which is
also the condition $\cH\in\mathfrak H_{A'}$. If $A'\ge A_0$ always, then
$\{a>A'\}\subseteq\{a>A_0\}$; for deterministic $A'$,
\cref{eq:restricted-coverage} is \cref{eq:random-membership-coverage} at $A'$.
\end{proof}

\begin{corollary}[Removing the least-trusted nodes]
\label{cor:judge-trust-pruning}
Let $S\subseteq[K]$ with $|S|\le A$ be a removed set chosen as in
\cref{prop:restricted-family}, and take
$\mathcal G=\{H\in\mathfrak H_A:H\cap S=\emptyset\}$, the family the fixed-set
rule enumerates when it is run on the $K-|S|$ survivors at budget $A-|S|$. Then
$\cH\in\mathcal G$ if and only if $a\le A$ and every removed node was Byzantine,
$S\subseteq\cA$; on that event \cref{prop:restricted-family}'s integrity
guarantee applies. Further,
\[
 \varepsilon_{\mathcal G}=\Pp\{|\cH\setminus S|<K-A\}.
\]
\end{corollary}

\begin{proof}
A set $H$ lies in $\mathcal G$ and satisfies $H\subseteq\cH$ exactly when
$H\subseteq\cH\setminus S$ and $|H|\ge K-A$. Since $K-A\ge1$, such an $H$ exists
if and only if $|\cH\setminus S|\ge K-A$. Membership $\cH\in\mathcal G$ asks
$a\le A$ and $\cH\cap S=\emptyset$, that is $S\subseteq\cA$.
\end{proof}

An audited budget keeps every node in the reference, but its advertised tail can
understate the true one, since the audit supplies a realized count of flagged
nodes rather than a bound on the law. Removal changes the estimand to the
survivors' mean, so V1's smaller sets answer a different question.

\subsection{What the measurements show}

\paragraph{The proposed rules hold coverage in every regime; the unguarded mean does not.}
A judge that drew the injected instruction is misled at calibration only, at the
query only, or in both phases (\cref{tab:judge-coverage}); the three regimes
share each replicate's misled set, oracle and audit. The three proposed rules
clear $k/(n+1)$ in every regime and budget, over all replicates and over the
budget-respecting ones ($a\le A$). Misleading only calibration drops the all-node mean
to $0.8862$; with per-cell Monte Carlo standard errors of about $0.001$, that is about 15
standard errors below $k/(n+1)$.

\paragraph{Why the calibration channel is the weak point.}
The injection compresses the upper tail of the correct-answer score, which
lowers the calibration threshold. Across all questions it raises the node-mean
correct-answer score (median $0.294$ to $0.694$). But the threshold reads an
upper order statistic, and on the 100 questions in the faithful panel's upper
decile the injected judges score \emph{lower}: $0.656$ against $0.925$. The
node-mean score's ninetieth percentile is lowest when about a quarter of the panel is misled, and the realized
corruption averages $1.27$ judges, so the sparse regime is the most exposed.

\paragraph{Fixed membership removes 45--75\% of deletion's excess.}
At $A=2$ fixed-set saves $0.21$--$0.30$ answers per query over deletion across
the three regimes, 45--75\% of deletion's excess over the honest-mean oracle
(\cref{tab:judge-sizes}). Each extra unit of budget costs deletion more than
fixed-set, so from $A=2$ to $A=5$ the saving roughly doubles, to $0.43$--$0.58$
answers per query (\cref{tab:judge-sizes}; \cref{fig:judge-panel}b shows the
both-phases regime).

\paragraph{Trust-based pruning answers a different question.}
At $A=2$, V1 loses clean-score integrity on 308 of the 464 budget-respecting
replicates, while fixed-set and V2 lose it on none (\cref{tab:judge-trust}). The
losses sit where pruning discards honest judges (\cref{cor:judge-trust-pruning});
the table notes give the breakdown and the limits of V2's tail.

\raggedbottom
\begingroup
\let\resulttablefloat\table
\renewcommand{\table}[1][]{\resulttablefloat[H]}%
\begin{table}[H]
\centering
\footnotesize
\caption{Every proposed rule clears its guaranteed floor in every regime and at every declared budget, while misleading only the calibration channel drops the unguarded all-node mean to 0.8862, below the nominal 0.90. Panel headings name the phases in which a misled judge misreports. Because membership is drawn, $\varepsilon_A$ is the exact probability that more than $A$ judges are misled and the guaranteed floor is $\{k/(n+1)\}(1-\varepsilon_A)$ (\cref{cor:random-membership}). The final panel keeps, for the proposed rules, only replicates with $a\le A$. Every entry in the final panel clears $k/(n+1)=0.9012$, the tightest being fixed-set at $A=2$ with only the calibration channel misled, at $0.9115$ over 464 replicates. The honest-mean oracle (not deployable) and the all-node mean (no guarantee) are references. Means over 500 replicates; $n=333$.}
\label{tab:judge-coverage}
\begin{tabular}{@{}lrrrr@{}}
\toprule
Rule & $A=2$ & $A=3$ & $A=4$ & $A=5$\\
\midrule
\multicolumn{5}{@{}l}{\emph{Misled at: both phases}}\\
$\varepsilon_A$ & 0.0777 & 0.0090 & 0.0007 & $3.7\times10^{-5}$\\
Guaranteed floor & 0.8312 & 0.8931 & 0.9006 & 0.9012\\
\midrule
Honest-mean oracle & 0.9010 & 0.9010 & 0.9010 & 0.9010\\
All-node mean & 0.9009 & 0.9009 & 0.9009 & 0.9009\\
Fixed-set & 0.9247 & 0.9316 & 0.9379 & 0.9444\\
Joint-threshold & 0.9332 & 0.9420 & 0.9495 & 0.9564\\
Deletion & 0.9516 & 0.9642 & 0.9735 & 0.9791\\
$p$-merger & 0.9814 & 0.9808 & 0.9802 & 0.9779\\
\midrule
\multicolumn{5}{@{}l}{\emph{Misled at: calibration only}}\\
$\varepsilon_A$ & 0.0777 & 0.0090 & 0.0007 & $3.7\times10^{-5}$\\
Guaranteed floor & 0.8312 & 0.8931 & 0.9006 & 0.9012\\
\midrule
Honest-mean oracle & 0.9010 & 0.9010 & 0.9010 & 0.9010\\
All-node mean & 0.8862 & 0.8862 & 0.8862 & 0.8862\\
Fixed-set & 0.9105 & 0.9172 & 0.9237 & 0.9301\\
Joint-threshold & 0.9151 & 0.9240 & 0.9325 & 0.9405\\
Deletion & 0.9362 & 0.9523 & 0.9636 & 0.9689\\
$p$-merger & 0.9820 & 0.9812 & 0.9805 & 0.9782\\
\midrule
\multicolumn{5}{@{}l}{\emph{Misled at: query only}}\\
$\varepsilon_A$ & 0.0777 & 0.0090 & 0.0007 & $3.7\times10^{-5}$\\
Guaranteed floor & 0.8312 & 0.8931 & 0.9006 & 0.9012\\
\midrule
Honest-mean oracle & 0.9010 & 0.9010 & 0.9010 & 0.9010\\
All-node mean & 0.9172 & 0.9172 & 0.9172 & 0.9172\\
Fixed-set & 0.9349 & 0.9411 & 0.9468 & 0.9523\\
Joint-threshold & 0.9384 & 0.9456 & 0.9523 & 0.9585\\
Deletion & 0.9539 & 0.9654 & 0.9730 & 0.9784\\
$p$-merger & 0.9797 & 0.9795 & 0.9790 & 0.9770\\
\midrule
\multicolumn{5}{@{}l}{\emph{Proposed rules, restricted to replicates with $a\le A$}}\\
\midrule
Fixed-set (both phases) & 0.9243 & 0.9317 & 0.9379 & 0.9444\\
Joint-threshold (both phases) & 0.9329 & 0.9421 & 0.9495 & 0.9564\\
Deletion (both phases) & 0.9516 & 0.9643 & 0.9735 & 0.9791\\
Fixed-set (calibration only) & 0.9115 & 0.9173 & 0.9237 & 0.9301\\
Joint-threshold (calibration only) & 0.9162 & 0.9241 & 0.9325 & 0.9405\\
Deletion (calibration only) & 0.9373 & 0.9524 & 0.9636 & 0.9689\\
Fixed-set (query only) & 0.9339 & 0.9411 & 0.9468 & 0.9523\\
Joint-threshold (query only) & 0.9374 & 0.9457 & 0.9523 & 0.9585\\
Deletion (query only) & 0.9534 & 0.9655 & 0.9730 & 0.9784\\
\bottomrule
\end{tabular}
\end{table}
\endgroup{}

\begingroup
\let\resulttablefloat\table
\renewcommand{\table}[1][]{\resulttablefloat[H]}%
\begin{table}[H]
\centering
\footnotesize
\caption{Fixed-set is the smallest of the three proposed rules in every panel and removes 45--75\% of deletion's excess over the oracle at $A=2$; each extra unit of declared budget costs it 0.048 to 0.067 candidates per query, against 0.108 to 0.175 for deletion. Entries are mean numbers of candidates per query out of $M=4$. No rule returned an empty set on any replicate, so every miss is a nonempty set that excludes the gold answer. Panels are split by misled phase, as in \cref{tab:judge-coverage}. The oracle's gap to a proposed rule is the price of not knowing which judges are honest; the all-node mean carries no guarantee. Means over 500 replicates; $n=333$.}
\label{tab:judge-sizes}
\begin{tabular}{@{}lrrrr@{}}
\toprule
Rule & $A=2$ & $A=3$ & $A=4$ & $A=5$\\
\midrule
\multicolumn{5}{@{}l}{\emph{Mean set size}, misled at: both phases}\\
\midrule
Honest-mean oracle & 1.718 & 1.718 & 1.718 & 1.718\\
All-node mean & 1.653 & 1.653 & 1.653 & 1.653\\
Fixed-set & 1.869 & 1.935 & 2.000 & 2.067\\
Joint-threshold & 1.959 & 2.056 & 2.145 & 2.230\\
Deletion & 2.167 & 2.342 & 2.505 & 2.642\\
$p$-merger & 2.820 & 2.756 & 2.726 & 2.651\\
\midrule
\multicolumn{5}{@{}l}{\emph{Mean set size}, misled at: calibration only}\\
\midrule
Honest-mean oracle & 1.718 & 1.718 & 1.718 & 1.718\\
All-node mean & 1.606 & 1.606 & 1.606 & 1.606\\
Fixed-set & 1.787 & 1.839 & 1.890 & 1.938\\
Joint-threshold & 1.831 & 1.901 & 1.964 & 2.024\\
Deletion & 1.995 & 2.133 & 2.257 & 2.365\\
$p$-merger & 2.833 & 2.771 & 2.741 & 2.669\\
\midrule
\multicolumn{5}{@{}l}{\emph{Mean set size}, misled at: query only}\\
\midrule
Honest-mean oracle & 1.718 & 1.718 & 1.718 & 1.718\\
All-node mean & 1.806 & 1.806 & 1.806 & 1.806\\
Fixed-set & 1.986 & 2.053 & 2.119 & 2.186\\
Joint-threshold & 2.022 & 2.103 & 2.181 & 2.263\\
Deletion & 2.203 & 2.364 & 2.495 & 2.625\\
$p$-merger & 2.785 & 2.728 & 2.700 & 2.624\\
\bottomrule
\end{tabular}
\end{table}
\endgroup{}

\begingroup
\let\resulttablefloat\table
\renewcommand{\table}[1][]{\resulttablefloat[H]}%
\begin{table}[H]
\centering
\scriptsize
\setlength{\tabcolsep}{3pt}
\caption{Fixed-set never loses clean-score integrity. Pruning the two least-trusted judges (V1) returns sets 0.034 to 0.054 candidates per query smaller across declared budgets when only the calibration channel is misled but excludes an answer the honest-mean oracle admits on up to 308 of 464 budget-respecting replicates; the audited budget (V2) keeps every judge in the reference and loses integrity on no budget-respecting replicate here. ``Lost integrity'' counts budget-respecting replicates on which a rule excluded at least one query--candidate pair the honest-mean oracle admits, so fixed-set's zeros confirm part 1 of \cref{thm:fixed-set}. The notes after the table define both comparators and qualify their rows. Means over 500 replicates; $n=333$.}
\label{tab:judge-trust}
\begin{tabular}{@{}llrrrrr@{}}
\toprule
Misled at & Rule & Coverage & Cov.\ ($a\le A$) & Size & Nodes in ref. & Lost integrity ($a\le A$)\\
\midrule
\multicolumn{7}{@{}l}{\emph{By regime, declared $A=2$}}\\
\midrule
Both phases & Fixed-set & 0.9247 & 0.9243 & 1.869 & 16 & 0 / 464\\
 & V1, trust pruning & 0.9041 & 0.9039 & 1.736 & 14 & 308 / 464\\
 & V2, audited budget & 0.9316 & 0.9303 & 1.940 & 16 & 0 / 464\\
\midrule
Calibration only & Fixed-set & 0.9105 & 0.9115 & 1.787 & 16 & 0 / 464\\
 & V1, trust pruning & 0.9030 & 0.9039 & 1.733 & 14 & 308 / 464\\
 & V2, audited budget & 0.9172 & 0.9172 & 1.840 & 16 & 0 / 464\\
\midrule
Query only & Fixed-set & 0.9349 & 0.9339 & 1.986 & 16 & 0 / 464\\
 & V1, trust pruning & 0.9052 & 0.9039 & 1.745 & 14 & 308 / 464\\
 & V2, audited budget & 0.9408 & 0.9396 & 2.055 & 16 & 0 / 464\\
\midrule
\multicolumn{7}{@{}l}{\emph{By declared budget, misled at calibration only}}\\
\midrule
$A=2$ & Fixed-set & 0.9105 & 0.9115 & 1.787 & 16 & 0 / 464\\
 & V1, trust pruning & 0.9030 & 0.9039 & 1.733 & 14 & 308 / 464\\
 & V2, audited budget & 0.9172 & 0.9172 & 1.840 & 16 & 0 / 464\\
\midrule
$A=3$ & Fixed-set & 0.9172 & 0.9173 & 1.839 & 16 & 0 / 498\\
 & V1, trust pruning & 0.9116 & 0.9117 & 1.797 & 14 & 242 / 498\\
 & V2, audited budget & 0.9172 & 0.9173 & 1.840 & 16 & 0 / 498\\
\midrule
$A=4$ & Fixed-set & 0.9237 & 0.9237 & 1.890 & 16 & 0 / 500\\
 & V1, trust pruning & 0.9186 & 0.9186 & 1.852 & 14 & 163 / 500\\
 & V2, audited budget & 0.9172 & 0.9172 & 1.840 & 16 & 0 / 500\\
\midrule
$A=5$ & Fixed-set & 0.9301 & 0.9301 & 1.938 & 16 & 0 / 500\\
 & V1, trust pruning & 0.9256 & 0.9256 & 1.904 & 14 & 80 / 500\\
 & V2, audited budget & 0.9172 & 0.9172 & 1.840 & 16 & 0 / 500\\
\midrule
\multicolumn{7}{@{}l}{\emph{What the audit classified, per replicate (common to all three regimes)}}\\
\midrule
$A'$ & & Replicates & Share & Flagged & Misled missed & Faithful flagged\\
2 & & 82 & 0.164 & 0.000 & 0.000 & 0.000\\
3 & & 364 & 0.728 & 1.387 & 0.000 & 0.022\\
4 & & 38 & 0.076 & 3.000 & 0.000 & 0.237\\
5 & & 14 & 0.028 & 4.286 & 0.000 & 2.214\\
7 & & 2 & 0.004 & 7.000 & 0.000 & 4.500\\
\bottomrule
\end{tabular}
\end{table}
{\footnotesize\noindent \emph{Notes to \cref{tab:judge-trust}.} Both comparators read an audit pool of 100 questions disjoint from calibration and evaluation. V1 prunes the two judges with the highest audit error and runs the fixed-set rule on the rest at budget $\max(0,A-2)$; V2 flags judges whose audit error exceeds $0.4$ and runs the fixed-set rule on all sixteen at the audited budget $A'$, the smallest budget whose tail under the corner law (flagged judges at $\bar p_2=0.5$, the rest at $\bar p_1=0.02$) is at most $\delta=0.01$. V2 ignores $A$, so its row repeats down the by-budget panel. Broken out by realized misled count at $A=2$, V1 loses integrity on 86 of the 86 replicates at a realized count of 0, on 222 of the 230 at 1, and on 0 of the 148 at 2: pruning two judges preserves the honest mean only when exactly two were misled and the audit ranked them first, and the losses sit where it discards honest judges instead. The upper panel holds $A=2$ and varies the regime: on budget-respecting replicates the calibration-only and query-only rows differ by 2.2 points for both fixed-set and V2. V1's rows coincide on all 464 of them, an invariance of the construction rather than robustness: the audit is not split by regime, and there the two pruned judges include every misled one. Its marginal coverage moves by 0.2 points only through the 36 over-budget replicates. The audited budget exceeds the declared $A=2$ on 418 of the 500 replicates (modal $A'=3$, on 364), so V2 is a restriction of the declared family only at larger declared budgets; the bottom panel is common to all three regimes. V2's advertised tail assumes the flagged judges bound how many carry a failure probability above $\bar p_1$ (\cref{lem:membership-monotone}); the flagged count is instead a realized one, and it falls below the law's 2 suspects on 305 of the 500 replicates. Where no judge is flagged (82 replicates) it returns $A'=2$ and advertises a tail of 0.0037, against 0.0777 under the law. The audit missed no misled judge but flagged 57 faithful ones in total, so $A'\ge a$ held on all 500 replicates while $A'$ is not a function of the realized count; the tail \cref{prop:restricted-family} spends, $\mathbb{P}\{a>A'\}$, is bounded through the smallest audited budget, $A'\ge2$, by the law's 0.0777. Fixed-set is decided in exact rational arithmetic and both comparators in floating point; the two ways of deciding fixed-set disagree on 2,496 of 13,608,000 decisions.\par}
\endgroup{}
{}
{\section{Changing identities, planted content, generated lists, and exact certificates}
\label{app:revision-theory}

This appendix proves the formal results added during the discussion. Deletion keeps its
coverage when the corrupted nodes change from question to question, while fixed-set and
joint-threshold need not (\cref{sec:rotating}); corruption that is the same in both phases
needs no budget (\cref{sec:phase-blind-corruption}); lists generated at the hub with a
``none of these'' entry keep every guarantee (\cref{sec:generated-lists}); every miss of
fixed-set is a miss of the clean oracle (\cref{sec:integrity-transfers}); and two exact
certificates settle most fixed-set decisions without enumeration (\cref{sec:certificates}).

Notation is that of \cref{sec:setup} and \cref{sec:method}: $K$ nodes with frozen scores $s_i$,
decoded calibration reports $V=(v_{ij})$ and query reports $w(x,y)$, the feasible
family $\mathfrak H_A=\{H\subseteq[K]:|H|\ge K-A\}$, the deletion means $\Up_A$ and
$\Dn_A$, the padding $g_\rho=\rho_{\calib}^-+\rho_{\qry}^+$, and
$k=\lceil(n+1)(1-\alpha)\rceil$. For $H\subseteq[K]$, write
$s_H(x,y)=|H|^{-1}\sum_{i\in H}s_i(x,y)$.

\subsection{Corrupted nodes that change with the question}
\label{sec:rotating}

The deletion rule keeps its guarantee when the corrupted nodes change from question to question.
\cref{thm:coverage} fixes one Byzantine set for calibration and prediction; deletion needs less:
at most $A$ nodes corrupted \emph{on each question}, and corrupted sets that are exchangeable
together with the examples, so that who is corrupted is not tied to the phase.

\begin{theorem}[Deletion when the corrupted set changes with the question]
\label{thm:rotating}
Let every node $i\in[K]$ have a frozen score function $s_i$. For each example
$j\in[n+1]$, let $\cA_j\subseteq[K]$ be a random set with $|\cA_j|\le A<K$,
and put $\cH_j=[K]\setminus\cA_j$. On example $j$, every node $i\in\cH_j$ reports
$Q_{i,r}\circ s_i$ on the channel of that example ($r=\calib$ for $j\le n$, $r=\qry$ for
$j=n+1$) with \cref{eq:directional-quantizer}, and the reports of the nodes in $\cA_j$ are arbitrary finite
values, which may depend on the labels, on the honest reports, and on the phase. Thus
$\cA_j$ contains every node that misreports on example $j$; it may contain others, and
the oracle below is that of the chosen sets. Suppose
the triples $(X_j,Y_j,\cA_j)$, $j\in[n+1]$, are exchangeable conditional on the frozen
score functions. Write
\[
 R'_j=s_{\cH_j}(X_j,Y_j),\qquad j\in[n+1],
\]
and, for $k\le n$, let $R'_{(k)}$ be the $k$th smallest of $R'_1,\ldots,R'_n$. Then:
\begin{enumerate}[itemsep=1pt,topsep=2pt]
 \item \emph{(Integrity for the per-question oracle.)} For $k\le n$, on every sample and
 for every such transcript,
 \[
  \bigl\{y:\ s_{\cH_{n+1}}(X_{n+1},y)\le R'_{(k)}\bigr\}\subseteq\cC^{\rm del}_\alpha(X_{n+1}).
 \]
 \item \emph{(Coverage.)}
 $\Pp\{Y_{n+1}\in\cC^{\rm del}_\alpha(X_{n+1})\}\ge k/(n+1)\ge1-\alpha$.
 \item \emph{(Size.)} If all decoded reports on channel $r$ lie in a common interval of
 width $D_r$, then for $k\le n$, with $W_{\rm del}$ as in \cref{thm:set-sandwich},
 \[
  \cC^{\rm del}_\alpha(X_{n+1})\subseteq
  \bigl\{y:\ s_{\cH_{n+1}}(X_{n+1},y)\le R'_{(k)}+W_{\rm del}\bigr\}.
 \]
\end{enumerate}
\cref{thm:coverage} and the deletion part of \cref{thm:set-sandwich} are the case
$\cA_1=\cdots=\cA_{n+1}$.
\end{theorem}

\begin{proof}
If $k=n+1$ the rule returns $\cY$; let $k\le n$.

\emph{Part 1.} Fix $j\le n$. Since $|\cH_j|\ge K-A$, the set $\cH_j$ is feasible, so by
\cref{thm:extremal-envelopes} the upper deletion mean of column $j$ is at least the mean report of
$\cH_j$ on that column, and each of those reports is at least the clean score minus
$\rho_{\calib}^-$:
\[
 \widehat R_j=\Up_A(v_{\cdot j})\ \ge\ \frac1{|\cH_j|}\sum_{i\in\cH_j}v_{ij}
 \ \ge\ R'_j-\rho_{\calib}^-.
\]
Order statistics are nondecreasing in each argument and commute with adding a constant,
so $R'_{(k)}\le\widehat R_{(k)}+\rho_{\calib}^-$. At the query, $\cH_{n+1}$ is feasible and
its reports exceed the clean scores by at most $\rho_{\qry}^+$, so for every $y$
\[
 \Dn_A\{w(X_{n+1},y)\}\ \le\ \bar w_{\cH_{n+1}}(X_{n+1},y)
 \ \le\ s_{\cH_{n+1}}(X_{n+1},y)+\rho_{\qry}^+.
\]
If $s_{\cH_{n+1}}(X_{n+1},y)\le R'_{(k)}$, chaining the two displays gives
$\Dn_A\{w(X_{n+1},y)\}\le\widehat R_{(k)}+\rho_{\calib}^-+\rho_{\qry}^+
=\widehat q^{\rm del}_\alpha$, that is, $y\in\cC^{\rm del}_\alpha(X_{n+1})$. Nothing was
assumed about the reports of the nodes in any $\cA_j$.

\emph{Part 2.} Given the frozen score functions, $R'_j=\varphi(X_j,Y_j,\cA_j)$ for the one
fixed function $\varphi(x,y,\cA)=s_{[K]\setminus\cA}(x,y)$. The triples are exchangeable,
so $R'_1,\ldots,R'_{n+1}$ are, and \cref{lem:rank} gives
$\Pp\{R'_{n+1}\le R'_{(k)}\}\ge k/(n+1)$. By Part 1 with $y=Y_{n+1}$ this event implies
$Y_{n+1}\in\cC^{\rm del}_\alpha(X_{n+1})$.

\emph{Part 3.} Recall from \cref{thm:set-sandwich} the membership term
$\mu_r=\min\{1,A/(K-A)\}D_r$ and $W_{\rm del}=\sum_r(\mu_r+\rho_r^-+\rho_r^+)$.
\cref{lem:bounded-deletion} is a statement about one report vector and one honest
set with at most $A$ nodes outside it. Applied to column $j\le n$ with honest set $\cH_j$,
and using $\min\{1,A/|\cH_j|\}\le\min\{1,A/(K-A)\}$, it gives
$\widehat R_j\le R'_j+\mu_{\calib}+\rho_{\calib}^+$, hence
$\widehat R_{(k)}\le R'_{(k)}+\mu_{\calib}+\rho_{\calib}^+$. Applied to the query vector
with honest set $\cH_{n+1}$, it gives
$\Dn_A\{w(X_{n+1},y)\}\ge s_{\cH_{n+1}}(X_{n+1},y)-\mu_{\qry}-\rho_{\qry}^-$. If
$y\in\cC^{\rm del}_\alpha(X_{n+1})$, combining these with the definition of the threshold
gives $s_{\cH_{n+1}}(X_{n+1},y)\le R'_{(k)}+W_{\rm del}$, exactly as in the proof of
\cref{thm:set-sandwich}.
\end{proof}

\begin{remark}[When the hypothesis holds, and when it does not]
\label{rem:rotating-scope}
The triples are exchangeable when one fixed map selects the corrupted set,
$\cA_j=\psi(X_j,Y_j,\Xi_j)$ with $\Xi_1,\ldots,\Xi_{n+1}$ i.i.d.\ and independent of the
examples (\cref{lem:phase-blind}). The cap $|\cA_j|\le A$ is a separate requirement. Two
instances, each under the cap: the nodes whose retrieval for $X_j$ returns a
document planted before calibration, without knowledge of which questions form the
calibration sample, so that a node is misled only on the questions that
retrieve it; and nodes that fail at random, independently from question to question,
with the same failure law in both phases. What the corrupted
nodes then \emph{send} is unrestricted and may switch direction between the phases.

\emph{Without the cap}, the bound drops only by the chance that the cap fails at the
query. If the triples are exchangeable but $|\cA_j|$ may exceed $A$, then
\[
 \Pp\{Y_{n+1}\in\cC^{\rm del}_\alpha(X_{n+1})\}\ \ge\ \frac{k}{n+1}-\Pp\{|\cA_{n+1}|>A\}.
\]
Indeed, the case $k=n+1$ is trivial; for $k\le n$, put $R''_j=R'_j$ when $|\cA_j|\le A$ and
$R''_j=-\infty$ otherwise. These are
exchangeable and satisfy $R''_j\le\widehat R_j+\rho_{\calib}^-$ for every $j\le n$ (by Part 1 of the
proof when $|\cA_j|\le A$, trivially otherwise), so
$\Pp\{R''_{n+1}\le R''_{(k)}\}\ge k/(n+1)$ by \cref{lem:rank} (applied to $\arctan R''_j$, with
value $-\pi/2$ at $-\infty$), with
$R''_{(k)}\le\widehat R_{(k)}+\rho_{\calib}^-$,
and on the event $\{|\cA_{n+1}|\le A\}$ the query chain of Part 1 holds with
$R''_{n+1}=R'_{n+1}$. Hence the label is covered on
$\{R''_{n+1}\le R''_{(k)}\}\cap\{|\cA_{n+1}|\le A\}$, which has probability at least
$k/(n+1)-\Pp\{|\cA_{n+1}|>A\}$. The bound cannot be improved in general. Let all nodes
share one score with a continuous distribution, and let the corrupted nodes report far
below every score at calibration and far above every score at the query. Then a column on
which the cap fails counts as $-\infty$, a column on which it holds gives the shared
score, and a query on which it fails is never covered, so the coverage is the floor
itself up to the chance that more than $k$ of the $n+1$ examples break the cap.

The hypothesis fails when \emph{who} is corrupted depends on the phase: one set of $A$
nodes at calibration and a disjoint set of $A$ at the query. At most $A$ nodes misreport
at any moment, yet $2A$ nodes misreport between calibration and the query, and
\cref{thm:coverage} applies with budget $2A$, not $A$, provided $2A<K$ and the hub declares
$2A$.
\end{remark}

The fixed-set and joint-threshold rules pair one group's calibration cutoff with a
group's query mean, which presumes that a group honest on one example is honest on all.
Under the hypothesis of \cref{thm:rotating} they can fail.

\begin{proposition}[Fixed-set and joint-threshold can lose coverage when identities rotate]
\label{prop:two-node}
Let $K=2$, $A=1$, with no rounding ($g_\rho=0$) and reports in $[0,1]$. Both nodes hold
the same frozen score $s$, the examples $(X_j,Y_j)$, $j\in[n+1]$, are i.i.d., and
$S_j=s(X_j,Y_j)$ has a continuous
distribution on $[1/4,3/4]$. Let $\cA_j=\{B_j\}$ with $B_1,\ldots,B_{n+1}$ i.i.d.\ uniform
on $\{1,2\}$ and independent of the examples. The corrupted node reports $0$ on a
calibration example and $1$, for every candidate, at the query. Let $k\le n$, let
$S_{(k)}$ be the $k$th smallest of $S_1,\ldots,S_n$, and let
$N\sim\mathrm{Binomial}(n,1/2)$. Then
\begin{align*}
 \Pp\{Y_{n+1}\in\cC^{\rm del}_\alpha(X_{n+1})\}&=\frac{k}{n+1},\\
 \Pp\{Y_{n+1}\in\cC^{\rm fs}_\alpha(X_{n+1})\}
 &=\E\!\left[\frac{(k-N)_+}{n-N+1}\right]
 =\frac{2k}{n+1}-1+\delta_n,
\end{align*}
where $0\le\delta_n\le2^{-n}+(n-k)\,\Pp\{N>k\}$.
When $(n+1)(1-\alpha)$ is an integer and $\alpha<1/2$, the fixed-set rule therefore covers
with probability $1-2\alpha$ up to a term that vanishes exponentially in $n$. The
joint-threshold rule covers with probability
$\Pp\{S_{n+1}\le\max(T_k(\{1\};V),T_k(\{2\};V),S_{(k)}/2)\}$, which is at least the
fixed-set probability and has the same limit $1-2\alpha$ when the $S_j$ are uniform on
$[1/4,3/4]$.
\end{proposition}

\begin{proof}
The triples $(X_j,Y_j,\cA_j)$ are i.i.d., so \cref{thm:rotating} applies.

\emph{Deletion.} A calibration column holds the reports $0$ and $S_j$, so
$\widehat R_j=\Up_1(v_{\cdot j})=S_j$ (the larger of the two reports, which also exceeds
their mean). At the query the reports for $Y_{n+1}$ are $1$ and $S_{n+1}$, so
$\Dn_1=S_{n+1}$. The rule covers exactly when $S_{n+1}\le S_{(k)}$, which has probability
$k/(n+1)$ for i.i.d.\ continuous scores.

\emph{Fixed-set.} The feasible groups are $\{1\}$, $\{2\}$, and $\{1,2\}$. Let $i^\ast$ be
the node that is honest at the query. The group $\{B_{n+1}\}$ has query mean $1$, above
every calibration report, so it never supports $Y_{n+1}$. The group $\{1,2\}$ has query
mean $(1+S_{n+1})/2\ge5/8$ and cutoff $\OS_k\{S_j/2\}\le3/8$, so it never does. The group
$\{i^\ast\}$ has query mean $S_{n+1}$. Its calibration reports are $0$ on the
$N=\#\{j\le n:B_j=i^\ast\}$ examples on which $i^\ast$ was corrupted and $S_j$ on the other
$n-N$, so its cutoff is $0$ if $k\le N$ and otherwise the $(k-N)$th smallest of those
$n-N$ scores. Hence fixed-set covers exactly when $N<k$ and $S_{n+1}$ is at most the
$(k-N)$th smallest of $n-N$ scores. Conditional on $(B_1,\ldots,B_{n+1})$ the scores are
i.i.d.\ and continuous and $N\sim\mathrm{Binomial}(n,1/2)$, which gives the expectation
displayed. For the closed form, with $J=n-N\sim\mathrm{Binomial}(n,1/2)$,
\begin{align*}
 \E\!\left[\frac1{J+1}\right]&=\sum_{m=0}^n\binom nm\frac{2^{-n}}{m+1}
 =\frac{2-2^{-n}}{n+1},\\
 \E\!\left[\frac{k-N}{n-N+1}\right]&=1-(n+1-k)\,\frac{2-2^{-n}}{n+1}
 =\frac{2k}{n+1}-1+\Bigl(1-\frac{k}{n+1}\Bigr)2^{-n},
\end{align*}
and $(k-N)_+=(k-N)+(N-k)_+$ with $0\le(N-k)_+/(n-N+1)\le(n-k)\ind\{N>k\}$.

\emph{Joint-threshold.} The smallest feasible query mean for $Y_{n+1}$ is $S_{n+1}$ and
the largest cutoff is the maximum of the three group cutoffs, which gives the stated
event; it contains the fixed-set event. As $n\to\infty$ with $\alpha<1/2$, each
single-node cutoff converges to the $(1-2\alpha)$ quantile of $S$, and $S_{(k)}/2$ to half
its $(1-\alpha)$ quantile. For the uniform law on $[1/4,3/4]$ these are
$3/4-\alpha$ and $3/8-\alpha/4$, and the first is the larger for every $\alpha<1/2$.
\end{proof}

For uniform scores the joint-threshold rule approaches $1-2\alpha$ slowly: in simulation it
covers $0.816$ at $n=499$ and $\alpha=0.1$, where fixed-set covers $0.800$.

\subsection{Corruption that is the same in both phases}
\label{sec:phase-blind-corruption}

A document planted in a shard before calibration, or a compromised prompt left in place,
changes what a node computes; it does not make the node behave differently in the two
phases. \cref{sec:setup} counts a node misled by an injected prompt as Byzantine. That reading
is always safe: the node then needs room in the budget, and the guarantee is relative to
the scores of the other nodes. There is a second reading, which needs no budget. If the
alteration is in place before calibration, is left unchanged, acts through the question
and fresh randomness only, and was not chosen with
knowledge of which questions form the calibration sample, the node is honest with
respect to the score it actually computes, and the guarantee is relative to those scores.

\begin{corollary}[Planted content and other phase-blind corruption]
\label{cor:phase-blind}
Suppose that, for every node $i$ and every example $j\in[n+1]$, the decoded reports are
$Q_{i,r}\{s^{\rm e}_i(X_j,\Xi_j,y)\}$ on the channel of that example, where the quantizers
obey \cref{eq:directional-quantizer} and:
\begin{enumerate}[itemsep=1pt,topsep=2pt]
 \item the \emph{effective score} functions $s^{\rm e}_1,\ldots,s^{\rm e}_K$ are fixed before
 calibration;
 \item the examples $Z_j=(X_j,Y_j)$, $j\in[n+1]$, are exchangeable conditional on the
 effective score functions;
 \item conditional on the effective score functions, $\Xi_1,\ldots,\Xi_{n+1}$ are i.i.d.\
 and independent of the examples.
\end{enumerate}
No node is assumed to compute the score its operator intended, and $s^{\rm e}_i$ may reflect
planted text, a compromised instruction, or any other alteration made before calibration
and left unchanged. Condition~2 excludes an alteration chosen with knowledge of which
questions form the calibration sample. Write $s^{\rm e}_{[K]}$ for the mean of the effective
scores, $R^{\rm e}_j=s^{\rm e}_{[K]}(X_j,\Xi_j,Y_j)$, and $R^{\rm e}_{(k)}$ for the $k$th
smallest of $R^{\rm e}_1,\ldots,R^{\rm e}_n$. Then for every declared budget $0\le A<K$ and each
$\cC\in\{\cC^{\rm fs},\cC^{\rm jt},\cC^{\rm del}\}$,
\[
 \bigl\{y:\ s^{\rm e}_{[K]}(X_{n+1},\Xi_{n+1},y)\le R^{\rm e}_{(k)}\bigr\}
 \subseteq\cC_\alpha(X_{n+1})\ \ (k\le n),
 \qquad
 \Pp\{Y_{n+1}\in\cC_\alpha(X_{n+1})\}\ge\frac{k}{n+1}.
\]
\end{corollary}

\begin{proof}
Apply \cref{thm:coverage} and part 1 of \cref{thm:fixed-set} to the federation whose frozen score
functions are $(x,\xi,y)\mapsto s^{\rm e}_i(x,\xi,y)$, on the examples
$Z^{\rm e}_j=((X_j,\Xi_j),Y_j)$, with Byzantine set $\cA=\emptyset$. The $Z^{\rm e}_j$ are
exchangeable conditional on that federation's frozen scores: the $Z_j$ are by
condition~2, and, conditional on the effective scores, the $\Xi_j$ are i.i.d.\ and
independent of them by condition~3 (\cref{lem:phase-blind}). Further $a=0\le A$, and every node's reports
obey \cref{eq:directional-quantizer} relative to its effective score. The honest set is $[K]$, so the
oracle of that federation is the set displayed.
\end{proof}

Without condition~2 the corollary is false: effective scores equal to $0$ on the
calibration questions and to the intended scores elsewhere satisfy every other
condition and give $R^{\rm e}_{(k)}=0$, so a fresh query is covered only when its label's reports
are within $g_\rho$ of zero. The
oracle of the corollary depends on $\Xi_{n+1}$, so the hub cannot recompute it from the
question; the rules read reports only, and need not.

\begin{remark}[What the budget is spent on]
\label{rem:budget-spent}
Under conditions 1--3 of \cref{cor:phase-blind}, the effective scores $s^{\rm e}_i$ can take
the place of $s_i$ in two ways.
\begin{enumerate}[itemsep=1pt,topsep=2pt]
 \item \emph{Fixed identities.} If one fixed set of at most $A$ nodes contains every node
 whose reports ever deviate from its rounded effective score, then \cref{thm:coverage}, part~1
 of \cref{thm:fixed-set}, and \cref{thm:set-sandwich} hold, each under its own conditions, for all three
 rules, with the oracle of the remaining nodes computed from their effective scores.
 \item \emph{Changing identities.} If $\cA_j$ contains the nodes that deviate on example
 $j$, $|\cA_j|\le A$, and the triples $((X_j,\Xi_j),Y_j,\cA_j)$ are exchangeable conditional
 on the effective score functions, then \cref{thm:rotating} holds. That is a statement
 about the deletion rule only (\cref{prop:two-node}).
\end{enumerate}
In both cases the budget is spent only on nodes whose reports are not one fixed function
of the question and of fresh per-question randomness. Content planted before calibration, under condition~2 of
\cref{cor:phase-blind}, is not such a deviation, however many nodes it reaches. The
guarantee is then for the system as it is: it is relative to the scores the nodes
actually compute, which are those of a worse scorer, so the sets are typically larger.
Nothing here restores the accuracy an attack takes away.
\end{remark}

\subsection{Generated candidate lists and a missing correct answer}
\label{sec:generated-lists}

Every guarantee extends to candidate lists that a frozen generator writes at the hub, provided
a ``none of these'' entry stands in for a correct answer the list lacks.

\begin{proposition}[Generated lists with a ``none of these'' option]
\label{prop:generated-lists}
Let a generator, frozen before calibration and trusted, produce for each question (using
per-example randomness, if any) an ordered list of $M$ candidates whose last
entry is ``none of these''. Write $\Lambda_j$ for the list produced for $X_j$, and let the
hub send $(X_j,\Lambda_j)$ unchanged to every node in
both phases. Let a matching rule, frozen before calibration, assign to a reference answer
and a list the position of the first listed candidate that matches the reference, and
position $M$ if none does; let $Y_j\in[M]$ be that position for example $j$. If the triples
$(X_j,\Lambda_j,Y_j)$, $j\in[n+1]$, are exchangeable conditional on the frozen system, then
\cref{thm:coverage}, \cref{cor:random-membership}, \cref{thm:fixed-set}, and \cref{thm:set-sandwich} hold as stated, with
the pair $(X_j,\Lambda_j)$ in the role of the question and $[M]$ in the role of $\cY$. In
particular, under the conditions of \cref{thm:coverage}, each proposed set contains the position
of the matching candidate, or the ``none of these'' position when no candidate matches,
with probability at least $1-\alpha$.
\end{proposition}

\begin{proof}
The cited results use a finite label set, frozen scores defined on pairs of a question
and a label, and exchangeable examples; \cref{sec:setup} already takes the label to be a
position in the list shown. With $X'_j=(X_j,\Lambda_j)$ as the question, node $i$'s frozen
score is a function of $(X'_j,y)$, the examples $(X'_j,Y_j)$ are exchangeable, and no
proof compares the lists of two questions.
\end{proof}

\begin{remark}[Limits]
\label{rem:generated-limits}
The generator and the matching rule are part of the trusted, frozen hub; their errors
change what the label means and are outside the guarantee. A list proposed by the nodes
themselves is not covered, since a Byzantine node could then choose the candidates. The
guarantee is about positions: it does not say that the scores of ``none of these'' are
informative, and when they are not, the sets that contain it can be large.
\end{remark}

\subsection{What clean-score integrity transfers}
\label{sec:integrity-transfers}

Within the budget, corrupt nodes can enlarge fixed-set's sets but cannot create an error:
fixed-set misses the correct answer only where the clean oracle misses it, on every sample.

\begin{proposition}[Errors of the fixed-set rule are errors of the clean oracle]
\label{prop:integrity-transfers}
Assume the conditions of \cref{thm:fixed-set} and $k\le n$. On every sample and every
admissible transcript:
\begin{enumerate}[itemsep=1pt,topsep=2pt]
 \item $\{Y_{n+1}\notin\cC^{\rm fs}_\alpha(X_{n+1})\}
 \subseteq\{Y_{n+1}\notin\cC^\circ_\alpha(X_{n+1})\}$: the rule misses the correct answer
 only on questions where the clean honest-mean oracle misses it too.
 \item If $\cC^{\rm fs}_\alpha(x)$ is a single answer $y$, then $\cC^\circ_\alpha(x)$ is
 $\{y\}$ or empty.
\end{enumerate}
Consequently, for every event $E$ of positive probability, which may depend on the
question, the label, the calibration sample, and the attacker's own randomness,
\[
 \Pp\{Y_{n+1}\in\cC^{\rm fs}_\alpha(X_{n+1})\mid E\}
 \ \ge\ \Pp\{Y_{n+1}\in\cC^\circ_\alpha(X_{n+1})\mid E\},
\]
uniformly over Byzantine reports. (The right-hand side does not depend on those reports
when $E$ is determined by the clean data.) The same statements hold for
$\cC^{\rm jt}_\alpha$ and $\cC^{\rm del}_\alpha$.
\end{proposition}

\begin{proof}
Part 1 of \cref{thm:fixed-set} gives
$\cC^\circ_\alpha(x)\subseteq\cC^{\rm fs}_\alpha(x)\subseteq\cC^{\rm jt}_\alpha(x)
\subseteq\cC^{\rm del}_\alpha(x)$ for every query and admissible transcript. Items 1 and 2
restate the first inclusion. For the conditional statement, intersect the event of item 1
with $E$ and divide by $\Pp(E)$.
\end{proof}

A rule that is valid but lacks this inclusion can keep its average coverage while moving
its errors: toward a subgroup of questions, or toward questions the clean system answers
correctly. The conditional inequality rules that out for any subgroup an attacker might
choose, at whatever coverage the clean oracle has there.

\subsection{Exact certificates for the fixed-set decision}
\label{sec:certificates}

Two exact certificates decide most candidates without enumerating the possible honest groups;
only the candidates they leave open need the enumeration.

\begin{proposition}[Two certificates]
\label{prop:certificates}
Fix the reports $V$ and $w=w(x,y)$, with $k\le n$. For $d\in\{0,\ldots,A\}$, let $P_d$ be
a set of $K-d$ nodes with the smallest query reports for $y$ (ties broken arbitrarily).
\begin{enumerate}[itemsep=1pt,topsep=2pt]
 \item If $\Dn_A(w)>\tau_{A,k}(V)+g_\rho$, then $y\notin\cC^{\rm fs}_\alpha(x)$.
 \item If $\bar w_{P_d}\le T_k(P_d;V)+g_\rho$ for some $d\in\{0,\ldots,A\}$, then
 $y\in\cC^{\rm fs}_\alpha(x)$.
\end{enumerate}
Both conclusions are exact. At query time, only a candidate that neither test settles
requires the enumeration of $\mathfrak H_A$.
\end{proposition}

\begin{proof}
(1) The condition says $y\notin\cC^{\rm jt}_\alpha(x)$ by \cref{eq:joint-threshold-set}, and
$\cC^{\rm fs}_\alpha\subseteq\cC^{\rm jt}_\alpha$ on every transcript. (2) $P_d$ has
$K-d\ge K-A$ nodes, so it is feasible, and the condition says that it supports $y$ in the
sense of \cref{eq:fixed-set-rule}.
\end{proof}

The threshold $\tau_{A,k}(V)$ is computed once per calibration set, by one enumeration of
$\mathfrak H_A$ there, so at query time test (1) costs one
partial sort of $K$ reports per candidate. Test (2) needs, for each $d$, the $k$th smallest
of $n$ group means. With the column totals of $V$ stored ($O(nK)$ operations, once),
removing one node updates all $n$ means in $O(n)$ operations, and selection is linear, so
the $A+1$ prefix groups cost $O(K\log K+(A+1)n)$ per candidate. Among groups of a given size
the prefix group has the smallest query mean, which is why it is the natural witness to
try.
}
{\section{Stronger scorers, attacks through retrieved text, and changing identities}
\label{app:revision-experiments}

This appendix reports the experiments added during the discussion. Each design was fixed
in a committed file before its held-out questions were scored, every frozen analysis is
reported, and the supplement holds the code, the designs, and the frozen results.

\subsection{Design}
\label{app:revision-design}

\textbf{Scorers.}\quad Three instruction-tuned models score the candidates, each as the
scorer of all sixteen nodes: \texttt{Llama-3.2-3B-Instruct}~\citep{grattafiori2024llama3}
(the scorer of \cref{app:judge-panel}), \texttt{Qwen2.5-7B-Instruct}, and
\texttt{Qwen2.5-14B-Instruct}~\citep{qwen2024qwen25}, at pinned revisions and in single
precision. A node's score for a candidate
is one minus the model's probability of that candidate's letter, renormalized over the
listed letters, as in \cref{app:judge-panel}. A prompt holds the node's four passages, cut
where needed so that the prompt fits 2{,}048 tokens.

\textbf{Benchmarks and federations.}\quad The three medical benchmarks of
\cref{sec:experiments}, with the hash partition of the medical textbook corpus into sixteen
shards and the GPT-2 study's frozen retrieval (four passages per node), and OpenBookQA with
the sixteen fact shards of \cref{app:judge-panel}. Pools: 989 (MMLU-Med), 1{,}173
(MedQA-US), 4{,}083 (MedMCQA), and 1{,}000 (OpenBookQA) questions; calibration on 329, 391,
499, and 333 of them; 100 random splits for the medical benchmarks and 500 for OpenBookQA,
with registered seeds. The 7B model was not run on MedMCQA. With the Qwen scorers the
OpenBookQA nodes are those of the passage-attack plans (\cref{app:revision-passages}), whose
clean views equal the submission's retrieval; questions, splits, and corrupt nodes are those
of the Llama panel.

\textbf{Threats, rules, and reports.}\quad The seven constructed threats of \cref{tab:real-main}
with two corrupt nodes of sixteen and a declared budget of two (and three), at targets of
90\% and 80\%, and every rule of \cref{app:baselines}. Reports are rounded to eight bits as
registered; because the two Qwen scorers' scores are sharply peaked, every panel is also
analyzed at 16 and 24 bits with the same splits and corrupt nodes
(\cref{app:revision-resolution}).

\textbf{Anchors.}\quad Three single-scorer pipelines with the same evidence give
split-conformal sets at the same targets: the model without retrieval, retrieval of four
passages from the whole corpus, and one node.

\subsection{The ladder of scorers}
\label{app:revision-ladder}

The stronger the scorer, the smaller the sets at the same guarantee, and fixed-set keeps its
advantage on every panel: at the resolution of \cref{tab:consolidated} it removes
41.0--74.7\% of deletion's excess over the oracle at 90\%. \Cref{tab:ladder} gives every
scored medical and OpenBookQA panel at both targets and at each report resolution analyzed: the oracle, the three proposed rules, and the merger, the
saving of fixed-set over deletion, and the shares of single-answer and full sets.
\Cref{tab:matched-baselines} sets the federation beside three single-scorer pipelines with
the same evidence and target. \Cref{tab:aggregators-two-phase} gives every rule of
\cref{app:baselines} under low--high reporting, the attack that stresses coverage.

\begingroup
\let\resulttablefloat\table
\renewcommand{\table}[1][]{\resulttablefloat[H]}%
\begingroup\scriptsize
\rev{%
\setlength{\tabcolsep}{2.6pt}
\begin{longtable}{@{}llrrrrrrrrrrrr@{}}
\caption{Fixed-set is never larger than joint-threshold or deletion, and at the resolution of \cref{tab:consolidated}, fixed-set removes 41.0--74.7\% of deletion's excess over the oracle at 90\%; where reports are coarse enough that every proposed rule keeps all four answers (three eight-bit Qwen rows at 90\%), finer reports restore informative sets. Two of sixteen nodes send the largest score in both phases, declared budget $A=2$; means over 100 calibration splits (500 on OpenBookQA). Llama-3.2-3B on OpenBookQA: the Llama-3.2-3B panel rescored for the revision, on the splits all scorers share. BioASQ-Y/N and PubMedQA* (two and three options): the rows of \cref{tab:consolidated}, from the disjoint federation of \cref{tab:literature}. Accuracy: top-answer accuracy of the sixteen-node mean with nobody corrupt (\%), on unrounded scores like the oracle, so both are the same at every resolution. Saved: deletion's size minus fixed-set's; Cut: that saving as a share of deletion's excess over the oracle (\%; a dash where deletion matches the oracle); Single and Full: shares of fixed-set's sets with one answer and with all four (\%). Fixed: fixed-set; Fixed cov.: fixed-set's coverage (\%); Joint: joint-threshold; Del.: deletion; Merger: $p$-value merger.}\label{tab:ladder}\\
\toprule
Scorer & Benchmark & Bits & Accuracy & Oracle & Fixed & Fixed cov. & Joint & Del. & Merger & Saved & Cut & Single & Full\\
\midrule
\endfirsthead
\toprule
Scorer & Benchmark & Bits & Accuracy & Oracle & Fixed & Fixed cov. & Joint & Del. & Merger & Saved & Cut & Single & Full\\
\midrule
\endhead
\bottomrule
\endfoot
\multicolumn{14}{@{}l}{\emph{Target 90\%}}\\
\midrule
Llama-3.2-3B & OpenBookQA & 8 & 73.0 & 1.721 & 1.865 & 91.99 & 1.933 & 2.104 & 2.857 & 0.239 & 62.5 & 43.7 & 6.4\\
Llama-3.2-3B & OpenBookQA & 16 & 73.0 & 1.721 & 1.837 & 91.63 & 1.904 & 2.066 & 2.783 & 0.230 & 66.5 & 44.6 & 5.8\\
Llama-3.2-3B & OpenBookQA & 24 & 73.0 & 1.721 & 1.837 & 91.63 & 1.904 & 2.066 & 2.782 & 0.229 & 66.5 & 44.6 & 5.8\\
Llama-3.2-3B & MedQA-US & 8 & 55.3 & 2.625 & 2.802 & 92.63 & 2.889 & 2.979 & 3.483 & 0.177 & 49.9 & 17.4 & 36.5\\
Llama-3.2-3B & MedQA-US & 16 & 55.3 & 2.625 & 2.732 & 91.70 & 2.825 & 2.914 & 3.425 & 0.182 & 63.1 & 18.7 & 33.4\\
Llama-3.2-3B & MedQA-US & 24 & 55.3 & 2.625 & 2.732 & 91.70 & 2.825 & 2.914 & 3.425 & 0.182 & 63.1 & 18.7 & 33.4\\
Llama-3.2-3B & MedMCQA & 8 & 47.5 & 2.734 & 2.905 & 92.39 & 2.976 & 3.114 & 3.581 & 0.209 & 55.0 & 14.2 & 38.3\\
Llama-3.2-3B & MedMCQA & 16 & 47.5 & 2.734 & 2.856 & 91.73 & 2.930 & 3.066 & 3.526 & 0.210 & 63.1 & 14.9 & 35.8\\
Llama-3.2-3B & MedMCQA & 24 & 47.5 & 2.734 & 2.856 & 91.73 & 2.929 & 3.066 & 3.526 & 0.210 & 63.2 & 14.9 & 35.8\\
Llama-3.2-3B & MMLU-Med & 8 & 66.9 & 1.948 & 2.114 & 92.28 & 2.196 & 2.322 & 2.866 & 0.208 & 55.5 & 36.6 & 11.4\\
Llama-3.2-3B & MMLU-Med & 16 & 66.9 & 1.948 & 2.076 & 91.75 & 2.162 & 2.285 & 2.815 & 0.208 & 61.9 & 37.4 & 10.0\\
Llama-3.2-3B & MMLU-Med & 24 & 66.9 & 1.948 & 2.076 & 91.75 & 2.162 & 2.285 & 2.815 & 0.208 & 61.9 & 37.4 & 10.0\\
Qwen2.5-7B & OpenBookQA & 8 & 78.1 & 1.419 & 1.565 & 92.27 & 1.576 & 1.942 & 4.000 & 0.376 & 72.0 & 57.1 & 1.2\\
Qwen2.5-7B & OpenBookQA & 16 & 78.1 & 1.419 & 1.548 & 92.04 & 1.558 & 1.903 & 4.000 & 0.355 & 73.4 & 57.9 & 1.1\\
Qwen2.5-7B & OpenBookQA & 24 & 78.1 & 1.419 & 1.548 & 92.04 & 1.558 & 1.903 & 2.971 & 0.355 & 73.4 & 57.9 & 1.1\\
Qwen2.5-7B & MedQA-US & 8 & 55.6 & 2.549 & 4.000 & 100.00 & 4.000 & 4.000 & 4.000 & 0.000 & 0.0 & 0.0 & 100.0\\
Qwen2.5-7B & MedQA-US & 16 & 55.6 & 2.549 & 2.680 & 91.19 & 2.691 & 2.960 & 4.000 & 0.280 & 68.1 & 20.1 & 32.0\\
Qwen2.5-7B & MedQA-US & 24 & 55.6 & 2.549 & 2.673 & 91.12 & 2.684 & 2.930 & 3.592 & 0.257 & 67.4 & 20.2 & 31.7\\
Qwen2.5-7B & MMLU-Med & 8 & 73.3 & 1.742 & 1.907 & 91.88 & 1.916 & 3.128 & 4.000 & 1.221 & 88.1 & 44.2 & 8.3\\
Qwen2.5-7B & MMLU-Med & 16 & 73.3 & 1.742 & 1.865 & 91.44 & 1.875 & 2.153 & 3.991 & 0.288 & 70.0 & 45.4 & 7.0\\
Qwen2.5-7B & MMLU-Med & 24 & 73.3 & 1.742 & 1.865 & 91.44 & 1.875 & 2.152 & 2.827 & 0.287 & 70.0 & 45.4 & 7.0\\
Qwen2.5-14B & OpenBookQA & 8 & 86.1 & 1.115 & 1.188 & 92.05 & 1.196 & 1.290 & 4.000 & 0.102 & 58.2 & 82.7 & 0.0\\
Qwen2.5-14B & OpenBookQA & 16 & 86.1 & 1.115 & 1.183 & 91.94 & 1.191 & 1.282 & 4.000 & 0.099 & 59.3 & 83.1 & 0.0\\
Qwen2.5-14B & OpenBookQA & 24 & 86.1 & 1.115 & 1.183 & 91.94 & 1.191 & 1.282 & 2.442 & 0.099 & 59.3 & 83.1 & 0.0\\
Qwen2.5-14B & MedQA-US & 8 & 64.1 & 2.179 & 4.000 & 100.00 & 4.000 & 4.000 & 4.000 & 0.000 & 0.0 & 0.0 & 100.0\\
Qwen2.5-14B & MedQA-US & 16 & 64.1 & 2.179 & 2.686 & 93.10 & 2.692 & 3.970 & 4.000 & 1.284 & 71.7 & 23.3 & 37.5\\
Qwen2.5-14B & MedQA-US & 24 & 64.1 & 2.179 & 2.309 & 91.29 & 2.316 & 2.628 & 4.000 & 0.319 & 71.0 & 29.8 & 19.5\\
Qwen2.5-14B & MedMCQA & 8 & 54.2 & 2.665 & 4.000 & 100.00 & 4.000 & 4.000 & 4.000 & 0.000 & 0.0 & 0.0 & 100.0\\
Qwen2.5-14B & MedMCQA & 16 & 54.2 & 2.665 & 2.785 & 91.37 & 2.791 & 3.739 & 4.000 & 0.953 & 88.8 & 20.0 & 39.7\\
Qwen2.5-14B & MedMCQA & 24 & 54.2 & 2.665 & 2.774 & 91.25 & 2.779 & 3.080 & 4.000 & 0.306 & 73.7 & 20.2 & 39.2\\
Qwen2.5-14B & MMLU-Med & 8 & 78.7 & 1.394 & 1.508 & 91.73 & 1.515 & 2.779 & 4.000 & 1.271 & 91.8 & 62.1 & 2.0\\
Qwen2.5-14B & MMLU-Med & 16 & 78.7 & 1.394 & 1.487 & 91.47 & 1.496 & 1.745 & 4.000 & 0.258 & 73.6 & 63.4 & 1.8\\
Qwen2.5-14B & MMLU-Med & 24 & 78.7 & 1.394 & 1.486 & 91.45 & 1.496 & 1.745 & 3.634 & 0.258 & 73.7 & 63.4 & 1.8\\
\midrule
Llama-3.2-3B & BioASQ-Y/N & 8 & 71.0 & 1.417 & 1.512 & 92.84 & 1.525 & 1.658 & 2.000 & 0.146 & 60.5 & 48.8 & 51.2\\
Llama-3.2-3B & PubMedQA* & 8 & 52.5 & 2.028 & 2.303 & 92.67 & 2.342 & 2.494 & 3.000 & 0.191 & 41.0 & 3.7 & 34.0\\
Qwen2.5-14B & BioASQ-Y/N & 24 & 73.9 & 1.366 & 1.457 & 93.02 & 1.483 & 1.728 & 2.000 & 0.271 & 74.7 & 54.3 & 45.7\\
Qwen2.5-14B & PubMedQA* & 24 & 39.2 & 2.373 & 2.635 & 95.29 & 2.642 & 2.970 & 3.000 & 0.335 & 56.1 & 4.3 & 67.8\\
\midrule
\multicolumn{14}{@{}l}{\emph{Target 80\%}}\\
\midrule
Llama-3.2-3B & OpenBookQA & 8 & 73.0 & 1.189 & 1.287 & 83.07 & 1.336 & 1.436 & 2.396 & 0.148 & 60.1 & 72.0 & 0.0\\
Llama-3.2-3B & OpenBookQA & 16 & 73.0 & 1.189 & 1.278 & 82.87 & 1.326 & 1.422 & 2.354 & 0.145 & 62.0 & 72.8 & 0.0\\
Llama-3.2-3B & OpenBookQA & 24 & 73.0 & 1.189 & 1.278 & 82.87 & 1.326 & 1.422 & 2.354 & 0.145 & 62.0 & 72.8 & 0.0\\
Llama-3.2-3B & MedQA-US & 8 & 55.3 & 2.000 & 2.132 & 83.06 & 2.228 & 2.370 & 3.165 & 0.238 & 64.3 & 32.4 & 10.8\\
Llama-3.2-3B & MedQA-US & 16 & 55.3 & 2.000 & 2.096 & 82.36 & 2.189 & 2.326 & 3.112 & 0.230 & 70.5 & 33.2 & 9.6\\
Llama-3.2-3B & MedQA-US & 24 & 55.3 & 2.000 & 2.096 & 82.36 & 2.189 & 2.326 & 3.112 & 0.230 & 70.5 & 33.2 & 9.6\\
Llama-3.2-3B & MedMCQA & 8 & 47.5 & 2.189 & 2.372 & 83.67 & 2.474 & 2.670 & 3.282 & 0.298 & 61.9 & 23.1 & 14.7\\
Llama-3.2-3B & MedMCQA & 16 & 47.5 & 2.189 & 2.330 & 82.85 & 2.433 & 2.626 & 3.239 & 0.296 & 67.7 & 23.9 & 13.2\\
Llama-3.2-3B & MedMCQA & 24 & 47.5 & 2.189 & 2.330 & 82.85 & 2.433 & 2.626 & 3.239 & 0.296 & 67.7 & 23.9 & 13.1\\
Llama-3.2-3B & MMLU-Med & 8 & 66.9 & 1.401 & 1.515 & 81.96 & 1.593 & 1.674 & 2.484 & 0.159 & 58.1 & 57.3 & 0.2\\
Llama-3.2-3B & MMLU-Med & 16 & 66.9 & 1.401 & 1.497 & 81.62 & 1.576 & 1.655 & 2.451 & 0.157 & 62.0 & 58.3 & 0.1\\
Llama-3.2-3B & MMLU-Med & 24 & 66.9 & 1.401 & 1.497 & 81.62 & 1.576 & 1.655 & 2.451 & 0.157 & 62.0 & 58.3 & 0.1\\
Qwen2.5-7B & OpenBookQA & 8 & 78.1 & 1.063 & 1.164 & 83.60 & 1.174 & 1.265 & 4.000 & 0.102 & 50.3 & 83.7 & 0.0\\
Qwen2.5-7B & OpenBookQA & 16 & 78.1 & 1.063 & 1.158 & 83.41 & 1.169 & 1.259 & 3.781 & 0.101 & 51.4 & 84.1 & 0.0\\
Qwen2.5-7B & OpenBookQA & 24 & 78.1 & 1.063 & 1.158 & 83.41 & 1.169 & 1.259 & 2.364 & 0.101 & 51.4 & 84.1 & 0.0\\
Qwen2.5-7B & MedQA-US & 8 & 55.6 & 1.927 & 2.068 & 82.73 & 2.075 & 3.631 & 4.000 & 1.563 & 91.7 & 33.4 & 8.5\\
Qwen2.5-7B & MedQA-US & 16 & 55.6 & 1.927 & 2.034 & 82.05 & 2.041 & 2.393 & 4.000 & 0.358 & 77.0 & 34.3 & 7.5\\
Qwen2.5-7B & MedQA-US & 24 & 55.6 & 1.927 & 2.034 & 82.05 & 2.041 & 2.392 & 3.194 & 0.357 & 77.0 & 34.3 & 7.5\\
Qwen2.5-7B & MMLU-Med & 8 & 73.3 & 1.196 & 1.288 & 82.27 & 1.300 & 1.428 & 4.000 & 0.140 & 60.2 & 73.7 & 0.0\\
Qwen2.5-7B & MMLU-Med & 16 & 73.3 & 1.196 & 1.282 & 82.11 & 1.293 & 1.418 & 3.199 & 0.136 & 61.3 & 74.2 & 0.0\\
Qwen2.5-7B & MMLU-Med & 24 & 73.3 & 1.196 & 1.282 & 82.11 & 1.293 & 1.418 & 2.323 & 0.136 & 61.3 & 74.2 & 0.0\\
Qwen2.5-14B & OpenBookQA & 8 & 86.1 & 0.879 & 0.943 & 83.90 & 0.952 & 0.969 & 4.000 & 0.025 & 28.2 & 93.9 & 0.0\\
Qwen2.5-14B & OpenBookQA & 16 & 86.1 & 0.879 & 0.941 & 83.76 & 0.950 & 0.966 & 3.287 & 0.025 & 29.1 & 93.6 & 0.0\\
Qwen2.5-14B & OpenBookQA & 24 & 86.1 & 0.879 & 0.941 & 83.76 & 0.950 & 0.966 & 1.637 & 0.025 & 29.1 & 93.6 & 0.0\\
Qwen2.5-14B & MedQA-US & 8 & 64.1 & 1.496 & 1.632 & 82.74 & 1.643 & 2.914 & 4.000 & 1.282 & 90.4 & 51.4 & 1.4\\
Qwen2.5-14B & MedQA-US & 16 & 64.1 & 1.496 & 1.597 & 82.01 & 1.618 & 1.920 & 4.000 & 0.323 & 76.2 & 53.4 & 1.1\\
Qwen2.5-14B & MedQA-US & 24 & 64.1 & 1.496 & 1.596 & 81.96 & 1.618 & 1.919 & 3.718 & 0.323 & 76.5 & 53.5 & 1.1\\
Qwen2.5-14B & MedMCQA & 8 & 54.2 & 2.023 & 2.174 & 82.40 & 2.183 & 3.974 & 4.000 & 1.799 & 92.3 & 31.2 & 12.2\\
Qwen2.5-14B & MedMCQA & 16 & 54.2 & 2.023 & 2.139 & 81.81 & 2.147 & 2.553 & 4.000 & 0.414 & 78.2 & 31.9 & 10.9\\
Qwen2.5-14B & MedMCQA & 24 & 54.2 & 2.023 & 2.139 & 81.81 & 2.147 & 2.552 & 3.996 & 0.413 & 78.1 & 32.0 & 10.8\\
Qwen2.5-14B & MMLU-Med & 8 & 78.7 & 1.040 & 1.110 & 82.64 & 1.118 & 1.171 & 4.000 & 0.061 & 46.4 & 88.7 & 0.0\\
Qwen2.5-14B & MMLU-Med & 16 & 78.7 & 1.040 & 1.106 & 82.50 & 1.115 & 1.166 & 4.000 & 0.060 & 47.7 & 89.0 & 0.0\\
Qwen2.5-14B & MMLU-Med & 24 & 78.7 & 1.040 & 1.106 & 82.50 & 1.115 & 1.166 & 2.024 & 0.060 & 47.7 & 89.0 & 0.0\\
\midrule
Llama-3.2-3B & BioASQ-Y/N & 8 & 71.0 & 1.174 & 1.265 & 84.49 & 1.284 & 1.351 & 1.953 & 0.086 & 48.5 & 73.5 & 26.5\\
Llama-3.2-3B & PubMedQA* & 8 & 52.5 & 1.655 & 1.739 & 83.15 & 1.756 & 1.817 & 2.966 & 0.077 & 47.7 & 26.1 & 0.0\\
Qwen2.5-14B & BioASQ-Y/N & 24 & 73.9 & 1.140 & 1.230 & 84.19 & 1.239 & 1.308 & 2.000 & 0.079 & 46.8 & 77.0 & 23.0\\
Qwen2.5-14B & PubMedQA* & 24 & 39.2 & 2.145 & 2.284 & 87.19 & 2.358 & 2.820 & 3.000 & 0.536 & 79.4 & 12.3 & 40.7\\
\end{longtable}
\par}%
\endgroup
\endgroup{}
\begingroup
\let\resulttablefloat\table
\renewcommand{\table}[1][]{\resulttablefloat[H]}%
\begin{table}[tbp]
\rev{%
\centering\scriptsize
\setlength{\tabcolsep}{2.6pt}
\caption{With nobody corrupt, fixed-set costs 0.062--0.227 answers per question over the split-conformal set of the sixteen-node mean at the same 90\% target. Accuracy: top-answer accuracy (\%); Size: mean split-conformal set size at 90\%; the last three columns: the robust rules at declared budget $A=2$ with nobody corrupt. Bits is the report resolution: the Qwen rows are at 24 bits, the finest analyzed for each. BioASQ-Y/N and PubMedQA* (two and three options): the disjoint federation of \cref{tab:literature}. Llama-3.2-3B on OpenBookQA: the Llama-3.2-3B panel rescored for the revision, on the splits all scorers share. Llama-3.2-3B's OpenBookQA cells without retrieval and from the whole corpus come from the design-1 panel of the false claim (\cref{tab:passage-answers}), whose sixteen clean nodes have the same scores and splits. No retr.: no retrieval; Corpus: retrieval from the whole corpus; One node: a single node, averaged over the sixteen; Mean: sixteen-node mean; Fixed: fixed-set; Del.: deletion; Merger: $p$-value merger.}
\label{tab:matched-baselines}
\begin{tabular}{@{}llrrrrrrrrrrrr@{}}
\toprule
 & & & \multicolumn{4}{c}{Accuracy} & \multicolumn{4}{c}{Size} & \multicolumn{3}{c}{Nobody corrupt}\\
\cmidrule(lr){4-7}\cmidrule(lr){8-11}\cmidrule(l){12-14}
Scorer & Benchmark & Bits & No retr. & Corpus & One node & Mean & No retr. & Corpus & One node & Mean & Fixed & Del. & Merger\\
\midrule
Llama-3.2-3B & OpenBookQA & 8 & 74.0 & 73.5 & 70.2 & 73.0 & 1.752 & 1.803 & 1.926 & 1.719 & 1.849 & 2.027 & 2.804\\
Llama-3.2-3B & MedQA-US & 8 & 54.8 & 54.8 & 52.5 & 55.3 & 2.602 & 2.669 & 2.690 & 2.633 & 2.786 & 2.906 & 3.454\\
Llama-3.2-3B & MedMCQA & 8 & 50.0 & 48.5 & 46.6 & 47.5 & 2.648 & 2.789 & 2.880 & 2.733 & 2.882 & 3.049 & 3.547\\
Llama-3.2-3B & MMLU-Med & 8 & 66.2 & 66.4 & 64.9 & 66.9 & 1.862 & 1.959 & 2.068 & 1.944 & 2.098 & 2.257 & 2.819\\
Qwen2.5-7B & OpenBookQA & 24 & 82.1 & 75.8 & 73.5 & 78.1 & 1.282 & 1.576 & 1.818 & 1.412 & 1.530 & 1.836 & 2.880\\
Qwen2.5-7B & MedQA-US & 24 & 57.6 & 56.0 & 54.1 & 55.6 & 2.457 & 2.566 & 2.689 & 2.550 & 2.662 & 2.923 & 3.560\\
Qwen2.5-7B & MMLU-Med & 24 & 75.0 & 71.1 & 70.2 & 73.3 & 1.536 & 1.788 & 1.928 & 1.738 & 1.854 & 2.140 & 2.754\\
Qwen2.5-14B & OpenBookQA & 24 & 87.2 & 86.4 & 82.5 & 86.1 & 1.101 & 1.159 & 1.337 & 1.113 & 1.175 & 1.251 & 2.342\\
Qwen2.5-14B & MedQA-US & 24 & 66.3 & 65.6 & 63.2 & 64.1 & 2.198 & 2.249 & 2.346 & 2.176 & 2.291 & 2.611 & 4.000\\
Qwen2.5-14B & MedMCQA & 24 & 57.4 & 54.6 & 51.4 & 54.2 & 2.588 & 2.612 & 2.820 & 2.662 & 2.766 & 3.066 & 4.000\\
Qwen2.5-14B & MMLU-Med & 24 & 80.6 & 79.6 & 76.6 & 78.7 & 1.392 & 1.414 & 1.591 & 1.390 & 1.476 & 1.732 & 3.580\\
\midrule
Llama-3.2-3B & BioASQ-Y/N & 8 & 53.3 & 83.8 & 69.7 & 71.0 & 1.750 & 1.159 & 1.552 & 1.412 & 1.503 & 1.627 & 2.000\\
Llama-3.2-3B & PubMedQA* & 8 & 36.0 & 76.0 & 51.3 & 52.5 & 2.236 & 1.696 & 2.258 & 2.026 & 2.253 & 2.405 & 3.000\\
Qwen2.5-14B & BioASQ-Y/N & 24 & 51.4 & 87.3 & 72.1 & 73.9 & 1.735 & 1.052 & 1.441 & 1.364 & 1.446 & 1.629 & 2.000\\
Qwen2.5-14B & PubMedQA* & 24 & 9.5 & 78.5 & 39.5 & 39.2 & 2.837 & 1.432 & 2.648 & 2.358 & 2.558 & 2.959 & 3.000\\
\bottomrule
\end{tabular}
\par}%
\end{table}
\endgroup{}
\begingroup
\let\resulttablefloat\table
\renewcommand{\table}[1][]{\resulttablefloat[H]}%
\begin{table}[tbp]
\rev{%
\centering\scriptsize
\setlength{\tabcolsep}{2pt}
\caption{Under low--high reporting the three proposed rules coincide, with mean coverage 89.78--91.09\%, while the all-node mean's falls to 63.31--86.40\% and that of the winsorized, trimmed, and median-of-means summaries to 62.90--87.12\%. Two of sixteen nodes report the lowest score at calibration and the highest at the query, declared budget $A=2$, nominal coverage 90\%; means over 100 calibration splits (500 on OpenBookQA), with standard errors of at most 1.0 points of coverage. Upper block: coverage (\%); lower block: mean set size. Bits is the report resolution: the Qwen rows are at 24 bits, the finest analyzed for each. BioASQ-Y/N and PubMedQA* (two and three options): the disjoint federation of \cref{tab:literature}. Llama-3.2-3B on OpenBookQA: the Llama-3.2-3B panel rescored for the revision, on the splits all scorers share. Fixed: fixed-set; Joint: joint-threshold; Del.: deletion; Merger: $p$-value merger; Sym.: guarded symmetric trimming; Mean: all-node mean; Wins.: winsorized mean; Trim.: unguarded trimmed mean; MoM: median of means; Local: local marginals.}
\label{tab:aggregators-two-phase}
\begin{tabular}{@{}llrrrrrrrrrrrrr@{}}
\toprule
Scorer & Benchmark & Bits & Fixed & Joint & Del. & Merger & Sym. & Mean & Median & Wins. & Trim. & MoM & Krum & Local\\
\midrule
\multicolumn{15}{@{}l}{\emph{Coverage (\%)}}\\
\midrule
Llama-3.2-3B & OpenBookQA & 8 & 90.41 & 90.41 & 90.41 & 97.27 & 100.00 & 82.08 & 87.33 & 85.23 & 86.16 & 84.88 & 90.23 & 91.78\\
Llama-3.2-3B & MedQA-US & 8 & 91.09 & 91.09 & 91.09 & 97.33 & 100.00 & 70.09 & 88.14 & 86.11 & 87.12 & 85.05 & 90.60 & 90.83\\
Llama-3.2-3B & MedMCQA & 8 & 90.66 & 90.66 & 90.66 & 97.43 & 100.00 & 65.30 & 86.71 & 83.06 & 84.22 & 82.40 & 90.41 & 92.52\\
Llama-3.2-3B & MMLU-Med & 8 & 90.34 & 90.34 & 90.34 & 96.90 & 100.00 & 77.76 & 87.30 & 84.60 & 85.73 & 84.16 & 89.99 & 90.72\\
Qwen2.5-7B & OpenBookQA & 24 & 90.09 & 90.09 & 90.09 & 97.11 & 99.98 & 83.64 & 86.05 & 83.35 & 83.97 & 82.96 & 89.61 & 95.87\\
Qwen2.5-7B & MedQA-US & 24 & 89.80 & 89.80 & 89.80 & 97.02 & 100.00 & 70.52 & 86.32 & 83.13 & 83.31 & 84.20 & 89.84 & 93.79\\
Qwen2.5-7B & MMLU-Med & 24 & 89.96 & 89.96 & 89.96 & 96.60 & 100.00 & 82.84 & 88.03 & 83.92 & 84.34 & 84.56 & 90.12 & 94.20\\
Qwen2.5-14B & OpenBookQA & 24 & 90.16 & 90.16 & 90.16 & 97.20 & 93.91 & 86.40 & 87.02 & 86.36 & 86.43 & 86.36 & 89.41 & 95.92\\
Qwen2.5-14B & MedQA-US & 24 & 89.93 & 89.93 & 89.93 & 100.00 & 100.00 & 76.93 & 87.87 & 84.14 & 84.35 & 85.09 & 89.82 & 93.02\\
Qwen2.5-14B & MedMCQA & 24 & 89.78 & 89.78 & 89.78 & 100.00 & 100.00 & 69.54 & 87.21 & 82.15 & 82.41 & 83.35 & 89.57 & 94.02\\
Qwen2.5-14B & MMLU-Med & 24 & 89.92 & 89.92 & 89.92 & 98.90 & 100.00 & 85.05 & 87.37 & 84.65 & 85.33 & 84.49 & 89.87 & 94.80\\
\midrule
Llama-3.2-3B & BioASQ-Y/N & 8 & 90.14 & 90.14 & 90.14 & 100.00 & 99.70 & 82.54 & 85.69 & 82.34 & 83.46 & 81.67 & 90.17 & 95.60\\
Llama-3.2-3B & PubMedQA* & 8 & 90.57 & 90.57 & 90.57 & 100.00 & 100.00 & 79.58 & 88.79 & 86.77 & 86.97 & 86.94 & 90.48 & 91.48\\
Qwen2.5-14B & BioASQ-Y/N & 24 & 90.12 & 90.12 & 90.12 & 100.00 & 99.87 & 82.30 & 100.00 & 82.28 & 82.30 & 82.30 & 95.77 & 100.00\\
Qwen2.5-14B & PubMedQA* & 24 & 91.07 & 91.07 & 91.07 & 100.00 & 100.00 & 63.31 & 100.00 & 62.90 & 63.82 & 71.48 & 100.00 & 100.00\\
\midrule
\multicolumn{15}{@{}l}{\emph{Mean set size}}\\
\midrule
Llama-3.2-3B & OpenBookQA & 8 & 1.742 & 1.742 & 1.742 & 2.637 & 4.000 & 1.244 & 1.614 & 1.420 & 1.485 & 1.405 & 1.762 & 1.879\\
Llama-3.2-3B & MedQA-US & 8 & 2.685 & 2.685 & 2.685 & 3.348 & 4.000 & 1.559 & 2.461 & 2.308 & 2.371 & 2.256 & 2.644 & 2.651\\
Llama-3.2-3B & MedMCQA & 8 & 2.778 & 2.778 & 2.778 & 3.426 & 4.000 & 1.558 & 2.546 & 2.339 & 2.396 & 2.318 & 2.796 & 2.917\\
Llama-3.2-3B & MMLU-Med & 8 & 1.980 & 1.980 & 1.980 & 2.660 & 4.000 & 1.301 & 1.808 & 1.629 & 1.693 & 1.618 & 1.984 & 2.021\\
Qwen2.5-7B & OpenBookQA & 24 & 1.419 & 1.419 & 1.419 & 2.582 & 3.995 & 1.164 & 1.303 & 1.155 & 1.175 & 1.148 & 1.550 & 2.074\\
Qwen2.5-7B & MedQA-US & 24 & 2.549 & 2.549 & 2.549 & 3.422 & 4.000 & 1.477 & 2.317 & 2.088 & 2.100 & 2.160 & 2.550 & 2.878\\
Qwen2.5-7B & MMLU-Med & 24 & 1.742 & 1.742 & 1.742 & 2.548 & 4.000 & 1.307 & 1.598 & 1.370 & 1.387 & 1.411 & 1.774 & 2.099\\
Qwen2.5-14B & OpenBookQA & 24 & 1.115 & 1.115 & 1.115 & 2.051 & 1.767 & 0.998 & 1.051 & 0.996 & 1.001 & 0.999 & 1.144 & 1.479\\
Qwen2.5-14B & MedQA-US & 24 & 2.179 & 2.179 & 2.179 & 4.000 & 4.000 & 1.361 & 2.079 & 1.767 & 1.778 & 1.849 & 2.274 & 2.509\\
Qwen2.5-14B & MedMCQA & 24 & 2.665 & 2.665 & 2.665 & 4.000 & 4.000 & 1.532 & 2.462 & 2.153 & 2.163 & 2.226 & 2.692 & 3.083\\
Qwen2.5-14B & MMLU-Med & 24 & 1.394 & 1.394 & 1.394 & 3.296 & 4.000 & 1.177 & 1.304 & 1.168 & 1.187 & 1.176 & 1.443 & 1.778\\
\midrule
Llama-3.2-3B & BioASQ-Y/N & 8 & 1.422 & 1.422 & 1.422 & 2.000 & 1.981 & 1.219 & 1.332 & 1.214 & 1.246 & 1.214 & 1.464 & 1.682\\
Llama-3.2-3B & PubMedQA* & 8 & 2.112 & 2.112 & 2.112 & 3.000 & 3.000 & 1.627 & 2.085 & 1.887 & 1.907 & 1.906 & 2.171 & 2.250\\
Qwen2.5-14B & BioASQ-Y/N & 24 & 1.370 & 1.370 & 1.370 & 2.000 & 1.993 & 1.187 & 2.000 & 1.187 & 1.187 & 1.183 & 1.806 & 2.000\\
Qwen2.5-14B & PubMedQA* & 24 & 2.380 & 2.380 & 2.380 & 3.000 & 3.000 & 1.682 & 3.000 & 1.660 & 1.722 & 2.081 & 3.000 & 3.000\\
\bottomrule
\end{tabular}
\par}%
\end{table}
\endgroup{}

\subsection{Report resolution and the \texorpdfstring{$p$}{p}-value merger}
\label{app:revision-resolution}

Finer reports restore informative sets for peaked scorers without touching the guarantee
(\cref{tab:report-resolution}). Where the target far exceeds such a scorer's accuracy (MedQA-US
with the 7B and the 14B, MedMCQA with the 14B, 54--64\% accurate), the conformal cutoff lies
within one eight-bit rounding step of the top of the score range, so at eight bits the proposed
rules' rounding guard keeps every answer under stable maximum. The merger fails for another
reason. It tests each node at level $\alpha\ell/(K-A)$ with $\ell=A+1$ and drops a candidate
only when $A+\ell$ nodes reject it, and a node cannot reject anything when its calibration
sample gives the correct answer the top report of the grid more often than that level.
\Cref{tab:grid-top} gives, for every node, how often that happens.

\begingroup
\let\resulttablefloat\table
\renewcommand{\table}[1][]{\resulttablefloat[H]}%
\begin{table}[tbp]
\rev{%
\centering\scriptsize
\setlength{\tabcolsep}{2.6pt}
\caption{Finer reports never enlarge fixed-set or deletion and bring both toward the oracle, while on the Qwen scorers the merger keeps all four answers below 16 bits, and on two of the seven Qwen panels even at 24 bits. Mean set size with nobody corrupt, declared budget $A=2$, nominal coverage 90\%, by the number of bits per report; the oracle uses unrounded scores and is the same at every resolution. Llama-3.2-3B on OpenBookQA: the Llama-3.2-3B panel rescored for the revision, on the splits all scorers share. Part.: partition, H hash and T topic; Fixed: fixed-set; Del.: deletion; Merger: $p$-value merger.}
\label{tab:report-resolution}
\begin{tabular}{@{}llllrrrrrrr@{}}
\toprule
 & & & & \multicolumn{6}{c}{Bits per report} & \\
\cmidrule(lr){5-10}
Scorer & Benchmark & Part. & Rule & 6 & 8 & 10 & 12 & 16 & 24 & Oracle\\
\midrule
Llama-3.2-3B & OpenBookQA &  & Fixed & 1.929 & 1.849 & 1.832 & 1.827 & 1.826 & 1.826 & 1.719\\
 &  &  & Del. & 2.129 & 2.027 & 2.004 & 1.999 & 1.997 & 1.997 & \\
 &  &  & Merger & 2.994 & 2.804 & 2.747 & 2.734 & 2.730 & 2.729 & \\
\midrule
Llama-3.2-3B & MedQA-US & H & Fixed & 2.970 & 2.786 & 2.739 & 2.728 & 2.725 & 2.725 & 2.633\\
 &  &  & Del. & 3.089 & 2.906 & 2.861 & 2.850 & 2.846 & 2.846 & \\
 &  &  & Merger & 3.589 & 3.454 & 3.409 & 3.396 & 3.392 & 3.392 & \\
\midrule
Llama-3.2-3B & MedQA-US & T & Fixed & 2.926 & 2.753 & 2.716 & 2.706 & 2.703 & 2.703 & 2.595\\
 &  &  & Del. & 3.105 & 2.907 & 2.862 & 2.851 & 2.847 & 2.847 & \\
 &  &  & Merger & 3.572 & 3.449 & 3.396 & 3.382 & 3.377 & 3.377 & \\
\midrule
Llama-3.2-3B & MedMCQA & H & Fixed & 3.015 & 2.882 & 2.850 & 2.842 & 2.840 & 2.840 & 2.733\\
 &  &  & Del. & 3.185 & 3.049 & 3.016 & 3.009 & 3.006 & 3.006 & \\
 &  &  & Merger & 3.638 & 3.547 & 3.505 & 3.494 & 3.491 & 3.491 & \\
\midrule
Llama-3.2-3B & MedMCQA & T & Fixed & 3.089 & 2.953 & 2.920 & 2.912 & 2.910 & 2.909 & 2.791\\
 &  &  & Del. & 3.262 & 3.120 & 3.086 & 3.078 & 3.075 & 3.075 & \\
 &  &  & Merger & 3.658 & 3.565 & 3.521 & 3.511 & 3.507 & 3.507 & \\
\midrule
Llama-3.2-3B & MMLU-Med & H & Fixed & 2.197 & 2.098 & 2.073 & 2.067 & 2.065 & 2.065 & 1.944\\
 &  &  & Del. & 2.359 & 2.257 & 2.232 & 2.226 & 2.225 & 2.224 & \\
 &  &  & Merger & 3.000 & 2.819 & 2.780 & 2.771 & 2.768 & 2.768 & \\
\midrule
Llama-3.2-3B & MMLU-Med & T & Fixed & 2.254 & 2.147 & 2.121 & 2.115 & 2.113 & 2.113 & 1.997\\
 &  &  & Del. & 2.425 & 2.307 & 2.281 & 2.275 & 2.273 & 2.273 & \\
 &  &  & Merger & 3.016 & 2.834 & 2.796 & 2.786 & 2.783 & 2.783 & \\
\midrule
Qwen2.5-7B & OpenBookQA &  & Fixed & 1.585 & 1.543 & 1.533 & 1.531 & 1.530 & 1.530 & 1.412\\
 &  &  & Del. & 1.990 & 1.859 & 1.841 & 1.837 & 1.836 & 1.836 & \\
 &  &  & Merger & 4.000 & 4.000 & 4.000 & 4.000 & 4.000 & 2.880 & \\
\midrule
Qwen2.5-7B & MedQA-US & H & Fixed & 4.000 & 3.985 & 3.643 & 2.818 & 2.667 & 2.662 & 2.550\\
 &  &  & Del. & 4.000 & 4.000 & 4.000 & 3.927 & 2.942 & 2.923 & \\
 &  &  & Merger & 4.000 & 4.000 & 4.000 & 4.000 & 4.000 & 3.560 & \\
\midrule
Qwen2.5-7B & MMLU-Med & H & Fixed & 2.318 & 1.885 & 1.861 & 1.856 & 1.854 & 1.854 & 1.738\\
 &  &  & Del. & 3.888 & 2.813 & 2.178 & 2.149 & 2.140 & 2.140 & \\
 &  &  & Merger & 4.000 & 4.000 & 4.000 & 4.000 & 3.991 & 2.754 & \\
\midrule
Qwen2.5-14B & OpenBookQA &  & Fixed & 1.192 & 1.179 & 1.176 & 1.175 & 1.175 & 1.175 & 1.113\\
 &  &  & Del. & 1.276 & 1.257 & 1.252 & 1.251 & 1.251 & 1.251 & \\
 &  &  & Merger & 4.000 & 4.000 & 4.000 & 4.000 & 4.000 & 2.342 & \\
\midrule
Qwen2.5-14B & MedQA-US & H & Fixed & 4.000 & 4.000 & 3.982 & 3.744 & 2.419 & 2.291 & 2.176\\
 &  &  & Del. & 4.000 & 4.000 & 4.000 & 4.000 & 3.927 & 2.611 & \\
 &  &  & Merger & 4.000 & 4.000 & 4.000 & 4.000 & 4.000 & 4.000 & \\
\midrule
Qwen2.5-14B & MedMCQA & H & Fixed & 4.000 & 4.000 & 3.935 & 3.175 & 2.772 & 2.766 & 2.662\\
 &  &  & Del. & 4.000 & 4.000 & 4.000 & 4.000 & 3.507 & 3.066 & \\
 &  &  & Merger & 4.000 & 4.000 & 4.000 & 4.000 & 4.000 & 4.000 & \\
\midrule
Qwen2.5-14B & MMLU-Med & H & Fixed & 1.558 & 1.498 & 1.485 & 1.479 & 1.477 & 1.476 & 1.390\\
 &  &  & Del. & 3.396 & 2.366 & 1.834 & 1.737 & 1.732 & 1.732 & \\
 &  &  & Merger & 4.000 & 4.000 & 4.000 & 4.000 & 4.000 & 3.580 & \\
\bottomrule
\end{tabular}
\par}%
\end{table}
\endgroup{}
\begingroup
\let\resulttablefloat\table
\renewcommand{\table}[1][]{\resulttablefloat[H]}%
\begingroup\scriptsize
\rev{%
\setlength{\tabcolsep}{2.6pt}
\begin{longtable}{@{}lllrrrrrrrr@{}}
\caption{On a replicate, the merger can drop nothing, whatever the query reports, if and only if at least $K-2A=12$ of the sixteen nodes report the correct answer at the top of the grid on at least $c$ of its calibration questions; at 90\% this holds on every replicate of every Qwen panel up to 12 bits, and from 8 bits up on no replicate of a Llama panel. With $n$ calibration questions, node $k$'s smallest $p$-value is $(1+t_k)/(n+1)$, $t_k$ counting those whose correct answer it reports at the top of the grid; the merger drops a candidate exactly when at least $2A+1=5$ nodes give it a $p$-value at or below $\alpha(A+1)/(K-A)$, a numerator at most $c=\lfloor\alpha(A+1)(n+1)/(K-A)\rfloor$ (7--10 here, with $A=2$ and $K=16$). Kept: the replicates, of 100 calibration splits (500 on OpenBookQA), on which the condition holds at 90\%; the Merger column gives the merger's mean set size with nobody corrupt at 90\%, all four candidates wherever Kept is every replicate. Lowest, median, and highest: over the nodes, the share of the pool's questions whose correct answer's report is the top grid level (\%, a description of the panel, not the condition); Level: the per-node level at 90\% (\%); Above: the nodes whose share exceeds it. Llama-3.2-3B on OpenBookQA: the Llama-3.2-3B panel rescored for the revision, on the splits all scorers share. Part.: partition, H hash and T topic; Merger: $p$-value merger.}\label{tab:grid-top}\\
\toprule
 & & & & \multicolumn{3}{c}{Top share (\%)} & & & & \\
\cmidrule(lr){5-7}
Scorer & Benchmark & Part. & Bits & Lowest & Median & Highest & Level & Above & Kept & Merger\\
\midrule
\endfirsthead
\toprule
 & & & & \multicolumn{3}{c}{Top share (\%)} & & & & \\
\cmidrule(lr){5-7}
Scorer & Benchmark & Part. & Bits & Lowest & Median & Highest & Level & Above & Kept & Merger\\
\midrule
\endhead
\bottomrule
\endfoot
Llama-3.2-3B & OpenBookQA &  & 6 & 0.8 & 1.1 & 1.7 & 2.14 & 0 & 0 & 2.994\\
 &  &  & 8 & 0.0 & 0.1 & 0.5 & 2.14 & 0 & 0 & 2.804\\
 &  &  & 10 & 0.0 & 0.0 & 0.1 & 2.14 & 0 & 0 & 2.747\\
 &  &  & 12 & 0.0 & 0.0 & 0.0 & 2.14 & 0 & 0 & 2.734\\
 &  &  & 16 & 0.0 & 0.0 & 0.0 & 2.14 & 0 & 0 & 2.730\\
 &  &  & 24 & 0.0 & 0.0 & 0.0 & 2.14 & 0 & 0 & 2.729\\
\midrule
Llama-3.2-3B & MedQA-US & H & 6 & 0.7 & 1.1 & 1.6 & 2.14 & 0 & 0 & 3.589\\
 &  &  & 8 & 0.2 & 0.3 & 0.4 & 2.14 & 0 & 0 & 3.454\\
 &  &  & 10 & 0.0 & 0.2 & 0.3 & 2.14 & 0 & 0 & 3.409\\
 &  &  & 12 & 0.0 & 0.0 & 0.1 & 2.14 & 0 & 0 & 3.396\\
 &  &  & 16 & 0.0 & 0.0 & 0.0 & 2.14 & 0 & 0 & 3.392\\
 &  &  & 24 & 0.0 & 0.0 & 0.0 & 2.14 & 0 & 0 & 3.392\\
\midrule
Llama-3.2-3B & MedQA-US & T & 6 & 0.3 & 1.1 & 1.7 & 2.14 & 0 & 0 & 3.572\\
 &  &  & 8 & 0.0 & 0.3 & 0.4 & 2.14 & 0 & 0 & 3.449\\
 &  &  & 10 & 0.0 & 0.2 & 0.3 & 2.14 & 0 & 0 & 3.396\\
 &  &  & 12 & 0.0 & 0.0 & 0.1 & 2.14 & 0 & 0 & 3.382\\
 &  &  & 16 & 0.0 & 0.0 & 0.0 & 2.14 & 0 & 0 & 3.377\\
 &  &  & 24 & 0.0 & 0.0 & 0.0 & 2.14 & 0 & 0 & 3.377\\
\midrule
Llama-3.2-3B & MedMCQA & H & 6 & 1.1 & 1.4 & 1.6 & 2.14 & 0 & 1 & 3.638\\
 &  &  & 8 & 0.1 & 0.2 & 0.4 & 2.14 & 0 & 0 & 3.547\\
 &  &  & 10 & 0.0 & 0.0 & 0.1 & 2.14 & 0 & 0 & 3.505\\
 &  &  & 12 & 0.0 & 0.0 & 0.0 & 2.14 & 0 & 0 & 3.494\\
 &  &  & 16 & 0.0 & 0.0 & 0.0 & 2.14 & 0 & 0 & 3.491\\
 &  &  & 24 & 0.0 & 0.0 & 0.0 & 2.14 & 0 & 0 & 3.491\\
\midrule
Llama-3.2-3B & MedMCQA & T & 6 & 0.7 & 1.2 & 1.6 & 2.14 & 0 & 0 & 3.658\\
 &  &  & 8 & 0.1 & 0.2 & 0.4 & 2.14 & 0 & 0 & 3.565\\
 &  &  & 10 & 0.0 & 0.0 & 0.1 & 2.14 & 0 & 0 & 3.521\\
 &  &  & 12 & 0.0 & 0.0 & 0.0 & 2.14 & 0 & 0 & 3.511\\
 &  &  & 16 & 0.0 & 0.0 & 0.0 & 2.14 & 0 & 0 & 3.507\\
 &  &  & 24 & 0.0 & 0.0 & 0.0 & 2.14 & 0 & 0 & 3.507\\
\midrule
Llama-3.2-3B & MMLU-Med & H & 6 & 0.6 & 1.0 & 1.2 & 2.14 & 0 & 0 & 3.000\\
 &  &  & 8 & 0.1 & 0.2 & 0.7 & 2.14 & 0 & 0 & 2.819\\
 &  &  & 10 & 0.0 & 0.1 & 0.2 & 2.14 & 0 & 0 & 2.780\\
 &  &  & 12 & 0.0 & 0.0 & 0.1 & 2.14 & 0 & 0 & 2.771\\
 &  &  & 16 & 0.0 & 0.0 & 0.0 & 2.14 & 0 & 0 & 2.768\\
 &  &  & 24 & 0.0 & 0.0 & 0.0 & 2.14 & 0 & 0 & 2.768\\
\midrule
Llama-3.2-3B & MMLU-Med & T & 6 & 0.5 & 0.8 & 1.5 & 2.14 & 0 & 0 & 3.016\\
 &  &  & 8 & 0.1 & 0.4 & 0.6 & 2.14 & 0 & 0 & 2.834\\
 &  &  & 10 & 0.0 & 0.1 & 0.3 & 2.14 & 0 & 0 & 2.796\\
 &  &  & 12 & 0.0 & 0.0 & 0.1 & 2.14 & 0 & 0 & 2.786\\
 &  &  & 16 & 0.0 & 0.0 & 0.0 & 2.14 & 0 & 0 & 2.783\\
 &  &  & 24 & 0.0 & 0.0 & 0.0 & 2.14 & 0 & 0 & 2.783\\
\midrule
Qwen2.5-7B & OpenBookQA &  & 6 & 14.4 & 15.5 & 17.1 & 2.14 & 16 & 500 & 4.000\\
 &  &  & 8 & 11.2 & 12.8 & 14.4 & 2.14 & 16 & 500 & 4.000\\
 &  &  & 10 & 9.7 & 10.3 & 11.0 & 2.14 & 16 & 500 & 4.000\\
 &  &  & 12 & 7.7 & 8.2 & 8.9 & 2.14 & 16 & 500 & 4.000\\
 &  &  & 16 & 4.4 & 5.1 & 5.6 & 2.14 & 16 & 500 & 4.000\\
 &  &  & 24 & 0.7 & 1.1 & 1.4 & 2.14 & 0 & 0 & 2.880\\
\midrule
Qwen2.5-7B & MedQA-US & H & 6 & 28.9 & 29.8 & 30.9 & 2.14 & 16 & 100 & 4.000\\
 &  &  & 8 & 23.2 & 24.7 & 26.4 & 2.14 & 16 & 100 & 4.000\\
 &  &  & 10 & 18.0 & 19.5 & 22.0 & 2.14 & 16 & 100 & 4.000\\
 &  &  & 12 & 13.5 & 14.9 & 16.5 & 2.14 & 16 & 100 & 4.000\\
 &  &  & 16 & 6.0 & 7.2 & 8.6 & 2.14 & 16 & 100 & 4.000\\
 &  &  & 24 & 0.3 & 1.0 & 1.5 & 2.14 & 0 & 0 & 3.560\\
\midrule
Qwen2.5-7B & MMLU-Med & H & 6 & 15.2 & 16.6 & 17.7 & 2.14 & 16 & 100 & 4.000\\
 &  &  & 8 & 11.6 & 13.5 & 15.0 & 2.14 & 16 & 100 & 4.000\\
 &  &  & 10 & 9.5 & 10.7 & 12.0 & 2.14 & 16 & 100 & 4.000\\
 &  &  & 12 & 7.3 & 8.3 & 9.8 & 2.14 & 16 & 100 & 4.000\\
 &  &  & 16 & 3.4 & 4.3 & 5.1 & 2.14 & 16 & 99 & 3.991\\
 &  &  & 24 & 0.4 & 0.8 & 1.1 & 2.14 & 0 & 0 & 2.754\\
\midrule
Qwen2.5-14B & OpenBookQA &  & 6 & 9.5 & 10.2 & 11.5 & 2.14 & 16 & 500 & 4.000\\
 &  &  & 8 & 7.9 & 8.9 & 10.1 & 2.14 & 16 & 500 & 4.000\\
 &  &  & 10 & 6.6 & 7.8 & 9.0 & 2.14 & 16 & 500 & 4.000\\
 &  &  & 12 & 5.8 & 6.7 & 7.8 & 2.14 & 16 & 500 & 4.000\\
 &  &  & 16 & 3.9 & 4.8 & 5.6 & 2.14 & 16 & 500 & 4.000\\
 &  &  & 24 & 1.0 & 1.5 & 2.0 & 2.14 & 0 & 11 & 2.342\\
\midrule
Qwen2.5-14B & MedQA-US & H & 6 & 26.8 & 27.9 & 28.4 & 2.14 & 16 & 100 & 4.000\\
 &  &  & 8 & 24.4 & 25.4 & 26.3 & 2.14 & 16 & 100 & 4.000\\
 &  &  & 10 & 22.1 & 22.8 & 24.2 & 2.14 & 16 & 100 & 4.000\\
 &  &  & 12 & 19.6 & 20.6 & 21.5 & 2.14 & 16 & 100 & 4.000\\
 &  &  & 16 & 15.3 & 15.8 & 17.2 & 2.14 & 16 & 100 & 4.000\\
 &  &  & 24 & 3.7 & 4.8 & 6.0 & 2.14 & 16 & 100 & 4.000\\
\midrule
Qwen2.5-14B & MedMCQA & H & 6 & 32.8 & 34.0 & 34.9 & 2.14 & 16 & 100 & 4.000\\
 &  &  & 8 & 28.8 & 29.8 & 30.8 & 2.14 & 16 & 100 & 4.000\\
 &  &  & 10 & 25.0 & 26.0 & 27.2 & 2.14 & 16 & 100 & 4.000\\
 &  &  & 12 & 21.5 & 22.4 & 23.8 & 2.14 & 16 & 100 & 4.000\\
 &  &  & 16 & 15.4 & 16.1 & 17.0 & 2.14 & 16 & 100 & 4.000\\
 &  &  & 24 & 5.7 & 6.0 & 6.2 & 2.14 & 16 & 100 & 4.000\\
\midrule
Qwen2.5-14B & MMLU-Med & H & 6 & 15.2 & 16.3 & 18.1 & 2.14 & 16 & 100 & 4.000\\
 &  &  & 8 & 12.9 & 14.7 & 16.2 & 2.14 & 16 & 100 & 4.000\\
 &  &  & 10 & 11.4 & 13.1 & 14.3 & 2.14 & 16 & 100 & 4.000\\
 &  &  & 12 & 10.0 & 11.3 & 12.1 & 2.14 & 16 & 100 & 4.000\\
 &  &  & 16 & 7.2 & 8.5 & 9.5 & 2.14 & 16 & 100 & 4.000\\
 &  &  & 24 & 1.8 & 2.6 & 3.4 & 2.14 & 13 & 70 & 3.580\\
\end{longtable}
\par}%
\endgroup
\endgroup{}

\subsection{The worst case along the target}
\label{app:revision-sweep}

The worst case is vacuous only while the target is well above the scorer's accuracy.
\Cref{tab:target-sweep} moves the target from 95\% to 30\%. Under high--low reporting the
proposed rules return nearly every answer while the target is well above the scorer's
accuracy, and their sets shrink fast once it is not; under low--high reporting the all-node mean
misses every target.

\begingroup
\let\resulttablefloat\table
\renewcommand{\table}[1][]{\resulttablefloat[H]}%
\begin{table}[tbp]
\rev{%
\centering\scriptsize
\setlength{\tabcolsep}{2.2pt}
\caption{Fixed-set's largest mean size over no attack and three attacks never grows as the target is lowered, while under low--high reporting the all-node mean's mean coverage is below every target on every panel. Upper block: fixed-set's largest mean size over no attack, the one-phase maximum at the query, low--high, and high--low reporting; lower block: the mean coverage of the all-node mean (\%) under low--high reporting. Two corrupt nodes of sixteen, declared budget $A=2$; columns are the coverage targets; means over 100 calibration splits (500 on OpenBookQA), with standard errors of at most 1.9 points of coverage. Llama-3.2-3B on OpenBookQA: the Llama-3.2-3B panel rescored for the revision, on the splits all scorers share. Part.: partition, H hash and T topic.}
\label{tab:target-sweep}
\begin{tabular}{@{}lllrrrrrrrrrrr@{}}
\toprule
 & & & & \multicolumn{10}{c}{Target}\\
\cmidrule(l){5-14}
Scorer & Benchmark & Part. & Bits & 95\% & 90\% & 85\% & 80\% & 75\% & 70\% & 60\% & 50\% & 40\% & 30\%\\
\midrule
\multicolumn{14}{@{}l}{\emph{Fixed-set's largest mean size over the four cells}}\\
\midrule
Llama-3.2-3B & OpenBookQA &  & 8 & 4.000 & 3.990 & 2.855 & 1.907 & 1.576 & 1.367 & 1.074 & 0.852 & 0.702 & 0.596\\
Llama-3.2-3B & MedQA-US & H & 8 & 4.000 & 4.000 & 4.000 & 4.000 & 3.969 & 3.310 & 2.051 & 1.498 & 1.035 & 0.733\\
Llama-3.2-3B & MedQA-US & T & 8 & 4.000 & 4.000 & 4.000 & 4.000 & 3.953 & 3.232 & 2.069 & 1.534 & 1.042 & 0.767\\
Llama-3.2-3B & MedMCQA & H & 8 & 4.000 & 4.000 & 4.000 & 4.000 & 3.974 & 3.611 & 2.784 & 2.038 & 1.432 & 0.914\\
Llama-3.2-3B & MedMCQA & T & 8 & 4.000 & 4.000 & 4.000 & 4.000 & 3.976 & 3.675 & 2.878 & 2.167 & 1.519 & 0.977\\
Llama-3.2-3B & MMLU-Med & H & 8 & 4.000 & 4.000 & 3.924 & 2.744 & 2.075 & 1.698 & 1.259 & 0.911 & 0.707 & 0.578\\
Llama-3.2-3B & MMLU-Med & T & 8 & 4.000 & 4.000 & 3.913 & 2.830 & 2.120 & 1.761 & 1.340 & 0.970 & 0.732 & 0.593\\
Qwen2.5-7B & OpenBookQA &  & 8 & 4.000 & 3.893 & 1.831 & 1.431 & 1.222 & 1.091 & 0.910 & 0.808 & 0.757 & 0.725\\
Qwen2.5-7B & OpenBookQA &  & 16 & 4.000 & 3.839 & 1.782 & 1.421 & 1.216 & 1.086 & 0.906 & 0.804 & 0.753 & 0.722\\
Qwen2.5-7B & MedQA-US & H & 8 & 4.000 & 4.000 & 4.000 & 4.000 & 3.982 & 3.767 & 1.530 & 1.093 & 0.875 & 0.780\\
Qwen2.5-7B & MedQA-US & H & 16 & 4.000 & 4.000 & 4.000 & 4.000 & 3.982 & 3.633 & 1.517 & 1.087 & 0.871 & 0.776\\
Qwen2.5-7B & MMLU-Med & H & 8 & 4.000 & 4.000 & 3.863 & 1.857 & 1.360 & 1.157 & 0.955 & 0.823 & 0.755 & 0.740\\
Qwen2.5-7B & MMLU-Med & H & 16 & 4.000 & 4.000 & 3.816 & 1.774 & 1.351 & 1.151 & 0.951 & 0.819 & 0.751 & 0.735\\
Qwen2.5-14B & OpenBookQA &  & 8 & 4.000 & 1.606 & 1.168 & 1.052 & 0.966 & 0.915 & 0.867 & 0.828 & 0.826 & 0.826\\
Qwen2.5-14B & OpenBookQA &  & 16 & 3.996 & 1.540 & 1.163 & 1.049 & 0.963 & 0.913 & 0.863 & 0.825 & 0.821 & 0.820\\
Qwen2.5-14B & MedQA-US & H & 8 & 4.000 & 4.000 & 4.000 & 4.000 & 3.249 & 1.545 & 1.105 & 0.901 & 0.823 & 0.786\\
Qwen2.5-14B & MedQA-US & H & 16 & 4.000 & 4.000 & 4.000 & 4.000 & 3.143 & 1.490 & 1.101 & 0.898 & 0.819 & 0.781\\
Qwen2.5-14B & MedMCQA & H & 8 & 4.000 & 4.000 & 4.000 & 4.000 & 4.000 & 3.983 & 1.753 & 1.198 & 0.940 & 0.782\\
Qwen2.5-14B & MedMCQA & H & 16 & 4.000 & 4.000 & 4.000 & 4.000 & 4.000 & 3.969 & 1.731 & 1.190 & 0.935 & 0.777\\
Qwen2.5-14B & MMLU-Med & H & 8 & 4.000 & 4.000 & 2.330 & 1.280 & 1.106 & 0.978 & 0.877 & 0.830 & 0.825 & 0.825\\
Qwen2.5-14B & MMLU-Med & H & 16 & 4.000 & 4.000 & 2.260 & 1.273 & 1.102 & 0.975 & 0.874 & 0.827 & 0.821 & 0.821\\
\midrule
\multicolumn{14}{@{}l}{\emph{Coverage of the all-node mean under low--high reporting (\%)}}\\
\midrule
Llama-3.2-3B & OpenBookQA &  & 8 & 84.97 & 82.08 & 75.61 & 69.13 & 64.41 & 59.90 & 48.91 & 37.18 & 17.57 & 0.03\\
Llama-3.2-3B & MedQA-US & H & 8 & 73.31 & 70.09 & 67.15 & 64.07 & 61.14 & 56.91 & 47.07 & 39.76 & 30.57 & 18.68\\
Llama-3.2-3B & MedQA-US & T & 8 & 73.58 & 70.41 & 67.15 & 63.91 & 60.34 & 55.44 & 46.93 & 40.55 & 29.85 & 20.56\\
Llama-3.2-3B & MedMCQA & H & 8 & 70.70 & 65.30 & 61.41 & 57.90 & 54.58 & 51.25 & 44.08 & 37.01 & 29.99 & 20.92\\
Llama-3.2-3B & MedMCQA & T & 8 & 69.82 & 65.19 & 61.23 & 57.54 & 53.53 & 49.99 & 42.96 & 36.54 & 29.39 & 20.64\\
Llama-3.2-3B & MMLU-Med & H & 8 & 81.01 & 77.76 & 74.09 & 68.31 & 62.77 & 58.29 & 50.07 & 38.54 & 25.40 & 3.43\\
Llama-3.2-3B & MMLU-Med & T & 8 & 80.83 & 77.88 & 74.12 & 67.71 & 61.08 & 56.80 & 49.39 & 38.62 & 26.46 & 9.30\\
Qwen2.5-7B & OpenBookQA &  & 8 & 88.51 & 83.64 & 78.94 & 74.26 & 68.14 & 62.33 & 44.96 & 0.43 & 0.00 & 0.00\\
Qwen2.5-7B & OpenBookQA &  & 16 & 88.51 & 83.64 & 78.94 & 74.25 & 68.14 & 62.32 & 44.92 & 0.42 & 0.00 & 0.00\\
Qwen2.5-7B & MedQA-US & H & 8 & 70.71 & 70.55 & 69.76 & 67.74 & 65.39 & 61.99 & 53.79 & 42.80 & 24.87 & 0.00\\
Qwen2.5-7B & MedQA-US & H & 16 & 70.71 & 70.52 & 69.76 & 67.74 & 65.38 & 61.98 & 53.79 & 42.80 & 24.84 & 0.00\\
Qwen2.5-7B & MMLU-Med & H & 8 & 84.72 & 82.85 & 79.56 & 74.36 & 69.92 & 62.68 & 50.46 & 1.09 & 0.00 & 0.00\\
Qwen2.5-7B & MMLU-Med & H & 16 & 84.71 & 82.84 & 79.56 & 74.35 & 69.92 & 62.68 & 50.43 & 1.07 & 0.00 & 0.00\\
Qwen2.5-14B & OpenBookQA &  & 8 & 91.47 & 86.40 & 80.24 & 74.05 & 64.43 & 40.74 & 0.00 & 0.00 & 0.00 & 0.00\\
Qwen2.5-14B & OpenBookQA &  & 16 & 91.46 & 86.40 & 80.23 & 74.05 & 64.40 & 40.55 & 0.00 & 0.00 & 0.00 & 0.00\\
Qwen2.5-14B & MedQA-US & H & 8 & 77.19 & 77.06 & 75.62 & 72.85 & 69.46 & 64.64 & 54.82 & 25.66 & 0.00 & 0.00\\
Qwen2.5-14B & MedQA-US & H & 16 & 77.10 & 76.93 & 75.61 & 72.84 & 69.45 & 64.63 & 54.81 & 25.63 & 0.00 & 0.00\\
Qwen2.5-14B & MedMCQA & H & 8 & 69.79 & 69.57 & 68.60 & 66.11 & 63.51 & 60.73 & 53.69 & 43.43 & 30.80 & 1.80\\
Qwen2.5-14B & MedMCQA & H & 16 & 69.73 & 69.54 & 68.58 & 66.10 & 63.50 & 60.73 & 53.68 & 43.42 & 30.77 & 1.48\\
Qwen2.5-14B & MMLU-Med & H & 8 & 88.01 & 85.05 & 80.69 & 76.32 & 71.35 & 62.01 & 1.10 & 0.00 & 0.00 & 0.00\\
Qwen2.5-14B & MMLU-Med & H & 16 & 87.94 & 85.05 & 80.69 & 76.32 & 71.35 & 62.00 & 1.10 & 0.00 & 0.00 & 0.00\\
\bottomrule
\end{tabular}
\par}%
\end{table}
\endgroup{}

\subsection{Attacks through retrieved passages}
\label{app:revision-passages}

Planted text misleads the nodes that read it, yet costs the proposed rules no coverage while it
is present in both phases. A false sentence asserting one wrong option, written in the
corpus's style, reaches a node in three ways: as its fourth passage; as a sentence in its shard, where its own retriever must find
it; or as four such sentences that flood its shard. In an earlier design, the submission's
injected instruction or a forged answer key, placed in a passage, barely moved a node's answers,
so replacing the instruction itself, as in \cref{sec:experiments}, is the harsher test; the
tables report the false sentence. \Cref{tab:passage-answers} gives what a poisoned node then answers, and
\cref{tab:passage-coverage} the coverage of the rules with a fixed number of poisoned nodes in
each phase. Text present in both phases leaves every proposed rule's coverage in place, as
\cref{cor:phase-blind} says; a source that is truthful at calibration and false afterward
is a Byzantine node, and the guarantee holds while such nodes are within the budget.

\begingroup
\let\resulttablefloat\table
\renewcommand{\table}[1][]{\resulttablefloat[H]}%
\begin{table}[tbp]
\rev{%
\centering\scriptsize
\setlength{\tabcolsep}{2.6pt}
\caption{A planted text makes a poisoned node's top answer the planted option on 52.4--99.5\% of the questions, while the sixteen-node mean loses at most 3.0 points of accuracy with two poisoned nodes and 11.8 with five. Shares of questions (\%) whose top answer is correct or is the planted option: the mean of the sixteen nodes with nobody poisoned; one poisoned node that reads the planted text at the query; and the sixteen-node mean with two or five poisoned nodes. Design 1 plants the text as the fourth retrieved passage of the submission's OpenBookQA design; design 2 repeats it and adds a truthful version of the text, plants in the shard, and two medical benchmarks. On OpenBookQA the clean nodes of Llama-3.2-3B are those of the Llama-3.2-3B panel rescored for the revision, on the splits all scorers share. Rows equal at the precision shown are shown once, with their designs and resolutions; a row of designs 1 and 2 is one scored result that both designs contain, not two designs that agree.}
\label{tab:passage-answers}
\begin{tabular}{@{}lllllrrrrr@{}}
\toprule
 & & & & & Honest mean & \multicolumn{2}{c}{Poisoned node} & \multicolumn{2}{c}{Mean, poisoned}\\
\cmidrule(lr){7-8}\cmidrule(l){9-10}
Benchmark & Design & Planted text & Scorer & Bits & Correct & Planted & Correct & 2: correct & 5: correct\\
\midrule
OpenBookQA & 1, 2 & false answer & Llama-3.2-3B & 8 & 73.0 & 65.5 & 28.6 & 71.2 & 64.8\\
OpenBookQA & 1, 2 & false claim & Llama-3.2-3B & 8 & 73.0 & 72.0 & 23.0 & 70.8 & 63.5\\
OpenBookQA & 2 & false claim in the shard & Llama-3.2-3B & 8 & 73.0 & 75.1 & 20.2 & 70.9 & 63.2\\
OpenBookQA & 2 & flooded shard & Llama-3.2-3B & 8 & 73.0 & 97.3 & 2.1 & 70.8 & 61.2\\
MedQA-US & 2 & false answer & Llama-3.2-3B & 8 & 55.3 & 67.2 & 21.3 & 52.3 & 45.1\\
MMLU-Med & 2 & false answer & Llama-3.2-3B & 8 & 66.8 & 52.4 & 38.2 & 65.3 & 61.0\\
OpenBookQA & 2 & false answer & Qwen2.5-7B & 8, 16 & 78.1 & 76.4 & 21.8 & 77.2 & 72.9\\
OpenBookQA & 2 & false claim & Qwen2.5-7B & 8, 16 & 78.1 & 83.4 & 15.1 & 77.1 & 72.6\\
OpenBookQA & 2 & false claim in the shard & Qwen2.5-7B & 8, 16 & 78.1 & 92.3 & 7.1 & 77.0 & 72.2\\
OpenBookQA & 2 & flooded shard & Qwen2.5-7B & 8, 16 & 78.1 & 99.5 & 0.2 & 76.9 & 71.8\\
OpenBookQA & 2 & false answer & Qwen2.5-14B & 8, 16 & 86.1 & 85.1 & 14.3 & 85.4 & 82.0\\
OpenBookQA & 2 & false claim & Qwen2.5-14B & 8, 16 & 86.1 & 92.5 & 7.0 & 85.4 & 82.0\\
OpenBookQA & 2 & false claim in the shard & Qwen2.5-14B & 8, 16 & 86.1 & 93.7 & 6.0 & 85.4 & 82.0\\
OpenBookQA & 2 & flooded shard & Qwen2.5-14B & 8, 16 & 86.1 & 98.5 & 1.4 & 85.4 & 81.9\\
\bottomrule
\end{tabular}
\par}%
\end{table}
\endgroup{}
\begingroup
\let\resulttablefloat\table
\renewcommand{\table}[1][]{\resulttablefloat[H]}%
\begingroup\scriptsize
\rev{%
\setlength{\tabcolsep}{2.6pt}
\begin{longtable}{@{}lrrrrrr@{}}
\caption{Within the declared budget, fixed-set's mean coverage is at least 90.09\% under every passage-borne threat, while the all-node mean's falls to 78.11\%. Coverage (\%) and mean size of fixed-set and the all-node mean with $a$ poisoned nodes, declared budget $A=2$, nominal coverage 90\%; means over 100 calibration splits (500 on OpenBookQA), with standard errors of at most 0.3 points of coverage; Oracle is the size of the honest-mean oracle of the unpoisoned nodes. The poisoned nodes read the planted text at the query only, in both phases, at calibration only, or a true version at calibration and the false one at the query. $^\dagger$More poisoned nodes than the declared budget: outside the guarantee. Design 2 repeats design 1 and adds a truthful version of the text, plants in the shard, and two medical benchmarks. On OpenBookQA the clean nodes of Llama-3.2-3B are those of the Llama-3.2-3B panel rescored for the revision, on the splits all scorers share. The design-1 panels of the false claim and the false answer (Llama-3.2-3B) are not shown: under every threat they run, their rows are those of design 2, number for number, and design 2 also runs ``True, then false''.}\label{tab:passage-coverage}\\
\toprule
 & & & \multicolumn{2}{c}{Fixed-set} & \multicolumn{2}{c}{All-node mean}\\
\cmidrule(lr){4-5}\cmidrule(l){6-7}
Threat & $a$ & Oracle & Coverage & Size & Coverage & Size\\
\midrule
\endfirsthead
\toprule
 & & & \multicolumn{2}{c}{Fixed-set} & \multicolumn{2}{c}{All-node mean}\\
\cmidrule(lr){4-5}\cmidrule(l){6-7}
Threat & $a$ & Oracle & Coverage & Size & Coverage & Size\\
\midrule
\endhead
\bottomrule
\endfoot
\multicolumn{7}{@{}l}{\emph{OpenBookQA, design 2, false answer; Llama-3.2-3B, 8 bits}}\\
No attack & 0 & 1.719 & 91.76 & 1.849 & 90.14 & 1.718\\
Query only & 2 & 1.721 & 91.12 & 2.109 & 89.35 & 1.906\\
Both phases & 2 & 1.721 & 91.94 & 2.193 & 90.15 & 1.991\\
Calibration only & 2 & 1.721 & 92.73 & 1.923 & 90.86 & 1.778\\
True, then false & 2 & 1.721 & 90.41 & 1.808 & 84.12 & 1.461\\
True, then false & 3$^\dagger$ & 1.723 & 87.39 & 1.668 & 80.70 & 1.362\\
True, then false & 4$^\dagger$ & 1.723 & 83.89 & 1.545 & 74.91 & 1.251\\
True, then false & 5$^\dagger$ & 1.725 & 79.48 & 1.439 & 68.67 & 1.144\\
\midrule
\multicolumn{7}{@{}l}{\emph{OpenBookQA, design 2, false claim; Llama-3.2-3B, 8 bits}}\\
No attack & 0 & 1.719 & 91.76 & 1.849 & 90.14 & 1.718\\
Query only & 2 & 1.721 & 91.01 & 2.170 & 89.26 & 1.955\\
Both phases & 2 & 1.721 & 91.92 & 2.258 & 90.14 & 2.062\\
Calibration only & 2 & 1.721 & 92.88 & 1.934 & 90.98 & 1.788\\
True, then false & 2 & 1.721 & 90.41 & 1.795 & 83.46 & 1.423\\
True, then false & 3$^\dagger$ & 1.723 & 86.90 & 1.630 & 79.26 & 1.311\\
True, then false & 4$^\dagger$ & 1.723 & 83.08 & 1.504 & 72.60 & 1.184\\
True, then false & 5$^\dagger$ & 1.725 & 77.69 & 1.395 & 65.56 & 1.077\\
\midrule
\multicolumn{7}{@{}l}{\emph{OpenBookQA, design 2, false claim in the shard; Llama-3.2-3B, 8 bits}}\\
No attack & 0 & 1.719 & 91.76 & 1.849 & 90.14 & 1.718\\
Query only & 2 & 1.721 & 90.95 & 2.190 & 89.18 & 1.967\\
Both phases & 2 & 1.721 & 91.95 & 2.287 & 90.15 & 2.085\\
Calibration only & 2 & 1.721 & 92.96 & 1.940 & 91.04 & 1.794\\
True, then false & 2 & 1.721 & 90.41 & 1.798 & 83.40 & 1.427\\
True, then false & 3$^\dagger$ & 1.723 & 86.88 & 1.634 & 78.93 & 1.314\\
True, then false & 4$^\dagger$ & 1.723 & 82.99 & 1.511 & 72.32 & 1.192\\
True, then false & 5$^\dagger$ & 1.725 & 77.25 & 1.401 & 65.48 & 1.091\\
\midrule
\multicolumn{7}{@{}l}{\emph{OpenBookQA, design 2, flooded shard; Llama-3.2-3B, 8 bits}}\\
No attack & 0 & 1.719 & 91.76 & 1.849 & 90.14 & 1.718\\
Query only & 2 & 1.721 & 90.80 & 2.392 & 89.02 & 2.094\\
Both phases & 2 & 1.721 & 91.96 & 2.509 & 90.14 & 2.275\\
Calibration only & 2 & 1.721 & 93.21 & 1.961 & 91.20 & 1.806\\
True, then false & 2 & 1.721 & 90.41 & 1.817 & 83.18 & 1.436\\
True, then false & 3$^\dagger$ & 1.723 & 86.70 & 1.658 & 78.09 & 1.323\\
True, then false & 4$^\dagger$ & 1.723 & 82.56 & 1.555 & 71.19 & 1.208\\
True, then false & 5$^\dagger$ & 1.725 & 76.08 & 1.475 & 63.86 & 1.118\\
\midrule
\multicolumn{7}{@{}l}{\emph{MedQA-US, design 2, false answer; Llama-3.2-3B, 8 bits}}\\
No attack & 0 & 2.633 & 92.53 & 2.786 & 90.21 & 2.632\\
Query only & 2 & 2.625 & 91.94 & 2.951 & 89.66 & 2.803\\
Both phases & 2 & 2.625 & 92.46 & 2.991 & 90.26 & 2.842\\
Calibration only & 2 & 2.625 & 93.17 & 2.841 & 90.91 & 2.677\\
True, then false & 2 & 2.625 & 91.10 & 2.782 & 78.11 & 2.139\\
True, then false & 3$^\dagger$ & 2.620 & 84.48 & 2.447 & 73.23 & 2.015\\
True, then false & 4$^\dagger$ & 2.619 & 78.31 & 2.242 & 69.14 & 1.901\\
True, then false & 5$^\dagger$ & 2.617 & 73.16 & 2.087 & 65.12 & 1.788\\
\midrule
\multicolumn{7}{@{}l}{\emph{MMLU-Med, design 2, false answer; Llama-3.2-3B, 8 bits}}\\
No attack & 0 & 1.944 & 92.07 & 2.098 & 89.90 & 1.942\\
Query only & 2 & 1.948 & 91.44 & 2.253 & 88.91 & 2.076\\
Both phases & 2 & 1.948 & 92.24 & 2.314 & 89.92 & 2.146\\
Calibration only & 2 & 1.948 & 92.98 & 2.165 & 90.82 & 2.007\\
True, then false & 2 & 1.948 & 90.36 & 2.043 & 82.27 & 1.642\\
True, then false & 3$^\dagger$ & 1.954 & 86.11 & 1.881 & 78.68 & 1.519\\
True, then false & 4$^\dagger$ & 1.956 & 82.10 & 1.733 & 74.57 & 1.414\\
True, then false & 5$^\dagger$ & 1.950 & 78.07 & 1.595 & 69.24 & 1.313\\
\midrule
\multicolumn{7}{@{}l}{\emph{OpenBookQA, design 2, false answer; Qwen2.5-7B, 8 bits}}\\
No attack & 0 & 1.412 & 92.03 & 1.543 & 90.08 & 1.413\\
Query only & 2 & 1.419 & 91.34 & 1.912 & 89.12 & 1.554\\
Both phases & 2 & 1.419 & 92.21 & 2.064 & 90.08 & 1.675\\
Calibration only & 2 & 1.419 & 92.87 & 1.607 & 90.88 & 1.463\\
True, then false & 2 & 1.419 & 90.30 & 1.467 & 83.98 & 1.230\\
True, then false & 3$^\dagger$ & 1.424 & 87.30 & 1.361 & 80.58 & 1.146\\
True, then false & 4$^\dagger$ & 1.429 & 83.57 & 1.265 & 76.86 & 1.060\\
True, then false & 5$^\dagger$ & 1.433 & 79.67 & 1.172 & 72.01 & 0.975\\
\midrule
\multicolumn{7}{@{}l}{\emph{OpenBookQA, design 2, false claim; Qwen2.5-7B, 8 bits}}\\
No attack & 0 & 1.412 & 92.03 & 1.543 & 90.08 & 1.413\\
Query only & 2 & 1.419 & 91.33 & 1.968 & 89.10 & 1.573\\
Both phases & 2 & 1.419 & 92.26 & 2.138 & 90.09 & 1.711\\
Calibration only & 2 & 1.419 & 92.90 & 1.610 & 90.90 & 1.464\\
True, then false & 2 & 1.419 & 90.30 & 1.469 & 83.92 & 1.230\\
True, then false & 3$^\dagger$ & 1.424 & 87.27 & 1.364 & 80.42 & 1.146\\
True, then false & 4$^\dagger$ & 1.429 & 83.46 & 1.269 & 76.60 & 1.059\\
True, then false & 5$^\dagger$ & 1.433 & 79.40 & 1.174 & 71.60 & 0.973\\
\midrule
\multicolumn{7}{@{}l}{\emph{OpenBookQA, design 2, false claim in the shard; Qwen2.5-7B, 8 bits}}\\
No attack & 0 & 1.412 & 92.03 & 1.543 & 90.08 & 1.413\\
Query only & 2 & 1.419 & 91.32 & 2.041 & 89.07 & 1.595\\
Both phases & 2 & 1.419 & 92.26 & 2.231 & 90.09 & 1.749\\
Calibration only & 2 & 1.419 & 92.92 & 1.611 & 90.92 & 1.465\\
True, then false & 2 & 1.419 & 90.29 & 1.471 & 83.80 & 1.230\\
True, then false & 3$^\dagger$ & 1.424 & 87.19 & 1.365 & 80.20 & 1.145\\
True, then false & 4$^\dagger$ & 1.429 & 83.25 & 1.269 & 76.24 & 1.057\\
True, then false & 5$^\dagger$ & 1.433 & 79.04 & 1.173 & 70.88 & 0.967\\
\midrule
\multicolumn{7}{@{}l}{\emph{OpenBookQA, design 2, flooded shard; Qwen2.5-7B, 8 bits}}\\
No attack & 0 & 1.412 & 92.03 & 1.543 & 90.08 & 1.413\\
Query only & 2 & 1.419 & 91.28 & 2.110 & 89.01 & 1.615\\
Both phases & 2 & 1.419 & 92.27 & 2.321 & 90.09 & 1.798\\
Calibration only & 2 & 1.419 & 92.97 & 1.615 & 90.96 & 1.469\\
True, then false & 2 & 1.419 & 90.27 & 1.471 & 83.79 & 1.230\\
True, then false & 3$^\dagger$ & 1.424 & 87.13 & 1.364 & 80.12 & 1.146\\
True, then false & 4$^\dagger$ & 1.429 & 83.18 & 1.268 & 75.99 & 1.057\\
True, then false & 5$^\dagger$ & 1.433 & 78.85 & 1.172 & 70.50 & 0.967\\
\midrule
\multicolumn{7}{@{}l}{\emph{OpenBookQA, design 2, false answer; Qwen2.5-7B, 16 bits}}\\
No attack & 0 & 1.412 & 91.83 & 1.530 & 90.08 & 1.412\\
Query only & 2 & 1.419 & 91.15 & 1.866 & 89.12 & 1.553\\
Both phases & 2 & 1.419 & 91.99 & 2.030 & 90.08 & 1.675\\
Calibration only & 2 & 1.419 & 92.64 & 1.589 & 90.87 & 1.463\\
True, then false & 2 & 1.419 & 90.12 & 1.454 & 83.98 & 1.230\\
True, then false & 3$^\dagger$ & 1.424 & 87.10 & 1.350 & 80.58 & 1.146\\
True, then false & 4$^\dagger$ & 1.429 & 83.37 & 1.253 & 76.85 & 1.060\\
True, then false & 5$^\dagger$ & 1.433 & 79.47 & 1.159 & 72.01 & 0.975\\
\midrule
\multicolumn{7}{@{}l}{\emph{OpenBookQA, design 2, false claim; Qwen2.5-7B, 16 bits}}\\
No attack & 0 & 1.412 & 91.83 & 1.530 & 90.08 & 1.412\\
Query only & 2 & 1.419 & 91.13 & 1.915 & 89.10 & 1.572\\
Both phases & 2 & 1.419 & 92.02 & 2.106 & 90.08 & 1.710\\
Calibration only & 2 & 1.419 & 92.68 & 1.592 & 90.90 & 1.464\\
True, then false & 2 & 1.419 & 90.11 & 1.457 & 83.92 & 1.230\\
True, then false & 3$^\dagger$ & 1.424 & 87.07 & 1.352 & 80.41 & 1.145\\
True, then false & 4$^\dagger$ & 1.429 & 83.25 & 1.256 & 76.59 & 1.059\\
True, then false & 5$^\dagger$ & 1.433 & 79.19 & 1.160 & 71.59 & 0.973\\
\midrule
\multicolumn{7}{@{}l}{\emph{OpenBookQA, design 2, false claim in the shard; Qwen2.5-7B, 16 bits}}\\
No attack & 0 & 1.412 & 91.83 & 1.530 & 90.08 & 1.412\\
Query only & 2 & 1.419 & 91.13 & 1.980 & 89.08 & 1.595\\
Both phases & 2 & 1.419 & 92.03 & 2.197 & 90.08 & 1.748\\
Calibration only & 2 & 1.419 & 92.70 & 1.594 & 90.91 & 1.465\\
True, then false & 2 & 1.419 & 90.11 & 1.458 & 83.80 & 1.230\\
True, then false & 3$^\dagger$ & 1.424 & 86.99 & 1.353 & 80.20 & 1.144\\
True, then false & 4$^\dagger$ & 1.429 & 83.03 & 1.255 & 76.23 & 1.056\\
True, then false & 5$^\dagger$ & 1.433 & 78.83 & 1.159 & 70.87 & 0.967\\
\midrule
\multicolumn{7}{@{}l}{\emph{OpenBookQA, design 2, flooded shard; Qwen2.5-7B, 16 bits}}\\
No attack & 0 & 1.412 & 91.83 & 1.530 & 90.08 & 1.412\\
Query only & 2 & 1.419 & 91.08 & 2.040 & 89.02 & 1.614\\
Both phases & 2 & 1.419 & 92.04 & 2.290 & 90.09 & 1.797\\
Calibration only & 2 & 1.419 & 92.74 & 1.597 & 90.96 & 1.468\\
True, then false & 2 & 1.419 & 90.09 & 1.458 & 83.78 & 1.230\\
True, then false & 3$^\dagger$ & 1.424 & 86.93 & 1.352 & 80.12 & 1.145\\
True, then false & 4$^\dagger$ & 1.429 & 82.96 & 1.254 & 75.98 & 1.057\\
True, then false & 5$^\dagger$ & 1.433 & 78.63 & 1.158 & 70.50 & 0.967\\
\midrule
\multicolumn{7}{@{}l}{\emph{OpenBookQA, design 2, false answer; Qwen2.5-14B, 8 bits}}\\
No attack & 0 & 1.113 & 91.85 & 1.179 & 90.15 & 1.114\\
Query only & 2 & 1.115 & 91.08 & 1.214 & 89.09 & 1.115\\
Both phases & 2 & 1.115 & 92.06 & 1.297 & 90.16 & 1.160\\
Calibration only & 2 & 1.115 & 92.75 & 1.218 & 91.04 & 1.147\\
True, then false & 2 & 1.115 & 90.28 & 1.136 & 86.47 & 1.027\\
True, then false & 3$^\dagger$ & 1.116 & 88.46 & 1.090 & 83.82 & 0.980\\
True, then false & 4$^\dagger$ & 1.117 & 86.06 & 1.041 & 80.32 & 0.929\\
True, then false & 5$^\dagger$ & 1.118 & 82.79 & 0.989 & 75.37 & 0.863\\
\midrule
\multicolumn{7}{@{}l}{\emph{OpenBookQA, design 2, false claim; Qwen2.5-14B, 8 bits}}\\
No attack & 0 & 1.113 & 91.85 & 1.179 & 90.15 & 1.114\\
Query only & 2 & 1.115 & 91.07 & 1.216 & 89.07 & 1.115\\
Both phases & 2 & 1.115 & 92.07 & 1.304 & 90.17 & 1.161\\
Calibration only & 2 & 1.115 & 92.77 & 1.218 & 91.06 & 1.148\\
True, then false & 2 & 1.115 & 90.27 & 1.136 & 86.43 & 1.027\\
True, then false & 3$^\dagger$ & 1.116 & 88.42 & 1.090 & 83.79 & 0.980\\
True, then false & 4$^\dagger$ & 1.117 & 85.99 & 1.041 & 80.30 & 0.929\\
True, then false & 5$^\dagger$ & 1.118 & 82.71 & 0.988 & 75.24 & 0.861\\
\midrule
\multicolumn{7}{@{}l}{\emph{OpenBookQA, design 2, false claim in the shard; Qwen2.5-14B, 8 bits}}\\
No attack & 0 & 1.113 & 91.85 & 1.179 & 90.15 & 1.114\\
Query only & 2 & 1.115 & 91.05 & 1.216 & 89.07 & 1.115\\
Both phases & 2 & 1.115 & 92.05 & 1.307 & 90.16 & 1.161\\
Calibration only & 2 & 1.115 & 92.78 & 1.219 & 91.06 & 1.148\\
True, then false & 2 & 1.115 & 90.27 & 1.136 & 86.45 & 1.027\\
True, then false & 3$^\dagger$ & 1.116 & 88.44 & 1.090 & 83.79 & 0.980\\
True, then false & 4$^\dagger$ & 1.117 & 86.03 & 1.041 & 80.23 & 0.928\\
True, then false & 5$^\dagger$ & 1.118 & 82.71 & 0.988 & 75.07 & 0.860\\
\midrule
\multicolumn{7}{@{}l}{\emph{OpenBookQA, design 2, flooded shard; Qwen2.5-14B, 8 bits}}\\
No attack & 0 & 1.113 & 91.85 & 1.179 & 90.15 & 1.114\\
Query only & 2 & 1.115 & 91.05 & 1.217 & 89.05 & 1.115\\
Both phases & 2 & 1.115 & 92.06 & 1.311 & 90.16 & 1.162\\
Calibration only & 2 & 1.115 & 92.78 & 1.219 & 91.07 & 1.148\\
True, then false & 2 & 1.115 & 90.26 & 1.136 & 86.40 & 1.026\\
True, then false & 3$^\dagger$ & 1.116 & 88.39 & 1.089 & 83.69 & 0.979\\
True, then false & 4$^\dagger$ & 1.117 & 85.91 & 1.040 & 80.06 & 0.927\\
True, then false & 5$^\dagger$ & 1.118 & 82.51 & 0.987 & 74.79 & 0.857\\
\midrule
\multicolumn{7}{@{}l}{\emph{OpenBookQA, design 2, false answer; Qwen2.5-14B, 16 bits}}\\
No attack & 0 & 1.113 & 91.75 & 1.175 & 90.15 & 1.113\\
Query only & 2 & 1.115 & 90.97 & 1.204 & 89.09 & 1.115\\
Both phases & 2 & 1.115 & 91.94 & 1.279 & 90.16 & 1.160\\
Calibration only & 2 & 1.115 & 92.64 & 1.212 & 91.04 & 1.147\\
True, then false & 2 & 1.115 & 90.18 & 1.132 & 86.47 & 1.027\\
True, then false & 3$^\dagger$ & 1.116 & 88.34 & 1.086 & 83.81 & 0.980\\
True, then false & 4$^\dagger$ & 1.117 & 85.91 & 1.036 & 80.32 & 0.929\\
True, then false & 5$^\dagger$ & 1.118 & 82.60 & 0.984 & 75.36 & 0.863\\
\midrule
\multicolumn{7}{@{}l}{\emph{OpenBookQA, design 2, false claim; Qwen2.5-14B, 16 bits}}\\
No attack & 0 & 1.113 & 91.75 & 1.175 & 90.15 & 1.113\\
Query only & 2 & 1.115 & 90.96 & 1.205 & 89.07 & 1.115\\
Both phases & 2 & 1.115 & 91.95 & 1.283 & 90.16 & 1.161\\
Calibration only & 2 & 1.115 & 92.66 & 1.213 & 91.06 & 1.148\\
True, then false & 2 & 1.115 & 90.17 & 1.132 & 86.43 & 1.027\\
True, then false & 3$^\dagger$ & 1.116 & 88.31 & 1.086 & 83.78 & 0.980\\
True, then false & 4$^\dagger$ & 1.117 & 85.85 & 1.036 & 80.30 & 0.929\\
True, then false & 5$^\dagger$ & 1.118 & 82.53 & 0.983 & 75.23 & 0.861\\
\midrule
\multicolumn{7}{@{}l}{\emph{OpenBookQA, design 2, false claim in the shard; Qwen2.5-14B, 16 bits}}\\
No attack & 0 & 1.113 & 91.75 & 1.175 & 90.15 & 1.113\\
Query only & 2 & 1.115 & 90.94 & 1.204 & 89.07 & 1.115\\
Both phases & 2 & 1.115 & 91.93 & 1.284 & 90.16 & 1.161\\
Calibration only & 2 & 1.115 & 92.67 & 1.214 & 91.06 & 1.148\\
True, then false & 2 & 1.115 & 90.17 & 1.132 & 86.45 & 1.027\\
True, then false & 3$^\dagger$ & 1.116 & 88.32 & 1.086 & 83.78 & 0.980\\
True, then false & 4$^\dagger$ & 1.117 & 85.88 & 1.036 & 80.23 & 0.928\\
True, then false & 5$^\dagger$ & 1.118 & 82.52 & 0.983 & 75.06 & 0.859\\
\midrule
\multicolumn{7}{@{}l}{\emph{OpenBookQA, design 2, flooded shard; Qwen2.5-14B, 16 bits}}\\
No attack & 0 & 1.113 & 91.75 & 1.175 & 90.15 & 1.113\\
Query only & 2 & 1.115 & 90.94 & 1.204 & 89.05 & 1.114\\
Both phases & 2 & 1.115 & 91.95 & 1.288 & 90.16 & 1.162\\
Calibration only & 2 & 1.115 & 92.68 & 1.214 & 91.07 & 1.148\\
True, then false & 2 & 1.115 & 90.16 & 1.132 & 86.40 & 1.026\\
True, then false & 3$^\dagger$ & 1.116 & 88.27 & 1.085 & 83.69 & 0.979\\
True, then false & 4$^\dagger$ & 1.117 & 85.77 & 1.035 & 80.05 & 0.926\\
True, then false & 5$^\dagger$ & 1.118 & 82.32 & 0.981 & 74.77 & 0.857\\
\end{longtable}
\par}%
\endgroup
\endgroup{}

\subsection{Changing identities}
\label{app:revision-identities}

When the corrupt nodes change from question to question, deletion covers at least what the
per-question oracle covers in every row, as \cref{thm:rotating} requires, while fixed-set and
joint-threshold lose coverage when the reports switch direction between phases, as
\cref{prop:two-node} allows. \Cref{tab:changing-identities-panels} redraws the two corrupt
nodes for every question, on every scored panel that has the analysis;
\cref{tab:changing-identities-llama} gives the same scenarios on the submitted Llama panel.

\begingroup
\let\resulttablefloat\table
\renewcommand{\table}[1][]{\resulttablefloat[H]}%
\begingroup\scriptsize
\rev{%
\setlength{\tabcolsep}{2.6pt}
\begin{longtable}{@{}lrrrrrrrr@{}}
\caption{When the two corrupt nodes are redrawn on every question, deletion's mean coverage is at least 89.83\% in every scenario on every panel, while low--high reporting takes that of fixed-set and joint-threshold, which assume a fixed honest group, to 70.51--89.10\%. Coverage (\%) and mean size, declared budget $A=2$, nominal coverage 90\%; means over 100 calibration splits (500 on OpenBookQA), with standard errors of at most 0.3 points of coverage; two nodes per question send the lowest or the highest report in the phases named. The oracle of a redrawn pair is the mean of the nodes honest on that question. These scenarios read only the clean nodes' scores, so panels with the same clean nodes give the same rows and are shown once. The scenarios that inject nodes or swap the pairs between the phases are recorded on every panel and not shown here; \cref{tab:changing-identities-llama} gives them on the submission's scores. Llama-3.2-3B on OpenBookQA: the Llama-3.2-3B panel rescored for the revision, on the splits all scorers share. Joint: joint-threshold.}\label{tab:changing-identities-panels}\\
\toprule
 & \multicolumn{2}{c}{Oracle} & \multicolumn{2}{c}{Fixed-set} & \multicolumn{2}{c}{Joint} & \multicolumn{2}{c}{Deletion}\\
\cmidrule(lr){2-3}\cmidrule(lr){4-5}\cmidrule(lr){6-7}\cmidrule(l){8-9}
Scenario & Coverage & Size & Coverage & Size & Coverage & Size & Coverage & Size\\
\midrule
\endfirsthead
\toprule
 & \multicolumn{2}{c}{Oracle} & \multicolumn{2}{c}{Fixed-set} & \multicolumn{2}{c}{Joint} & \multicolumn{2}{c}{Deletion}\\
\cmidrule(lr){2-3}\cmidrule(lr){4-5}\cmidrule(lr){6-7}\cmidrule(l){8-9}
Scenario & Coverage & Size & Coverage & Size & Coverage & Size & Coverage & Size\\
\midrule
\endhead
\bottomrule
\endfoot
\multicolumn{9}{@{}l}{\emph{OpenBookQA; Llama-3.2-3B, 8 bits; 11 panels with these clean nodes}}\\
Same pair, low--high & 90.15 & 1.721 & 90.41 & 1.742 & 90.41 & 1.742 & 90.41 & 1.742\\
Redrawn pair, high in both & 90.16 & 1.723 & 92.09 & 1.875 & 92.65 & 1.920 & 94.89 & 2.102\\
Redrawn pair, low--high & 90.16 & 1.723 & 84.10 & 1.339 & 84.70 & 1.370 & 90.43 & 1.744\\
Redrawn pair, high--low & 90.16 & 1.723 & 99.99 & 3.985 & 99.99 & 3.987 & 100.00 & 4.000\\
\midrule
\multicolumn{9}{@{}l}{\emph{MedQA-US, hash partition; Llama-3.2-3B, 8 bits; 2 panels with these clean nodes}}\\
Same pair, low--high & 90.26 & 2.625 & 91.09 & 2.685 & 91.09 & 2.685 & 91.09 & 2.685\\
Redrawn pair, high in both & 90.22 & 2.624 & 93.00 & 2.826 & 93.53 & 2.874 & 94.40 & 2.975\\
Redrawn pair, low--high & 90.22 & 2.624 & 73.29 & 1.703 & 74.38 & 1.749 & 91.02 & 2.683\\
Redrawn pair, high--low & 90.22 & 2.624 & 100.00 & 4.000 & 100.00 & 4.000 & 100.00 & 4.000\\
\midrule
\multicolumn{9}{@{}l}{\emph{MedQA-US, topic partition; Llama-3.2-3B, 8 bits}}\\
Same pair, low--high & 90.27 & 2.599 & 90.96 & 2.653 & 90.96 & 2.653 & 90.96 & 2.653\\
Redrawn pair, high in both & 90.20 & 2.597 & 92.66 & 2.791 & 93.18 & 2.841 & 94.55 & 2.988\\
Redrawn pair, low--high & 90.20 & 2.597 & 73.62 & 1.694 & 74.65 & 1.738 & 90.91 & 2.651\\
Redrawn pair, high--low & 90.20 & 2.597 & 100.00 & 4.000 & 100.00 & 4.000 & 100.00 & 4.000\\
\midrule
\multicolumn{9}{@{}l}{\emph{MedMCQA, hash partition; Llama-3.2-3B, 8 bits}}\\
Same pair, low--high & 89.97 & 2.734 & 90.66 & 2.778 & 90.66 & 2.778 & 90.66 & 2.778\\
Redrawn pair, high in both & 89.94 & 2.732 & 92.62 & 2.918 & 93.13 & 2.958 & 94.72 & 3.109\\
Redrawn pair, low--high & 89.94 & 2.732 & 70.93 & 1.773 & 72.50 & 1.840 & 90.63 & 2.777\\
Redrawn pair, high--low & 89.94 & 2.732 & 100.00 & 4.000 & 100.00 & 4.000 & 100.00 & 4.000\\
\midrule
\multicolumn{9}{@{}l}{\emph{MedMCQA, topic partition; Llama-3.2-3B, 8 bits}}\\
Same pair, low--high & 90.03 & 2.794 & 90.71 & 2.841 & 90.71 & 2.841 & 90.71 & 2.841\\
Redrawn pair, high in both & 90.07 & 2.796 & 92.84 & 2.990 & 93.38 & 3.029 & 95.14 & 3.182\\
Redrawn pair, low--high & 90.07 & 2.796 & 70.51 & 1.791 & 72.00 & 1.856 & 90.74 & 2.843\\
Redrawn pair, high--low & 90.07 & 2.796 & 100.00 & 4.000 & 100.00 & 4.000 & 100.00 & 4.000\\
\midrule
\multicolumn{9}{@{}l}{\emph{MMLU-Med, hash partition; Llama-3.2-3B, 8 bits; 2 panels with these clean nodes}}\\
Same pair, low--high & 89.89 & 1.948 & 90.34 & 1.980 & 90.34 & 1.980 & 90.34 & 1.980\\
Redrawn pair, high in both & 89.92 & 1.949 & 92.55 & 2.135 & 93.15 & 2.184 & 94.70 & 2.325\\
Redrawn pair, low--high & 89.92 & 1.949 & 80.41 & 1.425 & 81.27 & 1.469 & 90.35 & 1.979\\
Redrawn pair, high--low & 89.92 & 1.949 & 100.00 & 4.000 & 100.00 & 4.000 & 100.00 & 4.000\\
\midrule
\multicolumn{9}{@{}l}{\emph{MMLU-Med, topic partition; Llama-3.2-3B, 8 bits}}\\
Same pair, low--high & 89.73 & 1.995 & 90.22 & 2.025 & 90.22 & 2.025 & 90.22 & 2.025\\
Redrawn pair, high in both & 89.77 & 1.995 & 92.56 & 2.181 & 93.23 & 2.232 & 94.99 & 2.387\\
Redrawn pair, low--high & 89.77 & 1.995 & 80.39 & 1.448 & 81.05 & 1.490 & 90.19 & 2.024\\
Redrawn pair, high--low & 89.77 & 1.995 & 100.00 & 4.000 & 100.00 & 4.000 & 100.00 & 4.000\\
\midrule
\multicolumn{9}{@{}l}{\emph{OpenBookQA; Qwen2.5-7B, 8 bits; 4 panels with these clean nodes}}\\
Same pair, low--high & 90.09 & 1.419 & 90.27 & 1.430 & 90.27 & 1.430 & 90.27 & 1.430\\
Redrawn pair, high in both & 90.06 & 1.420 & 92.14 & 1.555 & 92.24 & 1.564 & 96.00 & 1.939\\
Redrawn pair, low--high & 90.06 & 1.420 & 86.78 & 1.264 & 87.04 & 1.273 & 90.26 & 1.430\\
Redrawn pair, high--low & 90.06 & 1.420 & 99.77 & 3.872 & 99.79 & 3.887 & 100.00 & 4.000\\
\midrule
\multicolumn{9}{@{}l}{\emph{OpenBookQA; Qwen2.5-7B, 16 bits; 4 panels with these clean nodes}}\\
Same pair, low--high & 90.09 & 1.419 & 90.09 & 1.419 & 90.09 & 1.419 & 90.09 & 1.419\\
Redrawn pair, high in both & 90.06 & 1.420 & 91.92 & 1.539 & 92.03 & 1.548 & 95.80 & 1.904\\
Redrawn pair, low--high & 90.06 & 1.420 & 86.55 & 1.256 & 86.83 & 1.265 & 90.06 & 1.420\\
Redrawn pair, high--low & 90.06 & 1.420 & 99.65 & 3.799 & 99.70 & 3.829 & 100.00 & 4.000\\
\midrule
\multicolumn{9}{@{}l}{\emph{MedQA-US, hash partition; Qwen2.5-7B, 8 bits}}\\
Same pair, low--high & 89.80 & 2.549 & 99.56 & 3.950 & 99.56 & 3.950 & 99.56 & 3.950\\
Redrawn pair, high in both & 89.90 & 2.555 & 100.00 & 4.000 & 100.00 & 4.000 & 100.00 & 4.000\\
Redrawn pair, low--high & 89.90 & 2.555 & 71.12 & 1.500 & 71.99 & 1.546 & 99.40 & 3.926\\
Redrawn pair, high--low & 89.90 & 2.555 & 100.00 & 4.000 & 100.00 & 4.000 & 100.00 & 4.000\\
\midrule
\multicolumn{9}{@{}l}{\emph{MedQA-US, hash partition; Qwen2.5-7B, 16 bits}}\\
Same pair, low--high & 89.80 & 2.549 & 89.83 & 2.552 & 89.83 & 2.552 & 89.83 & 2.552\\
Redrawn pair, high in both & 89.90 & 2.555 & 91.18 & 2.679 & 91.26 & 2.688 & 93.99 & 2.962\\
Redrawn pair, low--high & 89.90 & 2.555 & 70.69 & 1.485 & 71.79 & 1.536 & 89.94 & 2.558\\
Redrawn pair, high--low & 89.90 & 2.555 & 100.00 & 4.000 & 100.00 & 4.000 & 100.00 & 4.000\\
\midrule
\multicolumn{9}{@{}l}{\emph{MMLU-Med, hash partition; Qwen2.5-7B, 8 bits}}\\
Same pair, low--high & 89.96 & 1.742 & 90.25 & 1.766 & 90.25 & 1.766 & 90.25 & 1.766\\
Redrawn pair, high in both & 89.99 & 1.740 & 91.89 & 1.918 & 91.98 & 1.926 & 97.73 & 3.236\\
Redrawn pair, low--high & 89.99 & 1.740 & 84.54 & 1.382 & 84.70 & 1.390 & 90.29 & 1.765\\
Redrawn pair, high--low & 89.99 & 1.740 & 100.00 & 4.000 & 100.00 & 4.000 & 100.00 & 4.000\\
\midrule
\multicolumn{9}{@{}l}{\emph{MMLU-Med, hash partition; Qwen2.5-7B, 16 bits}}\\
Same pair, low--high & 89.96 & 1.742 & 89.96 & 1.742 & 89.96 & 1.742 & 89.96 & 1.742\\
Redrawn pair, high in both & 89.99 & 1.740 & 91.45 & 1.862 & 91.56 & 1.869 & 94.56 & 2.154\\
Redrawn pair, low--high & 89.99 & 1.740 & 84.35 & 1.372 & 84.52 & 1.381 & 89.99 & 1.740\\
Redrawn pair, high--low & 89.99 & 1.740 & 100.00 & 4.000 & 100.00 & 4.000 & 100.00 & 4.000\\
\midrule
\multicolumn{9}{@{}l}{\emph{OpenBookQA; Qwen2.5-14B, 8 bits; 4 panels with these clean nodes}}\\
Same pair, low--high & 90.16 & 1.115 & 90.26 & 1.118 & 90.26 & 1.118 & 90.26 & 1.118\\
Redrawn pair, high in both & 90.14 & 1.114 & 91.96 & 1.186 & 92.07 & 1.192 & 93.78 & 1.289\\
Redrawn pair, low--high & 90.14 & 1.114 & 88.90 & 1.069 & 89.10 & 1.076 & 90.25 & 1.118\\
Redrawn pair, high--low & 90.14 & 1.114 & 96.02 & 1.558 & 96.18 & 1.582 & 98.20 & 2.650\\
\midrule
\multicolumn{9}{@{}l}{\emph{OpenBookQA; Qwen2.5-14B, 16 bits; 4 panels with these clean nodes}}\\
Same pair, low--high & 90.16 & 1.115 & 90.16 & 1.115 & 90.16 & 1.115 & 90.16 & 1.115\\
Redrawn pair, high in both & 90.14 & 1.114 & 91.84 & 1.180 & 91.96 & 1.187 & 93.65 & 1.281\\
Redrawn pair, low--high & 90.14 & 1.114 & 88.79 & 1.065 & 89.01 & 1.072 & 90.14 & 1.114\\
Redrawn pair, high--low & 90.14 & 1.114 & 95.86 & 1.502 & 96.03 & 1.524 & 98.08 & 2.571\\
\end{longtable}
\par}%
\endgroup
\endgroup{}
\begingroup
\let\resulttablefloat\table
\renewcommand{\table}[1][]{\resulttablefloat[H]}%
\begin{table}[tbp]
\rev{%
\centering\scriptsize
\setlength{\tabcolsep}{2.6pt}
\caption{When corrupt nodes change from question to question, deletion's mean coverage is at least 90.38\% in every scenario, while that of fixed-set and joint-threshold, which assume a fixed honest group, can fall below the target (to 84.10\%). OpenBookQA, the submission's Llama-3.2-3B scores with the splits of the first analysis; declared budget $A=2$, nominal coverage 90\%; means over 500 calibration splits, with standard errors of at most 0.1 points of coverage. Upper block: coverage (\%); middle: mean size; lower: the share of replicates (\%) in which a proposed rule drops an answer the oracle keeps. $^\dagger$More corrupt nodes than the declared budget: four on a question, or two at calibration and two others at the query. Fixed: fixed-set; Joint: joint-threshold; Del.: deletion; Merger: $p$-value merger; Mean: all-node mean; Trim.: unguarded trimmed mean.}
\label{tab:changing-identities-llama}
\begin{tabular}{@{}lrrrrrrrr@{}}
\toprule
Scenario & Oracle & Fixed & Joint & Del. & Merger & Mean & Median & Trim.\\
\midrule
\multicolumn{9}{@{}l}{\emph{Coverage (\%)}}\\
\midrule
Same pair, low--high & 90.17 & 90.43 & 90.43 & 90.43 & 97.19 & 82.05 & 87.32 & 86.18\\
Injected, 1 per question & 90.13 & 91.97 & 92.36 & 95.17 & 98.29 & 90.08 & 90.10 & 90.11\\
Injected, 2 per question & 90.12 & 92.24 & 92.56 & 95.79 & 98.42 & 90.03 & 90.12 & 90.08\\
Injected, 4 per question$^\dagger$ & 90.15 & 92.33 & 92.65 & 94.23 & 98.65 & 90.03 & 90.06 & 90.05\\
Redrawn pair, low--high & 90.12 & 84.10 & 84.70 & 90.38 & 96.69 & 81.99 & 87.32 & 86.12\\
Redrawn pair, high--low & 90.12 & 99.99 & 99.99 & 100.00 & 100.00 & 99.55 & 92.70 & 94.05\\
Redrawn pair, high in both & 90.12 & 92.12 & 92.70 & 94.85 & 100.00 & 90.12 & 90.20 & 90.14\\
Swapped pairs, low then high$^\dagger$ & 90.19 & 89.05 & 90.45 & 90.45 & 92.49 & 82.05 & 87.32 & 86.15\\
\midrule
\multicolumn{9}{@{}l}{\emph{Mean set size}}\\
\midrule
Same pair, low--high & 1.725 & 1.746 & 1.746 & 1.746 & 2.630 & 1.245 & 1.615 & 1.488\\
Injected, 1 per question & 1.721 & 1.801 & 1.838 & 2.117 & 2.872 & 1.655 & 1.776 & 1.724\\
Injected, 2 per question & 1.722 & 1.810 & 1.851 & 2.296 & 2.928 & 1.605 & 1.769 & 1.698\\
Injected, 4 per question$^\dagger$ & 1.724 & 1.942 & 1.976 & 2.220 & 3.008 & 1.698 & 1.739 & 1.676\\
Redrawn pair, low--high & 1.722 & 1.339 & 1.370 & 1.743 & 2.542 & 1.243 & 1.615 & 1.486\\
Redrawn pair, high--low & 1.722 & 3.990 & 3.992 & 4.000 & 4.000 & 3.678 & 1.971 & 2.033\\
Redrawn pair, high in both & 1.722 & 1.877 & 1.923 & 2.103 & 4.000 & 1.721 & 1.804 & 1.755\\
Swapped pairs, low then high$^\dagger$ & 1.727 & 1.639 & 1.747 & 1.747 & 2.151 & 1.244 & 1.615 & 1.488\\
\midrule
\multicolumn{9}{@{}l}{\emph{Replicates dropping an answer the oracle keeps (\%)}}\\
\midrule
Same pair, low--high &  & 0.0 & 0.0 & 0.0 &  &  &  & \\
Injected, 1 per question &  & 70.8 & 32.4 & 0.0 &  &  &  & \\
Injected, 2 per question &  & 99.4 & 94.0 & 0.0 &  &  &  & \\
Injected, 4 per question$^\dagger$ &  & 100.0 & 100.0 & 86.0 &  &  &  & \\
Redrawn pair, low--high &  & 100.0 & 100.0 & 0.0 &  &  &  & \\
Redrawn pair, high--low &  & 0.0 & 0.0 & 0.0 &  &  &  & \\
Redrawn pair, high in both &  & 0.0 & 0.0 & 0.0 &  &  &  & \\
Swapped pairs, low then high$^\dagger$ &  & 99.8 & 94.4 & 94.4 &  &  &  & \\
\bottomrule
\end{tabular}
\par}%
\end{table}
{\footnotesize\noindent\emph{Notes to \cref{tab:changing-identities-llama,tab:targeted-subgroup,tab:integrity-ledger,tab:worst-case-inflation,tab:enumeration-cost}.} These tables use the submission's Llama-3.2-3B scores with the splits of the first analysis, and the other tables the Llama-3.2-3B panel rescored for the revision, on the splits all scorers share; the two differ by at most 0.014 answers in the cells they share.\par}
\endgroup{}

\subsection{Federation size, targeted attacks, and the integrity ledger}
\label{app:revision-scale}

Within the budget, corrupt nodes can move the errors of an unguarded mean but not those of
fixed-set. \Cref{tab:federation-scale} varies the number of nodes on the medical scores.
\Cref{tab:targeted-subject} lets two nodes misreport only on the questions of one subject of
MMLU-Med, and \cref{tab:targeted-subgroup} only on a random share of the questions: the
all-node mean keeps its average coverage while losing it on the targeted or the untouched
questions, and fixed-set never falls below the oracle on either group, as
\cref{prop:integrity-transfers} requires. \Cref{tab:integrity-ledger} counts the replicates
in which a rule drops an answer the oracle keeps.

\begingroup
\let\resulttablefloat\table
\renewcommand{\table}[1][]{\resulttablefloat[H]}%
\begin{table}[tbp]
\rev{%
\centering\scriptsize
\setlength{\tabcolsep}{2.6pt}
\caption{Fixed-set is smaller than deletion under stable maximum in every federation, and under low--high reporting its mean coverage is 90.22--91.09\% while the all-node mean's falls to 65.11--83.20\%. Medical benchmarks scored by Llama-3.2-3B; as many corrupt nodes as the declared budget $A$, nominal coverage 90\%; means over 100 calibration splits, with standard errors of at most 0.3 points of coverage. Hash and Topic are the sixteen nodes of one partition; Both is the two partitions pooled; Hash, 8 is eight hash nodes drawn per replicate. Stable max.: stable maximum; Low--high cov.: coverage under low--high reporting (\%); Fixed: fixed-set; Del.: deletion; Merger: $p$-value merger; Mean: all-node mean.}
\label{tab:federation-scale}
\begin{tabular}{@{}llrrrrrrrrrr@{}}
\toprule
 & & & & \multicolumn{2}{c}{No attack} & \multicolumn{2}{c}{Stable max.} & \multicolumn{2}{c}{High--low} & \multicolumn{2}{c}{Low--high cov.}\\
\cmidrule(lr){5-6}\cmidrule(lr){7-8}\cmidrule(lr){9-10}\cmidrule(l){11-12}
Benchmark & Federation & $K$ & $A$ & Oracle & Fixed & Fixed & Del. & Fixed & Merger & Fixed & Mean\\
\midrule
MedQA-US & Hash & 16 & 2 & 2.633 & 2.786 & 2.802 & 2.979 & 4.000 & 3.483 & 91.09 & 70.09\\
MedQA-US & Topic & 16 & 2 & 2.595 & 2.753 & 2.770 & 2.984 & 4.000 & 3.483 & 90.96 & 70.41\\
MedQA-US & Both & 32 & 2 & 2.610 & 2.729 & 2.735 & 2.843 & 3.985 & 3.681 & 90.84 & 79.72\\
MedQA-US & Both & 32 & 4 & 2.610 & 2.778 & 2.798 & 2.990 & 4.000 & 3.521 & 90.88 & 70.38\\
MedQA-US & Hash, 8 & 8 & 1 & 2.608 & 2.760 & 2.779 & 2.956 & 4.000 & 3.377 & 91.04 & 69.97\\
MedMCQA & Hash & 16 & 2 & 2.733 & 2.882 & 2.905 & 3.114 & 4.000 & 3.581 & 90.66 & 65.30\\
MedMCQA & Topic & 16 & 2 & 2.791 & 2.953 & 2.973 & 3.183 & 4.000 & 3.599 & 90.71 & 65.19\\
MedMCQA & Both & 32 & 2 & 2.748 & 2.861 & 2.868 & 2.995 & 3.790 & 3.616 & 90.63 & 77.48\\
MedMCQA & Both & 32 & 4 & 2.748 & 2.914 & 2.935 & 3.152 & 4.000 & 3.640 & 90.68 & 65.11\\
MedMCQA & Hash, 8 & 8 & 1 & 2.742 & 2.888 & 2.907 & 3.101 & 4.000 & 3.483 & 90.63 & 65.61\\
MMLU-Med & Hash & 16 & 2 & 1.944 & 2.098 & 2.114 & 2.322 & 4.000 & 2.866 & 90.34 & 77.76\\
MMLU-Med & Topic & 16 & 2 & 1.997 & 2.147 & 2.159 & 2.389 & 4.000 & 2.892 & 90.22 & 77.88\\
MMLU-Med & Both & 32 & 2 & 1.976 & 2.091 & 2.097 & 2.229 & 2.885 & 3.100 & 90.38 & 83.20\\
MMLU-Med & Both & 32 & 4 & 1.976 & 2.156 & 2.173 & 2.390 & 4.000 & 3.051 & 90.41 & 78.03\\
MMLU-Med & Hash, 8 & 8 & 1 & 1.964 & 2.093 & 2.109 & 2.306 & 4.000 & 2.770 & 90.40 & 77.55\\
\bottomrule
\end{tabular}
\par}%
\end{table}
\endgroup{}
\begingroup
\let\resulttablefloat\table
\renewcommand{\table}[1][]{\resulttablefloat[H]}%
\begin{table}[tbp]
\rev{%
\centering\scriptsize
\setlength{\tabcolsep}{2pt}
\caption{Two nodes that misreport only on one subject's questions never take fixed-set's mean coverage on that subject below the oracle's (smallest margin +0.29 points), while under All high the all-node mean's falls up to 2.23 points below it. MMLU-Med, each subject targeted in turn, declared budget $A=2$, nominal coverage 90\%; means over 100 calibration splits, with standard errors of at most 0.6 points of coverage on a subject. All high: the corrupt nodes report the largest score; Invert: they reverse their ranking, in both phases. Ranges and extremes are over the subjects; differences are in percentage points. Losing and Missed are the shares (\%) of the subject's questions that lose an answer the oracle keeps or newly miss the correct answer. Fixed: fixed-set; Part.: partition, H hash and T topic.}
\label{tab:targeted-subject}
\begin{tabular}{@{}lllrrrrrrrrr@{}}
\toprule
 & & & Oracle & Fixed$-$oracle & \multicolumn{2}{c}{Oracle$-$mean, largest} & \multicolumn{2}{c}{Mean, on the subject (\%)} & Fixed & \multicolumn{2}{c}{Size on the subject}\\
\cmidrule(lr){6-7}\cmidrule(lr){8-9}\cmidrule(l){11-12}
Scorer & Part. & Bits & cov.\ (\%) & smallest & Subject & Rest & Losing & Missed & missed & Oracle & Fixed\\
\midrule
\multicolumn{12}{@{}l}{\emph{All high}}\\
Llama-3.2-3B & H & 8 & 86.57--95.55 & +0.37 & +2.23 & $-$0.08 & 8.0--15.2 & 0.6--2.2 & 0.0 & 1.774--2.116 & 1.817--2.185\\
Llama-3.2-3B & T & 8 & 85.41--94.66 & +0.29 & +2.17 & $-$0.04 & 7.3--12.1 & 0.7--2.2 & 0.0 & 1.797--2.167 & 1.843--2.236\\
Qwen2.5-7B & H & 8 & 83.94--98.16 & +0.52 & +0.65 & $-$0.12 & 2.7--4.4 & 0.2--0.6 & 0.0 & 1.628--1.926 & 1.730--2.052\\
Qwen2.5-7B & H & 16 & 83.94--98.16 & +0.39 & +0.65 & $-$0.11 & 2.7--4.5 & 0.3--0.6 & 0.0 & 1.628--1.926 & 1.699--2.013\\
Qwen2.5-14B & H & 8 & 84.91--95.91 & +0.79 & +0.80 & $-$0.02 & 1.8--4.7 & 0.4--0.8 & 0.0 & 1.274--1.569 & 1.325--1.676\\
Qwen2.5-14B & H & 16 & 84.91--95.91 & +0.58 & +0.83 & $-$0.01 & 1.8--4.8 & 0.4--0.8 & 0.0 & 1.274--1.569 & 1.310--1.653\\
\midrule
\multicolumn{12}{@{}l}{\emph{Invert}}\\
Llama-3.2-3B & H & 8 & 86.57--95.55 & +4.43 & $-$4.40 & +2.90 & 0.0--0.0 & 0.0--0.0 & 0.0 & 1.774--2.116 & 3.742--3.987\\
Llama-3.2-3B & T & 8 & 85.41--94.66 & +5.27 & $-$5.09 & +2.58 & 0.0--0.0 & 0.0--0.0 & 0.0 & 1.797--2.167 & 3.828--3.995\\
Qwen2.5-7B & H & 8 & 83.94--98.16 & +1.82 & $-$1.82 & +3.16 & 0.0--0.5 & 0.0--0.1 & 0.0 & 1.628--1.926 & 3.994--3.999\\
Qwen2.5-7B & H & 16 & 83.94--98.16 & +1.82 & $-$1.82 & +3.17 & 0.0--0.5 & 0.0--0.1 & 0.0 & 1.628--1.926 & 3.885--3.999\\
Qwen2.5-14B & H & 8 & 84.91--95.91 & +3.88 & $-$3.44 & +1.78 & 0.0--0.2 & 0.0--0.1 & 0.0 & 1.274--1.569 & 3.722--3.915\\
Qwen2.5-14B & H & 16 & 84.91--95.91 & +3.81 & $-$3.44 & +1.80 & 0.0--0.2 & 0.0--0.1 & 0.0 & 1.274--1.569 & 3.489--3.845\\
\bottomrule
\end{tabular}
\par}%
\end{table}
\endgroup{}
\begingroup
\let\resulttablefloat\table
\renewcommand{\table}[1][]{\resulttablefloat[H]}%
\begingroup\scriptsize
\rev{%
\setlength{\tabcolsep}{2.6pt}
\begin{longtable}{@{}rlrrrrrrrrr@{}}
\caption{Nodes that misreport on a random share of the questions leave the mean coverage of every proposed rule on the targeted questions at least 90.72\%, while that of the unguarded rules falls to 88.06\% there. OpenBookQA, the submission's Llama-3.2-3B scores with the splits of the first analysis; nominal coverage 90\%; as many corrupt nodes as the declared budget; means over 500 calibration splits, with standard errors of at most 0.2 points of coverage on the targeted questions. Coverage and its parts in percent; Losing and Missed are shares (\%) of the targeted questions that lose an answer the oracle keeps or newly miss the correct answer; Pairs is the mean number of oracle pairs lost per replicate; sizes are mean answers. Joint: joint-threshold; Merger: $p$-value merger.}\label{tab:targeted-subgroup}\\
\toprule
 & & \multicolumn{3}{c}{Coverage (\%)} & \multicolumn{2}{c}{Targeted (\%)} & & \multicolumn{3}{c}{Size}\\
\cmidrule(lr){3-5}\cmidrule(lr){6-7}\cmidrule(l){9-11}
Share & Rule & All & Target & Rest & Losing & Missed & Pairs & All & Target & Rest\\
\midrule
\endfirsthead
\toprule
 & & \multicolumn{3}{c}{Coverage (\%)} & \multicolumn{2}{c}{Targeted (\%)} & & \multicolumn{3}{c}{Size}\\
\cmidrule(lr){3-5}\cmidrule(lr){6-7}\cmidrule(l){9-11}
Share & Rule & All & Target & Rest & Losing & Missed & Pairs & All & Target & Rest\\
\midrule
\endhead
\bottomrule
\endfoot
\multicolumn{11}{@{}l}{\emph{Nobody misreports, 2 nodes labeled corrupt}}\\
0.1 & Oracle & 90.10 & 90.09 & 90.10 & 0.0 & 0.0 & 0.00 & 1.718 & 1.718 & 1.718\\
0.1 & Fixed-set & 91.70 & 91.69 & 91.70 & 0.0 & 0.0 & 0.00 & 1.846 & 1.849 & 1.845\\
0.1 & Joint & 92.07 & 92.10 & 92.06 & 0.0 & 0.0 & 0.00 & 1.876 & 1.879 & 1.876\\
0.1 & Deletion & 93.96 & 94.00 & 93.96 & 0.0 & 0.0 & 0.00 & 2.022 & 2.026 & 2.022\\
0.1 & Merger & 98.13 & 98.22 & 98.12 & 0.5 & 0.1 & 3.02 & 2.804 & 2.805 & 2.804\\
0.1 & All-node mean & 90.08 & 90.16 & 90.07 & 2.3 & 0.3 & 16.81 & 1.714 & 1.717 & 1.714\\
0.1 & Median & 90.13 & 90.21 & 90.12 & 5.2 & 1.1 & 34.96 & 1.791 & 1.792 & 1.791\\
0.1 & Trimmed mean & 90.12 & 90.19 & 90.10 & 3.5 & 0.6 & 23.88 & 1.746 & 1.748 & 1.745\\
0.1 & Winsorized & 90.12 & 90.24 & 90.10 & 3.5 & 0.4 & 23.63 & 1.728 & 1.731 & 1.728\\
0.2 & Oracle & 90.10 & 90.17 & 90.08 & 0.0 & 0.0 & 0.00 & 1.718 & 1.720 & 1.717\\
0.2 & Fixed-set & 91.70 & 91.69 & 91.70 & 0.0 & 0.0 & 0.00 & 1.846 & 1.849 & 1.845\\
0.2 & Joint & 92.07 & 92.07 & 92.07 & 0.0 & 0.0 & 0.00 & 1.876 & 1.880 & 1.876\\
0.2 & Deletion & 93.96 & 93.99 & 93.96 & 0.0 & 0.0 & 0.00 & 2.022 & 2.026 & 2.021\\
0.2 & Merger & 98.13 & 98.17 & 98.12 & 0.4 & 0.1 & 3.02 & 2.804 & 2.802 & 2.804\\
0.2 & All-node mean & 90.08 & 90.14 & 90.07 & 2.5 & 0.3 & 16.81 & 1.714 & 1.716 & 1.714\\
0.2 & Median & 90.13 & 90.15 & 90.13 & 5.1 & 1.2 & 34.96 & 1.791 & 1.795 & 1.790\\
0.2 & Trimmed mean & 90.12 & 90.12 & 90.11 & 3.6 & 0.6 & 23.88 & 1.746 & 1.748 & 1.745\\
0.2 & Winsorized & 90.12 & 90.13 & 90.12 & 3.5 & 0.6 & 23.63 & 1.728 & 1.731 & 1.728\\
\midrule
\multicolumn{11}{@{}l}{\emph{Nobody misreports, 3 nodes labeled corrupt}}\\
0.1 & Oracle & 90.08 & 90.09 & 90.08 & 0.0 & 0.0 & 0.00 & 1.719 & 1.721 & 1.719\\
0.1 & Fixed-set & 92.30 & 92.30 & 92.29 & 0.0 & 0.0 & 0.00 & 1.891 & 1.893 & 1.891\\
0.1 & Joint & 92.86 & 92.89 & 92.85 & 0.0 & 0.0 & 0.00 & 1.936 & 1.938 & 1.935\\
0.1 & Deletion & 95.43 & 95.51 & 95.41 & 0.0 & 0.0 & 0.00 & 2.151 & 2.154 & 2.151\\
0.1 & Merger & 98.03 & 98.09 & 98.02 & 0.3 & 0.0 & 1.83 & 2.733 & 2.735 & 2.733\\
0.1 & All-node mean & 90.08 & 90.16 & 90.07 & 3.1 & 0.4 & 21.30 & 1.714 & 1.717 & 1.714\\
0.1 & Median & 90.13 & 90.21 & 90.12 & 5.4 & 1.1 & 36.17 & 1.791 & 1.792 & 1.791\\
0.1 & Trimmed mean & 90.13 & 90.21 & 90.12 & 4.3 & 0.8 & 28.06 & 1.764 & 1.766 & 1.764\\
0.1 & Winsorized & 90.12 & 90.18 & 90.10 & 3.8 & 0.7 & 25.02 & 1.761 & 1.763 & 1.761\\
0.2 & Oracle & 90.08 & 90.13 & 90.07 & 0.0 & 0.0 & 0.00 & 1.719 & 1.720 & 1.719\\
0.2 & Fixed-set & 92.30 & 92.29 & 92.30 & 0.0 & 0.0 & 0.00 & 1.891 & 1.895 & 1.890\\
0.2 & Joint & 92.86 & 92.82 & 92.87 & 0.0 & 0.0 & 0.00 & 1.936 & 1.939 & 1.935\\
0.2 & Deletion & 95.43 & 95.49 & 95.41 & 0.0 & 0.0 & 0.00 & 2.151 & 2.155 & 2.150\\
0.2 & Merger & 98.03 & 98.05 & 98.02 & 0.3 & 0.0 & 1.83 & 2.733 & 2.731 & 2.733\\
0.2 & All-node mean & 90.08 & 90.14 & 90.07 & 3.1 & 0.4 & 21.30 & 1.714 & 1.716 & 1.714\\
0.2 & Median & 90.13 & 90.15 & 90.13 & 5.1 & 1.2 & 36.17 & 1.791 & 1.795 & 1.790\\
0.2 & Trimmed mean & 90.13 & 90.16 & 90.13 & 4.0 & 0.8 & 28.06 & 1.764 & 1.768 & 1.764\\
0.2 & Winsorized & 90.12 & 90.10 & 90.12 & 3.7 & 0.8 & 25.02 & 1.761 & 1.763 & 1.760\\
\midrule
\multicolumn{11}{@{}l}{\emph{All high, 2 corrupt nodes}}\\
0.1 & Oracle & 90.10 & 90.09 & 90.10 & 0.0 & 0.0 & 0.00 & 1.718 & 1.718 & 1.718\\
0.1 & Fixed-set & 91.74 & 90.72 & 91.85 & 0.0 & 0.0 & 0.00 & 1.849 & 1.773 & 1.857\\
0.1 & Joint & 92.14 & 91.32 & 92.22 & 0.0 & 0.0 & 0.00 & 1.881 & 1.820 & 1.888\\
0.1 & Deletion & 94.05 & 93.27 & 94.14 & 0.0 & 0.0 & 0.00 & 2.030 & 1.969 & 2.037\\
0.1 & Merger & 98.30 & 98.38 & 98.29 & 0.3 & 0.0 & 1.90 & 2.856 & 2.857 & 2.856\\
0.1 & All-node mean & 90.07 & 89.00 & 90.19 & 8.0 & 1.1 & 17.84 & 1.715 & 1.632 & 1.724\\
0.1 & Median & 90.14 & 88.95 & 90.27 & 8.0 & 2.0 & 35.85 & 1.792 & 1.720 & 1.800\\
0.1 & Trimmed mean & 90.10 & 89.06 & 90.22 & 7.0 & 1.3 & 24.72 & 1.746 & 1.671 & 1.755\\
0.1 & Winsorized & 90.12 & 89.02 & 90.24 & 7.3 & 1.2 & 24.51 & 1.729 & 1.653 & 1.738\\
0.2 & Oracle & 90.10 & 90.17 & 90.08 & 0.0 & 0.0 & 0.00 & 1.718 & 1.720 & 1.717\\
0.2 & Fixed-set & 91.77 & 90.84 & 92.01 & 0.0 & 0.0 & 0.00 & 1.850 & 1.779 & 1.868\\
0.2 & Joint & 92.20 & 91.46 & 92.39 & 0.0 & 0.0 & 0.00 & 1.886 & 1.832 & 1.899\\
0.2 & Deletion & 94.14 & 93.40 & 94.32 & 0.0 & 0.0 & 0.00 & 2.038 & 1.982 & 2.051\\
0.2 & Merger & 98.30 & 98.34 & 98.30 & 0.3 & 0.0 & 1.90 & 2.856 & 2.853 & 2.857\\
0.2 & All-node mean & 90.10 & 89.21 & 90.32 & 7.0 & 1.0 & 18.16 & 1.717 & 1.645 & 1.735\\
0.2 & Median & 90.15 & 89.12 & 90.41 & 7.7 & 1.9 & 36.50 & 1.793 & 1.732 & 1.808\\
0.2 & Trimmed mean & 90.13 & 89.20 & 90.36 & 6.5 & 1.2 & 25.30 & 1.747 & 1.681 & 1.763\\
0.2 & Winsorized & 90.12 & 89.15 & 90.37 & 6.8 & 1.2 & 24.88 & 1.730 & 1.663 & 1.747\\
\midrule
\multicolumn{11}{@{}l}{\emph{All high, 3 corrupt nodes}}\\
0.1 & Oracle & 90.08 & 90.09 & 90.08 & 0.0 & 0.0 & 0.00 & 1.719 & 1.721 & 1.719\\
0.1 & Fixed-set & 92.39 & 90.88 & 92.56 & 0.0 & 0.0 & 0.00 & 1.897 & 1.783 & 1.910\\
0.1 & Joint & 92.98 & 91.70 & 93.12 & 0.0 & 0.0 & 0.00 & 1.945 & 1.851 & 1.956\\
0.1 & Deletion & 95.54 & 94.55 & 95.65 & 0.0 & 0.0 & 0.00 & 2.164 & 2.071 & 2.174\\
0.1 & Merger & 98.31 & 98.37 & 98.30 & 0.1 & 0.0 & 0.79 & 2.839 & 2.838 & 2.839\\
0.1 & All-node mean & 90.09 & 88.40 & 90.28 & 12.3 & 1.7 & 23.21 & 1.716 & 1.587 & 1.731\\
0.1 & Median & 90.11 & 88.06 & 90.34 & 10.7 & 2.7 & 38.33 & 1.792 & 1.672 & 1.806\\
0.1 & Trimmed mean & 90.12 & 88.36 & 90.31 & 10.0 & 2.0 & 30.02 & 1.765 & 1.642 & 1.779\\
0.1 & Winsorized & 90.11 & 88.50 & 90.28 & 9.8 & 1.8 & 27.12 & 1.762 & 1.635 & 1.776\\
0.2 & Oracle & 90.08 & 90.13 & 90.07 & 0.0 & 0.0 & 0.00 & 1.719 & 1.720 & 1.719\\
0.2 & Fixed-set & 92.44 & 90.96 & 92.81 & 0.0 & 0.0 & 0.00 & 1.900 & 1.790 & 1.927\\
0.2 & Joint & 93.09 & 91.92 & 93.38 & 0.0 & 0.0 & 0.00 & 1.953 & 1.869 & 1.974\\
0.2 & Deletion & 95.64 & 94.78 & 95.85 & 0.0 & 0.0 & 0.00 & 2.176 & 2.092 & 2.197\\
0.2 & Merger & 98.31 & 98.33 & 98.31 & 0.1 & 0.0 & 0.79 & 2.839 & 2.836 & 2.840\\
0.2 & All-node mean & 90.09 & 88.63 & 90.45 & 10.6 & 1.5 & 24.91 & 1.718 & 1.604 & 1.746\\
0.2 & Median & 90.17 & 88.48 & 90.59 & 9.6 & 2.4 & 39.27 & 1.795 & 1.694 & 1.820\\
0.2 & Trimmed mean & 90.09 & 88.56 & 90.47 & 8.9 & 1.9 & 31.24 & 1.766 & 1.659 & 1.793\\
0.2 & Winsorized & 90.12 & 88.71 & 90.48 & 8.6 & 1.7 & 28.21 & 1.764 & 1.652 & 1.792\\
\midrule
\multicolumn{11}{@{}l}{\emph{Invert, 2 corrupt nodes}}\\
0.1 & Oracle & 90.10 & 90.09 & 90.10 & 0.0 & 0.0 & 0.00 & 1.718 & 1.718 & 1.718\\
0.1 & Fixed-set & 91.92 & 99.43 & 91.07 & 0.0 & 0.0 & 0.00 & 1.984 & 3.656 & 1.796\\
0.1 & Joint & 92.32 & 99.68 & 91.49 & 0.0 & 0.0 & 0.00 & 2.031 & 3.805 & 1.832\\
0.1 & Deletion & 94.32 & 99.99 & 93.68 & 0.0 & 0.0 & 0.00 & 2.201 & 3.990 & 2.000\\
0.1 & Merger & 98.24 & 98.37 & 98.22 & 0.3 & 0.0 & 2.45 & 2.835 & 2.857 & 2.833\\
0.1 & All-node mean & 90.12 & 97.52 & 89.29 & 0.0 & 0.0 & 44.46 & 1.768 & 2.798 & 1.652\\
0.1 & Median & 90.14 & 91.31 & 90.00 & 2.7 & 0.6 & 34.59 & 1.790 & 1.865 & 1.782\\
0.1 & Trimmed mean & 90.10 & 92.33 & 89.84 & 0.0 & 0.0 & 25.93 & 1.743 & 1.903 & 1.725\\
0.1 & Winsorized & 90.10 & 93.13 & 89.76 & 0.0 & 0.0 & 28.89 & 1.726 & 1.944 & 1.701\\
0.2 & Oracle & 90.10 & 90.17 & 90.08 & 0.0 & 0.0 & 0.00 & 1.718 & 1.720 & 1.717\\
0.2 & Fixed-set & 92.36 & 98.77 & 90.76 & 0.0 & 0.0 & 0.00 & 2.063 & 3.243 & 1.769\\
0.2 & Joint & 92.94 & 99.59 & 91.28 & 0.0 & 0.0 & 0.00 & 2.200 & 3.746 & 1.815\\
0.2 & Deletion & 94.74 & 99.97 & 93.44 & 0.0 & 0.0 & 0.00 & 2.382 & 3.989 & 1.981\\
0.2 & Merger & 98.28 & 98.34 & 98.26 & 0.3 & 0.0 & 2.22 & 2.847 & 2.853 & 2.845\\
0.2 & All-node mean & 90.13 & 96.49 & 88.54 & 0.0 & 0.0 & 66.44 & 1.775 & 2.486 & 1.597\\
0.2 & Median & 90.14 & 91.18 & 89.89 & 2.8 & 0.6 & 34.23 & 1.789 & 1.857 & 1.772\\
0.2 & Trimmed mean & 90.10 & 92.06 & 89.61 & 0.1 & 0.0 & 27.23 & 1.742 & 1.885 & 1.706\\
0.2 & Winsorized & 90.12 & 92.79 & 89.46 & 0.1 & 0.0 & 32.93 & 1.725 & 1.918 & 1.677\\
\midrule
\multicolumn{11}{@{}l}{\emph{Invert, 3 corrupt nodes}}\\
0.1 & Oracle & 90.08 & 90.09 & 90.08 & 0.0 & 0.0 & 0.00 & 1.719 & 1.721 & 1.719\\
0.1 & Fixed-set & 92.32 & 99.99 & 91.46 & 0.0 & 0.0 & 0.00 & 2.045 & 4.000 & 1.826\\
0.1 & Joint & 92.87 & 100.00 & 92.07 & 0.0 & 0.0 & 0.00 & 2.091 & 4.000 & 1.877\\
0.1 & Deletion & 95.55 & 100.00 & 95.05 & 0.0 & 0.0 & 0.00 & 2.306 & 4.000 & 2.116\\
0.1 & Merger & 98.19 & 98.37 & 98.17 & 0.1 & 0.0 & 1.41 & 2.788 & 2.838 & 2.783\\
0.1 & All-node mean & 90.10 & 99.75 & 89.02 & 0.0 & 0.0 & 56.28 & 1.863 & 3.897 & 1.635\\
0.1 & Median & 90.14 & 91.86 & 89.94 & 1.9 & 0.4 & 35.59 & 1.791 & 1.910 & 1.778\\
0.1 & Trimmed mean & 90.11 & 93.44 & 89.74 & 0.0 & 0.0 & 31.29 & 1.759 & 1.989 & 1.733\\
0.1 & Winsorized & 90.11 & 94.43 & 89.63 & 0.0 & 0.0 & 30.77 & 1.756 & 2.060 & 1.721\\
0.2 & Oracle & 90.08 & 90.13 & 90.07 & 0.0 & 0.0 & 0.00 & 1.719 & 1.720 & 1.719\\
0.2 & Fixed-set & 92.75 & 99.99 & 90.95 & 0.0 & 0.0 & 0.00 & 2.226 & 3.999 & 1.784\\
0.2 & Joint & 93.40 & 100.00 & 91.76 & 0.0 & 0.0 & 0.00 & 2.280 & 4.000 & 1.851\\
0.2 & Deletion & 95.79 & 100.00 & 94.73 & 0.0 & 0.0 & 0.00 & 2.469 & 4.000 & 2.087\\
0.2 & Merger & 98.26 & 98.33 & 98.25 & 0.1 & 0.0 & 1.23 & 2.815 & 2.836 & 2.810\\
0.2 & All-node mean & 90.13 & 99.34 & 87.84 & 0.0 & 0.0 & 91.22 & 1.965 & 3.624 & 1.551\\
0.2 & Median & 90.14 & 91.68 & 89.76 & 1.9 & 0.5 & 35.56 & 1.790 & 1.898 & 1.763\\
0.2 & Trimmed mean & 90.13 & 93.02 & 89.40 & 0.1 & 0.0 & 34.24 & 1.755 & 1.956 & 1.705\\
0.2 & Winsorized & 90.14 & 94.04 & 89.17 & 0.0 & 0.0 & 36.64 & 1.753 & 2.022 & 1.685\\
\end{longtable}
\par}%
\endgroup
\endgroup{}
\begingroup
\let\resulttablefloat\table
\renewcommand{\table}[1][]{\resulttablefloat[H]}%
\begin{table}[tbp]
\rev{%
\centering\scriptsize
\setlength{\tabcolsep}{2.6pt}
\caption{Fixed-set, joint-threshold, deletion, and the audited budget never drop an answer the honest-mean oracle keeps on a replicate within the budget, while the merger does on 409--434 of 464 such replicates and trust pruning on 308. OpenBookQA judges, the submission's Llama-3.2-3B scores with the splits of the first analysis, with misled judges drawn at random, declared budget $A=2$, nominal coverage 90\%, over the replicates whose number of misled judges is within the budget. Losing counts the replicates on which a rule drops at least one (question, answer) pair the oracle keeps; Pairs is the mean number of such pairs per replicate. Joint: joint-threshold; Merger: $p$-value merger.}
\label{tab:integrity-ledger}
\begin{tabular}{@{}lrrrrr@{}}
\toprule
Rule & Replicates & Losing & Pairs & Coverage & Size\\
\midrule
\multicolumn{6}{@{}l}{\emph{Both phases}}\\
Oracle & 464 & 0 & 0.00 & 90.10 & 1.718\\
Fixed-set & 464 & 0 & 0.00 & 92.43 & 1.862\\
Joint & 464 & 0 & 0.00 & 93.29 & 1.951\\
Deletion & 464 & 0 & 0.00 & 95.16 & 2.158\\
Merger & 464 & 420 & 2.32 & 98.12 & 2.816\\
All-node mean & 464 & 405 & 59.04 & 90.08 & 1.654\\
V1, trust pruning & 464 & 308 & 4.28 & 90.39 & 1.740\\
V2, audited budget & 464 & 0 & 0.00 & 93.03 & 1.919\\
\midrule
\multicolumn{6}{@{}l}{\emph{Calibration only}}\\
Oracle & 464 & 0 & 0.00 & 90.10 & 1.718\\
Fixed-set & 464 & 0 & 0.00 & 91.15 & 1.795\\
Joint & 464 & 0 & 0.00 & 91.62 & 1.841\\
Deletion & 464 & 0 & 0.00 & 93.73 & 2.004\\
Merger & 464 & 409 & 2.17 & 98.18 & 2.828\\
All-node mean & 464 & 405 & 59.73 & 88.76 & 1.615\\
V1, trust pruning & 464 & 308 & 4.28 & 90.39 & 1.740\\
V2, audited budget & 464 & 0 & 0.00 & 91.72 & 1.840\\
\midrule
\multicolumn{6}{@{}l}{\emph{Query only}}\\
Oracle & 464 & 0 & 0.00 & 90.10 & 1.718\\
Fixed-set & 464 & 0 & 0.00 & 93.39 & 1.966\\
Joint & 464 & 0 & 0.00 & 93.74 & 2.002\\
Deletion & 464 & 0 & 0.00 & 95.34 & 2.184\\
Merger & 464 & 434 & 2.98 & 97.97 & 2.786\\
All-node mean & 464 & 405 & 16.38 & 91.58 & 1.788\\
V1, trust pruning & 464 & 308 & 4.28 & 90.39 & 1.740\\
V2, audited budget & 464 & 0 & 0.00 & 93.96 & 2.029\\
\bottomrule
\end{tabular}
\par}%
\end{table}
\endgroup{}

\subsection{A candidate list that may lack the correct answer}
\label{app:revision-none}

A ``none of these'' entry keeps the guarantee when the correct answer may be missing from the
list. One option of every OpenBookQA question is replaced by ``None of the other answers'' (the
correct option on 250 of the 1{,}000 questions); each node retrieves again for the list it
is shown, and the Llama scorer scores it (\cref{tab:none-option}). The guarantee of
\cref{prop:generated-lists} is marginal, not per group.

\begingroup
\let\resulttablefloat\table
\renewcommand{\table}[1][]{\resulttablefloat[H]}%
\begin{table}[tbp]
\rev{%
\centering\scriptsize
\setlength{\tabcolsep}{2.6pt}
\caption{When the correct answer is replaced by ``None of the other answers'', every proposed rule's mean coverage of those questions, how often it keeps that option, is at least 90.80\%, and no proposed rule drops an answer the oracle keeps. OpenBookQA under Llama-3.2-3B, one option replaced by ``None''; declared budget $A=2$, nominal coverage 90\%; means over 500 calibration splits, with standard errors of at most 0.2 points of coverage. Questions is the number of evaluation questions of the group per replicate, a mean over the replicates. Joint: joint-threshold; Merger: $p$-value merger.}
\label{tab:none-option}
\begin{tabular}{@{}llrrrrr@{}}
\toprule
Rule & Correct answer & Questions & Coverage & Size & ``None'' kept & Pairs lost\\
\midrule
\multicolumn{7}{@{}l}{\emph{No attack}}\\
Oracle & replaced & 166.5 & 87.77 & 2.678 & 87.8 & 0.00\\
Oracle & present & 500.5 & 91.16 & 2.272 & 65.9 & 0.00\\
Fixed-set & replaced & 166.5 & 90.80 & 2.835 & 90.8 & 0.00\\
Fixed-set & present & 500.5 & 92.86 & 2.421 & 69.9 & 0.00\\
Joint & replaced & 166.5 & 91.48 & 2.867 & 91.5 & 0.00\\
Joint & present & 500.5 & 93.15 & 2.450 & 70.7 & 0.00\\
Deletion & replaced & 166.5 & 93.67 & 3.012 & 93.7 & 0.00\\
Deletion & present & 500.5 & 94.25 & 2.583 & 74.1 & 0.00\\
Merger & replaced & 166.5 & 97.28 & 3.589 & 97.3 & 0.76\\
Merger & present & 500.5 & 98.18 & 3.291 & 87.4 & 0.43\\
\midrule
\multicolumn{7}{@{}l}{\emph{Stable maximum, $a=2$}}\\
Oracle & replaced & 166.5 & 87.88 & 2.684 & 87.9 & 0.00\\
Oracle & present & 500.5 & 91.11 & 2.276 & 65.9 & 0.00\\
Fixed-set & replaced & 166.5 & 91.21 & 2.859 & 91.2 & 0.00\\
Fixed-set & present & 500.5 & 93.01 & 2.440 & 70.5 & 0.00\\
Joint & replaced & 166.5 & 92.68 & 2.944 & 92.7 & 0.00\\
Joint & present & 500.5 & 93.64 & 2.516 & 72.4 & 0.00\\
Deletion & replaced & 166.5 & 94.50 & 3.081 & 94.5 & 0.00\\
Deletion & present & 500.5 & 94.81 & 2.658 & 75.7 & 0.00\\
Merger & replaced & 166.5 & 97.69 & 3.622 & 97.7 & 0.57\\
Merger & present & 500.5 & 98.40 & 3.330 & 88.2 & 0.59\\
\bottomrule
\end{tabular}
\par}%
\end{table}
\endgroup{}

\subsection{Worst-case inflation, enumeration cost, and the choice of budget}
\label{app:revision-cost}

Larger federations and narrower report ranges shrink the worst case, and exact enumeration
takes at most 17.3\,ms per query for sixteen or thirty-two nodes failing independently at up
to 5\% each. \Cref{tab:worst-case-inflation} varies the federation size and the admissible
report range under high--low reporting, \cref{tab:enumeration-cost} measures the hub's computation and the
two certificates of \cref{prop:certificates}, and \cref{tab:budget-choice} gives the smallest
declared budget whose binomial tail is at most 1\% for independent node failures. The timings are
medians of five runs (three for the largest cells) of one batch of 667 queries, on one
thread of a laptop processor (x86-64, Python 3.12, NumPy 1.26).

\begingroup
\let\resulttablefloat\table
\renewcommand{\table}[1][]{\resulttablefloat[H]}%
\begin{table}[tbp]
\rev{%
\centering\scriptsize
\setlength{\tabcolsep}{2.6pt}
\caption{Fixed-set is smaller than the merger in every synthetic cell except high--low reporting at $K=16$ with $A=2$ or $A=3$, and clipping reports to $[0.7, 1]$ cuts its high--low size on the Llama-3.2-3B scores from 3.991 to 2.331. Upper table: synthetic scores, federation size $K$, and declared budget $A$ with $a=A$ corrupt nodes; lower table: OpenBookQA, the submission's Llama-3.2-3B scores with the splits of the first analysis, with reports clipped to the range shown, $A=2$; Unclipped is the oracle on the original scores. Nominal coverage 90\%; mean set sizes; fixed-set coverage and full-set rate in percent. Fixed: fixed-set; Joint: joint-threshold; Del.: deletion; Merger: $p$-value merger; Sym.: guarded symmetric trimming; Fixed cov.: fixed-set's coverage (\%); Stable max.: stable maximum.}
\label{tab:worst-case-inflation}
\begin{tabular}{@{}rrlrrrrrrr@{}}
\toprule
$K$ & $A$ & Threat & Oracle & Fixed & Joint & Del. & Merger & Sym. & Fixed cov.\\
\midrule
16 & 2 & None & 0.952 & 1.035 & 1.087 & 1.128 & 1.406 & 2.531 & 93.77\\
16 & 2 & Stable max. & 0.953 & 1.048 & 1.242 & 1.297 & 1.492 & 2.649 & 94.10\\
16 & 2 & High--low & 0.953 & 1.833 & 1.992 & 2.124 & 1.492 & 3.214 & 99.71\\
16 & 2 & Low--high & 0.953 & 0.963 & 0.963 & 0.963 & 1.225 & 1.941 & 90.42\\
32 & 2 & None & 0.945 & 0.993 & 1.018 & 1.035 & 1.314 & 1.361 & 92.43\\
32 & 2 & Stable max. & 0.945 & 0.996 & 1.064 & 1.082 & 1.340 & 1.369 & 92.54\\
32 & 2 & High--low & 0.945 & 1.193 & 1.233 & 1.262 & 1.340 & 1.575 & 97.57\\
32 & 2 & Low--high & 0.945 & 0.953 & 0.953 & 0.953 & 1.168 & 1.211 & 90.36\\
64 & 2 & None & 0.943 & 0.975 & 0.988 & 0.993 & 1.338 & 1.105 & 91.76\\
64 & 2 & Stable max. & 0.942 & 0.975 & 1.006 & 1.012 & 1.349 & 1.106 & 91.79\\
64 & 2 & High--low & 0.942 & 1.053 & 1.068 & 1.075 & 1.349 & 1.167 & 94.92\\
64 & 2 & Low--high & 0.942 & 0.951 & 0.951 & 0.951 & 1.196 & 1.054 & 90.45\\
16 & 3 & None & 0.952 & 1.070 & 1.156 & 1.224 & 1.471 & 3.894 & 94.86\\
16 & 3 & Stable max. & 0.955 & 1.102 & 1.517 & 1.634 & 1.665 & 3.985 & 95.47\\
16 & 3 & High--low & 0.955 & 3.196 & 3.410 & 3.578 & 1.665 & 3.999 & 100.00\\
16 & 3 & Low--high & 0.955 & 0.964 & 0.964 & 0.964 & 1.251 & 3.370 & 90.43\\
32 & 3 & None & 0.945 & 1.011 & 1.047 & 1.074 & 1.405 & 1.802 & 93.21\\
32 & 3 & Stable max. & 0.946 & 1.017 & 1.134 & 1.167 & 1.458 & 1.843 & 93.41\\
32 & 3 & High--low & 0.946 & 1.422 & 1.507 & 1.577 & 1.458 & 2.297 & 99.07\\
32 & 3 & Low--high & 0.946 & 0.954 & 0.954 & 0.954 & 1.218 & 1.462 & 90.39\\
64 & 3 & None & 0.942 & 0.983 & 1.003 & 1.012 & 1.340 & 1.211 & 92.15\\
64 & 3 & Stable max. & 0.942 & 0.985 & 1.033 & 1.044 & 1.358 & 1.214 & 92.25\\
64 & 3 & High--low & 0.942 & 1.116 & 1.141 & 1.155 & 1.358 & 1.335 & 96.44\\
64 & 3 & Low--high & 0.942 & 0.951 & 0.951 & 0.951 & 1.182 & 1.119 & 90.41\\
\bottomrule
\end{tabular}
\par\medskip
\begin{tabular}{@{}llrrrrrrrrr@{}}
\toprule
Range & Threat & Unclipped & Oracle & Fixed & Joint & Del. & Merger & Sym. & Fixed cov. & Full\\
\midrule
$[0, 1]$ & None & 1.713 & 1.713 & 1.842 & 1.873 & 2.019 & 2.797 & 4.000 & 91.62 & 6.0\\
$[0, 1]$ & Stable max. & 1.714 & 1.714 & 1.855 & 1.923 & 2.094 & 2.848 & 4.000 & 91.81 & 6.2\\
$[0, 1]$ & High--low & 1.714 & 1.714 & 3.991 & 3.991 & 4.000 & 2.848 & 4.000 & 99.99 & 99.5\\
$[0, 1]$ & Low--high & 1.714 & 1.714 & 1.735 & 1.735 & 1.735 & 2.628 & 4.000 & 90.26 & 4.0\\
$[0.25, 1]$ & None & 1.713 & 1.716 & 1.839 & 1.869 & 2.010 & 2.778 & 4.000 & 91.55 & 5.9\\
$[0.25, 1]$ & Stable max. & 1.714 & 1.717 & 1.851 & 1.918 & 2.083 & 2.829 & 4.000 & 91.74 & 6.2\\
$[0.25, 1]$ & High--low & 1.714 & 1.717 & 3.724 & 3.738 & 3.973 & 2.829 & 4.000 & 99.68 & 85.9\\
$[0.25, 1]$ & Low--high & 1.714 & 1.717 & 1.733 & 1.733 & 1.733 & 2.610 & 3.994 & 90.20 & 4.0\\
$[0.5, 1]$ & None & 1.713 & 1.724 & 1.842 & 1.872 & 1.994 & 2.759 & 3.950 & 91.53 & 6.0\\
$[0.5, 1]$ & Stable max. & 1.714 & 1.724 & 1.852 & 1.920 & 2.066 & 2.810 & 3.995 & 91.69 & 6.3\\
$[0.5, 1]$ & High--low & 1.714 & 1.724 & 2.714 & 2.756 & 3.263 & 2.810 & 3.997 & 97.86 & 37.3\\
$[0.5, 1]$ & Low--high & 1.714 & 1.724 & 1.736 & 1.736 & 1.736 & 2.590 & 3.448 & 90.13 & 4.1\\
$[0.7, 1]$ & None & 1.713 & 1.737 & 1.859 & 1.890 & 1.979 & 2.745 & 2.812 & 91.60 & 6.7\\
$[0.7, 1]$ & Stable max. & 1.714 & 1.741 & 1.873 & 1.944 & 2.047 & 2.797 & 3.191 & 91.79 & 7.0\\
$[0.7, 1]$ & High--low & 1.714 & 1.741 & 2.331 & 2.370 & 2.550 & 2.797 & 3.358 & 96.27 & 22.2\\
$[0.7, 1]$ & Low--high & 1.714 & 1.741 & 1.749 & 1.749 & 1.749 & 2.577 & 2.275 & 90.12 & 4.6\\
\bottomrule
\end{tabular}
\par}%
\end{table}
\endgroup{}
\begingroup
\let\resulttablefloat\table
\renewcommand{\table}[1][]{\resulttablefloat[H]}%
\begin{table}[tbp]
\rev{%
\centering\scriptsize
\setlength{\tabcolsep}{2.6pt}
\caption{The two certificates settle 94.3--100.0\% of the decisions in every cell without enumeration and never contradict the exact rule; fixed-set's exact path takes 0.00633--161 milliseconds per query from 17 to 679,121 groups. Upper block: groups the exact rule enumerates; middle: decisions settled by the certificates (means per replicate), the Llama-3.2-3B rows on the submission's Llama-3.2-3B scores with the splits of the first analysis; lower: milliseconds per query, calibration included. Fixed: fixed-set; Joint: joint-threshold; Del.: deletion; Stable max.: stable maximum; Cal.\ max.: the maximum at calibration only; Query max.: the maximum at the query only.}
\label{tab:enumeration-cost}
\begin{tabular}{@{}rrrrrr@{}}
\toprule
$K$ & $A=1$ & $A=2$ & $A=3$ & $A=4$ & $A=5$\\
\midrule
16 & 17 & 137 & 697 & 2,517 & 6,885\\
32 & 33 & 529 & 5,489 & 41,449 & 242,825\\
64 & 65 & 2,081 & 43,745 & 679,121 & 8,303,633\\
128 & 129 & 8,257 & 349,633 & 11,017,633 & 275,584,033\\
\bottomrule
\end{tabular}
\par\medskip
\begin{tabular}{@{}lrrlrrrrrrr@{}}
\toprule
Testbed & $K$ & $A$ & Threat & Replicates & Settled (\%) & Certain in & Certain out & Unsettled & Pairs & Contradictions\\
\midrule
Llama-3.2-3B & 16 & 2 & None, $a=0$ & 100 & 98.2 & 1202.7 & 1416.2 & 49.1 & 2668 & 0\\
Llama-3.2-3B & 16 & 3 & None, $a=0$ & 100 & 97.7 & 1227.2 & 1379.6 & 61.2 & 2668 & 0\\
Llama-3.2-3B & 16 & 2 & Stable max., $a=1$ & 100 & 96.8 & 1184.0 & 1397.7 & 86.3 & 2668 & 0\\
Llama-3.2-3B & 16 & 3 & Stable max., $a=1$ & 100 & 96.0 & 1209.5 & 1353.0 & 105.5 & 2668 & 0\\
Llama-3.2-3B & 16 & 2 & Stable max., $a=2$ & 100 & 95.3 & 1163.1 & 1379.7 & 125.2 & 2668 & 0\\
Llama-3.2-3B & 16 & 3 & Stable max., $a=2$ & 100 & 94.3 & 1189.0 & 1326.7 & 152.3 & 2668 & 0\\
Llama-3.2-3B & 16 & 2 & Cal.\ max., $a=2$ & 100 & 97.8 & 1278.5 & 1330.8 & 58.7 & 2668 & 0\\
Llama-3.2-3B & 16 & 3 & Cal.\ max., $a=2$ & 100 & 97.0 & 1304.1 & 1282.7 & 81.2 & 2668 & 0\\
Llama-3.2-3B & 16 & 2 & Query max., $a=2$ & 100 & 98.3 & 1159.8 & 1462.7 & 45.5 & 2668 & 0\\
Llama-3.2-3B & 16 & 3 & Query max., $a=2$ & 100 & 97.8 & 1188.5 & 1422.1 & 57.5 & 2668 & 0\\
Llama-3.2-3B & 16 & 2 & High--low, $a=2$ & 100 & 99.3 & 2649.1 & 0.0 & 18.9 & 2668 & 0\\
Llama-3.2-3B & 16 & 3 & High--low, $a=2$ & 100 & 99.8 & 2662.7 & 0.0 & 5.3 & 2668 & 0\\
Llama-3.2-3B & 16 & 2 & Low--high, $a=2$ & 100 & 99.8 & 1154.4 & 1508.2 & 5.5 & 2668 & 0\\
Llama-3.2-3B & 16 & 3 & Low--high, $a=2$ & 100 & 98.9 & 1187.2 & 1450.4 & 30.4 & 2668 & 0\\
Llama-3.2-3B & 16 & 2 & Low--high, $a=3$ & 100 & 99.9 & 945.3 & 1721.1 & 1.6 & 2668 & 0\\
Llama-3.2-3B & 16 & 3 & Low--high, $a=3$ & 100 & 99.8 & 1149.4 & 1512.4 & 6.2 & 2668 & 0\\
Synthetic & 32 & 2 & None, $a=0$ & 20 & 99.2 & 492.6 & 1491.8 & 15.6 & 2000 & 0\\
Synthetic & 32 & 2 & Stable max., $a=2$ & 20 & 97.3 & 477.6 & 1467.9 & 54.5 & 2000 & 0\\
Synthetic & 32 & 2 & High--low, $a=2$ & 20 & 98.7 & 589.1 & 1385.0 & 25.9 & 2000 & 0\\
Synthetic & 32 & 2 & Low--high, $a=2$ & 20 & 100.0 & 477.3 & 1522.7 & 0.0 & 2000 & 0\\
Synthetic & 32 & 3 & None, $a=0$ & 20 & 98.8 & 499.2 & 1476.5 & 24.2 & 2000 & 0\\
Synthetic & 32 & 3 & Stable max., $a=3$ & 20 & 95.6 & 478.4 & 1433.1 & 88.5 & 2000 & 0\\
Synthetic & 32 & 3 & High--low, $a=3$ & 20 & 97.1 & 694.6 & 1247.0 & 58.4 & 2000 & 0\\
Synthetic & 32 & 3 & Low--high, $a=3$ & 20 & 100.0 & 477.9 & 1522.0 & 0.0 & 2000 & 0\\
Synthetic & 64 & 2 & None, $a=0$ & 20 & 99.5 & 482.6 & 1508.2 & 9.2 & 2000 & 0\\
Synthetic & 64 & 2 & Stable max., $a=2$ & 20 & 98.6 & 473.2 & 1498.3 & 28.4 & 2000 & 0\\
Synthetic & 64 & 2 & High--low, $a=2$ & 20 & 99.5 & 523.1 & 1465.8 & 11.0 & 2000 & 0\\
Synthetic & 64 & 2 & Low--high, $a=2$ & 20 & 100.0 & 473.1 & 1526.8 & 0.0 & 2000 & 0\\
Synthetic & 64 & 3 & None, $a=0$ & 20 & 99.3 & 486.4 & 1500.0 & 13.7 & 2000 & 0\\
Synthetic & 64 & 3 & Stable max., $a=3$ & 20 & 97.8 & 473.1 & 1483.4 & 43.5 & 2000 & 0\\
Synthetic & 64 & 3 & High--low, $a=3$ & 20 & 99.1 & 551.8 & 1429.2 & 18.9 & 2000 & 0\\
Synthetic & 64 & 3 & Low--high, $a=3$ & 20 & 100.0 & 472.9 & 1527.2 & 0.0 & 2000 & 0\\
\bottomrule
\end{tabular}
\par\medskip
\begin{tabular}{@{}rrrrrrr@{}}
\toprule
$K$ & $A$ & Groups & Fixed & Joint & Del. & Certificates\\
\midrule
16 & 1 & 17 & 0.00633 & 0.00552 & 0.00453 & 0.0359\\
16 & 2 & 137 & 0.0127 & 0.00685 & 0.00490 & 0.0639\\
16 & 3 & 697 & 0.0313 & 0.0152 & 0.00451 & 0.0924\\
16 & 4 & 2,517 & 0.107 & 0.0402 & 0.00506 & 0.135\\
16 & 5 & 6,885 & 0.289 & 0.0942 & 0.00585 & 0.221\\
32 & 1 & 33 & 0.0111 & 0.0101 & 0.00938 & 0.0523\\
32 & 2 & 529 & 0.0302 & 0.0192 & 0.00968 & 0.0832\\
32 & 3 & 5,489 & 0.263 & 0.0982 & 0.0107 & 0.177\\
32 & 4 & 41,449 & 2.11 & 0.639 & 0.0105 & 0.758\\
32 & 5 & 242,825 & 17.3 & 3.68 & 0.00811 & 3.79\\
64 & 1 & 65 & 0.0218 & 0.0212 & 0.0203 & 0.0778\\
64 & 2 & 2,081 & 0.129 & 0.0638 & 0.0223 & 0.156\\
64 & 3 & 43,745 & 2.87 & 0.933 & 0.0186 & 1.04\\
64 & 4 & 679,121 & 161 & 14.2 & 0.0212 & 14.5\\
128 & 1 & 129 & 0.0525 & 0.0464 & 0.0481 & 0.152\\
128 & 2 & 8,257 & 0.695 & 0.303 & 0.0460 & 0.448\\
128 & 3 & 349,633 & 72.0 & 11.4 & 0.0437 & 11.6\\
\bottomrule
\end{tabular}
\par}%
\end{table}
\endgroup{}
\begingroup
\let\resulttablefloat\table
\renewcommand{\table}[1][]{\resulttablefloat[H]}%
\begin{table}[tbp]
\rev{%
\centering\footnotesize
\setlength{\tabcolsep}{2.6pt}
\caption{Sixteen nodes failing independently at 1--10\% each need a declared budget of only 2--5 for a 1\% tail; at $K=64$ the budget grows to 12. Smallest declared budget $A$ with $\Pr\{\mathrm{Binomial}(K,p)>A\}\le 0.01$ under independent failures at rate $p$, its tail, and the groups the fixed-set rule enumerates at that budget.}
\label{tab:budget-choice}
\begin{tabular}{@{}lrrrrrrrrr@{}}
\toprule
 & \multicolumn{3}{c}{$K=16$} & \multicolumn{3}{c}{$K=32$} & \multicolumn{3}{c}{$K=64$}\\
\cmidrule(lr){2-4}\cmidrule(lr){5-7}\cmidrule(l){8-10}
Failure rate & $A$ & Tail & Groups & $A$ & Tail & Groups & $A$ & Tail & Groups\\
\midrule
1\% & 2 & 0.0005 & 137 & 2 & 0.0040 & 529 & 3 & 0.0039 & 43,745\\
2\% & 2 & 0.0037 & 137 & 3 & 0.0037 & 5,489 & 4 & 0.0092 & 679,121\\
5\% & 3 & 0.0070 & 697 & 5 & 0.0046 & 242,825 & 8 & 0.0044 & 5,130,659,561\\
10\% & 5 & 0.0033 & 6,885 & 8 & 0.0033 & 15,033,173 & 12 & 0.0099 & 4,211,954,943,769\\
\bottomrule
\end{tabular}
\par}%
\end{table}
\endgroup{}

\subsection{Two tasks where retrieval matters}
\label{app:revision-literature}

On two tasks that need the literature, retrieval matters, and the federation gains accuracy as
well as validity under attack. BioASQ-Y/N (two
options) and PubMedQA* (three options) are the two tasks of the same medical benchmark suite that
need the literature. Sixteen nodes each retrieve four passages for a question from their own
hash shard of PubMed, in the disjoint federation; in the overlapping one every article is held by
$m$ nodes ($m = 1, 2, 4, 8$). The design (pools, retrieval, calibration on one third of each pool,
100 random splits, the anchors) was fixed on the pilot pools before any held-out question was
scored, and every held-out question was analyzed from a single scoring.
\Cref{tab:literature} gives the anchors' accuracy, the accuracy of the nodes whose passages hold
an article of a question's gold evidence, and the sets of the oracle, the three proposed rules,
and the merger;
its rows of the disjoint federation also appear in \cref{tab:consolidated,tab:ladder,tab:matched-baselines,tab:aggregators-two-phase}.

\begingroup
\let\resulttablefloat\table
\renewcommand{\table}[1][]{\resulttablefloat[H]}%
\begin{table}[tbp]
\rev{%
\centering\scriptsize
\setlength{\tabcolsep}{2pt}
\caption{On BioASQ-Y/N and PubMedQA*, top-answer accuracy is 76.0--87.3\% with four passages from the whole corpus and 9.5--53.3\% without retrieval; under stable maximum, fixed-set's mean coverage is 91.77--95.29\% in every federation, and under low--high reporting the all-node mean's is 63.31--83.82\%. BioASQ-Y/N has two options and PubMedQA* three. Upper table: top-answer accuracy (\%) on the disjoint federation; Holding and Other: accuracy on the (question, node) pairs whose node's four passages do and do not hold an article of the question's gold evidence; Share: the share of pairs that do (\%). Middle and lower tables: mean coverage (\%) and mean set size at 90\%, declared budget $A=2$, with nobody corrupt and with two of sixteen nodes sending the largest score in both phases; Low--high: the all-node mean's mean coverage when two nodes report the lowest score at calibration and the highest at the query. Coverages and sizes are means over 100 calibration splits, with standard errors of at most 0.5 points of coverage. Disjoint: sixteen hash shards of the corpus, each node retrieving four passages from its own; $m=1,2,4,8$: every article held by $m$ of the sixteen nodes, each ranking with the statistics of the whole corpus (Llama-3.2-3B only). Bits is the report resolution: Llama-3.2-3B at eight bits and Qwen2.5-14B at the finest resolution analyzed on these tasks (24 bits). No retr.: no retrieval; Corpus: four passages from the whole corpus; One node: a single node, averaged over the sixteen; Mean: sixteen-node mean; Stable max.: stable maximum; Fixed: fixed-set; Joint: joint-threshold; Del.: deletion; Merger: $p$-value merger.}
\label{tab:literature}
\begin{tabular}{@{}llrrrrrrr@{}}
\toprule
 & & \multicolumn{4}{c}{Top-answer accuracy (\%)} & \multicolumn{3}{c}{Gold article (\%)}\\
\cmidrule(lr){3-6}\cmidrule(l){7-9}
Benchmark & Scorer & No retr. & Corpus & One node & Mean & Holding & Other & Share\\
\midrule
BioASQ-Y/N & Llama-3.2-3B & 53.3 & 83.8 & 69.7 & 71.0 & 85.4 & 66.6 & 16.4\\
BioASQ-Y/N & Qwen2.5-14B & 51.4 & 87.3 & 72.1 & 73.9 & 91.1 & 68.4 & 16.4\\
PubMedQA* & Llama-3.2-3B & 36.0 & 76.0 & 51.3 & 52.5 & 75.7 & 49.7 & 6.2\\
PubMedQA* & Qwen2.5-14B & 9.5 & 78.5 & 39.5 & 39.2 & 82.0 & 36.7 & 6.2\\
\bottomrule
\end{tabular}
\par\medskip
\begin{tabular}{@{}llrrrrrrrrrrrr@{}}
\toprule
 & & & \multicolumn{11}{c}{Mean coverage at 90\% (\%)}\\
\cmidrule(l){4-14}
 & & & \multicolumn{5}{c}{Nobody corrupt} & \multicolumn{5}{c}{Stable max.} & Low--high\\
\cmidrule(lr){4-8}\cmidrule(lr){9-13}\cmidrule(l){14-14}
Scorer & Federation & Bits & Oracle & Fixed & Joint & Del. & Merger & Oracle & Fixed & Joint & Del. & Merger & Mean\\
\midrule
\multicolumn{14}{@{}l}{\emph{BioASQ-Y/N, two options}}\\
Llama-3.2-3B & Disjoint & 8 & 89.96 & 92.57 & 92.87 & 95.67 & 100.00 & 89.97 & 92.84 & 93.11 & 96.32 & 100.00 & 82.54\\
Llama-3.2-3B & $m=1$ & 8 & 89.93 & 92.48 & 92.70 & 95.29 & 100.00 & 89.97 & 92.76 & 93.05 & 96.03 & 100.00 & 82.60\\
Llama-3.2-3B & $m=2$ & 8 & 90.01 & 93.18 & 93.48 & 95.86 & 100.00 & 90.10 & 93.39 & 93.77 & 96.29 & 100.00 & 81.68\\
Llama-3.2-3B & $m=4$ & 8 & 90.00 & 94.31 & 94.83 & 96.36 & 99.96 & 90.02 & 94.32 & 94.84 & 96.57 & 99.97 & 80.41\\
Llama-3.2-3B & $m=8$ & 8 & 90.25 & 92.29 & 92.67 & 93.64 & 99.55 & 90.27 & 92.60 & 93.34 & 94.33 & 99.67 & 82.70\\
Qwen2.5-14B & Disjoint & 24 & 89.86 & 92.75 & 93.46 & 98.14 & 100.00 & 89.83 & 93.02 & 94.01 & 98.51 & 100.00 & 82.30\\
\midrule
\multicolumn{14}{@{}l}{\emph{PubMedQA*, three options}}\\
Llama-3.2-3B & Disjoint & 8 & 89.88 & 92.14 & 92.44 & 93.74 & 100.00 & 89.85 & 92.67 & 93.10 & 94.60 & 100.00 & 79.58\\
Llama-3.2-3B & $m=1$ & 8 & 89.88 & 92.21 & 92.45 & 93.84 & 100.00 & 89.87 & 92.67 & 93.11 & 94.70 & 100.00 & 79.69\\
Llama-3.2-3B & $m=2$ & 8 & 90.02 & 91.76 & 92.01 & 92.45 & 100.00 & 89.99 & 92.10 & 92.45 & 93.22 & 100.00 & 82.42\\
Llama-3.2-3B & $m=4$ & 8 & 89.85 & 91.41 & 91.63 & 92.18 & 100.00 & 89.87 & 91.77 & 92.04 & 92.66 & 100.00 & 83.51\\
Llama-3.2-3B & $m=8$ & 8 & 89.74 & 91.66 & 91.82 & 92.16 & 100.00 & 89.69 & 92.03 & 92.22 & 92.52 & 100.00 & 83.82\\
Qwen2.5-14B & Disjoint & 24 & 89.98 & 95.37 & 95.53 & 99.81 & 100.00 & 90.50 & 95.29 & 95.33 & 99.59 & 100.00 & 63.31\\
\bottomrule
\end{tabular}
\par\medskip
\begin{tabular}{@{}llrrrrrrrrrrr@{}}
\toprule
 & & & \multicolumn{10}{c}{Mean set size at 90\%}\\
\cmidrule(l){4-13}
 & & & \multicolumn{5}{c}{Nobody corrupt} & \multicolumn{5}{c}{Stable max.}\\
\cmidrule(lr){4-8}\cmidrule(l){9-13}
Scorer & Federation & Bits & Oracle & Fixed & Joint & Del. & Merger & Oracle & Fixed & Joint & Del. & Merger\\
\midrule
\multicolumn{13}{@{}l}{\emph{BioASQ-Y/N, two options}}\\
Llama-3.2-3B & Disjoint & 8 & 1.412 & 1.503 & 1.515 & 1.627 & 2.000 & 1.417 & 1.512 & 1.525 & 1.658 & 2.000\\
Llama-3.2-3B & $m=1$ & 8 & 1.409 & 1.500 & 1.510 & 1.615 & 2.000 & 1.414 & 1.510 & 1.523 & 1.650 & 2.000\\
Llama-3.2-3B & $m=2$ & 8 & 1.344 & 1.424 & 1.435 & 1.521 & 2.000 & 1.346 & 1.433 & 1.446 & 1.554 & 2.000\\
Llama-3.2-3B & $m=4$ & 8 & 1.265 & 1.371 & 1.385 & 1.444 & 1.997 & 1.271 & 1.374 & 1.390 & 1.463 & 1.998\\
Llama-3.2-3B & $m=8$ & 8 & 1.208 & 1.268 & 1.279 & 1.309 & 1.913 & 1.209 & 1.278 & 1.296 & 1.329 & 1.935\\
Qwen2.5-14B & Disjoint & 24 & 1.364 & 1.446 & 1.465 & 1.629 & 2.000 & 1.366 & 1.457 & 1.483 & 1.728 & 2.000\\
\midrule
\multicolumn{13}{@{}l}{\emph{PubMedQA*, three options}}\\
Llama-3.2-3B & Disjoint & 8 & 2.026 & 2.253 & 2.280 & 2.405 & 3.000 & 2.028 & 2.303 & 2.342 & 2.494 & 3.000\\
Llama-3.2-3B & $m=1$ & 8 & 2.030 & 2.260 & 2.287 & 2.419 & 3.000 & 2.033 & 2.313 & 2.352 & 2.508 & 3.000\\
Llama-3.2-3B & $m=2$ & 8 & 1.985 & 2.156 & 2.181 & 2.217 & 3.000 & 1.983 & 2.192 & 2.224 & 2.293 & 3.000\\
Llama-3.2-3B & $m=4$ & 8 & 1.938 & 2.096 & 2.116 & 2.174 & 3.000 & 1.936 & 2.139 & 2.164 & 2.224 & 3.000\\
Llama-3.2-3B & $m=8$ & 8 & 1.884 & 2.094 & 2.108 & 2.152 & 3.000 & 1.885 & 2.135 & 2.154 & 2.199 & 3.000\\
Qwen2.5-14B & Disjoint & 24 & 2.358 & 2.558 & 2.570 & 2.959 & 3.000 & 2.373 & 2.635 & 2.642 & 2.970 & 3.000\\
\bottomrule
\end{tabular}
\par}%
\end{table}
\endgroup{}

\subsection{Candidate lists written at the hub}
\label{app:revision-generated}

Lists written by a language model keep the guarantee of \cref{prop:generated-lists}; how small
the sets are depends on how often a list holds the answer and on the scorer. The setting of
\cref{prop:generated-lists}, with lists generated by a model: Qwen2.5-14B reads the question alone, without its options, and writes
three candidate answers; ``None of these'' is appended; each list is labeled by a rule fixed in
advance (Judge: the first candidate that the same model judges to mean the benchmark's answer, else
``None of these''; and, as a second system calibrated on its own labels, Strict: the first exact
match after normalization); the sixteen nodes retrieve and score the four options, with Llama-3.2-3B and
with Qwen2.5-14B as scorers. The design, the prompt, the labeling rules, and the analysis were
fixed before any list existed; as the design allowed, the prompt's last sentence was reworded
once after a short engineering run.
\Cref{tab:generated-lists} reports how often a list holds the answer, the scorer's accuracy, and
the sets on OpenBookQA and MedQA-US.

\begingroup
\let\resulttablefloat\table
\renewcommand{\table}[1][]{\resulttablefloat[H]}%
\begin{table}[tbp]
\rev{%
\centering\scriptsize
\setlength{\tabcolsep}{2pt}
\caption{Under each labeling rule, the proposed rules' mean coverage at 90\% is 91.97--100.00\% with nobody corrupt and under stable maximum, and 90.32--100.00\% on the questions whose list lacks the answer, where coverage is how often ``None of these'' is kept; the generated lists hold the answer on 0.1--40.8\% of the questions. Judge and Strict are the labeling rules of \cref{app:revision-generated}; each calibrates and evaluates its own sets with its own labels: two systems, not one scored twice. Upper table: Recall, the share of questions (\%) whose label is a candidate; Hidden, the questions whose answer is or begins with ``all of'' or ``none of'', options a generated list cannot show, then how many of them each rule labels ``None of these'' and the recall without them. Middle table: Acc., top-answer accuracy of the sixteen-node mean against each rule's labels on unrounded scores (\%); then mean coverage (\%) and mean set size at 90\%, declared budget $A=2$, with nobody corrupt and with two of sixteen nodes sending the largest score in both phases. Lower table: Share, the share of a split's evaluation questions whose label is a candidate (\%), then the same coverages within those and within those labeled ``None of these''. Coverages and sizes are means over 100 calibration splits (500 on OpenBookQA), with standard errors of at most 0.4 points of coverage; by group, a mean is over the splits with an evaluation question in the group, as few as 340. Bits is the report resolution: Llama-3.2-3B at eight bits and Qwen2.5-14B at the finest resolution analyzed for these lists (24 bits). Fixed: fixed-set; Joint: joint-threshold; Del.: deletion; Merger: $p$-value merger; Cov.: mean coverage (\%).}
\label{tab:generated-lists}
\begin{tabular}{@{}llrrrrr@{}}
\toprule
Benchmark & Labels & Questions & Recall (\%) & Hidden & Labeled ``None'' & Recall without (\%)\\
\midrule
OpenBookQA & Judge & 1000 & 27.5 & 3 & 3 & 27.6\\
OpenBookQA & Strict & 1000 & 0.1 & 3 & 3 & 0.1\\
MedQA-US & Judge & 1173 & 40.8 & 0 & -- & 40.8\\
MedQA-US & Strict & 1173 & 3.2 & 0 & -- & 3.2\\
\bottomrule
\end{tabular}
\par\medskip
\begin{tabular}{@{}llrlrrrrrrrrrrr@{}}
\toprule
 & & & & & \multicolumn{2}{c}{Oracle} & \multicolumn{2}{c}{Fixed} & \multicolumn{2}{c}{Joint} & \multicolumn{2}{c}{Del.} & \multicolumn{2}{c}{Merger}\\
\cmidrule(lr){6-7}\cmidrule(lr){8-9}\cmidrule(lr){10-11}\cmidrule(lr){12-13}\cmidrule(l){14-15}
Benchmark & Scorer & Bits & Labels & Acc. & Cov. & Size & Cov. & Size & Cov. & Size & Cov. & Size & Cov. & Size\\
\midrule
\multicolumn{15}{@{}l}{\emph{Nobody corrupt}}\\
OpenBookQA & Llama-3.2-3B & 8 & Judge & 17.1 & 90.09 & 3.770 & 96.19 & 3.894 & 96.50 & 3.900 & 97.74 & 3.933 & 99.96 & 3.999\\
OpenBookQA & Llama-3.2-3B & 8 & Strict & 4.1 & 90.07 & 3.848 & 96.49 & 3.947 & 96.68 & 3.950 & 97.61 & 3.963 & 100.00 & 4.000\\
OpenBookQA & Qwen2.5-14B & 24 & Judge & 47.8 & 90.25 & 2.644 & 92.26 & 2.777 & 92.30 & 2.780 & 96.38 & 3.198 & 100.00 & 4.000\\
OpenBookQA & Qwen2.5-14B & 24 & Strict & 44.6 & 90.16 & 2.883 & 91.97 & 2.992 & 92.03 & 2.995 & 95.53 & 3.362 & 100.00 & 4.000\\
MedQA-US & Llama-3.2-3B & 8 & Judge & 23.5 & 90.15 & 3.709 & 94.26 & 3.834 & 94.66 & 3.843 & 96.10 & 3.878 & 99.22 & 3.974\\
MedQA-US & Llama-3.2-3B & 8 & Strict & 2.6 & 90.24 & 3.869 & 96.72 & 3.959 & 96.78 & 3.960 & 97.90 & 3.975 & 100.00 & 4.000\\
MedQA-US & Qwen2.5-14B & 24 & Judge & 32.2 & 90.23 & 3.486 & 100.00 & 4.000 & 100.00 & 4.000 & 100.00 & 4.000 & 100.00 & 4.000\\
MedQA-US & Qwen2.5-14B & 24 & Strict & 6.4 & 90.09 & 3.766 & 100.00 & 4.000 & 100.00 & 4.000 & 100.00 & 4.000 & 100.00 & 4.000\\
\midrule
\multicolumn{15}{@{}l}{\emph{Stable maximum, two corrupt nodes}}\\
OpenBookQA & Llama-3.2-3B & 8 & Judge & 17.1 & 90.11 & 3.770 & 96.81 & 3.909 & 97.19 & 3.918 & 98.14 & 3.946 & 99.97 & 3.999\\
OpenBookQA & Llama-3.2-3B & 8 & Strict & 4.1 & 90.08 & 3.847 & 97.11 & 3.956 & 97.24 & 3.958 & 98.38 & 3.974 & 100.00 & 4.000\\
OpenBookQA & Qwen2.5-14B & 24 & Judge & 47.8 & 90.26 & 2.646 & 92.43 & 2.789 & 92.50 & 2.793 & 96.38 & 3.216 & 100.00 & 4.000\\
OpenBookQA & Qwen2.5-14B & 24 & Strict & 44.6 & 90.17 & 2.870 & 92.09 & 2.998 & 92.18 & 3.002 & 95.67 & 3.403 & 100.00 & 4.000\\
MedQA-US & Llama-3.2-3B & 8 & Judge & 23.5 & 90.20 & 3.712 & 95.15 & 3.854 & 95.77 & 3.869 & 96.89 & 3.900 & 99.37 & 3.979\\
MedQA-US & Llama-3.2-3B & 8 & Strict & 2.6 & 90.22 & 3.869 & 97.64 & 3.971 & 97.68 & 3.972 & 98.99 & 3.988 & 100.00 & 4.000\\
MedQA-US & Qwen2.5-14B & 24 & Judge & 32.2 & 90.23 & 3.483 & 100.00 & 4.000 & 100.00 & 4.000 & 100.00 & 4.000 & 100.00 & 4.000\\
MedQA-US & Qwen2.5-14B & 24 & Strict & 6.4 & 90.15 & 3.765 & 100.00 & 4.000 & 100.00 & 4.000 & 100.00 & 4.000 & 100.00 & 4.000\\
\bottomrule
\end{tabular}
\par\medskip
\begin{tabular}{@{}llrlrrrrrrrrrrr@{}}
\toprule
 & & & & & \multicolumn{5}{c}{Label a candidate: coverage (\%)} & \multicolumn{5}{c}{Label ``None of these'': coverage (\%)}\\
\cmidrule(lr){6-10}\cmidrule(l){11-15}
Benchmark & Scorer & Bits & Labels & Share & Oracle & Fixed & Joint & Del. & Merger & Oracle & Fixed & Joint & Del. & Merger\\
\midrule
\multicolumn{15}{@{}l}{\emph{Nobody corrupt}}\\
OpenBookQA & Llama-3.2-3B & 8 & Judge & 27.5 & 98.92 & 99.43 & 99.47 & 99.62 & 99.99 & 86.75 & 94.96 & 95.38 & 97.03 & 99.95\\
OpenBookQA & Llama-3.2-3B & 8 & Strict & 0.1 & 100.00 & 100.00 & 100.00 & 100.00 & 100.00 & 90.06 & 96.49 & 96.67 & 97.60 & 100.00\\
OpenBookQA & Qwen2.5-14B & 24 & Judge & 27.5 & 94.35 & 95.78 & 95.80 & 98.10 & 100.00 & 88.69 & 90.92 & 90.98 & 95.72 & 100.00\\
OpenBookQA & Qwen2.5-14B & 24 & Strict & 0.1 & 100.00 & 100.00 & 100.00 & 100.00 & 100.00 & 90.15 & 91.96 & 92.02 & 95.52 & 100.00\\
MedQA-US & Llama-3.2-3B & 8 & Judge & 40.8 & 99.95 & 100.00 & 100.00 & 100.00 & 100.00 & 83.40 & 90.32 & 90.98 & 93.42 & 98.68\\
MedQA-US & Llama-3.2-3B & 8 & Strict & 3.2 & 100.00 & 100.00 & 100.00 & 100.00 & 100.00 & 89.92 & 96.61 & 96.67 & 97.83 & 100.00\\
MedQA-US & Qwen2.5-14B & 24 & Judge & 40.8 & 98.98 & 100.00 & 100.00 & 100.00 & 100.00 & 84.20 & 100.00 & 100.00 & 100.00 & 100.00\\
MedQA-US & Qwen2.5-14B & 24 & Strict & 3.2 & 100.00 & 100.00 & 100.00 & 100.00 & 100.00 & 89.76 & 100.00 & 100.00 & 100.00 & 100.00\\
\midrule
\multicolumn{15}{@{}l}{\emph{Stable maximum, two corrupt nodes}}\\
OpenBookQA & Llama-3.2-3B & 8 & Judge & 27.5 & 98.91 & 99.54 & 99.59 & 99.63 & 100.00 & 86.78 & 95.77 & 96.29 & 97.58 & 99.97\\
OpenBookQA & Llama-3.2-3B & 8 & Strict & 0.1 & 100.00 & 100.00 & 100.00 & 100.00 & 100.00 & 90.07 & 97.11 & 97.24 & 98.38 & 100.00\\
OpenBookQA & Qwen2.5-14B & 24 & Judge & 27.5 & 94.07 & 95.15 & 95.17 & 97.90 & 100.00 & 88.82 & 91.40 & 91.49 & 95.80 & 100.00\\
OpenBookQA & Qwen2.5-14B & 24 & Strict & 0.1 & 100.00 & 100.00 & 100.00 & 100.00 & 100.00 & 90.16 & 92.08 & 92.17 & 95.67 & 100.00\\
MedQA-US & Llama-3.2-3B & 8 & Judge & 40.8 & 99.91 & 100.00 & 100.00 & 100.00 & 100.00 & 83.51 & 91.82 & 92.86 & 94.75 & 98.93\\
MedQA-US & Llama-3.2-3B & 8 & Strict & 3.2 & 99.93 & 100.00 & 100.00 & 100.00 & 100.00 & 89.90 & 97.56 & 97.60 & 98.95 & 100.00\\
MedQA-US & Qwen2.5-14B & 24 & Judge & 40.8 & 99.00 & 100.00 & 100.00 & 100.00 & 100.00 & 84.19 & 100.00 & 100.00 & 100.00 & 100.00\\
MedQA-US & Qwen2.5-14B & 24 & Strict & 3.2 & 100.00 & 100.00 & 100.00 & 100.00 & 100.00 & 89.82 & 100.00 & 100.00 & 100.00 & 100.00\\
\bottomrule
\end{tabular}
\par}%
\end{table}
\endgroup{}
}
\FloatBarrier

\begin{thebibliography}{63}
\providecommand{\natexlab}[1]{#1}
\providecommand{\url}[1]{\texttt{#1}}
\expandafter\ifx\csname urlstyle\endcsname\relax
  \providecommand{\doi}[1]{doi: #1}\else
  \providecommand{\doi}{doi: \begingroup \urlstyle{rm}\Url}\fi

\bibitem[Angelopoulos \& Bates(2023)Angelopoulos and
  Bates]{angelopoulos2023gentle}
Anastasios~N. Angelopoulos and Stephen Bates.
\newblock Conformal prediction: A gentle introduction.
\newblock \emph{Foundations and Trends in Machine Learning}, 16\penalty0
  (4):\penalty0 494--591, 2023.
\newblock \doi{10.1561/2200000101}.

\bibitem[Barber et~al.(2023)Barber, Cand{\`e}s, Ramdas, and
  Tibshirani]{barber2023beyond}
Rina~Foygel Barber, Emmanuel~J. Cand{\`e}s, Aaditya Ramdas, and Ryan~J.
  Tibshirani.
\newblock Conformal prediction beyond exchangeability.
\newblock \emph{The Annals of Statistics}, 51\penalty0 (2):\penalty0 816--845,
  2023.
\newblock \doi{10.1214/23-AOS2276}.

\bibitem[Bell et~al.(2023)Bell, Gasc{\'o}n, Lepoint, Li, Meiklejohn, Raykova,
  and Yun]{bell2023acorn}
James Bell, Adri{\`a} Gasc{\'o}n, Tancr{\`e}de Lepoint, Baiyu Li, Sarah
  Meiklejohn, Mariana Raykova, and Cathie Yun.
\newblock {ACORN}: Input validation for secure aggregation.
\newblock In \emph{32nd USENIX Security Symposium (USENIX Security 23)}, pp.\
  4805--4822. USENIX Association, 2023.
\newblock URL
  \url{https://www.usenix.org/conference/usenixsecurity23/presentation/bell}.

\bibitem[Benjamini \& Heller(2008)Benjamini and Heller]{benjamini2008partial}
Yoav Benjamini and Ruth Heller.
\newblock Screening for partial conjunction hypotheses.
\newblock \emph{Biometrics}, 64\penalty0 (4):\penalty0 1215--1222, 2008.
\newblock \doi{10.1111/j.1541-0420.2007.00984.x}.

\bibitem[Blanchard et~al.(2017)Blanchard, El~Mhamdi, Guerraoui, and
  Stainer]{blanchard2017krum}
Peva Blanchard, El~Mahdi El~Mhamdi, Rachid Guerraoui, and Julien Stainer.
\newblock Machine learning with adversaries: {B}yzantine tolerant gradient
  descent.
\newblock In \emph{Advances in Neural Information Processing Systems},
  volume~30, pp.\  119--129. Curran Associates, Inc., 2017.
\newblock URL
  \url{https://proceedings.neurips.cc/paper/2017/hash/f4b9ec30ad9f68f89b29639786cb62ef-Abstract.html}.

\bibitem[Bonawitz et~al.(2017)Bonawitz, Ivanov, Kreuter, Marcedone, McMahan,
  Patel, Ramage, Segal, and Seth]{bonawitz2017secure}
Keith Bonawitz, Vladimir Ivanov, Ben Kreuter, Antonio Marcedone, H.~Brendan
  McMahan, Sarvar Patel, Daniel Ramage, Aaron Segal, and Karn Seth.
\newblock Practical secure aggregation for privacy-preserving machine learning.
\newblock In \emph{Proceedings of the 2017 ACM SIGSAC Conference on Computer
  and Communications Security}, pp.\  1175--1191. ACM, 2017.
\newblock \doi{10.1145/3133956.3133982}.

\bibitem[Campos et~al.(2024)Campos, Farinhas, Zerva, Figueiredo, and
  Martins]{campos2024cpnlp}
Margarida Campos, Ant{\'o}nio Farinhas, Chrysoula Zerva, M{\'a}rio A.~T.
  Figueiredo, and Andr{\'e} F.~T. Martins.
\newblock Conformal prediction for natural language processing: A survey.
\newblock \emph{Transactions of the Association for Computational Linguistics},
  12:\penalty0 1497--1516, 2024.
\newblock \doi{10.1162/tacl_a_00715}.

\bibitem[Chakraborty et~al.(2025)Chakraborty, Dahal, and
  Gupta]{chakraborty2025fedragmap}
Abhijit Chakraborty, Chahana Dahal, and Vivek Gupta.
\newblock Federated retrieval-augmented generation: A systematic mapping study.
\newblock In \emph{Findings of the Association for Computational Linguistics:
  EMNLP 2025}, pp.\  7362--7374. Association for Computational Linguistics,
  2025.
\newblock \doi{10.18653/v1/2025.findings-emnlp.388}.

\bibitem[Chen et~al.(2017)Chen, Su, and Xu]{chen2017byzantine}
Yudong Chen, Lili Su, and Jiaming Xu.
\newblock Distributed statistical machine learning in adversarial settings:
  {Byzantine} gradient descent.
\newblock \emph{Proceedings of the ACM on Measurement and Analysis of Computing
  Systems}, 1\penalty0 (2):\penalty0 1--25, 2017.
\newblock \doi{10.1145/3154503}.

\bibitem[Clarkson et~al.(2024)Clarkson, Xu, Cucuringu, and
  Reinert]{clarkson2024contamination}
Jason Clarkson, Wenkai Xu, Mihai Cucuringu, and Gesine Reinert.
\newblock Split conformal prediction under data contamination.
\newblock In \emph{Proceedings of the Thirteenth Symposium on Conformal and
  Probabilistic Prediction with Applications}, volume 230 of \emph{Proceedings
  of Machine Learning Research}, pp.\  5--27. PMLR, 2024.
\newblock URL \url{https://proceedings.mlr.press/v230/clarkson24a.html}.

\bibitem[Dhasade et~al.(2026)Dhasade, Guerraoui, Kermarrec, Petrescu, Pires,
  Randl, and de~Vos]{dhasade2026ragroute}
Akash Dhasade, Rachid Guerraoui, Anne-Marie Kermarrec, Diana Petrescu, Rafael
  Pires, Mathis Randl, and Martijn de~Vos.
\newblock Efficient federated search for retrieval-augmented generation using
  lightweight routing.
\newblock In \emph{Distributed Applications and Interoperable Systems}, volume
  16591 of \emph{Lecture Notes in Computer Science}, pp.\  3--20. Springer,
  2026.
\newblock \doi{10.1007/978-3-032-27358-1_1}.

\bibitem[Dubey \& Huo(2026{\natexlab{a}})Dubey and Huo]{dubeyhuo2026anytime}
Prasanjit Dubey and Xiaoming Huo.
\newblock Anytime-valid federated conformal {RAG} for {LLM} swarms,
  2026{\natexlab{a}}.
\newblock URL \url{https://arxiv.org/abs/2605.29139}.
\newblock arXiv:2605.29139v1.

\bibitem[Dubey \& Huo(2026{\natexlab{b}})Dubey and Huo]{dubeyhuo2026bandwidth}
Prasanjit Dubey and Xiaoming Huo.
\newblock Federated language models under bandwidth budgets: Distillation rates
  and conformal coverage, 2026{\natexlab{b}}.
\newblock URL \url{https://arxiv.org/abs/2605.09986}.
\newblock arXiv:2605.09986v2.

\bibitem[Feng et~al.(2025)Feng, Sui, Hou, Cresswell, and
  Wu]{feng2025conditionalrag}
Naihe Feng, Yi~Sui, Shiyi Hou, Jesse~C. Cresswell, and Ga~Wu.
\newblock Response quality assessment for retrieval-augmented generation via
  conditional conformal factuality.
\newblock In \emph{Proceedings of the 48th International ACM SIGIR Conference
  on Research and Development in Information Retrieval}, pp.\  2832--2836. ACM,
  2025.
\newblock \doi{10.1145/3726302.3730244}.

\bibitem[Gasparin \& Ramdas(2024)Gasparin and Ramdas]{gasparin2024merging}
Matteo Gasparin and Aaditya Ramdas.
\newblock Merging uncertainty sets via majority vote, 2024.
\newblock URL \url{https://arxiv.org/abs/2401.09379}.

\bibitem[Gendler et~al.(2022)Gendler, Weng, Daniel, and
  Romano]{gendler2022adversarial}
Asaf Gendler, Tsui-Wei Weng, Luca Daniel, and Yaniv Romano.
\newblock Adversarially robust conformal prediction.
\newblock In \emph{International Conference on Learning Representations}, 2022.
\newblock URL \url{https://openreview.net/forum?id=9L1BsI4wP1H}.

\bibitem[Grattafiori et~al.(2024)Grattafiori, Dubey, Jauhri, Pandey, Kadian,
  et~al.]{grattafiori2024llama3}
Aaron Grattafiori, Abhimanyu Dubey, Abhinav Jauhri, Abhinav Pandey, Abhishek
  Kadian, et~al.
\newblock The {Llama} 3 herd of models, 2024.
\newblock URL \url{https://arxiv.org/abs/2407.21783}.

\bibitem[Guerraoui et~al.(2024)Guerraoui, Gupta, and
  Pinot]{guerraoui2024primer}
Rachid Guerraoui, Nirupam Gupta, and Rafael Pinot.
\newblock {Byzantine} machine learning: A primer.
\newblock \emph{ACM Computing Surveys}, 56\penalty0 (7):\penalty0 1--39, 2024.
\newblock \doi{10.1145/3616537}.

\bibitem[He et~al.(2025)He, Yuan, Wu, Liu, and Ni]{he2025pfedrag}
Hangyu He, Xin Yuan, Kai Wu, Ren~Ping Liu, and Wei Ni.
\newblock p{F}ed{RAG}: A personalized federated retrieval-augmented generation
  system with depth-adaptive tiered embedding tuning.
\newblock In \emph{Findings of the Association for Computational Linguistics:
  EMNLP 2025}, pp.\  14255--14268, Suzhou, China, 2025. Association for
  Computational Linguistics.
\newblock \doi{10.18653/v1/2025.findings-emnlp.769}.
\newblock URL \url{https://aclanthology.org/2025.findings-emnlp.769/}.

\bibitem[Hendrycks et~al.(2021)Hendrycks, Burns, Basart, Zou, Mazeika, Song,
  and Steinhardt]{hendrycks2021mmlu}
Dan Hendrycks, Collin Burns, Steven Basart, Andy Zou, Mantas Mazeika, Dawn
  Song, and Jacob Steinhardt.
\newblock Measuring massive multitask language understanding.
\newblock In \emph{International Conference on Learning Representations}, 2021.
\newblock URL \url{https://openreview.net/forum?id=d7KBjmI3GmQ}.

\bibitem[Humbert et~al.(2023)Humbert, Le~Bars, Bellet, and
  Arlot]{humbert2023oneshot}
Pierre Humbert, Batiste Le~Bars, Aur{\'e}lien Bellet, and Sylvain Arlot.
\newblock One-shot federated conformal prediction.
\newblock In \emph{Proceedings of the 40th International Conference on Machine
  Learning}, volume 202 of \emph{Proceedings of Machine Learning Research},
  pp.\  14153--14177. PMLR, 2023.
\newblock URL \url{https://proceedings.mlr.press/v202/humbert23a.html}.

\bibitem[Jeary et~al.(2024)Jeary, Kuipers, Hosseini, and
  Paoletti]{jeary2024verifiably}
Linus Jeary, Tom Kuipers, Mehran Hosseini, and Nicola Paoletti.
\newblock Verifiably robust conformal prediction.
\newblock In \emph{Advances in Neural Information Processing Systems},
  volume~37, pp.\  4295--4314. Curran Associates, Inc., 2024.
\newblock \doi{10.52202/079017-0140}.
\newblock URL
  \url{https://proceedings.neurips.cc/paper_files/paper/2024/file/0814a342597d65e0832fc7ec9b42c317-Paper-Conference.pdf}.

\bibitem[Jin et~al.(2021)Jin, Pan, Oufattole, Weng, Fang, and
  Szolovits]{jin2020medqa}
Di~Jin, Eileen Pan, Nassim Oufattole, Wei-Hung Weng, Hanyi Fang, and Peter
  Szolovits.
\newblock What disease does this patient have? {A} large-scale open domain
  question answering dataset from medical exams.
\newblock \emph{Applied Sciences}, 11\penalty0 (14):\penalty0 6421, 2021.
\newblock \doi{10.3390/app11146421}.

\bibitem[Kang et~al.(2024{\natexlab{a}})Kang, G{\"u}rel, Yu, Song, and
  Li]{kang2024crag}
Mintong Kang, Nezihe~Merve G{\"u}rel, Ning Yu, Dawn Song, and Bo~Li.
\newblock C-{RAG}: Certified generation risks for retrieval-augmented language
  models.
\newblock In \emph{Proceedings of the 41st International Conference on Machine
  Learning}, volume 235 of \emph{Proceedings of Machine Learning Research},
  pp.\  22963--23000. PMLR, 2024{\natexlab{a}}.
\newblock URL \url{https://proceedings.mlr.press/v235/kang24a.html}.

\bibitem[Kang et~al.(2024{\natexlab{b}})Kang, Lin, Sun, Xiao, and
  Li]{kang2024robfcp}
Mintong Kang, Zhen Lin, Jimeng Sun, Cao Xiao, and Bo~Li.
\newblock Certifiably {B}yzantine-robust federated conformal prediction.
\newblock In \emph{Proceedings of the 41st International Conference on Machine
  Learning}, volume 235 of \emph{Proceedings of Machine Learning Research},
  pp.\  23022--23057. PMLR, 2024{\natexlab{b}}.
\newblock URL \url{https://proceedings.mlr.press/v235/kang24c.html}.

\bibitem[Karimireddy et~al.(2022)Karimireddy, He, and
  Jaggi]{karimireddy2022bucketing}
Sai~Praneeth Karimireddy, Lie He, and Martin Jaggi.
\newblock {B}yzantine-robust learning on heterogeneous datasets via bucketing.
\newblock In \emph{International Conference on Learning Representations}, 2022.
\newblock URL \url{https://openreview.net/forum?id=jXKKDEi5vJt}.

\bibitem[Kumar et~al.(2023)Kumar, Lu, Gupta, Palepu, Bellamy, Raskar, and
  Beam]{kumar2023conformalmcqa}
Bhawesh Kumar, Charlie Lu, Gauri Gupta, Anil Palepu, David Bellamy, Ramesh
  Raskar, and Andrew Beam.
\newblock Conformal prediction with large language models for multi-choice
  question answering, 2023.
\newblock URL \url{https://arxiv.org/abs/2305.18404}.
\newblock ICML 2023 Workshop on Neural Conversational AI (TEACH).

\bibitem[Lari et~al.(2026{\natexlab{a}})Lari, Arablouei, and
  Werner]{lari2026partial}
Ehsan Lari, Reza Arablouei, and Stefan Werner.
\newblock Partial model sharing improves {Byzantine} resilience in federated
  conformal prediction, 2026{\natexlab{a}}.
\newblock URL \url{https://arxiv.org/abs/2605.11684}.
\newblock Accepted to the 34th European Signal Processing Conference (EUSIPCO
  2026).

\bibitem[Lari et~al.(2026{\natexlab{b}})Lari, Arablouei, and
  Werner]{lari2026prism}
Ehsan Lari, Reza Arablouei, and Stefan Werner.
\newblock Communication-efficient {Byzantine}-robust federated conformal
  prediction via partial sharing.
\newblock \emph{IEEE Transactions on Signal Processing}, pp.\  1--16,
  2026{\natexlab{b}}.
\newblock \doi{10.1109/TSP.2026.3737023}.

\bibitem[Lei et~al.(2018)Lei, G'Sell, Rinaldo, Tibshirani, and
  Wasserman]{lei2018distribution}
Jing Lei, Max G'Sell, Alessandro Rinaldo, Ryan~J. Tibshirani, and Larry
  Wasserman.
\newblock Distribution-free predictive inference for regression.
\newblock \emph{Journal of the American Statistical Association}, 113\penalty0
  (523):\penalty0 1094--1111, 2018.
\newblock \doi{10.1080/01621459.2017.1307116}.

\bibitem[Lewis et~al.(2020)Lewis, Perez, Piktus, Petroni, Karpukhin, Goyal,
  K{\"u}ttler, Lewis, Yih, Rockt{\"a}schel, Riedel, and
  Kiela]{lewis2020retrieval}
Patrick Lewis, Ethan Perez, Aleksandra Piktus, Fabio Petroni, Vladimir
  Karpukhin, Naman Goyal, Heinrich K{\"u}ttler, Mike Lewis, Wen-tau Yih, Tim
  Rockt{\"a}schel, Sebastian Riedel, and Douwe Kiela.
\newblock Retrieval-augmented generation for knowledge-intensive {NLP} tasks.
\newblock In \emph{Advances in Neural Information Processing Systems 33}, pp.\
  9459--9474, 2020.
\newblock URL
  \url{https://proceedings.neurips.cc/paper/2020/hash/6b493230205f780e1bc26945df7481e5-Abstract.html}.

\bibitem[Li et~al.(2024)Li, Park, Lee, and Bastani]{li2024traq}
Shuo Li, Sangdon Park, Insup Lee, and Osbert Bastani.
\newblock {TRAQ}: Trustworthy retrieval augmented question answering via
  conformal prediction.
\newblock In \emph{Proceedings of the 2024 Conference of the North American
  Chapter of the Association for Computational Linguistics: Human Language
  Technologies (Volume 1: Long Papers)}, pp.\  3799--3821, Mexico City, Mexico,
  2024. Association for Computational Linguistics.
\newblock \doi{10.18653/v1/2024.naacl-long.210}.
\newblock URL \url{https://aclanthology.org/2024.naacl-long.210/}.

\bibitem[Lu et~al.(2023)Lu, Yu, Karimireddy, Jordan, and
  Raskar]{lu2023federated}
Charles Lu, Yaodong Yu, Sai~Praneeth Karimireddy, Michael~I. Jordan, and Ramesh
  Raskar.
\newblock Federated conformal predictors for distributed uncertainty
  quantification.
\newblock In \emph{Proceedings of the 40th International Conference on Machine
  Learning}, volume 202 of \emph{Proceedings of Machine Learning Research},
  pp.\  22942--22964. PMLR, 2023.
\newblock URL \url{https://proceedings.mlr.press/v202/lu23i.html}.

\bibitem[Makarov(1982)]{makarov1982fixedmarginals}
G.~D. Makarov.
\newblock Estimates for the distribution function of a sum of two random
  variables when the marginal distributions are fixed.
\newblock \emph{Theory of Probability and Its Applications}, 26\penalty0
  (4):\penalty0 803--806, 1982.
\newblock \doi{10.1137/1126086}.

\bibitem[Mihaylov et~al.(2018)Mihaylov, Clark, Khot, and
  Sabharwal]{mihaylov2018openbookqa}
Todor Mihaylov, Peter Clark, Tushar Khot, and Ashish Sabharwal.
\newblock Can a suit of armor conduct electricity? {A} new dataset for open
  book question answering.
\newblock In \emph{Proceedings of the 2018 Conference on Empirical Methods in
  Natural Language Processing}, pp.\  2381--2391, Brussels, Belgium, 2018.
  Association for Computational Linguistics.
\newblock \doi{10.18653/v1/D18-1260}.
\newblock URL \url{https://aclanthology.org/D18-1260/}.

\bibitem[Mohri \& Hashimoto(2024)Mohri and Hashimoto]{mohri2024conformal}
Christopher Mohri and Tatsunori Hashimoto.
\newblock Language models with conformal factuality guarantees.
\newblock In \emph{Proceedings of the 41st International Conference on Machine
  Learning}, volume 235 of \emph{Proceedings of Machine Learning Research},
  pp.\  36029--36047. PMLR, 2024.
\newblock URL \url{https://proceedings.mlr.press/v235/mohri24a.html}.

\bibitem[Mu \& Li(2026)Mu and Li]{mu2026routinghijacking}
Junjie Mu and Qiongxiu Li.
\newblock A wolf in sheep's clothing: Targeted routing hijacking in federated
  {RAG}, 2026.
\newblock URL \url{https://arxiv.org/abs/2605.28112}.
\newblock Accepted to the EMNLP 2026 Main Conference.

\bibitem[Pal et~al.(2022)Pal, Umapathi, and Sankarasubbu]{pal2022medmcqa}
Ankit Pal, Logesh~Kumar Umapathi, and Malaikannan Sankarasubbu.
\newblock {MedMCQA}: A large-scale multi-subject multi-choice dataset for
  medical domain question answering.
\newblock In \emph{Proceedings of the Conference on Health, Inference, and
  Learning}, volume 174 of \emph{Proceedings of Machine Learning Research},
  pp.\  248--260. PMLR, 2022.
\newblock URL \url{https://proceedings.mlr.press/v174/pal22a.html}.

\bibitem[Park et~al.(2023)Park, Bastani, and Kim]{park2023acon2}
Sangdon Park, Osbert Bastani, and Taesoo Kim.
\newblock {ACon$^2$}: Adaptive conformal consensus for provable blockchain
  oracles.
\newblock In \emph{32nd USENIX Security Symposium (USENIX Security 23)}, pp.\
  3313--3330. USENIX Association, 2023.
\newblock URL
  \url{https://www.usenix.org/conference/usenixsecurity23/presentation/park}.

\bibitem[Pillutla et~al.(2022)Pillutla, Kakade, and
  Harchaoui]{pillutla2022robust}
Krishna Pillutla, Sham~M. Kakade, and Zaid Harchaoui.
\newblock Robust aggregation for federated learning.
\newblock \emph{IEEE Transactions on Signal Processing}, 70:\penalty0
  1142--1154, 2022.
\newblock \doi{10.1109/TSP.2022.3153135}.

\bibitem[Plassier et~al.(2023)Plassier, Makni, Rubashevskii, Moulines, and
  Panov]{plassier2023labelshift}
Vincent Plassier, Mehdi Makni, Aleksandr Rubashevskii, Eric Moulines, and Maxim
  Panov.
\newblock Conformal prediction for federated uncertainty quantification under
  label shift.
\newblock In \emph{Proceedings of the 40th International Conference on Machine
  Learning}, volume 202 of \emph{Proceedings of Machine Learning Research},
  pp.\  27907--27947. PMLR, 2023.
\newblock URL \url{https://proceedings.mlr.press/v202/plassier23a.html}.

\bibitem[{Qwen Team} et~al.(2024){Qwen Team}, Yang, Yang, Zhang, Hui, Zheng,
  et~al.]{qwen2024qwen25}
{Qwen Team}, An~Yang, Baosong Yang, Beichen Zhang, Binyuan Hui, Bo~Zheng,
  et~al.
\newblock {Qwen2.5} technical report, 2024.
\newblock URL \url{https://arxiv.org/abs/2412.15115}.

\bibitem[Radford et~al.(2019)Radford, Wu, Child, Luan, Amodei, and
  Sutskever]{radford2019gpt2}
Alec Radford, Jeffrey Wu, Rewon Child, David Luan, Dario Amodei, and Ilya
  Sutskever.
\newblock Language models are unsupervised multitask learners.
\newblock Technical report, OpenAI, 2019.
\newblock URL
  \url{https://cdn.openai.com/better-language-models/language_models_are_unsupervised_multitask_learners.pdf}.

\bibitem[Rammal et~al.(2024)Rammal, Gruntkowska, Fedin, Gorbunov, and
  Richt{\'a}rik]{rammal2024compression}
Ahmad Rammal, Kaja Gruntkowska, Nikita Fedin, Eduard Gorbunov, and Peter
  Richt{\'a}rik.
\newblock Communication compression for {B}yzantine robust learning: New
  efficient algorithms and improved rates.
\newblock In \emph{Proceedings of the 27th International Conference on
  Artificial Intelligence and Statistics}, volume 238 of \emph{Proceedings of
  Machine Learning Research}, pp.\  1207--1215. PMLR, 2024.
\newblock URL \url{https://proceedings.mlr.press/v238/rammal24a.html}.

\bibitem[Reimers \& Gurevych(2019)Reimers and Gurevych]{reimers2019sbert}
Nils Reimers and Iryna Gurevych.
\newblock Sentence-{BERT}: Sentence embeddings using {S}iamese {BERT}-networks.
\newblock In \emph{Proceedings of the 2019 Conference on Empirical Methods in
  Natural Language Processing and the 9th International Joint Conference on
  Natural Language Processing (EMNLP-IJCNLP)}, pp.\  3982--3992, Hong Kong,
  China, 2019. Association for Computational Linguistics.
\newblock \doi{10.18653/v1/D19-1410}.
\newblock URL \url{https://aclanthology.org/D19-1410/}.

\bibitem[Robertson \& Zaragoza(2009)Robertson and Zaragoza]{robertson2009bm25}
Stephen Robertson and Hugo Zaragoza.
\newblock The probabilistic relevance framework: {BM25} and beyond.
\newblock \emph{Foundations and Trends in Information Retrieval}, 3\penalty0
  (4):\penalty0 333--389, 2009.
\newblock \doi{10.1561/1500000019}.

\bibitem[R{\"u}ger(1978)]{ruger1978maximale}
B.~R{\"u}ger.
\newblock {Das maximale Signifikanzniveau des Tests: ``Lehne {$H_0$} ab, wenn
  {$k$} unter {$n$} gegebenen Tests zur Ablehnung f{\"u}hren''}.
\newblock \emph{Metrika}, 25\penalty0 (1):\penalty0 171--178, 1978.
\newblock \doi{10.1007/BF02204362}.

\bibitem[Scholten \& G{\"u}nnemann(2025)Scholten and
  G{\"u}nnemann]{scholten2025reliable}
Yan Scholten and Stephan G{\"u}nnemann.
\newblock Provably reliable conformal prediction sets in the presence of data
  poisoning.
\newblock In \emph{International Conference on Learning Representations}, pp.\
  55873--55897, 2025.
\newblock URL
  \url{https://proceedings.iclr.cc/paper_files/paper/2025/hash/8c4d21a4b33361c40c2142db9511c126-Abstract-Conference.html}.

\bibitem[Shi et~al.(2026)Shi, Li, Cao, and Yan]{shi2026jointshift}
Yuanjie Shi, Peihong Li, Xuanyu Cao, and Yan Yan.
\newblock Valid and efficient uncertainty quantification for federated joint
  shift.
\newblock In \emph{Proceedings of the 42nd Conference on Uncertainty in
  Artificial Intelligence}, volume 337 of \emph{Proceedings of Machine Learning
  Research}, pp.\  6232--6260. PMLR, 2026.
\newblock URL \url{https://proceedings.mlr.press/v337/shi26a.html}.

\bibitem[Suresh et~al.(2017)Suresh, Yu, Kumar, and
  McMahan]{suresh2017distributed}
Ananda~Theertha Suresh, Felix~X. Yu, Sanjiv Kumar, and H.~Brendan McMahan.
\newblock Distributed mean estimation with limited communication.
\newblock In \emph{Proceedings of the 34th International Conference on Machine
  Learning}, volume~70 of \emph{Proceedings of Machine Learning Research}, pp.\
   3329--3337. PMLR, 2017.
\newblock URL \url{https://proceedings.mlr.press/v70/suresh17a.html}.

\bibitem[Vovk \& Wang(2020)Vovk and Wang]{vovkwang2020combining}
Vladimir Vovk and Ruodu Wang.
\newblock Combining p-values via averaging.
\newblock \emph{Biometrika}, 107\penalty0 (4):\penalty0 791--808, 2020.
\newblock \doi{10.1093/biomet/asaa027}.

\bibitem[Vovk et~al.(2005)Vovk, Gammerman, and Shafer]{vovk2005algorithmic}
Vladimir Vovk, Alex Gammerman, and Glenn Shafer.
\newblock \emph{Algorithmic Learning in a Random World}.
\newblock Springer, 2005.
\newblock \doi{10.1007/b106715}.
\newblock URL \url{https://doi.org/10.1007/b106715}.

\bibitem[Wang \& Owen(2019)Wang and Owen]{wang2019partial}
Jingshu Wang and Art~B. Owen.
\newblock Admissibility in partial conjunction testing.
\newblock \emph{Journal of the American Statistical Association}, 114\penalty0
  (525):\penalty0 158--168, 2019.
\newblock \doi{10.1080/01621459.2017.1385465}.

\bibitem[Wang et~al.(2026)Wang, He, Zhang, Liu, Liu, Zeng, Qin, Li, Li, Yao,
  An, Liu, Li, Sun, Liu, and Zhu]{wang2026securecollarag}
Zhaoqi Wang, Daqing He, Zijian Zhang, Ye~Liu, Jiamou Liu, Zhirui Zeng, Zhan
  Qin, Zhen Li, Xin Li, Hongwei Yao, Jincheng An, Yong Liu, Yi~Li, Qi~Sun,
  Xiulei Liu, and Liehuang Zhu.
\newblock Combating knowledge corruption in agent systems: A
  {Byzantine}-tolerant secure collaborative {RAG} framework.
\newblock In \emph{Proceedings of the ACM Web Conference 2026}, pp.\
  2661--2672. ACM, 2026.
\newblock \doi{10.1145/3774904.3792200}.

\bibitem[Wen et~al.(2026)Wen, Simeone, and Xing]{wen2026groupconditional}
Haifeng Wen, Osvaldo Simeone, and Hong Xing.
\newblock Efficient federated conformal prediction with group-conditional
  guarantee.
\newblock In \emph{Proceedings of the 42nd Conference on Uncertainty in
  Artificial Intelligence}, volume 337 of \emph{Proceedings of Machine Learning
  Research}, pp.\  7290--7313. PMLR, 2026.
\newblock URL \url{https://proceedings.mlr.press/v337/wen26a.html}.

\bibitem[Xiang et~al.(2026)Xiang, Wu, Zhong, Wagner, Chen, and
  Mittal]{xiang2026robustrag}
Chong Xiang, Tong Wu, Zexuan Zhong, David Wagner, Danqi Chen, and Prateek
  Mittal.
\newblock Certifiably robust {RAG} against retrieval corruption.
\newblock In \emph{2026 IEEE Conference on Secure and Trustworthy Machine
  Learning (SaTML)}, pp.\  392--410. IEEE, 2026.
\newblock \doi{10.1109/SaTML68715.2026.00029}.

\bibitem[Xiong et~al.(2024)Xiong, Jin, Lu, and Zhang]{xiong2024benchmarking}
Guangzhi Xiong, Qiao Jin, Zhiyong Lu, and Aidong Zhang.
\newblock Benchmarking retrieval-augmented generation for medicine.
\newblock In \emph{Findings of the Association for Computational Linguistics:
  ACL 2024}, pp.\  6233--6251, Bangkok, Thailand, 2024. Association for
  Computational Linguistics.
\newblock \doi{10.18653/v1/2024.findings-acl.372}.
\newblock URL \url{https://aclanthology.org/2024.findings-acl.372/}.

\bibitem[Yan et~al.(2024)Yan, Romano, and Weng]{yan2024rscpplus}
Ge~Yan, Yaniv Romano, and Tsui-Wei Weng.
\newblock Provably robust conformal prediction with improved efficiency.
\newblock In \emph{International Conference on Learning Representations}, pp.\
  31974--32010, 2024.
\newblock URL
  \url{https://proceedings.iclr.cc/paper_files/paper/2024/hash/8759c20675d67ef3b91c4607cecbb27e-Abstract-Conference.html}.

\bibitem[Yin et~al.(2018)Yin, Chen, Ramchandran, and
  Bartlett]{yin2018byzantine}
Dong Yin, Yudong Chen, Kannan Ramchandran, and Peter Bartlett.
\newblock {B}yzantine-robust distributed learning: Towards optimal statistical
  rates.
\newblock In \emph{Proceedings of the 35th International Conference on Machine
  Learning}, volume~80 of \emph{Proceedings of Machine Learning Research}, pp.\
   5650--5659. PMLR, 2018.
\newblock URL \url{https://proceedings.mlr.press/v80/yin18a.html}.

\bibitem[Zargarbashi et~al.(2024)Zargarbashi, Akhondzadeh, and
  Bojchevski]{zargarbashi2024robust}
Soroush~H. Zargarbashi, Mohammad~Sadegh Akhondzadeh, and Aleksandar Bojchevski.
\newblock Robust yet efficient conformal prediction sets.
\newblock In \emph{Proceedings of the 41st International Conference on Machine
  Learning}, volume 235 of \emph{Proceedings of Machine Learning Research},
  pp.\  17123--17147. PMLR, 2024.
\newblock URL \url{https://proceedings.mlr.press/v235/h-zargarbashi24a.html}.

\bibitem[Zhong et~al.(2023)Zhong, Huang, Wettig, and Chen]{zhong2023poisoning}
Zexuan Zhong, Ziqing Huang, Alexander Wettig, and Danqi Chen.
\newblock Poisoning retrieval corpora by injecting adversarial passages.
\newblock In \emph{Proceedings of the 2023 Conference on Empirical Methods in
  Natural Language Processing}, pp.\  13764--13775. Association for
  Computational Linguistics, 2023.
\newblock \doi{10.18653/v1/2023.emnlp-main.849}.

\bibitem[Zhu et~al.(2024)Zhu, Zecchin, Park, Guo, Feng, and
  Simeone]{zhu2024wfcp}
Meiyi Zhu, Matteo Zecchin, Sangwoo Park, Caili Guo, Chunyan Feng, and Osvaldo
  Simeone.
\newblock Federated inference with reliable uncertainty quantification over
  wireless channels via conformal prediction.
\newblock \emph{IEEE Transactions on Signal Processing}, 72:\penalty0
  1235--1250, 2024.
\newblock \doi{10.1109/TSP.2024.3358615}.

\bibitem[Zou et~al.(2025)Zou, Geng, Wang, and Jia]{zou2025poisonedrag}
Wei Zou, Runpeng Geng, Binghui Wang, and Jinyuan Jia.
\newblock {PoisonedRAG}: Knowledge corruption attacks to retrieval-augmented
  generation of large language models.
\newblock In \emph{34th USENIX Security Symposium (USENIX Security 25)}, pp.\
  3827--3844. USENIX Association, 2025.
\newblock URL
  \url{https://www.usenix.org/conference/usenixsecurity25/presentation/zou-poisonedrag}.

\end{thebibliography}
\end{document}